\documentclass[a4paper,10pt]{book}

\usepackage[T1]{fontenc}
\usepackage[sc]{mathpazo}
\usepackage{cite}
\usepackage[linesnumbered,ruled,slide,boxed,vlined]{algorithm2e}
\usepackage{todonotes}
\usepackage[multiple]{footmisc}

\usepackage{rawfonts}
\usepackage{mathptmx}
\usepackage{rotating}
\usepackage{stmaryrd}

\usepackage{latexsym}
\usepackage[latin1]{inputenc}
\usepackage[british]{babel}
\usepackage[a4paper]{geometry}
\usepackage{amsmath, amssymb, amsthm, cancel}
\usepackage[pdftex,colorlinks=true,linkcolor=black,citecolor=black,menucolor=black,urlcolor=black]{hyperref}
\usepackage{cleveref}

\usepackage{enumitem}
\usepackage[basic]{complexity}

\usepackage{ifthen}
\usepackage{environ}
\newif\iflong
\NewEnviron{shortproof}{\iflong\else\expandafter\begin{proof}\BODY\end{proof}\fi}
\NewEnviron{longproof}{\iflong\expandafter\begin{proof}\BODY\end{proof}\fi}
\longtrue

\usepackage{setspace}
\usepackage{relsize}
\usepackage{multirow}
\usepackage{datetime}
\usepackage{graphics}
\usepackage{tikz,pgf}

\usetikzlibrary{backgrounds,shapes,decorations,snakes,arrows,shadows,matrix} 
\tikzstyle{snode} = [draw,rounded corners,minimum size=4mm,font=\small,fill=gray!15,label distance=-1mm,node distance=0.1cm]

\usetikzlibrary{arrows,positioning,shapes.misc,shapes.geometric,patterns}
\tikzset{
    >=stealth',
    args/.style={circle,draw=black, thick},
    minimum size=7mm
}

\usepackage{floatflt}
\usepackage{url}
\usepackage{color}

\usepackage{cite}
\usepackage{fancyhdr}
\usepackage[font=it]{caption}
\usepackage[strict]{changepage}
\usepackage{makeidx}
\usepackage{helvet}
\usepackage{courier}
\usepackage{verbatim}
\usepackage{float}
\usepackage{gnuplot-lua-tikz}

\usetikzlibrary{shapes,arrows}

\usepackage{subfig,tikz}
\tikzstyle{box}=[ellipse,minimum width=15mm, minimum height=6mm, very
  thick,draw=black, rounded corners, text badly centered,
  fill=black!4,inner sep=0pt, drop shadow]
\usepackage{float}

\tikzstyle{arg}=[draw,thick,circle,minimum size=0.7cm]
\tikzstyle{argm}=[draw,thick,circle,minimum size=0.7cm,scale=0.6]
\tikzstyle{attack}=[->,left,thick]
\tikzstyle{attackc}=[attack,trcolor]
\usetikzlibrary{shapes,shadows,arrows,automata}
\usetikzlibrary{decorations.markings}
\usetikzlibrary{decorations.pathmorphing}

\newcommand{\standardbox}[1]{$\textcolor{white}{d}#1\textcolor{white}{p}$}

\usepackage{ifthen}
\newcommand{\hideandshow}[1]{%
 \ifthenelse{\isundefined{\showme}}{}{#1}}

\newlength{\outsidemarginparwidth}
\newlength{\insidemarginparwidth}
\makeindex



\tikzstyle{literal}=[draw, rounded corners, text centered, text width=5mm, minimum height=5mm]
\tikzstyle{causation}=[->, semithick]

\tikzstyle{time}=[shape=circle, draw]
\tikzstyle{action}=[->, semithick]
\tikzstyle{inference}=[-latex, dashed]
\tikzstyle{reflexive}=[to path={.. controls +(-1,1) and +(1,1) .. (\tikztotarget) \tikztonodes}]

\tikzstyle{object}=[draw, rounded corners, fill=gray!15, text width=33mm, text centered, minimum height=1cm]
\tikzstyle{invisible}=[]
\tikzstyle{defrelation}=[->, semithick]
\tikzstyle{translation}=[-latex, dashed]
\tikzstyle{reflexive}=[to path={.. controls +(-1,1) and +(1,1) .. (\tikztotarget) \tikztonodes}]

\newcommand{\outerspace}{\vfill}
\newcommand{\innerspace}{\vspace*{8mm}}

\newenvironment{headedtextblock}[1]{%
  \begin{center}%
    \begin{minipage}[h]{0.75\linewidth}%
      \parindent 1em%
      \begin{center}%
        \larger[2]\bfseries#1
      \end{center}}{%
    \end{minipage}\end{center}}

\newenvironment{abstract}{\begin{headedtextblock}{Abstract}}{\end{headedtextblock}}
\newenvironment{acknowledgements}{\begin{headedtextblock}{Acknowledgements}}{\end{headedtextblock}}

\newtheorem{theorem}{Theorem}[chapter]
\newtheorem{lemma}[theorem]{Lemma}
\newtheorem{proposition}[theorem]{Proposition}
\newtheorem{corollary}[theorem]{Corollary}
\newtheorem{fact}[theorem]{Fact}
\newtheorem{conjecture}[theorem]{Conjecture}

\newtheorem{observation}{Observation}

\theoremstyle{definition}
\newtheorem{definition}[theorem]{Definition}
\newtheorem{example}[theorem]{Example}

\newcommand{\Cn}{\ensuremath{\mathit{Cn}}}
\newcommand{\Cnsem}{\ensuremath{\mathit{Cn}^\sem}}
\newcommand{\Cnsemp}{\ensuremath{\mathit{Cn}^{\sem'}}}
\newcommand{\CnsemC}{\ensuremath{\mathit{Cn}^{\sem_C}}}

\newcommand{\T}{\mathcal{T}}
\renewcommand{\K}{\mathcal{K}}
\renewcommand{\C}{\mathcal{C}}

\newcommand{\To}{\implies}
\newcommand{\oT}{\impliedby}
\newcommand{\ToT}{\iff}
\newcommand{\eqdef}{\mathrel{\mathop:}=}
\newcommand{\pair}{[B,S]}

\newcommand{\m}[1]{\ensuremath{\mathcal{#1}}}

\newcommand{\interm}[2]{$#1$-$#2$-intermediate}

\newcommand{\nop}[1]{}
\newcommand{\card}[1]{\ensuremath{\left|#1\right|}}
\newcommand{\N}{\mathbb{N}}

\renewcommand{\r}{\mathfrak{r}}

\renewcommand{\Mod}{\operatorname{Mod}}

\newcommand{\AF}{\mathit{F}}

\renewcommand{\AH}{\mathit{H}}
\renewcommand{\E}{\mathit{E}}

\renewcommand{\k}{\mathit{k}}
\newcommand{\iden}{\mathit{id}}

\newcommand{\AG}{\mathit{G}}

\newcommand{\F}{\mathit{F}}
\newcommand{\G}{\mathit{G}}
\renewcommand{\H}{\mathit{H}}

\newcommand{\I}{\mathit{I}}

\newcommand{\A}{\mathcal{A}}

\newcommand{\stable}{{\mathit{stb}}}
\newcommand{\stb}{\stable}
\newcommand{\adm}{\mathit{ad}}
\newcommand{\prf}{\mathit{pr}}
\newcommand{\grd}{\mathit{gr}}
\newcommand{\com}{\mathit{co}}
\newcommand{\stg}{\mathit{stg}}
\newcommand{\stgzwei}{\mathit{stg2}}
\newcommand{\nav}{\mathit{na}}

\newcommand{\id}{\mathit{il}}
\newcommand{\semi}{\mathit{ss}}
\newcommand{\eag}{\mathit{eg}}
\newcommand{\cf}{\mathit{cf}}
\newcommand{\cfzwei}{\mathit{cf2}}
\newcommand{\sad}{\textit{sad}}

\newcommand{\sta}{\textit{sta}}
\newcommand{\sm}{\setminus}
\renewcommand{\L}{\mathcal{L}}
\newcommand{\M}{\mathit{M}}
\newcommand{\Ext}{\m{E}}
\newcommand{\Lab}{\m{L}}
\newcommand{\XC}{\mathbb{X}}
\newcommand{\LI}{\L^\text{\tiny I}}
\newcommand{\LO}{\L^\text{\tiny O}}
\newcommand{\LU}{\L^\text{\tiny U}}
\newcommand{\MI}{\M^\text{\tiny I}}

\newcommand{\ddef}{\textit{def}}

\newcommand{\CAF}{\textit{CAF}}
\newcommand{\XAF}{\textit{XAF}}

\newcommand{\Def}{\mathcal{D}}
\newcommand{\CDef}{\mathcal{C\!D}}

\newcommand{\ssucc}{\operatorname{s}}

\newcommand{\set}[1]{\left\{#1\right\}}

\newcommand{\powerset}[1]{2^{#1}}

\newcommand{\Ss}{\mathbb{S}}
\newcommand{\Tt}{\mathbb{T}}

\newcommand{\Args}{\textit{Args}}
\newcommand{\Pairs}{\textit{Pairs}}
\newcommand{\dcl}{\textit{dcl}}

\newcommand{\Card}[1]{\left\|#1\right\|}

\newcommand{\lp}{\mathit{Loop}}
\newcommand{\lpaf}{\lp(\F\!,\!\G)}
\newcommand{\att}{\mathit{Att}}

\newcommand{\attsigma}{\att^{\sigma}}

\newcommand{\attadaf}{\att^{\adm}(\F\!,\!\G)}

\newcommand{\attsigmaaf}{\att^{\sigma}(\F\!,\!\G)}

\newcommand{\tvt}{\mathbf{t}}
\newcommand{\tvf}{\mathbf{f}}

\newcommand{\two}{\set{\tvt,\tvf}}

\newcommand{\cleq}{\sqsubseteq}
\newcommand{\cgeq}{\sqsupseteq}

\newcommand{\ccap}{\sqcap}
\newcommand{\bigccup}{\bigsqcup}
\newcommand{\bigccap}{\bigsqcap}

\mathchardef\mathhyphen="2D 

\newcommand{\citet}[1]{\citeA{#1}}
\newcommand{\citep}[1]{\cite{#1}}

\newcommand{\secsme}{covered\xspace} 

\newcommand{\define}[1]{\emph{#1}}

\newcommand{\guard}{\ \middle\vert\ }

\renewcommand{\eqdef}{=}

\renewcommand{\mod}{\ensuremath{\mathit{mod}}}

\newcommand{\phiff}{\mathrel{\phantom{\iff}}}
\newcommand{\pheq}{\mathrel{\phantom{=}}}

\newcommand{\dcup}{\mathrel{\dot{\cup}}}
\newcommand{\dbigcup}{\mathop{\dot{\bigcup}}}
\newcommand{\dsub}{\mathrel{\dot{\subseteq}}}

\newcommand{\LP}{\ensuremath{\mathit{LP}}}
\newcommand{\LLPfin}{\ensuremath{{\left(2^{\L_\LP}\right)_{\!\fin}}}}
\newcommand{\AFl}{\ensuremath{\mathit{AF}}}
\newcommand{\Can}{\ensuremath{\mathit{C}}}
\newcommand{\DL}{\ensuremath{\mathit{DL}}}

\newcommand{\sem}{\sigma}
\newcommand{\semc}{\sigma'}
\ifdefined\th\renewcommand{\th}{\tau}\else\newcommand{\th}{\tau}\fi
\renewcommand{\hat}{\widehat}

\newcommand{\fin}{\textrm{\smaller fin}}

\newcommand{\semf}{\sigma_{\!\fin}}
\newcommand{\semfc}{\sigma'_{\!\fin}}
\newcommand{\semfcc}{\sigma''_{\!\fin}}
\newcommand{\Lfin}{\left(2^\L\right)_{\!\fin}}
\newcommand{\LAFfin}{\left(2^{\L_\AFl}\right)_{\!\fin}}

\newcommand{\dom}{\mathop{dom}}

\newcommand{\naf}{\mathop{\sim\!}}

\newcommand{\SM}{\ensuremath{\mathit{stb}}}
\newcommand{\SU}{\ensuremath{\mathit{sup}}}

\newcommand{\default}[3]{\ensuremath{{#1}:{#2}/{#3}}}

\usepackage{ifthen}

\newcommand{\dkeyword}[1]{\text{\underbar{\normalfont\bfseries\ttfamily#1}}}

\newcommand{\deffect}[3]{\dkeyword{action}\ #1\ \dkeyword{causes}\ #3\ifthenelse{\equal{#2}{\top}\or\equal{#2}{}}{}{\ \dkeyword{if}\ #2}}
\newcommand{\dsd}[2]{\dkeyword{normally}\ #2\ifthenelse{\equal{#1}{\top}\or\equal{#1}{}}{}{\ \dkeyword{if}\ #1}}

\newcounter{cexample}

\begin{document}


\pagenumbering{roman}
\pagestyle{plain}

 \begin{titlepage}
 \begin{center}
   \vfill
   \outerspace
   {\LARGE On the Existence of Characterization Logics and Fundamental Properties of Argumentation Semantics
	}\\
  
   \innerspace
 {\Large Ringo Baumann}\\
   \innerspace
   
   \outerspace
   \vfill
 \end{center}
 \end{titlepage}

\begin{abstract} \label{abstract}

Given the large variety of existing logical formalisms it is of utmost importance
to select the most adequate one for a specific purpose, e.g. for representing the
knowledge relevant for a particular application or for using the formalism as a
modeling tool for problem solving. Awareness of the nature of a logical formalism, 
in other words, of its fundamental intrinsic properties, is indispensable and
provides the basis of an informed choice. 

One such intrinsic property of logic-based knowledge representation languages is the context-dependency of pieces of knowledge. In classical propositional logic, for example, there is no such context-dependence: whenever two sets of formulas are
equivalent in the sense of having the same models (\textit{ordinary equivalence}), then they are mutually replaceable in arbitrary contexts (\textit{strong equivalence}). However, a large number of commonly used
formalisms are not like classical logic which leads to a series of interesting developments. It turned out that sometimes, to characterize strong equivalence in formalism~\m{L}, we
can use ordinary equivalence in formalism~\m{L'}: for example, strong equivalence in normal
logic programs under stable models can be characterized by the standard semantics of the
logic of here-and-there. Such results about the existence of \textit{characterizing logics} has rightly been recognized
as important for the study of concrete knowledge representation formalisms and raise a fundamental question: \textit{Does every formalism have one?} In this thesis, we answer this question
with a qualified ``yes''. More precisely, we show that
the important case of considering only finite knowledge bases guarantees the existence of a canonical characterizing formalism. Furthermore, we argue that those characterizing formalisms can
be seen as \textit{classical}, monotonic logics which are uniquely
determined (up to isomorphism) regarding their model theory.

The other main part of this thesis is devoted to argumentation semantics which
play the flagship role in Dung's abstract argumentation theory. Almost
all of them are motivated by an easily understandable intuition of what
should be acceptable in the light of conflicts. However, although these
intuitions equip us with short and comprehensible formal definitions it
turned out that their intrinsic properties such as \textit{existence and uniqueness}, \textit{expressibility}, \textit{replaceability} and \textit{verifiability} are not that easily accessible. We review the mentioned properties for almost all semantics available in the literature. In doing so we include two main axes: namely first,
the distinction between extension-based and labelling-based versions and
secondly, the distinction of different kind of argumentation frameworks
such as finite or unrestricted ones.

\end{abstract}

\begin{acknowledgements} \label{acknowledgements}

Many factors have contributed in one way or another to the successful
completion of this thesis. Among them are: scientific challenges, an inspiring working atmosphere, thinking, asking questions, fruitful discussions with colleagues, research stays, listening to and giving talks at conferences, writing papers, getting them accepted or rejected, comments of anonymous reviewers, giving lessons, answering student questions, supervision of degree theses, watching and doing sports, listening to and making music and of course, spending time with friends and family. 

I was lucky to meet and get in contact with many inspiring colleagues. Some of them deserve a separate mentioning as they influenced my own research in a variety of ways: Pietro Baroni, Gerd Brewka, Wolfgang Dvo\v{r}\'ak, Dov Gabbay, Heinrich Herre, Thomas Linsbichler, Frank Loebe, Christof Spanring, Hannes Straß and Stefan Woltran.

\end{acknowledgements}


\cleardoublepage
\hypertarget{toc}{\phantomsection}
\label{toc}
\tableofcontents

\cleardoublepage
\hypertarget{lof}{\phantomsection}
\label{lof}
\listoffigures



\cleardoublepage
\pagenumbering{arabic}
\pagestyle{fancy}


\chapter{Introduction}
\label{sec:intro}

Given the large variety of existing logical formalisms it is of utmost importance to select the most adequate one for a specific purpose, e.g. for representing the knowledge relevant for a particular application or for using the formalism as a modeling tool for problem solving.
Awareness of the nature of a logical formalism or, in other words, of its fundamental intrinsic properties, is indispensable and provides the basis of an informed choice. Apart from a deeper understanding of the considered formalism, the study of such intrinsic properties can help to identify interesting fragments or to develop useful extensions of a formalism. Moreover, the obtained insights can be used to refine existing algorithms or even give rise to new ones.

Presumably, the best-known intrinsic property of logics is \textit{monotonicity}. \textit{Monotonic logics} like propositional logic or first order logic are perfectly suitable for the formalization of universal truths since in these logics, whenever a formula $\phi$ is a logical consequence of a theory $T$, it remains true forever and without any exception even
if we add further information to $T$. 
In contrast, formalisms which do not satisfy monotonicity, commonly referred to as \textit{nonmonotonic logics}, allow for defeasible reasoning, i.e.\ it is possible to withdraw former
conclusions (cf. \cite{Bre92,GabHR94} for excellent overviews). Both kinds of logics have their traditional application domains and apart from this fundamental choice there are many other criteria of comparison influencing the decision which logic or which specific semantics of a logic to use in a certain context.

One of the first intrinsic properties which comes to mind is \textit{computational complexity}, i.e.\ how expensive it is to solve typical decision problems in the candidate formalism. A related issue is \textit{modularity} which is, among other things, concerned with the question of whether it is possible to divide a given theory in subtheories, s.t.\ the formal semantics of the entire theory can be obtained by constructing the semantics of the subtheories. Both topics were studied in-depth for mainstream nonmonotonic formalisms such as default logic \cite{Got92,Tur96}, logic programming under certain semantics \cite{DanEGV97,LifT94} as well as abstract argumentation frameworks (AFs) under various argumentation semantics \cite{DvoD18,BarGL18}.

In this habilitation treatise we elaborate a theory of four further intrinsic properties of abstract argumentation semantics. 
In brief, Dung-style AFs consist of arguments and attacks which are treated as primitives, i.e., the internal
structure of arguments is not considered. 
The major focus
is on resolving conflicts. 
To this end a variety of semantics
have been defined, each of them specifying acceptable sets
of arguments, so-called extensions, in a particular way.
Another main approach to argumentation semantics are so-called labelling-based semantics which contain more information then their extension-based counterparts. A labelling explicitly classifies each argument either as accepted, rejected or undecided.

\pagebreak
\begin{enumerate}

\item \textit{existence and uniqueness}\quad Is it possible, and if so how, to \textit{guarantee} the existence of at least one or exactly one extension/labelling by considering the structure of a given AF $\AF$ only? \hfill{(Chapter~\ref{cha:ex})}

\item \textit{expressibility}\quad Is it possible, and if so how, to \textit{realize} a given candidate set of extensions/labellings within a single AF $\AF$? \hfill{(Chapter~\ref{cha:realizability})}

\item \textit{replaceability}\quad Is it possible, and if so how, to \textit{simplify} parts of a given AF $\AF$ s.t.\ the modified version $\AF'$ and $\AF$ cannot be distinguished semantically by further information
that might be added later to both simultaneously? \hfill{(Chapter~\ref{cha:replace})}

\item \textit{verifiability}\quad Is it possible, and if so how, to \textit{compute} the semantics of an AF $\AF$ unambiguously, given that we are faced with strictly less information than the entire framework $\F$? \hfill{(Chapter~\ref{cha:ver})}

\end{enumerate}

The question whether a certain formalism always provides one with a formal meaning or even with a uniquely determined semantical answer is a crucial factor for its suitability for the application in mind. 
For instance, in problem solving a plurality of solutions may possibly be desired, whereas in decision making one might be interested in guaranteeing a single answer. 
It is well-known that a given theory in propositional logic neither has to possess a model nor, in case of existence, has there to be exactly one. 
The same applies to logic programs under stable model semantics. 
In contrast, a propositional theory of only positive formulae is always satisfiable and definite logic programs constitute a subclass of logic programs where even uniqueness of a model is guaranteed. In Chapter~\ref{cha:ex} we will see that Dung's abstract argumentation semantics behave in a similar way, i.e.\ the existence and uniqueness of extensions/labellings depend on structural restrictions of argumentation frameworks.

Expressibility is concerned with the expressive power of logical formalisms. The question here is which kinds of models are realizable, that is, can be the set of models of a single knowledge base of the formalism. This is a decisive property from an application angle since potential necessary or sufficient properties of model sets may rule out a logic or make it perfectly appropriate for representing certain solutions. For instance, it is well-known that in case of propositional logic every finite set of two-valued interpretations is realizable. This means, given such a finite set~\m{I}, we always find a set of formulae $T$, s.t.\ $\Mod(T) = \m{I}$. In case of normal logic programs it is obvious that not all model sets can be expressed, since any set of stable models forms a $\subseteq$-antichain. Remarkably, being such an antichain is not only necessary but even sufficient for realizability w.r.t.\ stable model semantics \cite{EitFPTW13,Str15}. In case of abstract argumentation we are equipped with a high number of semantics, for which we will see in Chapter~\ref{cha:realizability} that characterizing properties are not that easy to find. Moreover, as expected, representational limits highly depend on the chosen semantics.

In case of propositional logic -- in contrast to all nonmonotonic logics available in the literature -- we have that standard equivalence, i.e.\ sharing the same models, even guarantees intersubstitutability in any logical context without loss of information. As an aside, it is not the monotonicity of a certain logic but rather the so-called \textit{intersection property} which guarantees this behavior \cite{BauS16}. Substitutability is of great importance for dynamically evolving scenarios since it allows one to simplify parts of a theory without looking at the rest. For this reason, much effort has been devoted to characterizing \textit{strong equivalence} for nonmonotonic formalisms, such as logic programs~\cite{DBLP:journals/tocl/LifschitzPV01}, causal theories \cite{DBLP:conf/lpnmr/Turner04}, default logic \cite{Turner01} and nonmonotonic logics in general~\cite{DBLP:journals/amai/Truszczynski06,BauS16}. In Chapter~\ref{cha:replace} we will see that characterization theorems in case of abstract argumentation are quite different
from those for the aforementioned formalisms in that being strongly equivalent can be decided by looking at the syntax only.

Verifiability is a topic that is very specific for abstract argumentation. Over the last two decades a bunch of argumentation semantics were introduced \cite{BaroniCG18}. The
motivations of these semantics range from the desired treatment of specific examples
to fulfilling a number of abstract principles. Mathematically speaking, a semantics takes as input an AF and returns acceptable sets of arguments.
Verifiability deals with the question whether we really need the entire AF to compute a certain argumentation semantics $\sigma$? In other words, is it possible to unambiguously determine acceptable sets  w.r.t.\ $\sigma$, given only partial information of the underlying framework $\AF$. In Chapter~\ref{cha:ver} we will see that most of the existing semantics can be computed with indeed less information than the entire framework. We will
categorize the amount of information required by taking the conflict-free sets as a basis and distinguishing between
different amounts of knowledge about the neighborhood of these sets, so-called \textit{verification
classes}. These classes allow us to categorize semantics 
with respect to the information needed to verify whether a
certain set of arguments is acceptable. The study of this topic as well as the properties mentioned before pave the way for a more general view on argumentation
semantics, their 
common features, and their inherent differences. 

Let us return to knowledge representation formalisms in general. Two knowledge bases are strongly equivalent if and only if they are mutually interchangeable in arbitrary contexts. This notion is of high interest for any logical formalism since it allows one
to locally replace, and thus potentially give rise to simplification, parts of a given theory
without changing the semantics of the latter.
As already said, since it is possible to find
ordinarily but not strongly equivalent objects for any nonmonotonic formalism available in the literature a lot of research has been devoted to characterizing strong equivalence.
For example, it turned out that strong equivalence for logic programs under stable models can be characterized by so-called HT-models \citep{DBLP:journals/tocl/LifschitzPV01}. 
More precisely, two logic programs are strongly equivalent
if and only if they are standard equivalent in the logic of here-and-there. This means the logic of here-and-there can be seen as a characterizing formalism for logic programs under stable model semantics.
In Chapter~\ref{cha:intersect} we will study whether the existence of such characterization logics can be guaranteed for every logic.
One main result is that every knowledge representation formalism that allows for a notion of strong equivalence on its finite knowledge bases also possesses a canonical characterizing formalism. In particular, we argue that those characterizing formalisms can be seen as classical, monotonic logics which are uniquely determined (up to isomorphism) regarding their model theory.

\section{Publications}

Most of the results presented in this habilitation thesis have already been published as conference papers or
journal articles, as a chapter of a handbook as well as in a book in honour of Gerhard Brewka. Moreover, an extended version of \cite{BauS16} is currently under review in the Journal of Artificial Intelligence Research (JAIR). In the
following we list the involved publications together with the chapters or sections where they
are mainly used. 

\begin{itemize}
\item conference papers
	
	\begin{enumerate}
	\item  \cite{Bau14} \hfill{(Sections \ref{sec:furthequi} and \ref{sec:exceptdel})}\\
		\textit{Context-free and context-sensitive kernels: Update and deletion equivalence in abstract argumentation}\\
		European Conference on Artificial Intelligence (ECAI)
		\item \cite{compact} \hfill{(Section \ref{sec:sigcompact})}\\
		\textit{Compact argumentation frameworks}\\
		European Conference on Artificial Intelligence (ECAI)
		\item \cite{BauB15} \hfill{(Section \ref{sec:argu})}\\
		\textit{AGM meets abstract argumentation: Expansion and revision for Dung frameworks}\\
		International Joint Conference on Artificial Intelligence (IJCAI)
		\item \cite{Bau16} \hfill{(Section~\ref{sec:equivlabel})}\\
\textit{Characterizing equivalence notions for labelling-based semantics}\\
 International Conference on Principles of Knowledge Representation and Reasoning (KR)
\item \cite{BauS16} \hfill{(Chapters \ref{cha:intersect})}\\
		 \textit{An abstract logical approach to characterizing strong equivalence in logic-based knowledge representation formalisms}\\
		International Conference on Principles of Knowledge Representation and Reasoning (KR)
\item \cite{BauLW16a} \hfill{(Chapter~\ref{cha:ver})}\\
\textit{Verifiability of argumentation semantics}\\
 Computational Models of Argument (COMMA)
		
		\item \cite{BauS17} \hfill{(Sections \ref{sec:exunrestrict}, \ref{sec:sigunrestrict} and \ref{chap:equiunrestr})}\\
		 \textit{A study of unrestricted abstract argumentation frameworks}\\
		International Joint Conference on Artificial Intelligence (IJCAI)

\end{enumerate}
\item journal articles
	
	\begin{enumerate} \setcounter{enumi}{7}
		\item \cite{compactj} \hfill{(Sections \ref{sec:sigcompact} and \ref{sec:analytic})}\\
		\textit{On rejected arguments and implicit conflicts: The hidden power of argumentation semantics}\\ 
Journal of Artificial Intelligence (AIJ)
\end{enumerate}
	\item book contribution
	
	\begin{enumerate} \setcounter{enumi}{8}
		\item \cite{Bau18} \mbox{} \hfill{(Chapters \ref{cha:ex}, \ref{cha:realizability}, \ref{cha:replace} and \ref{cha:ver})}\\
		\textit{On the nature of argumentation semantics: Existence and uniqueness, expressibility, and replaceability}\\
	  Handbook of Formal Argumentation
		\item \cite{BauS15} \hfill{(Chapters \ref{cha:ex})}\\
		\textit{Infinite argumentation frameworks - On the existence and uniqueness of extensions}\\
		Advances in Knowledge Representation, Logic Programming, and Abstract Argumentation -- Essays Dedicated to Gerhard Brewka on the Occassion of His 60th Birthday
	\end{enumerate}

\end{itemize}

\chapter{On the Existence of Characterization Logics for Knowledge Representation Formalisms}

\label{cha:intersect}

Reusability of human-made artifacts is of paramount importance in computer science.
To assess the reusability of (parts of) knowledge bases in logic-based knowledge representation languages, we have to know whether pieces of knowledge make certain context-dependent assumptions.
In classical propositional logic, for example, there is no such context-dependence:
whenever two (sets of) formulas are equivalent in the sense of having the same models, then they are mutually replaceable in arbitrary contexts.

In the field of knowledge representation and reasoning, however, not all commonly used formalisms are like classical logic in this regard.
For example, the answer-set-programming paradigm uses the formalism of normal logic programs to encode combinatorial problems such that the answer sets (stable models) of the logic programs correspond to and encode solutions to the encoded problem~\cite{gebser12asp}.
Alas, for normal logic programs, having the same stable models does \emph{not} amount to mutual replaceability.
For this, a stronger property is needed: it is called \define{strong equivalence}, and holds for two logic programs if and only if they keep the same stable models even if they are both extended with an arbitrary third logic program.
Formally -- as a logic program is a set of rules --, extending a logic program can be modeled by ordinary set union.
Consequently, the notion of strong equivalence can be defined in a similar way for other knowledge representation formalisms for which set union is an adequate formalization of appending or otherwise combining knowledge bases.

In a series of interesting developments, researchers have succeeded in precisely characterizing strong equivalence for several formalisms, among them normal logic programs under the stable model semantics~\citep{DBLP:journals/tocl/LifschitzPV01,Turner01}.
What is more, it turned out that sometimes, to characterize strong equivalence in formalism $\F$, we can use ordinary equivalence in formalism $\F'$:
for example, strong equivalence in normal logic programs under stable models can be characterized by the standard semantics of the logic of here-and-there~\citep{DBLP:journals/tocl/LifschitzPV01}.
Such results have rightly been recognized as important for the study of these concrete knowledge representation formalisms, as having a characterization of strong equivalence gives us a deeper level of understanding of the meaning of pieces of knowledge in that formalism.

However, such results about the existence of characterizing formalisms also raise a fundamental question:
\emph{Does every formalism have one?}
In the following we will see that the answer to this question is a qualified ``yes''.
More precisely, while not every formalism has one, we show that the important case of considering only \emph{finite} knowledge bases (but still possibly infinite languages) guarantees the existence of a characterizing formalism, and that in a very general setting.
Existing results on characterizing formalisms make use of specifics of each formalism~\cite{DBLP:journals/tocl/LifschitzPV01,Turner01,DBLP:conf/aaai/Truszczynski07,DBLP:conf/aaai/TruszczynskiW08,DBLP:conf/kr/CabalarD14}.
The presented approach completely abstract away from formalism specifics and addresses the core of the problem, the nature of strong equivalence itself.
In fact, we will not only show the existence of just any characterizing formalism, but of characterizing formalisms whose model theory is \emph{uniquely determined} (up to isomorphism), and structurally resembles that of classical logics.
At this point, we appeal to the reader's intuition on what makes logics classical;
we will later define what we mean by ``classical logic'' in a precise mathematical way.
From our point of view, the main is a surprising and important insight, as it tells us that for the overwhelming majority of knowledge representation formalisms, strong equivalence can be approached using established techniques from classical logic.

While our work is in its essence derived from first principles, building mostly upon classical logic and lattice theory, there have been important inspirations.
Foremost, \cite{DBLP:journals/amai/Truszczynski06} presented a general, algebraic account of strong equivalence within approximation fixpoint theory~\cite{denecker00approximations}.
His setting is indeed quite general, but most of this generality derives from algebraic commonalities in the semantics of logic programs and default logic. 
It is not immediately clear, for example, if and how it captures Dung's theory of abstract argumentation \cite{Dung95}, which will be discussed in depth in the subsequent chapters.
More precisely, while argumentation frameworks (AFs) with all their semantics can be captured by approximation fixpoint theory~\cite{strass13approximating}, Truszczy{\'n}ski's notion of expanding an operator does not coincide with the corresponding notion of expanding AFs and his results are not directly applicable.
In other words, while the operator associated to the union of two logic programs corresponds to the union of their respective associated operators, the same does not hold for the (componentwise) union of two AFs and their operators.
Thus although AFs are essentially a restricted subclass of normal logic programs with respect to the ordinary equivalence of having the same models, this does not carry over to strong equivalence because the respective notions of knowledge base union are different in AFs and normal logic programs.
For example, the AF where $b$ attacks $a$ corresponds to the logic program \mbox{$P_1=\set{a\gets\naf b}$};
likewise, \mbox{$P_2=\set{a\gets\naf c}$} corresponds to the AF where $c$ attacks $a$~\cite{osorio05inferring,wu09completeextensions,strass13approximating}.
However, the AF where both $b$ and $c$ attack $a$ (the union of the two AFs above) corresponds to the logic program \mbox{$P_3=\set{a\gets \naf b,\,\naf c}$}, where obviously the three programs are not subset-related: \mbox{$P_1\not\subseteq P_3$} and \mbox{$P_2\not\subseteq P_3$}.

In contrast, we show how the approach we develop in this paper can be directly applied to argumentation frameworks.
Thus as a consequence of our main theorem, we get the first semantical characterization of strong equivalence in AFs in contrast to the currently known syntactical characterizations \cite{strong}.

The chapter proceeds as follows.
In the next section, we introduce the general setting in which we derive our results and present our conception of the term ``classical logic''.
Afterwards, we define characterization logics and show two classes of formalisms that always possess them.
We next apply our results, chiefly to abstract argumentation frameworks but briefly also to normal logic programs. Finally, we conclude the chpater with a discussion of related and future work.

\section{An Abstract View on Model Theory}
\label{sec:model-theory}

What is a classical logic?

We will spend this section introducing an abstract notion of logics with model-theoretic semantics and explaining when we call some of them classical.
Formally, we consider logical languages $\L$, that is, non-empty sets of language elements.
We make no assumption on the internal structure of pieces of knowledge \mbox{$f\in\L$}.
These pieces of knowledge could be formulas of classical propositional logic, normal logic program rules, attacks between arguments, or Reiter-style defaults.
A model-theoretic semantics for a language $\L$ uses a set $\I$ of interpretations and a \emph{model} function \mbox{$\sem : 2^\L \to 2^\I$} with the intuition that $\sem$ assigns each language subset \mbox{$T\subseteq\L$}, a \define{theory}, the set $\sem(T)$ of its models.
We make no assumptions on the internal structure of interpretations -- there need not be an underlying vocabulary of atoms or the like (although in the concrete cases we consider there often will be) that are the same among syntax and semantics.
This is the main abstraction in our setting.
It goes beyond what is known from classical logic in that meaning is not assigned to language elements (formulas), but only to \emph{theories}, that is, sets of language elements.
This is a necessary requirement for being able to model a number of established knowledge representation formalisms:
for example, in normal logic programs, meaning is not assigned to single rules, but only to sets thereof.
Likewise, in default logic, meaning is not assigned to single defaults, but only to sets thereof.
We illustrate our definitions so far by showing more precisely how existing formalisms can be embedded into our setting.

\begin{example}
  \label{exm:known-formalisms}
  Consider a set $\A$ of propositional atoms.
  \begin{description}
  \item[\normalfont Classical propositional logic:]
    The underlying language $\L_\PL$ is the set of all classical propositional formulas over $\A$ and can be defined as usual by induction.
    The set of interpretations is then given by the set \mbox{$\I_\PL=\set{ v:\A\to\two }$} of all two-valued interpretations of $\A$.
    Lastly, $\sem_\mod(T)$ is the set of all models of the theory \mbox{$T\subseteq\L_\PL$}, that is, the set of all interpretations satisfying all formulas in $T$.
  \item[\normalfont Normal logic programs:]
    The underlying language $\L_\LP$ is the set of all normal logic program rules \linebreak \mbox{$a_0\gets a_1,\ldots,a_m,\naf a_{m+1},\ldots,\naf a_n$} with \mbox{$0\leq m\leq n$} and \mbox{$a_0,a_1,\ldots,a_n\in\A$}.
    The set $\I_\LP$ of interpretations is then the set \mbox{$\I_\LP=2^\A$} of all possible stable model candidates.
      Accordingly, $\sem_\SM(T)$ returns the set of stable models of the theory (normal logic program) \mbox{$T\subseteq\L_\LP$}~\cite{gelfond-lifschitz88thestablemodel};
      we could also define $\sem_\SU(T)$ returning the set of supported models of $T$~\citep{clark78negation}, or $\sem_\mod(T)$ the set of classical models of $T$, interpreting $\gets$ as material implication and $\naf$ as classical negation~\citep{vanemden76thesemantics}.%
    \item[\normalfont Propositional default logic~\citep{reiter80alogic}:]
    The underlying language $\L_\DL$ is the set of all defaults over formulas over $\A$, that is, \mbox{$\L_\DL=\set{ (\varphi,\Psi,\xi) \guard \varphi,\xi\in\L_\PL, \Psi\subseteq\L_\PL }$} where the triple \linebreak $(\varphi,\set{\psi_1,\ldots,\psi_n},\xi)$ represents the default \mbox{$\default{\varphi}{\psi_1,\ldots,\psi_n}{\xi}$}.
    The possible interpretations are given by \mbox{$\I_\DL=2^\L$}, as each default theory is assigned a set of logical theories called extensions:\linebreak
    \mbox{$\sem_\DL(T)=\set{ E\subseteq\L \guard E \text{ is an extension of } T }$}.
  \item[\normalfont Abstract argumentation frameworks~\citep{Dung95}:]
    The language $\L_\AFl$ contains the fundamental building blocks of AFs, i.e.\ arguments and attacks:
    $\L_\AFl=\set{ (\set{a},\emptyset), (\set{a,b},\set{(a,b)}) \guard a,b\in\A}.$
    Extension-based semantics can be incorporated by setting \mbox{$\I_\AFl = 2^\A$} and, depending on the argumentation semantics $\rho$ we use, we set \mbox{$\sem_\rho(T)=\rho(F_T)$}, where \mbox{$F_T=(\bigcup_{(A,R)\in T}A, \bigcup_{(A,R)\in T}R)$} represents the AF associated to \mbox{$T\subseteq\L_\AFl$} (cf.\ Section~\ref{sec:argu} for more detailed information and discussion).
  \end{description}
\end{example}

Before we delve into the main aim, namely characterizing strong equivalence, we briefly analyze some foundational properties of our way of abstractly modeling knowledge representation languages.
We begin with the relationship between the model function and the consequence function of a logical language.

\subsection{Models and Consequences}
\label{sec:models-and-consequences}

A consequence function for a language $\L$ is a function \mbox{$\Cn:2^\L\to 2^\L$} that assigns a given set $T$ of language elements another set $\Cn(T)$ of language elements.
Intuitively, $\Cn(T)$ is understood to be the set of logical consequences of the theory $T$.
Given a language, we can define the consequence function in terms of the semantics.
In words, the set of consequences of a given theory $T$ is the union of all theories $S$ such that any model of $T$ is a model of $S$.
\begin{definition}
  \label{def:consequence}
  Let $\L$ be a language and \mbox{$\sem:2^\L\to 2^\I$} be a model function.
  The \define{canonical consequence function of $\sem$} is defined as follows\/:
  \begin{gather*} 
    \Cnsem:2^\L\to 2^\L,\qquad T\mapsto\mathop{\bigcup_{S\subseteq\L,}}_{\sem(T)\subseteq\sem(S)}S
  \end{gather*}
\end{definition}

In classical definitions of logical consequence, one typically says that a single formula is a consequence of a theory if all models of the theory are models of the formula.
In our case, the focus is primarily on theories both for presumptions and consequences, so we stick to the above definition.
For classical logic $\L_\PL$, this definition coincides with the standard notion of logical consequence.

It will be of great interest that certain algebraic properties of the semantics induce certain useful properties of the consequence relation.
We now introduce the most important properties.

\begin{definition}
  \label{def:operator-properties}
  Let $\L$ be a language.
  \begin{itemize}
  \item A model function \mbox{$\sem:2^\L\to 2^\I$} is \define{antimonotone} iff
    \begin{quote}
      for all \mbox{$T_1, T_2\in 2^\L$}:
      \mbox{$T_1\subseteq T_2 \implies \sem(T_2)\subseteq\sem(T_1)$}.
    \end{quote}
  \item A consequence operator \mbox{$\Cn:2^\L\to 2^\L$} is \define{monotone} iff
    \begin{quote}
      for all \mbox{$T_1, T_2\in 2^\L$}:
      \mbox{$T_1\subseteq T_2 \implies \Cn(T_1) \subseteq \Cn(T_2)$}.
    \end{quote}
  \item A consequence operator \mbox{$\Cn:2^\L\to 2^\L$} is \define{increasing} iff
    \begin{quote}
      for all \mbox{$T\in 2^\L$}, we find \mbox{$T\subseteq \Cn(T)$}.
    \end{quote}
    \item A consequence operator \mbox{$\Cn:2^\L\to 2^\L$} is \define{idempotent} iff
      \begin{quote}
        for all \mbox{$T\in 2^\L$}, we find \mbox{$\Cn(\Cn(T))\subseteq\Cn(T)$}.
      \end{quote}
    \item A consequence operator \mbox{$\Cn:2^\L\to 2^\L$} is a \define{closure operator} iff
      \begin{quote}
        $\Cn$ is monotone, increasing, and idempotent.
      \end{quote}
  \end{itemize}
\end{definition}

In what follows, we define a \emph{logic} as a tuple $(\L,\I,\sem)$ consisting of a language $\L$, an interpretation set $\I$, and a model function \mbox{$\sem : 2^\L \to 2^\I$}.
We will sometimes associate the canonical consequence function of $\sem$ to the whole logic for convenience.

To start out in analyzing how semantics and consequence relate in our setting, we show some straightforward properties of the canonical consequence function defined above. More precisely, any induced consequence operator is increasing and moreover, it is antimonotone if the considered model function is antimonotone.
\begin{proposition}
  \label{thm:antimonotone-monotone-increasing}
  Let \mbox{$(\L,\I,\sem)$} be a logic and $\Cnsem$ its canonical consequence function.
  \begin{enumerate}
  \item\label{itm:increasing} $\Cnsem$ is increasing.
  \item\label{itm:antimonotone-monotone} If $\sem$ is antimonotone, then $\Cnsem$ is monotone.
  \end{enumerate}
	\pagebreak
  \begin{proof} \hfill
    \begin{enumerate}
    \item
      Clearly for any \mbox{$T\subseteq\L$} we find \mbox{$\sem(T)\subseteq\sem(T)$} whence \mbox{$\displaystyle T\subseteq\mathop{\bigcup_{S\subseteq\L,}}_{\sem(T)\subseteq\sem(S)}S=\Cnsem(T)$}.
    \item Let $\sem$ be antimonotone and \mbox{$T_1\subseteq T_2$}.
      Then clearly \mbox{$\sem(T_2)\subseteq\sem(T_1)$} by antimonotonicity.
      Thus for any \mbox{$S\subseteq\L$} with \mbox{$\sem(T_1)\subseteq\sem(S)$} we also find \mbox{$\sem(T_2)\subseteq\sem(S)$}.      
      It follows that $$\mathop{\bigcup_{S\subseteq\L,}}_{\sem(T_1)\subseteq\sem(S)}S\subseteq\mathop{\bigcup_{S\subseteq\L,}}_{\sem(T_2)\subseteq\sem(S)}S,$$ that is, \mbox{$\Cnsem(T_1)\subseteq\Cnsem(T_2)$}.
      Since $T_1$ and $T_2$ were arbitrary, $\Cnsem$ is monotone.
    \end{enumerate}
  \end{proof}
\end{proposition}

The reverse of Item~\ref{itm:antimonotone-monotone} does not hold, that is, monotone consequence functions do not necessitate antimonotone model functions.

\begin{example}
  \label{exm:monotone-increasing--not-antimonotone}
  Consider \mbox{$\L=\set{a}$} and \mbox{$\I=\set{1}$} with
  \mbox{$\sem(\emptyset)=\emptyset$} and \mbox{$\sem(\set{a})=\set{1}$}.
  We get the following canonical consequence function\/:
  \begin{align*}
    \Cnsem(\emptyset) &= \mathop{\bigcup_{S\subseteq\L,}}_{\sem(\emptyset)\subseteq\sem(S)} S = \mathop{\bigcup_{S\subseteq\L,}}_{\emptyset\subseteq\sem(S)}S = \bigcup_{S\subseteq\L}S = \L = \set{a} \\
    \Cnsem(\set{a}) &= \mathop{\bigcup_{S\subseteq\L,}}_{\sem(\set{a})\subseteq\sem(S)} S = \mathop{\bigcup_{S\subseteq\L,}}_{\set{1}\subseteq\sem(S)}S = \set{a}
  \end{align*}
  Thus \mbox{$\Cnsem(\emptyset)=\set{a}=\Cnsem(\set{a})$} and $\Cnsem$ is monotone and increasing albeit $\sem$ is not antimonotone.
\end{example}

Also, not every antimonotone model function induces a closure operator, that is, an operator that is monotone, increasing and idempotent.
\begin{example}
  \label{exm:antimonotone-idempotent}
  Consider the logic $(\L,\I,\sem)$ with language \mbox{$\L=\set{a,b,c}$}, interpretation set \mbox{$\I=\set{1}$} and model function \mbox{$\sem:2^\L\to 2^\I$} given by
  \begin{gather*}
    \sem(T) =
    \begin{cases}
      \set{1} & \text{if } T\in\set{\emptyset,\set{a},\set{b}} \\
      \emptyset & \text{otherwise}
    \end{cases}
  \end{gather*}
  In addition to $\sem$ being antimonotone we have \mbox{$\Cnsem(\emptyset)=\set{a,b}$} and \mbox{$\Cnsem(\set{a,b})=\set{a,b,c}$}, whence
  \begin{gather*}
    \set{a,b,c}=\Cnsem(\set{a,b})=\Cnsem(\Cnsem(\emptyset))\not\subseteq\Cnsem(\emptyset)=\set{a,b}
  \end{gather*}
  which shows that $\Cnsem$ is not idempotent.
\end{example}

We can show that our restriction to semantics via model functions is not overly limiting.
We could have chosen to start out from consequence functions as well, as long as these consequence functions are increasing (what is contained in a theory follows from it) and idempotent (all consequences are obtained in one step).
More precisely, when given a consequence operator satisfying these two requirements, we can also define a canonical model function whose associated canonical consequence function is exactly the given consequence function we started with.

\begin{proposition}
  \label{thm:con-sem}
  Let $\L$ be a language and \mbox{$C:2^\L\to 2^\L$} be a consequence function which is increasing and idempotent.
	
  Then the model function \mbox{$\sem_C:2^\L\to 2^\L$} with \mbox{$T \mapsto \L\setminus C(T)$} is such that \mbox{$\CnsemC=C$}.\footnote{Please note that the assigned interpretation set $\I$ of $\sem_C$ is just the language $\L$.}
  \begin{proof}
    Let \mbox{$T\subseteq\L$}.
    We find that
    \begin{align*}
      \CnsemC(T)
      &= \mathop{\bigcup_{S\subseteq\L,}}_{\sem_C(T)\subseteq\sem_C(S)}S & \text{(Def.~$\Cnsem$)} \\
      &= \mathop{\bigcup_{S\subseteq\L,}}_{\L\setminus C(T)\subseteq \L\setminus C(S)}S & \text{(Def.~$\sem_C$)} \\
      &= \mathop{\bigcup_{S\subseteq\L,}}_{C(S)\subseteq C(T)}S & \text{(set algebra)} \\
      &= \mathop{\bigcup_{S\subseteq\L,}}_{S\subseteq C(S)\subseteq C(T)}S & \text{($C$ is increasing)}
    \end{align*}
    Firstly, this shows that \mbox{$\CnsemC(T)\subseteq C(T)$}.
    Moreover, in the last equation we can substitute \mbox{$S=C(T)$} to obtain that \mbox{$C(T)\subseteq\CnsemC(T)$}.
  \end{proof}
\end{proposition}

It is clear from Item~\ref{itm:increasing} of Proposition~\ref{thm:antimonotone-monotone-increasing} that no non-increasing consequence function $C$ can be mimicked by $\Cnsem$ for any $\sem$.
However, we consider the restrictions of possible consequence functions $C$ having to be increasing and idempotent not to be too severe.

\subsection{Standard and strong equivalence}
\label{sec:strong-equivalence}

This paper is chiefly about characterizing strong equivalence in one logic via standard equivalence in another logic.
We will now formally introduce these concepts.

\begin{definition} \label{def:strongord}
  Let $(\L,\I,\sem)$ be a logic and \mbox{$T_1,T_2\subseteq\L$} theories.
  We say that $T_1$ and $T_2$ are
  \begin{itemize}
  \item \define{ordinarily equivalent} if and only if \mbox{$\sem(T_1)=\sem(T_2)$};
  \item \define{strongly equivalent} if and only if \mbox{$\forall U\subseteq\L : \sem(T_1\cup U)=\sem(T_2\cup U)$}.
  \end{itemize}
  
\end{definition}

The notion of strong equivalence is intimately connected with the possibility to simplify parts of a given
theory without affecting its semantics. Consider the following example. 

\begin{example} Given a logic $(\L,\I,\sem)$, a theory $S$ and a subtheory $T_1$ of it, i.e.\ $T_1\subseteq S$. Now, we may replace $T_1$ with any $T_2$ being strongly equivalent to it without changing the semantics of $S$. More precisely, $\sem(S) = \sem(S[T_1|T_2])$ with $S[T_1|T_2] = T_2 \cup (S\sm T_1)$. This can be seen as follows: Since $T_1\subseteq S$ we have $T_1 \cup (S\sm T_1) = S$. Moreover, due to the assumed strong equivalence of $T_1$ and $T_2$ we obtain \linebreak $\sigma(T_1 \cup (S\sm T_1)) = \sigma(T_1 \cup (S\sm T_1))$. Hence, $\sem(S) = \sem(S[T_1|T_2])$ as claimed.
\end{example}

Clearly, strongly equivalent theories are ordinarily equivalent by definition. What about the converse direction? It is a matter of fact that in case of well-known
nonmonotonic formalisms, such as logic programs \cite{DBLP:journals/tocl/LifschitzPV01}, default logic \cite{Turner01}, causal theories \cite{DBLP:conf/lpnmr/Turner04} and abstract argumentation \cite{strong,Bau16} strong equivalence and ordinary equivalence are indeed different concepts. However, there are logics like propositional logic or first order logic where both concepts coincide. In the following we will say that the model function $\sem$ has the \define{replacement property} if ordinary equivalence implies strong equivalence. The following natural question arises: What properties must a logic possess in order for ordinary and strong equivalence to coincide? Propositional as well as first order logic possess a monotone consequence function. Does monotony of the consequence operator ensure the coincidence of both concepts? The following example provides us with a negative answer.

\begin{example}
  \label{exm:monotone-strong}
  Consider the language \mbox{$\L=\set{a,b}$} with interpretation set \mbox{$\I=\set{1,2}$} and model function~$\sem$ given by%
  \begin{align*}
    \sem(\emptyset) &=\set{1,2} \\
    \sem(\set{a})   &=\set{1,2} \\
    \sem(\set{b})   &=\set{2} \\
    \sem(\set{a,b}) &=\emptyset
  \end{align*}
	
	It is easy to verify that the semantics $\sem$ is antimonotone.
  Therefore, by Proposition~\ref{thm:antimonotone-monotone-increasing}, its consequence function \mbox{$\Cnsem$} is monotone.
  However, while $\emptyset$ and $\set{a}$ are obviously ordinarily equivalent, they are not strongly equivalent, which can be seen by extending both with the theory $\set{b}$\/:
  \begin{gather*}
    \sem(\emptyset\cup\set{b})=\sem(\set{b})=\set{2}\neq\emptyset=\sem(\set{a,b})=\sem(\set{a}\cup\set{b})
  \end{gather*}
  We also inspect the induced consequence operator\/:
  \begin{align*}
    \Cnsem(\emptyset)&=\bigcup\set{ S\subseteq\L \guard \set{1,2}\subseteq\sem(S) }=\set{a} \\
    \Cnsem(\set{a})&=\bigcup\set{ S\subseteq\L \guard \set{1,2}\subseteq\sem(S) }=\set{a} \\
    \Cnsem(\set{b})&=\bigcup\set{ S\subseteq\L \guard \set{2}\subseteq\sem(S) }=\set{a,b} \\
    \Cnsem(\set{a,b})&=\bigcup\set{ S\subseteq\L \guard \emptyset\subseteq\sem(S) }=\set{a,b}
  \end{align*}
  Since the codomain of $\Cnsem$ consists entirely of fixpoints, $\Cnsem$ is idempotent.
  Therefore the induced consequence operator $\Cnsem$ is increasing, monotone, and idempotent, thus a closure operator.
  Yet, the inducing semantics does not have the replacement property.
\end{example}%

So having a monotone consequence function is, by itself, insufficient to guarantee the replacement property.
We can however identify a property that is strong enough to guarantee replacement on its own.
We call it the intersection property, because it basically says that the semantics of a theory can be obtained by only considering the semantics of the singleton sets constituting the theory.
\begin{definition}
  \label{def:intersection}
  Let $(\L,\I,\sem)$ be a logic.
  Its model function \mbox{$\sem : 2^\L \to 2^\I$} has the \define{intersection property} iff for all \mbox{$T\subseteq\L$}:
  \begin{gather*}
    \sem(T) = \bigcap_{F\in T}\sem(\set{F})
  \end{gather*}
\end{definition}
It follows from the definition that in particular for any two theories \mbox{$T_1,T_2\subseteq\L$}, we have that\linebreak \mbox{$\sem(T_1\cup T_2)=\sem(T_1)\cap\sem(T_2)$}.
The intersection property is a certain locality, independence, or compositionality criterion.
Towards an explanation of Example~\ref{exm:monotone-strong} we can now remark that its model function~$\sem$ does not have the intersection property\/:
\begin{gather*}
  \sem(\set{a,b})=\emptyset\neq \set{2}=\set{1,2}\cap\set{2}=\sem(\set{a})\cap\sem(\set{b})
\end{gather*}

Indeed, this is necessarily so:
as we will show next (and as is easy to show), satisfying the intersection property is sufficient for satisfying the replacement property.

\begin{proposition}
  \label{thm:ordinary-v-strong}
  Let $(\L,\I,\sem)$ be a logic.
  If $\sem$ satisfies the intersection property, then standard equivalence coincides with strong equivalence.
  \begin{proof}
    Let $\sem$ satisfy the intersection property.
    It is clear that strong equivalence implies ordinary equivalence (set \mbox{$U=\emptyset$}), so it remains to show the converse.
    Let \mbox{$T_1,T_2\subseteq\L$} such that \mbox{$\sem(T_1)=\sem(T_2)$} and consider any \mbox{$U\subseteq\L$}.
    We have
    \begin{align*}
      \sem(T_1\cup U) 
      &= \sem(T_1) \cap \sem(U) & \text{(intersection)} \\
      &= \sem(T_2) \cap \sem(U) & \text{(presumption)} \\
      &= \sem(T_2\cup U) & \text{(intersection)}
    \end{align*}%
  \end{proof}%
\end{proposition}%
Notably, monotonicity properties were not even needed in the above result.
So why is it that all formalisms we know of that have the replacement property also happen to have monotone consequence functions?
It holds because $\sem$ having the intersection property implies that $\sem$ is antimonotone (which in turn implies that $\Cnsem$ is monotone).
\begin{proposition}
  \label{thm:intersection-antimonotone}
  Let $(\L,\I,\sem)$ be a logic. If $\sem$ has the intersection property,
  then $\sem$ is antimonotone.
  \begin{proof}
    Let \mbox{$T_1\subseteq T_2\subseteq\L$}.
    Then \mbox{$T_1\cup T_2=T_2$}, and we conclude the desired subset-inclusion via \linebreak
    \mbox{$\sem(T_2)=\sem(T_1\cup T_2)=\sem(T_1)\cap\sem(T_2)\subseteq\sem(T_1)$}.
  \end{proof}
\end{proposition}

In the other direction, we can observe that the replacement property on its own does not guarantee antimonotonicity.
\begin{example}
  \label{exm:replacement-weak}
  Consider the language \mbox{$\L=\set{a}$} and interpretation set \mbox{$\I=\set{1}$}.
  For semantics $\sem$ with \mbox{$\sem(\emptyset)=\emptyset$} and \mbox{$\sem(\set{a})=\set{1}$}, we can see that the replacement property holds trivially since there are no semantically equivalent theories that are syntactically different.
  However, $\sem$ is not antimonotone (as \mbox{$\sem(\set{a})=\set{1}\not\subseteq\emptyset=\sem(\emptyset)$}) and does not have the intersection property\/:
  \begin{gather*}
    \sem(\emptyset\cup\set{a})=\sem(\set{a})=\set{1}\neq\emptyset=\emptyset\cap\set{1}=\sem(\emptyset)\cap\sem(\set{a})
  \end{gather*}
\end{example}

It is easy to see that classical propositional logic $\L_\PL$ has the intersection property simply \emph{by definition}:
the standard model semantics is typically firstly defined for single formulas \mbox{$\varphi\in\L_\PL$} and then generalized to theories $T$ by \emph{setting} \mbox{$\sem_\mod(T)=\bigcap_{\varphi\in T}\sem_\mod(\set{\varphi})$}.

As it turns out, for all logics, the intersection property also guarantees that each theory $T$ has the same models as the set of all canonical consequences of $T$.

\begin{proposition}
  \label{thm:intersection-theory-models}
  Let $(\L, \I, \sem)$ be a logic that has the intersection property.
  Then for each \mbox{$T\subseteq\L$} we find that \mbox{$\sem(T)=\sem\left(\Cnsem(T)\right)$}.
  \begin{proof}
    \begin{align*}
      \sem\left(\Cnsem(T)\right)
      &= \sem\!\left(\mathop{\bigcup_{S\subseteq\L,}}_{\sem(T)\subseteq\sem(S)}S\right) & \text{(Definition~\ref{def:consequence})} \\
      &= \mathop{\bigcap_{S\subseteq\L,}}_{\sem(T)\subseteq\sem(S)}\sem(S) & \text{(intersection property)} \\
      &= \sem(T)
    \end{align*}
    In the last equality, the $\supseteq$-direction is clear as we intersect only supersets of $\sem(T)$, and the $\subseteq$-direction is clear as we can substitute $T$ for $S$ in the line above.
  \end{proof}
\end{proposition}

Finally, this means that the intersection property only holds for semantics whose canonical consequence functions are closure operators.

\begin{proposition}
  \label{thm:intersection-closure}
  Let $(\L, \I, \sem)$ be a logic that has the intersection property.
  Then $\Cnsem$ is a closure operator.
  \begin{proof} We have to show that $\Cnsem$ is increasing, monotone and idempotent.
    First, $\Cnsem$ is increasing in any case (Proposition~\ref{thm:antimonotone-monotone-increasing}). Moreover, since $\sem$ has the intersection property, it is also antimonotone (Proposition~\ref{thm:intersection-antimonotone}) and thus, $\Cnsem$ is monotone (Proposition~\ref{thm:antimonotone-monotone-increasing}).
   Finally, it follows from Proposition~\ref{thm:intersection-theory-models} that $\Cnsem$ is also idempotent.  More precisely, if \mbox{$\varphi\in\Cnsem(\Cnsem(T))$} then there is an \mbox{$S\subseteq\L$} with \mbox{$\varphi\in S$} and \mbox{$\sem(\Cnsem(T))\subseteq\sem(S)$}, and by \mbox{$\sem(T)=\sem(\Cnsem(T))$} it is clear that in this case \mbox{$\sem(T)\subseteq\sem(S)$} and \mbox{$\varphi\in S\subseteq\Cnsem(T)$}. Hence, $\Cnsem$ is a closure operator.
  \end{proof}
\end{proposition}

We have shown in Proposition~\ref{thm:con-sem} that we also could have started with a consequence function.
Clearly the replacement property could be easily defined for consequence functions \mbox{$C:2^\L\to 2^\L$} in the sense that $C$ has the replacement property if and only if for all theories \mbox{$T_1,T_2\subseteq\L$}, ``classical consequence-equivalence'' \mbox{$C(T_1)=C(T_2)$} coincides with ``strong consequence-equivalence''  \mbox{$\forall U\subseteq\L:$} \linebreak $C(T_1\cup U)=C(T_2\cup U)$.
It is also clear that any semantics having the replacement property induces a canonical consequence function having the consequence-function version of the replacement property.
However, we must remark that we have not found a consequence-function equivalent of the intersection property.
Even setting \mbox{$C(T) = \bigcup_{F\in T}C(\set{F})$} would be too weak to capture interactions between different subtheories of $T$.
On the other hand, and on the positive side, we will see that the intersection property for model functions as we defined it in Definition~\ref{def:intersection} will give us a good handle on characterizing strong equivalence.

\subsection{Galois correspondences}
\label{sec:galois}

Up to here, we have considered various properties of logics in our abstract setting.
Mostly, those were algebraic properties of the model-theoretic semantics.
In this subsection, we conclude our argument for defining as ``classical'' those logics whose semantics satisfy the intersection property.

For this, it is firstly necessary to slightly extend (and, for the time being, also slightly constrain) our abstract notion of ``logic''.
Up to now, we only assumed the existence of a model function \mbox{$\sem:2^\L\to 2^\I$} that assigns a set of interpretations to a theory (intuitively, its models).
In the converse direction, we now also assume a \define{theory function} \mbox{$\th:2^\I\to 2^\L$}, that takes as input a set \mbox{$K\subseteq\I$} of interpretations and intuitively returns the set $\th(K)$ of all language elements that are true under all interpretations in $K$.
A \emph{logic} will now be a tuple $(\L,\I,\sem,\th)$ with $\L,\I,\sem$ as before and \mbox{$\th:2^\I\to 2^\L$} a theory function.

This immediately yields another way of defining a consequence operator:
given $\sem$ and~$\th$, we can define \mbox{$\Cn^{\sem,\th}(T) = \th(\sem(T))$}.
Symmetrically, we can define an operator on interpretation sets by \linebreak \mbox{$K\mapsto\sem(\th(K))$}.

We next consider a class of logics where the interplay of model function and theory function satisfies certain conditions.
Below, we denote the composition of functions \mbox{$f:A\to B$} and \mbox{$g:B\to C$} by \mbox{$g\circ f$}, that is, \mbox{$g\circ f:A\to C$} with \mbox{$x\mapsto g(f(x))$}.
\begin{definition}
  \label{def:galois}
  Let $(\L,\I,\sem,\th)$ be a logic.
  The functions $\sem$ and $\th$ are in \define{Galois correspondence} if and only if\/:
  \begin{enumerate}
  \item $\sem$ is antimonotone and $\th$ is antimonotone;
  \item $\sem\circ\th$ is increasing and $\th\circ\sem$ is increasing.
  \end{enumerate}
\end{definition}
If $\sem$ and $\th$ are in Galois correspondence, we also say that $(\L,\I,\sem,\th)$ is a \define{Galois logic}.

The first two properties together imply that the resulting consequence function is monotone; by the last property, it is also increasing.
Galois correspondences have been studied in model theory~\citep{goguen84introducing}, and in an even more abstract way in lattice theory~\citep{ore44galois,birkhoff73lattice,cohn81universal,davey-priestley}.

Indeed, there is also another (equivalent) formulation~\citep{davey-priestley}.
\begin{proposition}
  \label{thm:galois:alternative}
  Let $(\L,\I,\sem,\th)$ be a logic.
  It holds that $\sem$ and $\th$ are in {Galois correspondence} iff for all \mbox{$T\subseteq\L$} and \mbox{$K\subseteq\I$}\/:
  \begin{gather*}
    K\subseteq\sem(T) \iff T\subseteq\th(K)
  \end{gather*}%
  \begin{proof}
    \begin{description}
    \item[\normalfont if:]
      Assume that for all \mbox{$T\subseteq\L$} and \mbox{$K\subseteq\I$} we have, \mbox{$K\subseteq\sem(T)$} iff \mbox{$T\subseteq\th(K)$}.
      \begin{enumerate}
      \item We start with showing that $\th\circ\sem$ is increasing.\\
			Let \mbox{$T\subseteq\L$}.
        Obviously, \mbox{$\sem(T)\subseteq\sem(T)$}. Hence, if substituting $K=\sem(T)$ we obtain by presumption \mbox{$T\subseteq\th(\sem(T))$}.\\
				In the same spirit one may easily show that  $\sem\circ\th$ is increasing.
      \item We show now that $\sem$ is antimonotone.\\ 
			Let \mbox{$T_1\subseteq T_2$}.
        Then \mbox{$T_1\subseteq T_2\subseteq\th(\sem(T_2))$} by the above.
        By presumption (with \mbox{$K=\sem(T_2)$}), it follows that \mbox{$\sem(T_2)\subseteq\sem(T_1)$}.\\
				Analogously one may show that $\th$ is antimonotone.
       
      \end{enumerate}
    \item[\normalfont only if:]
      Let \mbox{$T\subseteq\L$} and \mbox{$K\subseteq\I$}.
      If \mbox{$K\subseteq\sem(T)$}, then \mbox{$\th(\sem(T))\subseteq\th(K)$}, whence we conclude that \linebreak \mbox{$T\subseteq\th(\sem(T))\subseteq\th(K)$}.
      The reverse implication follows symmetrically.
    \end{description}
  \end{proof}
\end{proposition}

It also follows that Galois correspondences induce closure operators, that is, operators that are monotone, increasing, and idempotent.
\begin{proposition}
  \label{thm:galois:closure}
  Let $(\L,\I,\sem,\th)$ be a Galois logic.
  Then the operators
  \begin{itemize}
  \item $\sem\circ\th:2^\I\to 2^\I$, $K\mapsto\sem(\th(K))$
  \item $\th\circ\sem:2^\L\to 2^\L$, $T\mapsto\th(\sem(T))$
  \end{itemize}
  are closure operators.
  \begin{proof}
    We have to show that \mbox{$\sem\circ\th$} and \mbox{$\th\circ\sem$} are (1) increasing, (2) monotone and (3) idempotent.
    We only consider \mbox{$\sem\circ\th$} as the proof for \mbox{$\th\circ\sem$} is absolutely symmetric.
    \begin{enumerate}
    \item The second item of Definition~\ref{def:galois} states verbatim that \mbox{$\sem\circ\th$} is increasing.
    \item It is easy to show that \mbox{$\sem\circ\th$} is monotone since both \mbox{$\sem$} and \mbox{$\th$} are antimonotone:
      if \mbox{$K_1\subseteq K_2$}, then \mbox{$\th(K_2)\subseteq\th(K_1)$} whence \mbox{$\sem(\th(K_1)) \subseteq \sem(\th(K_2))$}, which means \mbox{$(\sem\circ\th)(K_1)\subseteq (\sem\circ\th)(K_2)$}.
    \item Consider \mbox{$K\subseteq\I$}.
      We have to show that \mbox{$(\sem\circ\th)((\sem\circ\th)(K)) \subseteq (\sem\circ\th)(K)$}, that is,\linebreak \mbox{$\sem(\th(\sem(\th(K))))\subseteq\sem(\th(K))$}.

      Since \mbox{$\th\circ\sem$} is increasing by definition, we have that \mbox{$\th(K)\subseteq\th(\sem(\th(K)))$}.
      Since \mbox{$\sem$} is antimonotone, \mbox{$\sem(\th(\sem(\th(K)))) \subseteq \sem(\th(K))$}.
    \end{enumerate}
  \end{proof}
\end{proposition}

Towards showing that having the intersection property and being in a Galois correspondence are one and the same, we firstly derive a slightly modified characterization of the intersection property.
Instead of decomposing theories into its singletons, this version considers families of theories.

\begin{proposition}
  \label{thm:intersection:general}
  Let $(\L,\I,\sem)$ be a logic.
  Its model function \mbox{$\sem : 2^\L \to 2^\I$} satisfies the intersection property if and only if for all \mbox{$\T\subseteq 2^\L$}:
  \begin{gather*}
    \sem\!\left(\bigcup_{T\in\T}T\right) = \bigcap_{T\in\T}\sem(T)
  \end{gather*}
  \begin{proof}
    \begin{description}
    \item[\normalfont if:]
      Let \mbox{$T\subseteq\L$} and define \mbox{$\T=\set{ \set{F} \guard F\in T }$}.
      Clearly \mbox{$T=\bigcup_{F\in T}\set{F}$}, whence 
      \begin{align*}
        \sem(T)
        &=\sem\!\left(\bigcup_{F\in T}\set{F}\right) \\
        &=\sem\!\left(\bigcup_{\set{F}\in\T}\set{F}\right) \\
        &=\bigcap_{\set{F}\in\T}\sem(\set{F}) \\
        &=\bigcap_{F\in T}\sem(\set{F})
      \end{align*}
    \item[\normalfont only if:]
      Let \mbox{$\T\subseteq 2^\L$} and define \mbox{$T=\bigcup_{U\in\T}U$}.
      Now we have that
      \begin{align*}
        \sem\!\left(\bigcup_{U\in\T}U\right)
        &= \sem(T) \\
        &= \bigcap_{F\in T}\sem(\set{F}) \\
        &= \bigcap_{F\in\bigcup_{U\in\T}U}\sem(\set{F}) \\
        &= \bigcap_{U\in\T}\left(\bigcap_{F\in U}\sem(\set{F})\right)\\
        &= \bigcap_{U\in\T}\sem(U)
      \end{align*}
    \end{description}
  \end{proof}
\end{proposition}

In what follows, we will make implicit use of this result and consider the two formulations of the intersection property to be interchangeable.

We now conclude this section with its main result.
It states that for any logic $(\L,\I,\sem)$, the conditions ``$\sem$ has the intersection property'' and ``$\sem$ is in a Galois correspondence with some $\th$'' are equivalent.

The proof can be adapted from the literature \cite[Propositions~7.31 and 7.33]{davey-priestley} to our slightly different setting with acceptable effort.

\begin{theorem}
  \label{thm:galois-iff-intersection}
  Let $(\L,\I,\sem)$ be a logic.
  \begin{enumerate}
  \item If there is a \mbox{$\th: 2^\I\to 2^\L$} such that $\sem$ and $\th$ are in Galois correspondence, then $\sem$ has the intersection property.
  \item If $\sem$ has the intersection property, then we can define a theory function \mbox{$\th:2^\I\to 2^\L$} with
  \begin{gather*}
    K\mapsto\mathop{\bigcup_{T\subseteq\L,}}_{K\subseteq\sem(T)}T
  \end{gather*}
  such that $\sem$ and $\th$ are in Galois correspondence.
  \end{enumerate}
  \begin{proof}
    \begin{enumerate}
    \item 
      Assume the presumption.
      We have to show that for any \mbox{$\T\subseteq 2^\L$}, we find that
      \begin{gather*}
        \sem\!\left(\bigcup_{T\in \T}T\right) = \bigcap_{T\in \T}\sem(T)
      \end{gather*}
      Let \mbox{$Z=\bigcup_{T\in \T}T$}.
      We first show that $\sem(Z)$ is a lower bound for the set \mbox{$\sem(\T)=\set{\sem(T)\guard T\in \T}$}.
      Clearly, for each \mbox{$T\in \T$} we have \mbox{$T\subseteq Z$}, whence by antimonotonicity of $\sem$ we get \mbox{$\sem(Z)\subseteq\sem(T)$} for each \mbox{$T\in \T$}.
      Now let \mbox{$Q\subseteq\I$} be any lower bound for \mbox{$\sem(\T)$}.
      Then \mbox{$Q\subseteq\sem(T)$} for all \mbox{$T\in\T$}.
      By Proposition~\ref{thm:galois:alternative}, we get \mbox{$T\subseteq\th(Q)$} for all\ \,\mbox{$T\in\T$}.
      By definition, this entails that \mbox{$\bigcup_{T\in\T}T=Z\subseteq\th(Q)$}.
      Now using Proposition~\ref{thm:galois:alternative} again yields \mbox{$Q\subseteq\sem(Z)$} whence $\sem(Z)$ is the greatest lower bound of $\sem(\T)$.
    \item 
      Let $\sem$ have the intersection property.
      By Proposition~\ref{thm:intersection-antimonotone}, it follows that $\sem$ is antimonotone.
      For antimonotonicity of $\th$, consider \mbox{$K_1\subseteq K_2$}.
      Then
      \begin{align*}
        \th(K_2) = \mathop{\bigcup_{T\subseteq\L,}}_{K_2\subseteq\sem(T)}T \subseteq \mathop{\bigcup_{T\subseteq\L,}}_{K_1\subseteq\sem(T)}T  = \th(K_1)
      \end{align*}
      Now for showing that $\sem\circ\th$ and $\th\circ\sem$ are increasing.
      
      Let \mbox{$K\subseteq\I$}.
      We find
      \begin{align*}
        K \subseteq \mathop{\bigcap_{T\subseteq\L,}}_{K\subseteq\sem(T)}\sem(T) 
        = \sem\!\left(\mathop{\bigcup_{T\subseteq\L,}}_{K\subseteq\sem(T)}T\right) 
        = \sem(\th(K))
      \end{align*}
      Now let \mbox{$T\subseteq\L$}.
      By using $\sem(T)\subseteq \sem(T)$ it is easy to verify that
      \begin{align*}
        T 
        \subseteq \mathop{\bigcup_{S\subseteq\L,}}_{\sem(T)\subseteq\sem(S)}S 
        =\th(\sem(T))
      \end{align*}
    \end{enumerate}    
  \end{proof}
\end{theorem}


This is the main motivation for our definition saying that \define{classical} logics are exactly those that have the intersection property, or, equivalently, that have model and theory functions that are in Galois correspondence.
Furthermore, as partly mentioned before, many well-studied logics (that we would call classical due to their ubiquity alone) \emph{have} the intersection property simply by definition, as their fundamental building blocks are \emph{formulas} instead of theories.

\section{Characterization Logics}
\label{sec:characterization:general}

From now on we omit $\I$ from the presentation of logics and thus write $(\L, \sem)$, since concrete interpretations are immaterial for strong equivalence.
Furthermore, we will distinguish between two important cases regarding the domain of $\sem$. The first one is \mbox{$\dom(\sem) = 2^\L$} (called \define{full} logics) and the second one is \mbox{$\dom(\sem) = \Lfin = \{T\in 2^\L\mid T \text{ is finite}\}$} (\define{finite-theory} logics), the restriction of $\L$ to finite knowledge bases.

\begin{definition}
  Let $(\L, \sem)$ be a logic.
  The binary relation \define{strong equivalence} \mbox{${\equiv^\sem_s} \subseteq \dom(\sem)\times \dom(\sem)$} is defined by \mbox{$T_1 \equiv^\sem_s T_2\iff\forall U\in\dom(\sem):\sem(T_1\cup U)=\sem(T_2\cup U)$}.
\end{definition}
It is straightforward to show that $\equiv^\sem_s$ is an equivalence relation;
we denote the equivalence class of a theory \mbox{$T\in\dom(\sem)\subseteq 2^\L$} by $[T]^\sem_s$.
We recall and will pervasively use that for all theories \mbox{$T_1,T_2\subseteq\L$}, we have \mbox{$T_1\in [T_2]^\sem_s$ if and only if $[T_1]^\sem_s=[T_2]^\sem_s$}.

Given an arbitrary logic $(\L,\sem)$, we want to find a characterizing classical logic, that is, a semantics~$\semc$ that has the intersection property and whose ordinary equivalence coincides with strong $\sem$-equivalence.
Such logics get a name.%
\begin{definition}
  \label{def:characterization-logic}
  Let $(\L, \sem)$ be a (full) logic.
  The logic $(\L, \sem')$ is a (full) \define{characterization logic for $(\L, \sem)$} if and only if\/:
  \begin{align*}
    &\forall T_1,T_2\subseteq\L : \sem'(T_1)=\sem'(T_2) \iff [T_1]^\sem_s=[T_2]^\sem_s \tag{\text{characterization}} \\
    &\forall \T\subseteq 2^\L: \sem'\!\left(\bigcup_{T\in \T}T\right) = \bigcap_{T\in\T}\sem'(T) \tag{\text{intersection}}
  \end{align*}
\end{definition}

Since characterization logics are the centerpiece of our study we would like to take a look to this central definition from an other angle, namely in terms of consequence functions. Note that the analysis of consequence relations have been very prominent in the formative years of nonmonotonic reasoning. We refer the interested reader to \cite{Gabbay85,Kraus90,Makinson94}. Let us consider the canonical consequence functions of $\sigma$ and $\sigma'$ according to Definition~\ref{def:consequence}. Doing so reveals that the characterizing and characterized consequence relations fit
together appropriately as stated in the following assertion. It is part of future work to study further properties in terms of consequence relations.

\begin{proposition}
 Given a logic $(\L, \sem)$ and a characterization logic $(\L, \sem')$ of it. Let $\Cnsem$ and $\Cnsemp$ be the canonical consequence functions of the corresponding logics. Given $T\subseteq\L$.
\begin{enumerate}
	\item $\Cnsemp\big(T\big) \subseteq \Cnsem\big(T\big)$ \hfill{(sublogic)}
	\item $\Cnsem\big(T\big) = \Cnsemp\big(\Cnsem\big(T\big)\big)$ \hfill{(left absorption)}
	\item $\Cnsem\big(T\big) = \Cnsem\big(\Cnsemp\big(T\big)\big)$ \hfill{(right absorption)}
\end{enumerate}
\end{proposition}

\begin{proof} 
In the following proofs we will often use that $(\L, \sem)$ and $(\L, \sem')$ are characterization logics implying that $\sem$ and $\sem'$ possess the intersection property (Definition~\ref{def:characterization-logic}) which in turn means that $\Cnsem$ and $\Cnsemp$ are closure operators (Proposition~\ref{thm:intersection-closure}).
\begin{enumerate}
\item In order to show $\Cnsemp\big(T\big) \subseteq \Cnsem\big(T\big)$ we first prove the subsequent property (*). For any $U\subseteq\L$: If $U\subseteq \Cnsemp\big(T\big)$, then $\Cnsemp\big(T\big) = \Cnsemp(T\cup U)$. Since $T\subseteq T\cup U$ we derive $\Cnsemp\big(T\big) \subseteq \Cnsemp(T\cup U)$ (monotonicity). Moreover, we have $T \subseteq \Cnsemp\big(T\big)$ (increasing) and $U \subseteq \Cnsemp\big(T\big)$ (assumption) justifying $T\cup U \subseteq \Cnsemp\big(T\big)$. Hence, $\Cnsemp(T\cup U) \subseteq \Cnsemp\big(\Cnsemp\big(T\big)\big)$ (monotonicity) and finally, $\Cnsemp(T\cup U) \subseteq \Cnsemp\big(T\big)$ (idempotency) concluding the proof for property (*).\\
   Now, let $U\subseteq\Cnsemp\big(T\big)$. We have to show $U\subseteq\Cnsem\big(T\big)$. Due to Proposition~\ref{thm:intersection-theory-models} and property (*) we obtain $\sem'(T) = \sem'\big(\Cnsemp\big(T\big)\big) = \sem'\big(\Cnsemp\big(T\cup U\big)\big) = \sem'(T\cup U)$. Since $(\L, \sem')$ is assumed to be a characterization logic we derive $[T]^\sem_s=[T\cup U]^\sem_s$. Consequently, $\sem(T) = \sem(T\cup U)$ implying $\Cnsem(T) = \Cnsem(T\cup U)$ in consideration of Definition~\ref{def:consequence}. Using the previous equality justifies $U\subseteq T\cup U \subseteq \Cnsem(T)$.  
\item On the one hand, we have $\Cnsem\big(T\big) \subseteq \Cnsemp\big(\Cnsem\big(T\big)\big)$ since $\Cnsemp$ is increasing. On the other hand, $\Cnsemp\big(\Cnsem\big(T\big)\big)\subseteq\Cnsem\big(\Cnsem\big(T\big)\big)$ due to item~1 (sublogic) and finally, $\Cnsemp\big(\Cnsem\big(T\big)\big)\subseteq\Cnsem\big(T\big)$ since $\Cnsem$ is increasing too. Thus, $\Cnsem\big(T\big) = \Cnsemp\big(\Cnsem\big(T\big)\big)$ as claimed. 
	\item Due to Proposition~\ref{thm:intersection-theory-models} we have $\sem'(T) = \sem'\big(\Cnsemp\big(T\big)\big)$. Since $(\L, \sem')$ is a characterization logic we further conclude $[T]^\sem_s = \big[\Cnsemp\big(T\big)\big]^\sem_s$. Hence, $\sem(T) = \sem\big(\Cnsemp\big(T\big)\big)$ which guarantees $\Cnsem\big(T\big) = \Cnsem\big(\Cnsemp\big(T\big)\big)$ by Definition~\ref{def:consequence}. 
\end{enumerate}
\end{proof}

We now start our analysis of characterization logics in terms of model functions. We first show that characterization logics are unique up to isomorphism.
More precisely, for any model function $\sem$, the algebras corresponding to the model theories of any two characterizing model functions $\sem'$ and $\sem''$ are isomorphic.
To do that, we first show that the model theory of any characterization logic\footnote{The proof of Theorem~\ref{thm:characterization-lattice} reveals that $(\L, \sem')$ does not necessarily has to be a characterization logic of $(\L, \sem)$. Indeed, the stated properties regarding the model theory hold for any logic possessing the intersection property.} is a complete lattice, that is, a partially ordered set where each subset of the carrier set has both a greatest lower bound (glb) and a least upper bound (lub).
\begin{theorem}
  \label{thm:characterization-lattice}
  Let $(\L, \sem)$ be a full logic with characterization logic $(\L, \sem')$.
  The pair $\left(\sem'\!\left(2^\L\right), \subseteq\right)$ is a complete lattice where 
  glb $\bigwedge$ and lub $\bigvee$ are given such that
  for all \mbox{$\K\subseteq\sem'\!\left(2^\L\right)$}, 
  \begin{gather*}
    \hfill
    \bigwedge_{K\in \K}K=\bigcap_{K\in\K}K
    \hfill\text{ and }\hfill
    \bigvee_{K\in\K}K=\bigwedge_{L\in\K^u}L
    \hfill
  \end{gather*}
  where \mbox{$\K^u = \set{ L\in\sem'\!\left(2^\L\right) \guard \forall K\in\K:K\subseteq L }$} is the set of upper bounds of $\K$.
  \begin{proof}
    Let \mbox{$\K\subseteq\sem'\!\left(2^\L\right)$};
    we first show \mbox{$\bigcap_{K\in\K}K\in\sem'\!\left(2^\L\right)$}.
    Clearly for each \mbox{$K\in\K\subseteq\sem'\!\left(2^\L\right)$} there exists a \mbox{$T\subseteq\L$} with \mbox{$\sem'(T)=K$}.
    Thus by the axiom of choice there is a \mbox{$\T\subseteq 2^\L$} that contains a \mbox{$T\in\T$} with \mbox{$\sem'(T)=K$} for each \mbox{$K\in\K$}.
    Since \mbox{$\bigcup_{T\in\T}T\in 2^\L$} and $\sem'$ has the intersection property, 
    \mbox{$\bigcap_{K\in\K}K=\bigcap_{T\in\T}\sem'(T)=\sem'\!\left(\bigcup_{T\in\T}T\right)\in\sem'\!\left(2^\L\right)$}.

    Now consider \mbox{$\bigvee_{K\in\K}K$}.
    We show that \mbox{$\bigwedge_{L\in\K^u}L$} is the least element of $\K^u$, the set of all upper bounds of $\K$.
    Clearly, \mbox{$L\in\K^u$} implies that \mbox{$\forall K\in\K:K\subseteq L$}.
    Thus, \mbox{$\forall K\in\K:$} \mbox{$K\subseteq\bigcap_{L\in\K^u}L$} whence \mbox{$\bigcap_{L\in\K^u}L\in\K^u$}.
    In particular, if \mbox{$M\in\K^u$} then \mbox{$\bigcap_{L\in\K^u}L\subseteq M$} and \mbox{$\bigcap_{L\in\K^u}L$} is the least upper bound of~$\K$.
  \end{proof}
\end{theorem}

It is vital that the least upper bound is defined in terms of the greatest lower bound, as ordinary set union will not work.
\begin{example}
  \label{exm:se-cleq-lattice-not-distributive}
  Consider \mbox{$\L=\set{a,b,c}$} and semantics $\sem$ with
  \begin{align*}
    [\emptyset]^\sem_s&=\set{\emptyset}, \\
    [\set{a}]^\sem_s&=\set{\set{a}},\\
    [\set{b}]^\sem_s&=\set{\set{b}},\\
    [\set{c}]^\sem_s&=\set{\set{c}},\\
    [\set{a,b}]^\sem_s&=[\set{a,c}]^\sem_s=[\set{b,c}]^\sem_s=[\set{a,b,c}]^\sem_s =\set{\set{a,b},\set{a,c},\set{b,c},\set{a,b,c}}
  \end{align*}
  Assume that $(\L, \sem')$ is a characterization logic for $(\L, \sem)$.
  We claim that
  $$\sem'(\set{b}) \cup \sem'(\set{c})\notin\sem'\!\left(2^\L\right)$$
  whence $\sem'\!\left(2^\L\right)$ is not closed under set union and thus cannot be the carrier set of any sublattice of $(2^{\I'},\subseteq)$.
  Assume to the contrary that there is a \mbox{$T\subseteq\L$} with \mbox{$\sem'(T)=\sem'(\set{b})\cup\sem'(\set{c})$}.
  Then
  \begin{align*}
    \sem'(\set{a}\cup T)
    &= \sem'(\set{a}) \cap \sem'(T) \\
    &= \sem'(\set{a}) \cap ( \sem'(\set{b}) \cup \sem'(\set{c}) ) \\
    &= ( \sem'(\set{a}) \cap \sem'(\set{b}) ) \cup ( \sem'(\set{a}) \cap \sem'(\set{c}) ) \\
    &= \sem'(\set{a,b}) \cup \sem'(\set{a,c}) \\
    &= \sem'(\set{a,b}) = \sem'(\set{a,c}) = \sem'(\set{b,c}) = \sem'(\set{a,b,c})
  \end{align*}
  Thus \mbox{$T=\set{b}$} or \mbox{$T=\set{c}$} or \mbox{$T\in [\set{a,b,c}]^\sem_s$}.
  Since $\sem'$ has the intersection property, it is in particular antimonotone~(Proposition~\ref{thm:intersection-antimonotone}).
  Hence for \mbox{$T\in [\set{a,b,c}]^\sem_s$}, we get \mbox{$\sem'(T)\subseteq\sem'(\set{b})$} and since \mbox{$[\set{b}]^\sem_s\neq [\set{a,b,c}]^\sem_s$} even \mbox{$\sem'(T)\subsetneq\sem'(\set{b})$}.
  Substituting $c$ for $b$ in the argument above, we obtain \linebreak \mbox{$\sem'(T)\subsetneq\sem'(\set{c})$} in the same fashion.
  In combination \mbox{$\sem'(T)\subsetneq\sem'(\set{b})\cup\sem'(\set{c})$}, contradiction.
  Thus \mbox{$T=\set{b}$} or \mbox{$T=\set{c}$}.
  Clearly \mbox{$[\set{b}]^\sem_s\neq [\set{b,c}]^\sem_s$} and \mbox{$[\set{c}]^\sem_s\neq [\set{b,c}]^\sem_s$} imply that 
  \begin{align*}
    \sem'(\set{b})\cap\sem'(\set{c})=\sem'(\set{b,c})&\subsetneq\sem'(\set{b}) \\
    \sem'(\set{b})\cap\sem'(\set{c})=\sem'(\set{b,c})&\subsetneq\sem'(\set{c})
  \end{align*}
  Thus \mbox{$\sem'(\set{b})\setminus\sem'(\set{c})\neq\emptyset$} and \mbox{$\sem'(\set{c})\setminus\sem'(\set{b})\neq\emptyset$}.
  Therefore
  $$\sem'(\set{b})\subsetneq\sem'(\set{b})\cup\sem'(\set{c})=\sem'(T)$$
  as well as
  $$\sem'(\set{c})\subsetneq\sem'(\set{b})\cup\sem'(\set{c})=\sem'(T).$$
  Contradiction.
  Thus \mbox{$T$} does not exist and \mbox{$\sem'(\set{b}) \cup \sem'(\set{c})\notin\sem'\!\left(2^\L\right)$}.
\end{example}
This example shows in particular that the resulting complete lattice induced by the characterization logic is not necessarily distributive, not even in the case of finite logics.

After these necessary preliminaries, we now present the result on uniqueness of characterization logics.
\begin{theorem}
  \label{thm:unique-characterization}
  Let $(\L, \sem)$ be a full logic with characterization logics $(\L, \sem')$ and $(\L, \sem'')$.
  Then the complete lattices $\left(\sem'\!\left(2^\L\right), \subseteq\right)$ and $\left(\sem''\!\left(2^\L\right), \subseteq\right)$ are isomorphic.
  \begin{proof}
    We provide a bijection \mbox{$\phi:\sem'\!\left(2^\L\right)\to\sem''\!\left(2^\L\right)$} with
    \begin{gather*}
      \hfill
      \phi\!\left(\bigwedge_{K\in\K}K\right) = \bigwedge_{K\in\K}\phi(K)
      \hfill\text{ and }\hfill
      \phi\!\left(\bigvee_{K\in\K}K\right) = \bigvee_{K\in\K}\phi(K)
      \hfill
    \end{gather*}
    Given any \mbox{$K\in\sem'\!\left(2^\L\right)$}, there clearly exists a \mbox{$T\subseteq\L$} with \mbox{$\sem'(T)=K$}; now define $\phi(K)$ such that \mbox{$\phi(K)=\sem''(T)$}. This means, $\phi(\sem'(T)) = \phi(\sem''(T))$.
    \begin{itemize}
    \item $\phi$ is injective:
      Let \mbox{$K_1,K_2\in\sem'\!\left(2^\L\right)$} with \mbox{$\phi(K_1)=\phi(K_2)$}.
      Clearly there exist \mbox{$T_1,T_2\subseteq\L$} such that \mbox{$\sem'(T_1)=K_1$} and \mbox{$\sem'(T_2)=K_2$}.
      Thus \mbox{$\phi(K_1)=\phi(K_2)$} implies \mbox{$\sem''(T_1)=\phi(K_1)=\phi(K_2)=\sem''(T_2)$}.
      Since $\sem''$ has the characterization property, we get \mbox{$[T_1]^\sem_s=[T_2]^\sem_s$}.
      Since $\sem'$ has the characterization property, we get \mbox{$K_1=\sem'(T_1)=\sem'(T_2)=K_2$}.
    \item $\phi$ is surjective:
      Let \mbox{$M\in\sem''\!\left(2^\L\right)$}.
      Then there exists a \mbox{$T\subseteq\L$} with \mbox{$\sem''(T)=M$}.
      It follows by definition that \mbox{$\phi(\sem'(T))=\sem''(T)=M$}.
    \item $\phi$ is structure-preserving: 
      Let \mbox{$\K\subseteq\sem'\!\left(2^\L\right)$} and define \mbox{$\T\subseteq 2^\L$} such that for each \mbox{$K\in\K$}, the family $\T$ contains an element of the preimage of $K$ with respect to $\sem'$.
      ($\T$ need not be unique, but exists by the axiom of choice.)
      \begin{align*}
         \pheq \phi\!\left(\bigwedge_{K\in\K}K\right) 
        &= \phi\!\left(\bigcap_{K\in\K}K\right) 
        = \phi\!\left(\bigcap_{T\in\T}\sem'(T)\right) 
        = \phi\!\left(\sem'\!\left(\bigcup_{T\in\T}T\right)\right)\\ 
        &= \sem''\!\left(\bigcup_{T\in\T}T\right)
        = \bigcap_{T\in\T}\sem''(T) 
        = \bigcap_{K\in\K}\phi(K) 
        = \bigwedge_{K\in\K}\phi(K)
      \end{align*}
      Since $\bigvee$ can be defined in terms of $\bigwedge$, it follows that
      \begin{gather*}
        \hspace*{-1em}\phi\!\left(\bigvee_{K\in\K}K\right) 
        = \phi\!\left(\bigwedge_{K\in\K^u}K\right)
        = \bigwedge_{K\in\K^u}\phi(K)
        = \bigvee_{K\in\K}\phi(K)
      \end{gather*}
    \end{itemize}
    Thus $\left(\sem'\!\left(2^\L\right), \subseteq\right)$ and $\left(\sem''\!\left(2^\L\right), \subseteq\right)$ are isomorphic.
  \end{proof}
\end{theorem}

Thus if a classical characterization logic exists, it is (up to isomorphism on its model theory) uniquely determined.
However, as we show next, in some cases there simply is no characterization logic.
\begin{example}
  \label{exm:uncharacterizable}
  Let \mbox{$\L=\N$} be the natural numbers and \mbox{$\I\neq\emptyset$} arbitrary.
  We define the semantics \linebreak \mbox{$\sem:2^\L\to 2^\I$} such that
  \begin{gather*}
    \sem(T) =
    \begin{cases}
      \emptyset & \text{if } T \text{ is finite,} \\
      \I & \text{otherwise.}
    \end{cases}
  \end{gather*}
  There are two strong equivalence classes:
  $[\emptyset]^\sem_s$, the set of all finite subsets of $\N$, and $[\N]^\sem_s$, the set of all infinite subsets of $\N$.
  Assume that $(\L,\sem')$ is a characterization logic for $(\L,\sem)$.
  By the model intersection property, we get
  \begin{gather*}
    \sem'(\N) 
    = \sem'\!\left(\bigcup_{n\in\N}\set{n}\right) 
    = \bigcap_{n\in\N}\sem'(\set{n})
    = \sem'(\emptyset)
  \end{gather*}
  in contradiction to the characterization property.
\end{example}

The example above motivates the study of special features of logics warranting the existence of characterization logics.
We now proceed with some useful properties that will be needed in this endeavour later on.
Most importantly, we show that strong equivalence classes have an expansion property:
It is not completely obvious, but follows easily from the definition of strong equivalence that two strongly equivalent theories can both be expanded (via set union) with the same theory and are again strongly equivalent;
the converse holds as well.
Furthermore, the union of two strongly equivalent theories is again strongly equivalent to the two theories.
Finally, for any two strongly equivalent theories that are in subset relation, every theory in between them is also strongly equivalent to them.
\begin{lemma}
  \label{thm:se-expansion}
  Let $(\L, \sem)$ be a full logic and \mbox{$T, T_1, T_2, T_3\subseteq\L$}.
  \begin{enumerate}
  \item Strong equivalence is invariant to expansion\/:
    $$[T_1]^\sem_s = [T_2]^\sem_s \iff \Bigl( \forall U\subseteq\L: [T_1\cup U]^\sem_s = [T_2\cup U]^\sem_s \Bigr)$$
  \item\label{itm:semilattice} Each strong equivalence class is a join-semilattice\/:
    $$T_1,T_2\in [T]^\sem_s \implies T_1\cup T_2\in [T]^\sem_s$$
  \item\label{itm:convex} Each strong equivalence class is convex\/:
    $$T_1\in [T]^\sem_s \land T_3\in [T]^\sem_s \land T_1\subseteq T_2\subseteq T_3 \implies T_2\in [T]^\sem_s$$
  \end{enumerate}
  \begin{proof}
    \begin{enumerate}
    \item 
      \begin{align*}
        &\phiff [T_1]^\sem_s = [T_2]^\sem_s \\
        &\iff \forall U\subseteq\L: \sem(T_1\cup U)=\sem(T_2\cup U) & \text{(Def.~$\equiv^\sem_s$)} \\
        &\iff \forall U'\subseteq\L: \forall U''\subseteq\L:\sem(T_1\cup (U'\cup U''))=\sem(T_2\cup (U'\cup U'')) & \text{(write $U=U'\cup U''$)} \\
        &\iff \forall U'\subseteq\L: \forall U''\subseteq\L:\sem((T_1\cup U')\cup U'')=\sem((T_2\cup U')\cup U'') & \text{(associativity $\cup$)} \\
        &\iff \forall U'\subseteq\L: [T_1\cup U']^\sem_s=[T_2\cup U']^\sem_s & \text{(Def.~$\equiv^\sem_s$)} \\
        &\iff \forall U\subseteq\L: [T_1\cup U]^\sem_s=[T_2\cup U]^\sem_s & \text{(rename $U'=U$)}
      \end{align*}
    \item 
      \begin{align*}
        &\phiff T_1,T_2\in [T]^\sem_s \\
        &\iff [T_1]^\sem_s = [T_2]^\sem_s = [T]^\sem_s \\
        &\implies [T_1\cup T_2]^\sem_s = [T_2 \cup T_2]^\sem_s = [T_2]^\sem_s = [T]^\sem_s \\
        &\implies T_1\cup T_2 \in [T]^\sem_s
      \end{align*}
    \item 
      \begin{align*}
        &\phiff T_1,T_3\in [T]^\sem_s \land T_1\subseteq T_2\subseteq T_3 \\
        &\implies \Bigl( T_1\cup T_2 \in [T]^\sem_s \iff T_3\cup T_2\in [T]^\sem_s \Bigr) \land T_1\subseteq T_2\subseteq T_3 \\
        &\implies \Bigl( T_2 \in [T]^\sem_s \iff T_3 \in [T]^\sem_s \Bigr) \\
        &\implies T_2 \in [T]^\sem_s
      \end{align*}
    \end{enumerate}
  \end{proof}
\end{lemma}

We also take a closer look on the structure of strong equivalence classes and their relationships with each other.
We first define a binary relation $\cleq$ on strong equivalence classes where two classes are in $\cleq$-relation if and only if there are $\subseteq$-comparable representatives of the classes.

\begin{definition}
  \label{def:cleq}
  Let $(\L, \sem)$ be a logic.
  Define the following relation on strong equivalence classes\/:
  \begin{gather*}
    [S]^\sem_s \cleq [T]^\sem_s \iff \exists S'\in [S]^\sem_s:\exists T'\in [T]^\sem_s: S'\subseteq T'
  \end{gather*}
\end{definition}

\noindent It is easy to show (and later useful) that $\cleq$ is a partial order.

\begin{lemma}
  \label{thm:cleq-partial-order}
  The relation $\cleq$ is a partial order, that is, reflexive, antisymmetric and transitive.
  \begin{proof}
    Reflexivity is clear.
    \begin{description}
    \item[\normalfont antisymmetric:]
      Let \mbox{$[S]^\sem_s \cleq [T]^\sem_s$} and \mbox{$[T]^\sem_s \cleq [S]^\sem_s$}.
      Then there are \mbox{$S_1, S_2\in [S]^\sem_s$} as well as \mbox{$T_1,T_2\in [T]^\sem_s$} with \mbox{$S_1\subseteq T_1$} and \mbox{$T_2\subseteq S_2$}.
      Clearly
      \begin{align*}
        [S]^\sem_s 
        &= [S_1\cup S_2]^\sem_s & \text{(} S_1,S_2\in [S]^\sem_s \text{; Lemma~\ref{thm:se-expansion}, Item~\ref{itm:semilattice})} \\
        &= [T_2\cup S_1\cup S_2]^\sem_s & \text{(} T_2\subseteq S_2 \text{)} \\
        &= [T_1\cup T_2\cup S_1\cup S_2]^\sem_s &\text{(} [T_2]^\sem_s=[T_1\cup T_2]^\sem_s \text{)} \\
        &= [T_1\cup T_2\cup S_1]^\sem_s &\text{(} [S_1\cup S_2]^\sem_s=[S_1]^\sem_s\text{)} \\
        &= [T_1\cup T_2]^\sem_s &\text{(} S_1\subseteq T_1 \text{)} \\
        &= [T]^\sem_s &\text{(} T_1,T_2\in [T]^\sem_s \text{; Lemma~\ref{thm:se-expansion}, Item~\ref{itm:semilattice})}
      \end{align*}
    \item[\normalfont transitive:] 
      Let \mbox{$[T_1]^\sem_s \cleq [T_2]^\sem_s$} and \mbox{$[T_2]^\sem_s \cleq [T_3]^\sem_s$}.
      Then there exist \mbox{$T_1'\in [T_1]^\sem_s$}, \mbox{$T_2',T_2''\in [T_2]^\sem_s$} and \mbox{$T_3''\in [T_3]^\sem_s$} with \mbox{$T_1'\subseteq T_2'$} and \mbox{$T_2''\subseteq T_3''$}.
      Since \mbox{$[T_2']^\sem_s=[T_2'']^\sem_s$}, we conclude that \mbox{$[T_2'\cup T_3'']=[T_2''\cup T_3'']=[T_3'']^\sem_s$}.
      Now it is clear that \mbox{$T_1'\subseteq T_2'\subseteq T_2'\cup T_3''$} with \mbox{$T_1'\in [T_1]^\sem_s$} and \mbox{$T_2'\cup T_3''\in [T_3]^\sem_s$} imply that \mbox{$[T_1]^\sem_s \cleq [T_3]^\sem_s$}.
    \end{description}
  \end{proof}
\end{lemma}

It follows from Lemma~\ref{thm:se-expansion} in particular that in the case of logics $(\L, \sem)$ with $\L$ finite, each strong equivalence class $[T]^\sem_s$ has a $\subseteq$-greatest element that equals the union of all elements.
However, for logics with infinite $\L$, this need not be the case:
in the logic of Example~\ref{exm:uncharacterizable}, the class $[\emptyset]^\sem_s$ has no maximal elements, in particular no greatest element;
the class $[\N]^\sem_s$ has no minimal elements, in particular no least element.
We will see that having a $\subseteq$-greatest element in each equivalence class is sufficient for the existence of a characterization logic.
We therefore decided to name this class of logics and study it in some greater detail.

\subsection{Covered Logics}
\label{sec:covered-logics}

\begin{definition}
  \label{def:secsme}
  Let $(\L, \sem)$ be a logic with strong equivalence relation $\equiv^\sem_s$.
  For an equivalence class \mbox{$[T]^\sem_s\in\left(2^\L\right)_{\!/{\equiv^\sem_s}}=\set{ [T]^\sem_s \guard T\subseteq\L }$}, we define its \define{cover} to be the set
  $$\widehat{[T]^\sem_s}=\bigcup_{S\in [T]^\sem_s}S$$
  and say that a logic $(\L, \sem)$ is \define{\secsme} if and only if \mbox{$\forall T\subseteq\L: \widehat{[T]^\sem_s} \in [T]^\sem_s$}.
\end{definition}
Roughly, the existence of greatest elements in equivalence classes guarantees that these classes are closed under arbitrary set union.\footnote{Strong equivalence classes are always closed under set unions with non-empty, finite index sets~(Lemma~\ref{thm:se-expansion}).}
Clearly any finite logic is \secsme.
Furthermore, two familiar representatives of \secsme logics are classical logic and abstract argumentation theory.
In the former case, it is clear that arbitrary unions of families of equivalent theories are again theories that are equivalent to each of its members.
In the latter case it is not immediately clear but can be shown with reasonable effort.

Towards the main result of this section, we show that \secsme logics behave ``nicely'' in an algebraic sense, which will pave the way for obtaining characterization logics for them.
Most importantly, while we know from Theorem~\ref{thm:characterization-lattice} that the strong equivalence classes of any logic form a complete lattice, for covered logics we can even specify the join and meet operations directly via set operations on covers and subsequent class formation.
\begin{lemma}
  \label{thm:se-cleq-lattice}
  Let $(\L,\sem)$ be a \secsme logic with strong equivalence relation $\equiv^\sem_s$.
  The pair \mbox{$\left(\left(2^\L\right)_{\!/{\equiv^\sem_s}},\cleq\right)$} is a complete lattice with operations
  \begin{gather*}
    \hfill
    \bigccup_{C\in \C}C = \left[\bigcup_{C\in\C}\hat{C}\right]^\sem_s
    \hfill\text{and}
    \bigccap_{C\in\C}C = \left[\bigcap_{C\in\C}\hat{C}\right]^\sem_s 
    \hfill
  \end{gather*}
	
  \begin{proof}
    Let \mbox{$\C\subseteq \set{ [T]^\sem_s \guard T\subseteq\L }$}.
    \begin{itemize}
    \item \mbox{$D=\bigccup_{C\in \C}C$} is the least upper bound of $\C$:

      For any \mbox{$C\in\C$}, we get \mbox{$\hat{C}\subseteq\bigcup_{B\in\C}\hat{B}$} immediately and
      \mbox{$\widehat{C}\in C$} since $(\L,\sem)$ is covered.
      The fact that \mbox{$\bigcup_{B\in\C}\hat{B}\in\left[\bigcup_{B\in\C}\hat{B}\right]^\sem_s=D$} is also immediate, whence \mbox{$C\cleq D$} and by arbitrary choice of $C$ we get that $D$ is an upper bound of $\C$.

      Now let $E$ be any upper bound of $\C$ and consider an arbitrary \mbox{$C\in\C$}.
      Since \mbox{$C\cleq E$}, there are \mbox{$T_C\in C$} and \mbox{$T_C'\in E$} such that \mbox{$T_C\subseteq T_C'$}.
      Since $(\L,\sem)$ is \secsme, \mbox{$[T_C]^\sem_s=C=\left[\hat{C}\right]^\sem_s$} whence by strong equivalence \mbox{$[T_C\cup T_C']^\sem_s=\left[\hat{C}\cup T_C'\right]^\sem_s$}.
      That is, \mbox{$\hat{C}\cup T_C'\in [T_C\cup T_C']^\sem_s=[T_C']^\sem_s=E$}.
      Since $C$ was arbitrary, we have that \mbox{$\forall C\in\C:\exists T_C'\in E: \hat{C}\cup T_C'\in E$}.
      In particular, \mbox{$\forall C\in\C:\exists T_C'\in E: \hat{C}\cup T_C'\subseteq \hat{E}$} whence \mbox{$\bigcup_{C\in\C}\left(\hat{C}\cup T_C'\right)\subseteq \hat{E}$}.
      Now fix a specific \mbox{$B\in\C$};
      then we have in particular that there exists a \mbox{$T_B'\in E$} such that \mbox{$\hat{B}\cup T_B'\in E$}.
      Furthermore, \mbox{$\hat{B}\cup T_B' \subseteq \bigcup_{C\in\C}\left(\hat{C}\cup T_C'\right) \subseteq \hat{E}$} with \mbox{$\hat{E}\in E$} since $(\L,\sem)$ is \secsme.
      By Lemma~\ref{thm:se-expansion}~(Item~\ref{itm:convex}) saying that strong equivalence classes are convex, we get \mbox{$\bigcup_{C\in\C}\left(\hat{C}\cup T_C'\right)\in E$}.
      Together with \mbox{$\bigcup_{C\in\C}\hat{C}\subseteq\bigcup_{C\in\C}\left(\hat{C}\cup T_C'\right)$} and \mbox{$\bigcup_{C\in\C}\hat{C}\in D$} we get \mbox{$D\cleq E$} and $D$ is the least upper bound of $\C$.
    \item \mbox{$D=\bigccap_{C\in \C}C$} is the greatest lower bound of $\C$:

      As above, \mbox{$C\in\C$} implies \mbox{$\bigcap_{B\in\C}\hat{B}\subseteq\hat{C}$}, whence by \mbox{$\hat{C}\in C$} and \mbox{$\bigcap_{B\in\C}\hat{B}\in D$} we get \mbox{$D\cleq C$}.
      Since \mbox{$C\in\C$} was arbitrarily chosen, $D$ is a lower bound of $\C$.

      Now let $E$ be any lower bound of $\C$.
      Thus by definition, for each \mbox{$C\in\C$} there exist \mbox{$T_C\in C$} and \mbox{$T_C'\in E$} such that \mbox{$T_C'\subseteq T_C$}.
      Now fix one \mbox{$C\in\C$} and consider an arbitrary \mbox{$B\in\C$}.
      Clearly \mbox{$[T_{C}']^\sem_s=[T_{B}']^\sem_s=E$} whence \mbox{$[T_{C}'\cup T_{B}]^\sem_s=[T_{B}'\cup T_{B}]^\sem_s=[T_B]^\sem_s$} by strong equivalence and \linebreak  since \mbox{$T_{B}'\subseteq T_{B}$}.
      Therefore, \mbox{$T_{C}'\cup T_{B}\in B$}, that is, \mbox{$T_{C}'\subseteq T_{C}'\cup T_{B}\subseteq \hat{B}$} with \mbox{$\hat{B}\in B$}  since $(\L,\sem)$ is covered.
      Since $B$ was chosen arbitrarily, we get the following:
      for every \mbox{$B\in\C$}, we have that \mbox{$T_{C}'\subseteq \hat{B}$}.
      Consequently \mbox{$T_{C}'\subseteq \bigcap_{B\in\C}\hat{B}$} with \mbox{$T_{C}'\in E$} and \mbox{$\bigcap_{B\in\C}\hat{B}\in D$}.
      Thus \mbox{$E\cleq D$} and $D$ is the greatest lower bound of $\C$.
    \end{itemize}
  \end{proof}
\end{lemma}

In particular, the least and greatest elements of that lattice are given by
\begin{gather*}
  \hfill
  \bigccup_{C\in \emptyset}C = \left[\bigcup_{C\in\emptyset}\hat{C}\right]^\sem_s = \left[\emptyset\right]^\sem_s
  \hfill\quad\text{and}\quad\hfill
  \bigccap_{C\in\emptyset}C = \left[\bigcap_{C\in\emptyset}\hat{C}\right]^\sem_s = \left[\L\right]^\sem_s.
  \hfill
\end{gather*}

As we show next, it follows that the mapping \mbox{$[\cdot]^\sem_s:2^\L\to {\left(2^\L\right)_{\!/\equiv^{\sem}_s}}$} assigning a theory $T$ its equivalence class $[T]^\sem_s$ is join-preserving for arbitrary joins.
This is akin to the intersection property for the characterizing semantics, only that the semantics is not yet a model theory in the form of a complete lattice of sets but, more generally, a complete lattice (that is not necessarily one of sets).
\begin{lemma}
  \label{thm:se-cleq-lattice-join}
  Let $(\L,\sem)$ be a \secsme logic with strong equivalence relation $\equiv^\sem_s$.
  For all \mbox{$\T\subseteq 2^\L$}, we have\/:
  \begin{gather*}
    \left[ \bigcup_{T\in\T}T \right]^\sem_s = \bigccup_{T\in\T}[T]^\sem_s
  \end{gather*}
  \begin{proof}
    Let \mbox{$\T\subseteq 2^\L$} and denote \mbox{$\C = \set{ [T]^\sem_s \guard T\in\T }$}.
    Clearly \mbox{$\bigccup_{T\in\T}[T]^\sem_s=\bigccup\C$}.
    Consequently, it suffices to show that
    $$\left[ \bigcup_{T\in\T}T \right]^\sem_s=\bigccup\C.$$
     This will be done by showing that \mbox{$D=\left[ \bigcup_{T\in\T}T \right]^\sem_s$} is the least upper bound of $\C$ in \mbox{$\left(\left(2^\L\right)_{\!/{\equiv^\sem_s}},\cleq\right)$}.

    Let \mbox{$C\in\C$}.
    Then there is a \mbox{$T\in\T$} with \mbox{$C=[T]^\sem_s$}.
    Clearly \mbox{$T\in C$} and
    $$T\subseteq\bigcup_{S\in\T}S\subseteq\widehat{\left[\bigcup_{S\in\T}S\right]^\sem_s}=\hat{D}$$
    with \mbox{$\hat{D}\in D$} since $(\L,\sem)$ is covered.
    Thus \mbox{$C\cleq D$}.
    Since \mbox{$C\in\C$} was arbitrarily chosen, $D$ is an upper bound of $\C$.

    Now let $E$ be any upper bound of $\C$.
    Consider an arbitrary \mbox{$C\in\C$}.
    There clearly is a \mbox{$T\in\T$} with \mbox{$C=[T]^\sem_s$} and since $E$ is an upper bound for $\C$ there also are \mbox{$T_C\in C$} and \mbox{$T_C'\in E$} with \mbox{$T_C\subseteq T_C'$}.
    
    We have \mbox{$[T\cup T_C']^\sem_s=[T_C\cup T_C']^\sem_s=[T_C']^\sem_s=[E]^\sem_s$} whence \mbox{$T\cup T_C'\in E$}.
    Now since every \mbox{$C\in\C$} originates in some \mbox{$T\in\T$} we get that for every \mbox{$T\in\T$} there exists a \mbox{$T_C'\in E$} with \mbox{$T\cup T_C'\in E$}, that is, \mbox{$T\cup T_C'\subseteq\hat{E}$}.
    Then clearly \mbox{$\bigcup_{T\in \T}T \subseteq \bigcup_{T\in\T}(T\cup T_C') \subseteq \hat{E}$}, and by \mbox{$\bigcup_{T\in\T}T\in D$} and \mbox{$\hat{E}\in E$} (the logic $(\L,\sem)$ is covered) we get \mbox{$D\cleq E$}.
    Since $E$ was an arbitrarily chosen upper bound of $\C$, $D$ is the least upper bound of $\C$.
  \end{proof}
\end{lemma}

While it preserves arbitrary joins, the mapping from theories to their strong equivalence classes does not necessarily preserve meets.
It is easy to see that for the meet operation in the complete lattice of strong equivalence classes we have 
\mbox{$\left[ \bigcap_{T\in\T}T \right]^\sem_s \cleq \bigccap_{T\in\T}[T]^\sem_s$}, as
$$\bigcap_{T\in\T}T\subseteq\bigcap_{T\in\T}\widehat{[T]^\sem_s}\in\left[ \bigcap_{T\in\T}\widehat{[T]^\sem_s} \right]^\sem_s = \bigccap_{T\in\T}[T]^\sem_s$$
The reverse relation does not hold, as is witnessed by the following small, finite logic.
\begin{example}
  \label{exm:se-cleq-lattice-meet}
  Let \mbox{$\L=\set{a,b}$} and \mbox{$\equiv^\sem_s$} such that \mbox{$\emptyset\not\equiv^\sem_s\set{a}\equiv^\sem_s\set{b}\equiv^\sem_s\set{a,b}$}.
  Then we get\/:
  \begin{align*}
    [\set{a}\cap\set{b}]^\sem_s & =[\emptyset]^\sem_s=\set{\emptyset} \\
    [\set{a}]^\sem_s\ccap [\set{b}]^\sem_s &= [\set{a,b}]^\sem_s\ccap [\set{a,b}]^\sem_s = [\set{a,b}]^\sem_s = \set{\set{a},\set{b},\set{a,b}} \neq \set{\emptyset}
  \end{align*}
\end{example}

However, this is not a hindrance since the intersection property only needs to hold for arbitrary unions of theories and does not say anything about theory intersection.

We conclude this section with its main theorem showing that any full logic being \secsme possesses a characterization logic. 
The relevant characterization logic can even be defined (more or less) explicitly via a Herbrand-style canonical construction:
roughly, the semantics of the characterization logic maps a theory $T$ of the original language to the union of all strong equivalence classes $[S]^\sem_s$ that are in $\cgeq$-relation to the class $[T]^\sem_s$ of the input theory.

\begin{theorem}
  \label{thm:main:covered}
  Let $(\L, \sem)$ be a logic.
  If $(\L, \sem)$ is \secsme then a characterization logic for $(\L, \sem)$ is given by $(\L, \semc)$ with
\begin{gather*}
    \hfill
    \semc : 2^\L\to 2^{2^\L}
    \hfill,\quad\hfill
    T \mapsto \mathop{\bigcup_{S\in 2^\L,}}_{[T]^\sem_s\cleq[S]^\sem_s}[S]^\sem_s
    \hfill
\end{gather*}
\begin{proof}
  \begin{description}
  \item[\normalfont characterization:]
    We have to show \mbox{$\forall T_1,T_2\subseteq\L : \sem'(T_1)=\sem'(T_2) \iff [T_1]^\sem_s=[T_2]^\sem_s$}.
    
    Let \mbox{$T_1,T_2\subseteq\L$}.
    \begin{description}
    \item[\normalfont $\implies$:]
      Let \mbox{$\sem'(T_1)=\sem'(T_2)$}.
      Then by definition of $\sem'$ we get
      \begin{gather*}
        \mathop{\bigcup_{S\in 2^\L,}}_{[T_1]^\sem_s\cleq[S]^\sem_s}[S]^\sem_s = \mathop{\bigcup_{S\in 2^\L,}}_{[T_2]^\sem_s\cleq[S]^\sem_s}[S]^\sem_s
      \end{gather*}
      that is, for any \mbox{$S\in 2^\L$} we find that \mbox{$[T_1]^\sem_s\cleq [S]^\sem_s$} iff \mbox{$[T_2]^\sem_s\cleq [S]^\sem_s$}.
      The set \mbox{$\set{ [S]^\sem_s \guard S\in 2^\L }$} corresponds to a partition of ${2^\L}$;
      thus for any element \mbox{$T\in\mathop{\bigcup_{S\in 2^\L,}}_{[T_1]^\sem_s\cleq[S]^\sem_s}[S]^\sem_s$} we have \mbox{$[T_1]^\sem_s\cleq [T]^\sem_s$}, and if also \mbox{$T\in\mathop{\bigcup_{S\in 2^\L,}}_{[T_2]^\sem_s\cleq[S]^\sem_s}[S]^\sem_s$}, then \mbox{$[T_2]^\sem_s\cleq [T]^\sem_s$} as well.
      Symmetry applies to obtain the above conclusion.

      In particular, \mbox{$T_1,T_2\in 2^\L$} whence \mbox{$[T_1]^\sem_s\cleq [T_1]^\sem_s$} iff \mbox{$[T_2]^\sem_s\cleq [T_1]^\sem_s$}, and \mbox{$[T_1]^\sem_s\cleq [T_2]^\sem_s$} iff \linebreak \mbox{$[T_2]^\sem_s\cleq [T_2]^\sem_s$}.
      This shows \mbox{$[T_1]^\sem_s\cleq [T_2]^\sem_s$} and \mbox{$[T_2]^\sem_s\cleq [T_1]^\sem_s$}, that is, \mbox{$[T_1]^\sem_s=[T_2]^\sem_s$}.
    \item[\normalfont $\impliedby$:]
      Let \mbox{$[T_1]^\sem_s=[T_2]^\sem_s$}.
      Then immediately
      \begin{gather*}
        \sem'(T_1) = \mathop{\bigcup_{S\in 2^\L,}}_{[T_1]^\sem_s\cleq[S]^\sem_s}[S]^\sem_s = \mathop{\bigcup_{S\in 2^\L,}}_{[T_2]^\sem_s\cleq[S]^\sem_s}[S]^\sem_s = \sem'(T_2)
      \end{gather*}
    \end{description}
  \item[\normalfont intersection:]
    We have to show \mbox{$\forall \T\subseteq 2^\L: \sem'\!\left(\bigcup_{T\in \T}T\right) = \bigcap_{T\in\T}\sem'(T)$}.
    
    Let \mbox{$\T\subseteq 2^\L$}.
    From Theorem~\ref{thm:se-cleq-lattice-join}, we know that \mbox{$\left[ \bigcup_{T\in\T}T \right]^\sem_s=\bigccup_{T\in\T}[T]^\sem_s$} is the least upper bound of \mbox{$\set{ [T]^\sem_s \guard T\in\T }$}, whence for all \mbox{$S\in 2^\L$}, we have that \mbox{$\bigccup_{T\in\T}[T]^\sem_s\cleq [S]^\sem_s$} if and only if $[S]^\sem_s$ is an upper bound of \mbox{$\set{ [T]^\sem_s \guard T\in\T }$}, that is,
    \begin{gather*}
      \left[\bigcup_{T\in \T}T\right]^\sem_s\cleq[S]^\sem_s
      \quad\iff\quad
      \forall T\in\T:[T]^\sem_s\cleq [S]^\sem_s
    \end{gather*}
    Using this relationship, we get that
    \begin{align*}
      \sem'\!\left(\bigcup_{T\in \T}T\right)
      &= \mathop{\bigcup_{S\in 2^\L,}}_{\left[\bigcup_{T\in \T}T\right]^\sem_s\cleq[S]^\sem_s}[S]^\sem_s \\
      &= \mathop{\bigcup_{S\in 2^\L,}}_{\forall T\in\T:[T]^\sem_s\cleq [S]^\sem_s}[S]^\sem_s \\
      &= \bigcap_{T\in\T}\left(\mathop{\bigcup_{S\in 2^\L,}}_{[T]^\sem_s\cleq[S]^\sem_s}[S]^\sem_s\right) \\
      &= \bigcap_{T\in\T}\sem'(T)
    \end{align*}
  \end{description}
\end{proof}
\end{theorem}

The construction from the statement of the theorem looks quite abstract, so we illustrate with a small concrete logic.

\begin{example}
  \label{exm:covered-characterization}
  We reconsider the logic from Example~\ref{exm:se-cleq-lattice-not-distributive}.
  There, \mbox{$\L=\set{a,b,c}$} and 
  \begin{align*}
    [\emptyset]^\sem_s&=\set{\emptyset}\!,\\
    [\set{a}]^\sem_s&=\set{\set{a}}\!,\\
    [\set{b}]^\sem_s&=\set{\set{b}}\!,\\
    [\set{c}]^\sem_s&=\set{\set{c}}\!,\\
    [\set{a,b}]^\sem_s&=[\set{a,c}]^\sem_s=[\set{b,c}]^\sem_s=[\set{a,b,c}]^\sem_s =\set{\set{a,b},\set{a,c},\set{b,c},\set{a,b,c}}\!.
  \end{align*}
  The resulting lattice of strong equivalence classes is depicted below\/: 
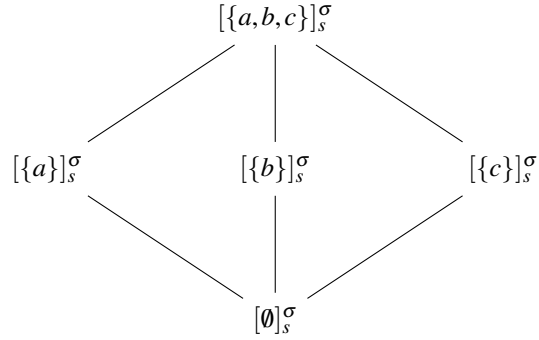
  \begin{center}
    \begin{tikzpicture}[xscale=3,yscale=2]
      \node (e) at (0,-1) {$[\emptyset]^\sem_s$};
      \node (a) at (-1,0) {$[\set{a}]^\sem_s$};
      \node (b) at (0,0)  {$[\set{b}]^\sem_s$};
      \node (c) at (1,0)  {$[\set{c}]^\sem_s$};
      \node (abc) at (0,1) {$[\set{a,b,c}]^\sem_s$};
      \path (e) edge[-] (a);
      \path (e) edge[-] (b);
      \path (e) edge[-] (c);
      \path (a) edge[-] (abc);
      \path (b) edge[-] (abc);
      \path (c) edge[-] (abc);
    \end{tikzpicture}
  \end{center}
  According to Theorem~\ref{thm:main:covered}, the characterization logic \mbox{$\sem':2^\L\to 2^{2^\L}$} assigns as follows\/:
  \begin{align*}
    \sem'(\emptyset) &= \mathop{\bigcup_{S\in 2^\L,}}_{[\emptyset]^\sem_s\cleq[S]^\sem_s}[S]^\sem_s = \bigcup_{S\in 2^\L}[S]^\sem_s&&= 2^\L \\
    \sem'(\set{a}) &= \mathop{\bigcup_{S\in 2^\L,}}_{[\set{a}]^\sem_s\cleq [S]^\sem_s}[S]^\sem_s = [\set{a}]^\sem_s \cup [\set{a,b,c}]^\sem_s &&= \set{ \set{a}, \set{a,b}, \set{a,c}, \set{b,c}, \set{a,b,c} } \\
    \sem'(\set{b}) &= \mathop{\bigcup_{S\in 2^\L,}}_{[\set{b}]^\sem_s\cleq [S]^\sem_s}[S]^\sem_s = [\set{b}]^\sem_s \cup [\set{a,b,c}]^\sem_s &&= \set{ \set{b}, \set{a,b}, \set{a,c}, \set{b,c}, \set{a,b,c} } \\
    \sem'(\set{c}) &= \mathop{\bigcup_{S\in 2^\L,}}_{[\set{c}]^\sem_s\cleq [S]^\sem_s}[S]^\sem_s = [\set{c}]^\sem_s \cup [\set{a,b,c}]^\sem_s &&= \set{ \set{c}, \set{a,b}, \set{a,c}, \set{b,c}, \set{a,b,c} } \\
    \sem'(\set{a,b,c}) &= \mathop{\bigcup_{S\in 2^\L,}}_{[\set{a,b,c}]^\sem_s\cleq [S]^\sem_s}[S]^\sem_s = [\set{a,b,c}]^\sem_s &&= \set{ \set{a,b}, \set{a,c}, \set{b,c}, \set{a,b,c} }
  \end{align*}
  Now it holds for example that
  \begin{align*}
    &\pheq \sem'(\set{a}) \cap \sem'(\set{b}) \\
    &= \set{ \set{a}, \set{a,b}, \set{a,c}, \set{b,c}, \set{a,b,c} } \cap \set{ \set{b}, \set{a,b}, \set{a,c}, \set{b,c}, \set{a,b,c} } \\
    &= \set{ \set{a,b}, \set{a,c}, \set{b,c}, \set{a,b,c} } \\
    &= \sem'(\set{a,b}) \\
    &= \sem'(\set{a} \cup \set{b})
  \end{align*}
\end{example}

\subsection{Finite-Theory Characterization Logics}
\label{sec:characterization:finitary}

In the field of knowledge representation it is a common assumption that knowledge bases are finite.
This is indeed not overly limiting, as finite knowledge bases will be most relevant for practical purposes.
The following definition translates this assumption into our setting:
the \textit{finite-theory version} of a given logic (or simply, a \textit{finite-theory logic}) considers only the finite knowledge bases of a language.

\begin{definition}
  \label{def:finitary}
  Given a full logic $(\L, \sem)$, the finite-theory version $(\L, \semf)$ of $(\L, \sem)$ is defined by the semantics
  $$\semf: \Lfin \to \sigma\!\left(2^\L\right) \qquad\text{with}\qquad \semf(T) = \sem(T)$$
  where $\Lfin = \{T\in 2^\L\mid T \text{ is finite}\}$.
\end{definition}

For finite-theory restrictions of logics, we adequately relax our requirements on characterization logics.

\begin{definition}
  \label{def:fin-characterization-logic}
  Let $(\L, \sem)$ be a full logic and $(\L, \semf)$ its finite-theory version.
  We say that $(\L, \semfc)$ is a \define{finite-theory characterization logic for $(\L, \sem)$} if and only if\/:
  \begin{enumerate}
  \item \mbox{$\forall T_1,T_2\in\Lfin: \semfc(T_1)=\semfc(T_2)$ iff $[T_1]^{\semf}_s=~[T_2]^{\semf}_s$}; \hfill (finite-theory characterization)
  \item \mbox{$\forall T_1,T_2\in\Lfin: \semfc (T_1 \cup T_2) = \semfc(T_1) \cap \semfc(T_2)$}. \hfill (finite-theory intersection)
  \end{enumerate}
\end{definition}
The second item requires binary intersection only;
this is due to the fact that arbitrary unions of (finite) theories are not necessarily finite, and thus their semantics might not be well-defined.

As we did in the general case before, we first analyze the algebraic structure of the resulting model theories.
We show that the model theory of any finite-theory characterization logic forms a lattice, that is, a partially ordered set where each non-empty finite subset has both a greatest lower bound and a least upper bound.
(This is in contrast to \emph{complete} lattices in the general case.)
The proof is, although similar in procedure, slightly more involved than in the general case.

\begin{proposition}
  \label{thm:fin-characterization-lattice}
  Let $(\L, \semf)$ be a finite-theory logic with characterization logic $(\L, \semfc)$.
  Denoting \mbox{$\K=\set{\semfc(T)\guard T\in\Lfin }$}, the pair $\left(\K, \subseteq\right)$ is a lattice where 
  glb and lub are given such that
  for all \mbox{$K_1,K_2\in\K$}\/:
  \begin{gather*}
    \hfill
    K_1\wedge K_2 = K_1\cap K_2
    \hfill\text{ and }\hfill
    K_1\vee K_2 = \bigwedge\set{K_1,K_2}^u
    \hfill
  \end{gather*}
  where \mbox{$\set{K_1,K_2}^u = \set{ K\in\K \guard K_1\subseteq K, K_2\subseteq K }$}. 
  \begin{proof}
    Let \mbox{$K_1,K_2\in\K$}.
    Clearly there exist finite \mbox{$T_1,T_2\subseteq\L$} with \mbox{$\semfc(T_1)=K_1$} as well as \mbox{$\semfc(T_2)=K_2$}.
    \begin{description}
    \item[\normalfont glb:] 
      Obviously \mbox{$T_1\cup T_2\in\Lfin$}.
      Therefore, it follows that 
      \begin{align*}
        K_1\wedge K_2=\semfc(T_1)\wedge\semfc(T_2)=\semfc(T_1)\cap\semfc(T_2)=\semfc(T_1\cup T_2)\in\K
      \end{align*}
			\pagebreak
    \item[\normalfont lub:] 
      It follows from the finite intersection property that $\semfc$ is antimonotone.
      Now \mbox{$\emptyset\subseteq T_1$} implies that \mbox{$K_1=\semfc(T_1)\subseteq\semfc(\emptyset)$} and likewise for $T_2$.
      Thus \mbox{$\semfc(\emptyset)$} is an upper bound of $K_1$ and $K_2$, whence the set \mbox{$\set{K_1,K_2}^u$} is non-empty.
      We will now show that $\set{K_1,K_2}^u$ is finite.  
     Clearly both $T_1$ and $T_2$ have only finitely many subsets $T_1'$ and $T_2'$.
      For each of these subsets, $\semfc$ being antimonotone means that \mbox{$\emptyset\subseteq T_1'\subseteq T_1$} implies \mbox{$\semfc(T_1)\subseteq\semfc(T_1')\subseteq\semfc(\emptyset)$}.
      Thus both $K_1$ and $K_2$ have only finitely many supersets in $\K$, that is, the sets \mbox{$K_1^\uparrow$} and \mbox{$K_2^\uparrow$} are finite, whence \mbox{$\set{K_1,K_2}^u\subseteq K_1^\uparrow\cup K_2^\uparrow$} is finite.
      Consequently, \mbox{$\set{K_1,K_2}^u$} is a finite, non-empty subset of $\K$.
			It therefore possesses a greatest lower bound \mbox{$K_l=\bigwedge\set{K_1,K_2}^u\in\K$}.
			Since $\set{K_1,K_2}^u$ is closed under intersection we have that \mbox{$K_l=\bigwedge\set{K_1,K_2}^u=\bigcap\set{K_1,K_2}^u\in\set{K_1,K_2}^u$} is the least element of $\set{K_1,K_2}^u$ and therefore the least upper bound of $K_1$ and $K_2$ concluding the proof.
    \end{description}
  \end{proof}
\end{proposition}
\noindent As before, we can show (with reasonable effort) that finite-theory characterization logics are unique up to isomorphism.%
\begin{theorem}
  \label{thm:fin-unique-characterization}
  Let $(\L, \sem)$ be a finite-theory logic having two finite-theory characterization logics $(\L, \semfc)$ and $(\L, \semfcc)$.
  Denoting the two carrier sets by \mbox{$\K'=\set{\semfc(T)\guard T\in\Lfin}$} and accordingly \mbox{$\K''=\set{\semfcc(T)\guard T\in\Lfin}$}, the lattices $\left(\K', \subseteq\right)$ and $\left(\K'', \subseteq\right)$ are isomorphic.
  \begin{proof}
    We provide a bijection \mbox{$\phi:\K'\to\K''$} such that for all \mbox{$K_1,K_2\in\K'$}, we find that \mbox{$\phi(K_1\wedge K_2)$} \mbox{$=\phi(K_1)\wedge\phi(K_2)$} and \mbox{$\phi(K_1\vee K_2)=\phi(K_1)\vee\phi(K_2)$}.
    Let \mbox{$K\in\K'$}.
    By definition, there exists a finite \mbox{$T\subseteq\L$} with \mbox{$\semfc(T)=K$}.
    Define \mbox{$\phi(K)=\semfcc(T)$}.
    \begin{description}
    \item[\normalfont$\phi$ is bijective:] 
      The proof is as in the general case.
    \item[\normalfont$\phi$ is structure-preserving:] 
      Let \mbox{$K_1,K_2\in\K'$}.
      Clearly there exist \mbox{$T_1,T_2\subseteq\L$} such that \mbox{$\semfc(T_1)=K_1$} and \mbox{$\semfc(T_2)=K_2$}.
      We have 
      \begin{align*}
        \phi(K_1\wedge K_2)
        &=\phi(K_1\cap K_2) & \text{(Def.~$\wedge$)} \\
        &=\phi(\semfc(T_1)\cap\semfc(T_2)) & \text{($\semfc$ is onto $\K'$)} \\
        &=\phi(\semfc(T_1\cup T_2)) & \text{(intersection $\semfc$)} \\
        &=\semfcc(T_1\cup T_2) & \text{(Def.~$\phi$)} \\
        &=\semfcc(T_1)\cap\semfcc(T_2) & \text{(intersection $\semfcc$)} \\
        &=\phi(\semfc(T_1))\cap\phi(\semfc(T_2)) & \text{(Def.~$\phi$)} \\
        &=\phi(K_1)\cap\phi(K_2) & \text{(assumption)} \\
        &=\phi(K_1)\wedge\phi(K_2) & \text{(Def.~$\wedge$)}
      \end{align*}
      With reasonable effort, we can also show that
      \begin{align*}
        &\pheq \phi(K_1\vee K_2) \\
        &= \phi\!\left(\bigwedge\set{K_1,K_2}^u\right) \\
        & \qquad\qquad\text{(Def.~$\vee$)} \\
        &= \bigwedge\phi\!\left(\set{K_1,K_2}^u\right) \\
        & \qquad\qquad\text{($\phi$ preserves $\wedge$)} \\
        &= \bigwedge\phi\!\left(\set{ K\in\K' \guard K_1\subseteq K, K_2\subseteq K }\right) \\
        & \qquad\qquad\text{(Def.~$\cdot^u$)} \\
        &= \bigwedge\set{ \phi(K) \guard K\in\K', K_1\subseteq K, K_2\subseteq K } \\
        & \qquad\qquad\text{(notation)} \\
        &= \bigwedge\set{ \phi(\semfc(T)) \guard T\in\Lfin, \semfc(T_1)\subseteq\semfc(T), \semfc(T_2)\subseteq\semfc(T) } \\
        & \qquad\qquad\text{($\semfc$ is onto $\K'$)} \\
        &= \bigwedge\set{ \phi(\semfc(T)) \guard T\in\Lfin, \semfc(T_1)\cap\semfc(T)=\semfc(T_1), \semfc(T_2)\cap\semfc(T)=\semfc(T_2) } \\
        & \qquad\qquad\text{(elementary)} \\
        &= \bigwedge\set{ \phi(\semfc(T)) \guard T\in\Lfin, \semfc(T_1\cup T)=\semfc(T_1), \semfc(T_2\cup T)=\semfc(T_2) } \\
        & \qquad\qquad\text{(intersection $\semfc$)} \\
        &= \bigwedge\set{ \phi(\semfc(T)) \guard T\in\Lfin, [T_1\cup T]^\sem_s=[T_1]^\sem_s, [T_2\cup T]^\sem_s=[T_2]^\sem_s } \\
        & \qquad\qquad\text{(characterization $\semfc$)} \\
        &= \bigwedge\set{ \phi(\semfc(T)) \guard T\in\Lfin, \semfcc(T_1\cup T)=\semfcc(T_1), \semfcc(T_2\cup T)=\semfcc(T_2) } \\
        & \qquad\qquad\text{(characterization $\semfcc$)} \\
        &= \bigwedge\set{ \semfcc(T) \guard T\in\Lfin, \semfcc(T_1)\cap\semfcc(T)=\semfcc(T_1), \semfcc(T_2)\cap\semfcc(T)=\semfcc(T_2) } \\
        & \qquad\qquad\text{(intersection $\semfcc$)} \\
        &= \bigwedge\set{ \semfcc(T) \guard T\in\Lfin, \semfcc(T_1)\subseteq\semfcc(T), \semfcc(T_2)\subseteq\semfcc(T) } \\
        & \qquad\qquad\text{(elementary)} \\
        &= \bigwedge\set{ K \guard K\in\K'', \semfcc(T_1)\subseteq K, \semfcc(T_2)\subseteq K } \\
        & \qquad\qquad\text{($\semfcc$ is onto $\K''$)} \\ 
        &= \bigwedge\set{ \semfcc(T_1),\semfcc(T_2)}^u \\
        & \qquad\qquad\text{(Def.~$\cdot^u$)} \\
        &= \semfcc(T_1)\vee\semfcc(T_2) \\
        & \qquad\qquad\text{(Def.~$\vee$)} \\
        &= \phi(\semfc(T_1))\vee\phi(\semfc(T_2)) \\
        & \qquad\qquad\text{(Def.~$\phi$)} \\
        &= \phi(K_1)\vee\phi(K_2) \\
        & \qquad\qquad\text{(assumption)}
      \end{align*}
      Thus $\phi$ is a structure-preserving bijection from $\K'$ to $\K''$, and the two lattices are isomorphic.
    \end{description}
  \end{proof}
\end{theorem}

The following theorem shows that any logic possesses a finite-theory characterization logic.
This means that the most important case for knowledge representation behaves well in the sense that characterization logics always exist.

\begin{theorem}
  \label{thm:characterization:finite-theory}
  Let $(\L, \sem)$ be a full logic.
  Then a finite-theory characterization logic for $(\L, \sem)$ is given by $(\L, \semfc)$ with
  \begin{gather*}
    \hfill
    \semfc : \Lfin \to 2^{2^\L}
    \hfill,\quad\hfill
    T \mapsto \mathop{\bigcup_{S\in \Lfin,}}_{T \subseteq S}[S]^{\semf}_s
    \hfill
  \end{gather*}
  \begin{proof}
    \begin{description}
    \item[\normalfont finite intersection:]
      We show that for all \mbox{$T\in\Lfin$}, we find \mbox{$\semfc(T) = \bigcap_{t\in T}\semfc(\set{t})$}.
      Consider\ \,\mbox{$T\in\Lfin$}.
      \begin{description}
      \item[\normalfont$\subseteq$:]
        Let \mbox{$U\in\semfc(T)$}. 
        Hence, there is an \mbox{$S\in \Lfin$} such that \mbox{$U\in [S]^{\semf}_s$} and \mbox{$T \subseteq S$}.
        Consequently, for any \mbox{$t\in T$}, we find \mbox{$\set{t}\subseteq T \subseteq S$} and thus \mbox{$[S]^{\semf}_s\subseteq \semfc(\set{t})$} showing \mbox{$U\in\bigcap_{t\in T}\semfc(\set{t})$}.
      \item[\normalfont$\supseteq$:]
        Consider now \mbox{$U\in\bigcap_{t\in T}\semfc(\set{t})$}. 
        This means that for any \mbox{$t\in T$} there exists an \mbox{$S_t\in\Lfin$} such that \mbox{$U\in [S_t]^{\semf}_s$} (that is, \mbox{$[U]^{\semf}_s=[S_t]^{\semf}_s$}) and \mbox{$\set{t}\subseteq S_t$}.
        It follows that for any \mbox{$t,t'\in T$}, we find \mbox{$[S_{t}]^{\semf}_s = [U]^{\semf}_s = [S_{t'}]^{\semf}_s$}.
        Let us fix a certain \mbox{$\overline{t}\in T$}. 
        In particular, \mbox{$U\in [S_{\overline{t}}]^{\semf}_s$}.
        Now consider the set \mbox{$\bigcup_{t\in T} S_t$}, which is finite since $T$ is finite and each $S_t$ is finite.
        Furthermore, \mbox{$\bigcup_{t\in T} S_t \in [S_{\overline{t}}]^{\semf}_s$}, that is, \mbox{$\left[\bigcup_{t\in T} S_t\right]^{\semf}_s = [S_{\overline{t}}]^{\semf}_s$} by Thoerem~\ref{thm:se-expansion}.
        Since for all \mbox{$t\in T$} we have \mbox{$t\in S_t$}, we conclude that \mbox{$T=\bigcup_{t\in T}\{t\}\subseteq\bigcup_{t\in T} S_t$}.
        In turn, this yields \mbox{$U\in\semfc(T)$}.
      \end{description}
    \item[\normalfont finite-theory characterization:]
      We show that for all \mbox{$T_1,T_2\in\Lfin$}, we find \mbox{$[T_1]^{\semf}_s=[T_2]^{\semf}_s$} if and only if \mbox{$\semfc(T_1)=\semfc(T_2)$}.
      \begin{description}
      \item[\normalfont if:]
        Let \mbox{$T_1,T_2\in\Lfin$} with \mbox{$\semfc(T_1)=\semfc(T_2)$}.
        Firstly, the definition of $\semfc$ yields that for each \mbox{$S\in\Lfin$} we have \mbox{$[S]^{\semf}_s\subseteq\semfc(S)$}.
        In combination with the presumption, this means that \mbox{$[T_1]^{\semf}_s\subseteq\semfc(T_1)=\semfc(T_2)$} and \mbox{$[T_2]^{\semf}_s\subseteq\semfc(T_2)=\semfc(T_1)$}.
        Hence, there is a set \mbox{$S_1\in[T_2]^{\semf}_s$} with \mbox{$T_1\subseteq S_1$}, whence \mbox{$[T_1]^{\semf}_s\cleq [T_2]^{\semf}_s$}.
        Likewise, there is a set \mbox{$S_2\in[T_1]^{\semf}_s$} with \mbox{$T_2\subseteq S_2$}, whence also \mbox{$[T_2]^{\semf}_s\cleq [T_1]^{\semf}_s$}.
        By Theorem~\ref{thm:cleq-partial-order}  (saying that $\cleq$ is antisymmetric), we get \mbox{$[T_1]^{\semf}_s=[T_2]^{\semf}_s$}.
      \item[\normalfont only if:]
        Let \mbox{$T_1,T_2\in\Lfin$} with \mbox{$[T_1]^{\semf}_s=[T_2]^{\semf}_s$} and consider any \mbox{$T\in\semfc(T_1)$}.
        Then there is an \mbox{$S\in \Lfin$} such that \mbox{$T\in [S]^{\semf}_s$} and \mbox{$T_1 \subseteq S$}.
        Clearly there is a \mbox{$U\in \Lfin$} with \mbox{$T_1 \cup U = S$}.
        Since \mbox{$[T_1]^{\semf}_s=[T_2]^{\semf}_s$} by presumption, we can apply Lemma~\ref{thm:se-expansion}, which yields \mbox{$T\in [S]^{\semf}_s$} \mbox{$= [T_1\cup U]^{\semf}_s = [T_2\cup U]^{\semf}_s$}.
        Obviously, we have \mbox{$T_2\cup U\in\Lfin$} with \mbox{$T_2\subseteq T_2\cup U$} and the definition of $\semfc$ yields \mbox{$T\in\semfc(T_2)$}.
        This shows \mbox{$\semfc(T_1)\subseteq\semfc(T_2)$};
        the reverse inclusion \mbox{$\semfc(T_2)\subseteq\semfc(T_1)$} holds by symmetry.
      \end{description}
    \end{description}
  \end{proof}
\end{theorem}

\noindent Intuitively, in this canonical construction of a characterization semantics $\semfc$ (akin to Herbrand interpretations in first-order logic), the model set of a theory $T$ is the set of all theories that are strongly equivalent to some supertheory of~$T$.
 
\section{Applying Canonical Constructions to Nonmonotonic Formalisms} 

In the previous section we have seen that (under certain condition) the existence of characterization logics for knowledge representation formalisms are guaranteed. 
These results are achieved by defining in a sense Herbrand-style canonical construction. 
More precisely, the characterization semantics of the new logic is defined in terms of certain unions of strong equivalence classes of the original language.
Such a characterization semantics is usually far from being intuitive or self-explanatory.
The intended role of this semantics was to serve as a witness for the existence of characterization logics. 
Nevertheless, we will discuss the application our general, abstract results to some of the formalisms presented in Example~\ref{exm:known-formalisms}. 
We start with abstract argumentation theory, which is a vibrant as well as immensely growing research area in AI \citep{Dung95}.
Surprisingly, we get a meaningful result very similar to the recently introduced Dung logics \cite{BauB15}.

\subsection{Abstract Argumentation Theory} \label{sec:argu}

We start with a very brief introduction to Dung's argumentation theory which is sufficient for the moment (cf.\ Section~\ref{sec:Dung} for more detailed as well as more general definitions and \cite{BaroniCG18} for a recent and comprehensive overview). 

An \define{argumentation framework} (AF) is a pair \mbox{$\AF = (A,R)$} such that \mbox{$R \subseteq A \times A$}.
Although there exists some work on unrestricted AFs \citep{BauS15,BauS17} it is common to assume that $A$, the set of arguments, is a finite subset of a fixed infinite background set~\m{U}. Let us denote the class of all finite AFs by~$\AFl_{\fin}$.
An (extension-based) \define{argumentation semantics} is a function \mbox{$\rho:\AFl_{\fin} \to 2^{2^\m{U}}$} where elements of $\rho(\AF)$ are called \define{$\rho$-extensions} of~$\AF$. The most prominent one is stable semantics (abbreviated by $\stb$) which was already defined by Dung in 1995. A set $E$ is a $\stb$-extension of $\AF$ if 1.\ there are no $a,b\in A$, s.t.\ $(a,b)\in R$ (\textit{conflict-freeness}) and 2.\ for any $c\in A\sm E$, there is an $a\in A,$ s.t.\ $(a,c)\in R$ (\textit{full range}). 

We proceed with some notational conventions and the precise definition of strong equivalence in case of AFs. The union \mbox{$\AF\dcup\AG$} as well as subset-relation \mbox{$\AF\dsub\AG$} of two AFs is understood to be pointwise, that is, \mbox{$(A_1,R_1)\dcup (A_2,R_2) = (A_1\cup A_2, R_1\cup R_2)$}, and, similarly, $(A_1,R_1)\dsub (A_2,R_2)$ if and only if $A_1\subseteq A_2$ and $R_1\subseteq R_2$. 

\begin{definition}
Given an argumentation semantics $\rho$. Two AFs $\AF$ and $\AG$ are \textit{strongly $\rho$-equivalent} if for any $\AH\in\AFl_{\fin}$, $\rho(\AF\dcup \AH) = \rho(\AG\dcup \AH)$. For short, $\AF \equiv^\rho_s \AG$.
\end{definition}

The first work regarding characterizing strong equivalence for AFs was \cite{strong}. It turned out that deciding this notion is deeply linked to the syntax of AFs. In general, any argument being part of an AF may contribute towards future extensions. However, for each semantics, there are patterns of redundant attacks captured by so-called \textit{kernels}. Formally, a kernel is a function $ k: \AFl_{\fin} \rightarrow \AFl_{\fin} $ where $ k(\AF) = \AF^k$ is obtained from $ \AF $ by deleting certain redundant attacks. Consider the following definition and characterization theorem. 

\begin{definition}\label{dfn:stbkernel}
    Given an AF $\AF = (A,R)$. The \textit{$ \stb $-kernel} $ \AF^{k(\stb)} = \big( A, R^{k(\stb)} \big) $ is defined with
    $ R^{k(\stb)} = R \setminus \{ (a, b)  \mid  a \not = b \land (a, a) \in R \} $.
		\end{definition}

\begin{theorem}[\cite{strong}]\label{the:stbkernel} For two AFs $\AF,\AG$ we have:
$$\AF \equiv^\stb_s \AG\ \iff \AF^{k(\stb)} = \AG^{k(\stb)}.$$
\end{theorem}

Let us illustrate the introduced concepts with an example.
\begin{example} Consider the following six AFs. 
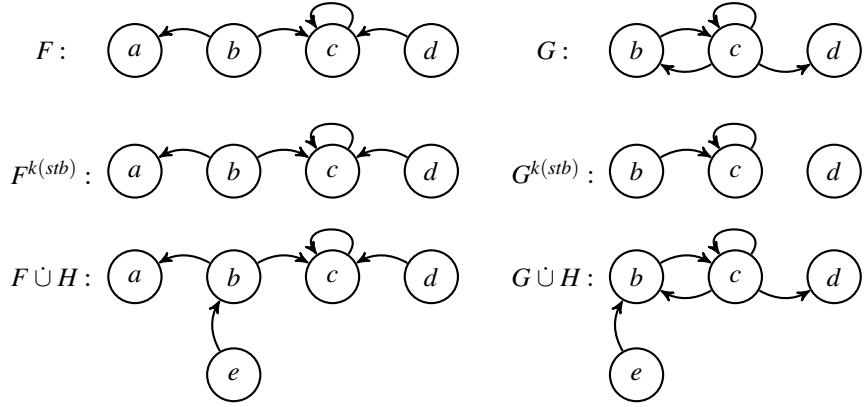
\label{ex:six}
\begin{center}
        \begin{tikzpicture}[node distance=13mm,on grid]
				
            \node[args] (d) {$ d $};
            \node[args] (c) [left of=d] {$ c $}
              edge[<-,thick,bend left] (d)
                edge[->,thick,loop,distance=15pt,in=125,out=55] (c);
            \node[args] (b) [left of=c] {$ b $}
						   edge[->,thick,bend left] (c);
             \node[args] (a) [left of=b] {$ a $};
             \node (F) [left = 1.1 cm of a] {$ \AF: $};
						\draw[->,thick,bend right] (b) to (a);
						
             \node (G) [right = 1.6 cm of d] {$ \AG: $};
            \node[args] (b_G) [right = 1.1 cm of G] {$ b $};
            \node[args] (c_G) [right of=b_G] {$ c $}
                edge[->,thick,loop,distance=15pt,in=125,out=55] (c_G)
                edge[<-,thick,bend right] (b_G)
                edge[->,thick,bend left] (b_G);
            \node[args] (d_G) [right of=c_G] {$ d $}
                edge[<-,thick,bend left] (c_G);

            \node[args] (d') [below = 1.6cm of d]{$ d $};
            \node[args] (c') [left of=d'] {$ c $}
              edge[<-,thick,bend left] (d')
                edge[->,thick,loop,distance=15pt,in=125,out=55] (c');
            \node[args] (b') [left of=c'] {$ b $}
						   edge[->,thick,bend left] (c');
             \node[args] (a') [left of=b'] {$ a $};
             \node (F') [left = 1.1 cm of a'] {$ \AF^{k(\stb)}: $};
						\draw[->,thick,bend right] (b') to (a');
						
             \node (G') [right = 1.6 cm of d'] {$ \AG^{k(\stb)}: $};
            \node[args] (b_G') [right = 1.1 cm of G'] {$ b $};
            \node[args] (c_G') [right of=b_G'] {$ c $}
                edge[->,thick,loop,distance=15pt,in=125,out=55] (c_G')
                edge[<-,thick,bend right] (b_G');
            \node[args] (d_G') [right of=c_G'] {$ d $};
								
								  \node[args] (d'') [below = 1.4cm of d'] {$ d $};
            \node[args] (c'') [left of=d''] {$ c $}
              edge[<-,thick,bend left] (d'')
                edge[->,thick,loop,distance=15pt,in=125,out=55] (c'');
            \node[args] (b'') [left of=c''] {$ b $}
						   edge[->,thick,bend left] (c'');
							\node[args] (e'') [below of=b''] {$ e $}
							edge[->,thick,bend left] (b'');
             \node[args] (a'') [left of=b''] {$ a $};
             \node (F'') [left = 1.1 cm of a''] {$ \AF\dcup\AH: $};
						\draw[->,thick,bend right] (b'') to (a'');
						
             \node (G'') [right = 1.6 cm of d''] {$ \AG\dcup\AH: $};
            \node[args] (b_G'') [right = 1.1 cm of G''] {$ b $};
						\node[args] (e_G'') [below of=b_G''] {$ e $}
							edge[->,thick,bend left] (b_G'');
            \node[args] (c_G'') [right of=b_G''] {$ c $}
                edge[->,thick,loop,distance=15pt,in=125,out=55] (c_G'')
                edge[<-,thick,bend right] (b_G'')
                edge[->,thick,bend left] (b_G'');
            \node[args] (d_G'') [right of=c_G''] {$ d $}
                edge[<-,thick,bend left] (c_G'');
         \end{tikzpicture}
			
    \end{center} 
The AFs $\AF$ and $\AG$ possess the same stable extensions, namely $\stb(\AF) = \stb(\AG) = \{\{b,d\}\}$. However, both frameworks are not strongly $\stb$-equivalent since $\AF^{k(\stb)}\neq\AG^{k(\stb)}$. A witnessing framework is given by $\AH$. Indeed $\stb(\AF\dcup\AH) =  \{\{a,d,e\}\} \neq \emptyset = \stb(\AG\dcup\AH)$.
		
\end{example}

Let us consider now argumentation theory in the general setup. In Example~\ref{exm:known-formalisms} we have seen that the embedding of abstract argumentation in our setting is a bit more involved than in case of propositional logic, logic programs or default logic.
The main reason for this is that in contrast to the other considered formalisms we have that abstract argumentation frameworks possess two sorts of building blocks, namely arguments and attacks. Moreover, the latter are dependent since adding attacks requires the presence of the corresponding arguments.
In order to cast AFs into our general setup we have to have theories which correspond to AFs, s.t.\ the standard set union $\cup$ of such theories correspond to $\dcup$ on the AF-level. Moreover, the semantics of theories has to correspond to the argumentation semantics of the associated AFs.

We start with the introduction of a $\rho$-logic which formally captures a specific argumentation semantics $\rho$ on the level of theories.

\begin{definition}
  \label{def:arglogic}
  Let $\m{U}$ be a background set of arguments and $\rho$ be an AF semantics.
  A \define{$\rho$-logic} is a triple \mbox{$(\L_\AFl,\I,\sem_{\rho})$} where \mbox{$\L_\AFl = \{(\{a\},\emptyset),(\{a,b\},\{(a,b)\})\mid a,b\in\m{U}\}$}, \mbox{$\I = 2^{\m{U}}$} and \mbox{$\sem_{\rho}: 2^{\L_\AFl} \to 2^\I$} with \mbox{$\sem_{\rho}(T) = \rho\!\left(\dbigcup_{t\in T} t \right)$}.
\end{definition}

The representational issue implies that two different theories may represent the same framework which causes some additional effort.
More precisely, in what follows it will be a typical task to show that the presented results are independent of the concrete representation of a certain framework. Let us start with an illustrating example.

\begin{example} \label{ex:repr} Let $T = \left\{(\{b,c\},\{(b,c)\}),(\{c,b\},\{(c,b)\}),(\{c,d\},\{(c,d)\}),(\{c\},\{(c,c)\})\right\}$ and $S = T\cup \{(\{a\},\emptyset),(\{b\},\emptyset),(\{c\},\emptyset)\}$. Please observe that $\dbigcup_{t\in T} t = \dbigcup_{s\in S} s = \AG$ as depicted in Example~\ref{ex:six}. Moreover by definition we have $\sigma_{\stb}(T) = \sigma_{\stb}(S) = \stb(\AG) = \{\{b,d\}\}$.
\end{example}

The following functions (restricted to the finite case) will be frequently used.
First, we define the \textit{associated AF} of a given theory via \mbox{$\AFl: \left(2^{\L_\AFl}\right)_{\fin} \to \AFl_{\fin}$} where \mbox{$T\mapsto\dbigcup_{t\in T}t$}.
As already discussed the function $\AFl(\cdot)$ is not injective as demonstrated in Example~\ref{ex:repr}.
Secondly, the \textit{canonical representation}~of~a given AF is defined by \mbox{$\Can:\AFl_{\fin} \to \left(2^{\L_\AFl}\right)_{\fin}$} where $(A,R)$ is represented by the $\L_\AFl$-theory \mbox{$\{(\{a\},\emptyset)\mid a\in A\} \cup \{(\{a,b\},\{(a,b)\})\mid (a,b)\in R\}$}.
Observe that for any AF $\AH$, we find \mbox{$\AFl\!\left(\Can(\AH)\right) = \AH$}.
Moreover, regarding Examples~\ref{ex:six} and~\ref{ex:repr} we have $\Can(\AG) = S\neq T$. 

Note that the assumption of finiteness of AFs can be reflected by considering the finite-theory versions of $\rho$-logics. 
Before applying our canonical construction presented in Theorem~\ref{thm:characterization:finite-theory} we have to ensure that a constructed $\rho$-logic correctly reflects AFs under semantics~$\rho$.
We start with two simple properties showing that the concrete representation (of an AF via a theory) is not ``seen'' by set union as well as semantics $\sigma_{\rho}$.

\begin{proposition}
  \label{obs:afs}
  Let $(\L_\AFl,\I,\sem_{\rho})$ be a $\rho$-logic and consider any theories \mbox{$S,T\subseteq\L_\AFl$}.
  \begin{enumerate}
  \item $\AFl(S\cup T) = \AFl(S) \dcup \AFl(T)$ and
  \item $\sigma_{\rho}\!\left(S\cup T\right) = \rho\!\left(\AFl(S) \dcup \AFl(T)\right)$.
  \end{enumerate}
\end{proposition}
\begin{proof}
\begin{enumerate}
  \item Both statements can be easily seen. Consider the following equations.
	\begin{align*}
      &\pheq \AFl(S\cup T) &\\
			&= \dbigcup_{u\in S\cup T} u &(\text{Definition } \AFl)\\
			&= \dbigcup_{s\in S} s \dcup \dbigcup_{t\in T}t &(\text{associativity } \dbigcup)\\
			&= \AFl(S) \dcup \AFl(T) &(\text{Definition } \AFl)
			\end{align*}
  \item \begin{align*}
      &\pheq \sigma_{\rho}\!\left(S\cup T\right)& \\
			&= \rho\!\left(\dbigcup_{u\in S\cup T} u\right) &(\text{Definition } \sigma_{\rho})\\
			&= \rho\!\left(\AFl(S\cup T)\right) &(\text{Definition } \AFl)\\
			&= \rho\!\left(\AFl(S) \dcup \AFl(T)\right)&(\text{Item } 1) 
			\end{align*} 
	
\end{enumerate}
\end{proof}
 
The following theorem shows that two $\L_\AFl$-theories $S$ and $T$ are strongly equivalent under $\sigma_{\rho}$ if and only if the AFs $\AFl(S)$ and $\AFl(T)$ are strongly equivalent under $\rho$ (denoted by $\AFl(S) \equiv^{\rho}_s \AFl(T)$).
We mention that the theorem does not require finiteness and is thus valid for arbitrary cardinalities of theories as well as AFs.


\begin{theorem}
  \label{the:afs}
  Let $(\L_\AFl,\I,\sem_{\rho})$ be a $\rho$-logic.
  For \mbox{$S,T\subseteq\L_\AFl$} we have 
  $$\AFl(S) \equiv^{\rho}_s \AFl(T)\ \iff [S]^{\sigma_{\rho}}_s = [T]^{\sigma_{\rho}}_s.$$
\begin{proof} 
  \begin{description}
  \item[\normalfont $\implies$:] 
    Let \mbox{$\AFl(S) \equiv^{\rho}_s \AFl(T)$} and \mbox{$V\in 2^{\L_\AFl}$}.
    We have to show that both theories are strongly equivalent, i.e.\ \mbox{$\sigma_{\rho}(S \cup V) = \sigma_{\rho}(T \cup V)$}.
   
    \begin{align*}
      &\pheq \sigma_{\rho}(S \cup V) &\\
			&=\rho(\AFl(S)\dcup\AFl(V)) &(Observation~\ref{obs:afs})\\
      &=\rho(\AFl(T)\dcup\AFl(V)) &(\text{assumption } \AFl(S) \equiv^{\rho}_s \AFl(T))\\
      &=\sigma_{\rho}(T \cup V) &(Observation~\ref{obs:afs})\\
    \end{align*}
  \item[\normalfont $\impliedby$:] 
    Let \mbox{$[S]^{\sigma_{\rho}}_s = [T]^{\sigma_{\rho}}_s$} and $\AH$ be an AF.
    We have to show that the corresponding AFs are strongly equivalent, i.e.\ \mbox{$\rho(\AFl(S) \dcup \AH) = \rho(\AFl(T) \dcup \AH)$}.
    \begin{align*}
      &\pheq \rho(\AFl(S) \dcup \AH) &\\
			&= \rho(\AFl(S) \dcup \AFl\!\left(\Can(\AH)\right)) &(\AFl\!\left(\Can(\AH)\right) = \AH)\\
      &= \sigma_{\rho}(S \cup \Can(\AH)) &(Observation~\ref{obs:afs})\\
      &= \sigma_{\rho}(T \cup \Can(\AH)) &(\text{assumption } [S]^{\sigma_{\rho}}_s = [T]^{\sigma_{\rho}}_s )\\
			&= \rho(\AFl(T) \dcup \AFl\!\left(\Can(\AH)\right)) &(Observation~\ref{obs:afs}) \\
      &= \rho(\AFl(T) \dcup \AH) &(\AFl\!\left(\Can(\AH)\right) = \AH)
    \end{align*}
  \end{description}
\end{proof}
\end{theorem}

Due to Theorem~\ref{thm:characterization:finite-theory} we are able to present finite-theory characterization logics for any $\rho$-logic. 

\begin{corollary}
  \label{def:rhochar}
  Let $(\L_\AFl,\I,\sem_{\rho})$ be a $\rho$-logic.
  The following logic $(\L_\AFl, \kappa)$ is a finite-theory characterization logic of $(\L_\AFl,\I,\sem_{\rho})$\/: 
  \begin{gather*}
    \hfill
    \kappa : \LAFfin \to 2^{2^{\L_\AFl}}
    \hfill,\quad\hfill
    T \mapsto \mathop{\bigcup_{S\in \LAFfin,}}_{T \subseteq S}[S]^{{(\sem_{\rho})}_{\fin}}_s
    \hfill
  \end{gather*}
\end{corollary}

So far, so good, but how can we interpret these finite-theory characterization logics in terms of argumentation theory?
In other words, what is the corresponding characterization semantics on the level of pure AFs (instead of theories associated with AFs)?
We extend the function $\AFl$ to sets of theories as usual, namely \mbox{$\AFl: 2^{2^{\L_\AFl}} \to 2^{\AFl_{\fin}}$} where \mbox{$\m{T}\mapsto \{\AFl(T)\mid T\in\m{T}\}$}.\footnote{We do not introduce a new symbol for the new function. Which function is meant will be clear from the context.}
Consider the following definition.
We will see that all crucial properties of $\kappa$ transfer to $\rho'$, that is, $\rho'$ satisfies finite intersection and furthermore, it characterizes strong equivalence under $\rho$.

\begin{definition}
  \label{def:rhoprime}
  Given an argumentation semantics \mbox{$\rho:\AFl_{\fin} \to 2^{2^\m{U}}$}. We define 
  \mbox{$\rho': \AFl_{\fin} \to 2^{{\AFl_{\fin}}}$}
  with
  \mbox{$\AF \mapsto \AFl\!\left(\kappa(\Can(\AF))\right)$}.
\end{definition}

\begin{proposition} 
  \label{thm:rhoprime}
  For any argumentation semantics $\rho$ and semantics $\rho'$ as defined above we have\/:
  \begin{enumerate}
  \item \mbox{$\forall \AF,\AG\in\AFl_{\fin}: \rho'(\AF)=\rho'(\AG) \iff \AF\equiv^{\rho}_s\AG$}; \hfill (finite-theory characterization)
  \item \mbox{$\forall \AF,\AG\in\AFl_{\fin}: \rho' (\AF \dcup \AG) = \rho'(\AF) \cap \rho'(\AG)$}. \hfill (finite-theory intersection)
  \end{enumerate}
\begin{proof}
  \begin{description}
  \item[\normalfont finite-theory characterization:]
	
    \begin{description}
	
    \item[\normalfont $\!\impliedby\!$:]
      Let \mbox{$\AF\equiv^{\rho}_s\AG$}.
      Hence, $\AFl\!\left(\Can(\AF)\right)\equiv^{\rho}_s\AFl\!\left(\Can(\AG)\right)$. Consequently, \mbox{$[\Can(\AF)]^{\sigma_{\rho}}_s = [\Can(\AG)]^{\sigma_{\rho}}_s$} (Theorem~\ref{the:afs}).
      Since $(\L_\AFl, \kappa)$ is a finite-theory characterization logic of $(\L_\AFl,\I,\sem_{\rho})$ (Definition~\ref{def:rhochar}) we deduce that \mbox{$\kappa(\Can(\AF)) = \kappa(\Can(\AG))$}.
      Obviously, \mbox{$\AFl\!\left(\kappa(\Can(\AF))\right) = \AFl\!\left(\kappa(\Can(\AG))\right)$} which means \mbox{$\rho'(\AF)=\rho'(\AG)$} (Definition~\ref{def:rhoprime}).\pagebreak
    \item[\normalfont $\!\implies\!$:]
      We prove the contrapositive.
      Hence, let \mbox{$\AF\not\equiv^{\rho}_s\AG$}.
      This means, $\AFl\!\left(\Can(\AF)\right)\not\equiv^{\rho}_s\AFl\!\left(\Can(\AG)\right)$ and we obtain \mbox{$[\Can(\AF)]^{\sigma_{\rho}}_s \neq [\Can(\AG)]^{\sigma_{\rho}}_s$} (Theorem~\ref{the:afs}).
      Consequently, \mbox{$\kappa(\Can(\AF)) \neq \kappa(\Can(\AG))$} since $\kappa$ characterizes strong equivalence under $\sigma_{\rho}$ (Definition~\ref{def:rhochar}).
      Since equivalence classes are disjoint we deduce the existence of a theory $U$, such that (without loss of generality) \mbox{$[U]^{\sigma_{\rho}}_s \subseteq \kappa(\Can(\AF)) \sm \kappa(\Can(\AG))$}.
      Consequently, \mbox{$\AFl(U) \in \AFl(\kappa(\Can(\AF))) \sm \AFl(\kappa(\Can(\AG)))$} since for all other representations of $U'$, s.t. $\AFl(U) = \AFl(U')$, we have \mbox{$U'\in [U]^{\sigma_{\rho}}_s$} (Theorem~\ref{the:afs}).
      Hence, \mbox{$\AFl\!\left(\kappa(\Can(\AF))\right) \neq \AFl\!\left(\kappa(\Can(\AG))\right)$} which means \mbox{$\rho'(\AF)\neq\rho'(\AG)$} (Definition~\ref{def:rhoprime}).
    \end{description}
  \item[\normalfont finite-theory intersection:]
      \begin{align*}
        &\pheq \rho' (\AF \dcup \AG) & \\
        &={\AFl\!\left(\kappa(\Can(\AF \dcup \AG))\right)} &(\text{Definition~\ref{def:rhoprime}})\\
        &={\AFl\!\left(\kappa(\Can(\AF) \cup \Can(\AG))\right)} &(\Can(\AF \dcup \AG) = \Can(\AF) \cup \Can(\AG))\\
        &={\AFl\!\left(\kappa(\Can(\AF)) \cap \kappa(\Can(\AG))\right)}&(\text{intersection } \kappa)\\
        &={\AFl\!\left(\kappa(\Can(\AF))\right) \cap \AFl\!\left(\kappa(\Can(\AG))\right)}&(\text{can be seen})\\
        &=\rho'(\AF) \cap \rho'(\AG) &(\text{Definition~\ref{def:rhoprime}})
      \end{align*}
  \end{description}
\end{proof}
\end{proposition}

Finally, we present an equivalent definition of $\rho'$ that does not rely on $\rho$-logics.
This means the evaluation of $\rho'$ can be done purely on the level of AFs.

\begin{proposition}
  \label{thm:rho-afs}
  Let \mbox{$\rho:\AFl_{\fin} \to 2^{\AFl_{\fin}}$} be a semantics and $\rho'$ as in Definition~\ref{def:rhoprime}.
  For any \mbox{$\AF\in\AFl_{\fin}$} we have\/:
  $$\rho'(\AF) 
  = \mathop{\bigcup_{\AG\in \AFl_{\fin},}}_{\AF \dsub \AG} \{\AH\mid \AH \equiv^{\rho}_s \AG\}$$
\begin{proof}
  \begin{align*}
    \rho'(\AF) &= {\AFl\!\left(\kappa(\Can(\AF))\right)} & \text{(Definition~\ref{def:rhoprime})} \\
    &={\AFl\left(\mathop{\bigcup_{S\in \LAFfin,}}_{\Can(\AF) \subseteq S}[S]^{{(\sem_{\rho})}_{\fin}}_s\right)}&(\text{Definition~\ref{def:rhochar}})\\ 
    &={\AFl\left(\mathop{\bigcup_{S\in \LAFfin,}}_{\Can(\AF) \subseteq \Can(\AFl(S))}[S]^{{(\sem_{\rho})}_{\fin}}_s\right)}&(S\subseteq \Can(\AFl(S)), [S]^{{(\sem_{\rho})}_{\fin}}_s = [\Can(\AFl(S))]^{{(\sem_{\rho})}_{\fin}}_s)\\
    &={\mathop{\bigcup_{S\in \LAFfin,}}_{\Can(\AF) \subseteq \Can(\AFl(S))}\AFl\left([S]^{{(\sem_{\rho})}_{\fin}}_s\right)}&(\text{Definition }\AFl(\cdot))\\
    &={\mathop{\bigcup_{S\in \LAFfin,}}_{\AF \dsub \AFl(S)}\AFl\left([S]^{{(\sem_{\rho})}_{\fin}}_s\right)}&(\Can(\AF) \subseteq \Can(\AFl(S)) \iff \AF \dsub \AFl(S))\\
		 &={\mathop{\bigcup_{S\in \LAFfin,}}_{\AF \dsub \AFl(S)}\left\{\AFl(T)\mid T\in [S]^{{(\sem_{\rho})}_{\fin}}_s\right\}}&(\text{Definition }\AFl(\cdot))\\
		 &={\mathop{\bigcup_{S\in \LAFfin,}}_{\AF \dsub \AFl(S)}\left\{\AFl(T)\mid [T]^{{(\sem_{\rho})}_{\fin}}_s = [S]^{{(\sem_{\rho})}_{\fin}}_s\right\}}&(\text{equivalence relation})\\
		&={\mathop{\bigcup_{\AG\in \AFl_{\text{\fin}},}}_{\AF \dsub \AG}\left\{\AH\mid \AH \equiv^{\rho}_s \AG\right\}}&(\AFl(S) = \AG, \AFl(T) = \AH, \text{Theorem~\ref{the:afs}})
  \end{align*}
  \[\mbox{}\]
\end{proof}
\end{proposition}

Recently, so-called \textit{Dung-logics} were introduced to be able to perform AGM-style revision for Dung's abstract argumentation frameworks \citep{BauB15}. These Dung-logics are very similar yet different from the characterization logics presented in Proposition~\ref{thm:rho-afs}. The main difference is that theories in Dung-logics are sets of AFs in contrast to the newly presented characterization logic where theories correspond to single AFs.
We mention that Dung-logics possess the intersection property (Definition~\ref{def:dunglogic}) and furthermore, two AFs $\AF$ and $\AG$ are strongly equivalent with respect to an argumentation semantics $\rho$ if and only if the singletons of $\AF$ and $\AG$ are ordinarily equivalent with respect to the semantics defined by Theorem~3 of~\cite{BauB15}.

We start with the formal definition of a Dung-logic in case of the most prominent argumentation semantics, namely stable semantics.

\begin{definition} \label{def:dunglogic} The Dung-logic in case of stable semantics is a pair $(\AFl_{\fin},\delta)$ with 
$$\delta:2^{\AFl_{\fin}} \to 2^{\AFl_{\fin}} \quad \m{U}\mapsto \bigcap_{\AF\in\m{U}} \Mod^{\k(\stb)}(\AF)$$
whereas \mbox{$\Mod^{\k(\stb)}(\AF) = \left\{\AG\in\AFl_{\fin}\mid \AF^{\k(\stb)}\dsub\AG^{\k(\stb)}\right\}$.}
\end{definition}
In order to see the similarity we consider the above definition for singletons of AFs.

\begin{observation} Let \mbox{$\stb:\AFl_{\fin} \to 2^{\AFl_{\fin}}$} be stable semantics and $\delta$ as in Definition~\ref{def:dunglogic}.
  For any \mbox{$\AF\in\AFl_{\fin}$} we have\/:
  $$\delta(\{\AF\}) 
  = \mathop{\bigcup_{\AG\in \AFl_{\fin},}}_{\AF^{\k(\stb)} \dsub \AG^{\k(\stb)}} \{\AH\mid \AH \equiv^{\stb}_s \AG\}$$
\end{observation}

The only difference in comparison to $\rho'$ is that for $\delta$, we include all equivalence classes of AFs $G$ whose \emph{kernels} are in superset relation with the kernel of $F$, instead of having the superset relation on the AFs themselves. A further analysis will be part of future work.

\subsection{Normal Logic Programs}
\label{sec:normal-logic-programs}


For normal logic programs under the stable model semantics as presented in Example~\ref{exm:known-formalisms}, applying Theorem~\ref{thm:characterization:finite-theory} yields\/:
\begin{corollary}
  \label{thm:logic-programs:characterization}
  For the finite-theory version of logic $(\L_\LP,\sem_\SM)$ of normal logic programs under stable model semantics, a finite-theory characterization logic is given by
  \begin{gather*}
    \hfill
    \sem_\SM' : \LLPfin \to 2^{\LLPfin}
    \hfill,\qquad\hfill
    T \mapsto \mathop{\bigcup_{S\in\LLPfin,}}_{T \subseteq S}[S]^{\sem_\SM}_s
    \hfill
  \end{gather*}
\end{corollary}

An existing, well-known characterizing semantics is given by SE-models.
\begin{definition}
  \label{def:logic-programs:se-models}
  Let $P$ be a normal logic program over $\A$ and \mbox{$X\subseteq Y\subseteq A$}.
  Define semantics \linebreak \mbox{$\sem_\SE:\LLPfin\to 2^{\A\times\A}$} by
  \begin{gather*}
    T\mapsto \set{ (X,Y) \guard X\subseteq Y \text{ and } X,Y\in\sem_\mod\!\left(T^Y\right) }
  \end{gather*}%
  where 
    \begin{align*}
      T^Y &= \set{ a_0\gets a_1,\ldots,a_m \guard a_0\gets a_1,\ldots,a_m,\naf a_{m+1},\ldots,\naf a_n\in T, a_{m+1},\ldots,a_n\notin Y } \\
      \sem_\mod(T) &= \{ M\subseteq\A \mid \forall a_0\gets a_1,\ldots,a_m,\naf a_{m+1},\ldots,\naf a_n\in T: \\
      &\qquad\qquad\qquad ( a_1,\ldots,a_m\in M \land a_{m+1},\ldots,a_n\notin M) \implies a_0\in M \}
    \end{align*}
\end{definition}

SE-models characterize strong equivalence of stable models~\cite[Theorem~1]{DBLP:conf/lpnmr/Turner01};
SE-model semantics also has the intersection property \cite[Lemma~3]{DBLP:journals/amai/Truszczynski06}.

\begin{proposition}
  \label{thm:logic-programs:se-models}
  $(\L_\LP,\sem_\SE)$ is a characterization logic for $(\L_\LP,\sem_\SM)$.
\end{proposition}

Since finite-theory characterization logics are unique up to isomorphism (Theorem~\ref{thm:fin-unique-characterization}), there is a one-to-one-correspondence between the model sets given by Corollary~\ref{thm:logic-programs:characterization} (as well as any other model set of a certain characterization logic) and sets of SE-models.
More precisely, for any two logic programs \mbox{$T_1,T_2\in\L_\LP$}, we find \mbox{$\sem_\SM'(T_1)\subseteq\sem_\SM'(T_2)$} if and only if \mbox{$\sem_\SE(T_1)\subseteq\sem_\SE(T_2)$}.
However, please note that the set of SE-models of a finite logic program is finite, while $\sem_\SM'$ maps logic programs to infinite model sets in general.
So in the concrete case of logic programs, SE-models are much easier to work with.

\section{Discussion}
\label{sec:discussion}

We presented a general framework for analyzing strong equivalence of knowledge representation formalisms.
The framework abstracts away from all language specifics other than that knowledge bases be expressible as sets of atomic language elements.
For two classes of formalisms, covered and finite-theory logics, we showed that they always possess a classical characterization logic.
We called characterization logics classical because they have the intersection property (that is, the semantics of theories can always be obtained by considering the semantics of its members independently).
We called characterization logics characterizing because their standard equivalence coincides with strong equivalence in the characterized formalism.
As an application of our results, we obtained a first characterization logic for abstract argumentation where single AFs are interpreted as theories. This new logic complements the already existing Dung-logics which consider single AFs as building blocks and hence, theories as sets of AFs~\citep{BauB15}.

Most previous work on characterizing strong equivalence in KR that we know of focused on specific formalisms or on a handful of related formalisms, such as work on strong equivalence in logic programs under stable models \citep{DBLP:journals/tocl/LifschitzPV01,Turner01} and supported models \citep{DBLP:conf/aaai/TruszczynskiW08}, that also give rise to similar developments in default logic~\citep{Turner01,DBLP:conf/aaai/Truszczynski07} and autoepistemic logic~\citep{DBLP:conf/aaai/TruszczynskiW08}.
By considering and exploiting formalism specifics, more fine-grained views on classical, strong and intermediate equivalence notions are possible~\citep{DBLP:conf/iclp/EiterF03,DBLP:conf/aaai/Truszczynski07,DBLP:journals/tplp/Woltran08}.
Such notions are at present not ``visible'' in our setting, but could be incorporated by restricting the set of theories that are allowed for expansion.

We have chosen to consider as ``classical'' all logics whose model function possesses the intersection property.
Other characterizations might be possible when choosing consequence functions instead of model functions as starting point.
For those, we considered closure operators as a special class (implying, for example, cumulativity).
Other properties for future consideration come to mind -- for example compactness, which is independent of the closure property and relevant to proof theory.



Our approach could be further generalized by abstracting away even from knowledge bases as sets and knowledge base expansion as set union.
We could assume a language as equipped with an expansion operator $\oplus$ under which the language is closed and then derive our results completely algebraically.
This would enable us to treat, for example, abstract dialectical frameworks~\citep{brewka-woltran10adfs,brewka13adfs}, a quite recent non-classical KR formalism encompassing both argumentation frameworks and logic programs~\citep{strass13approximating}, 
for which strong equivalence has not been studied yet.

\chapter{Existence and Uniqueness of Argumentation Semantics} \label{cha:ex}

Argumentation has become
one of the major fields within AI over the last two decades~
\cite{RahwanIS09,Bench-CaponD07}.
In particular, 
Dung's argumentation frameworks (AFs) \cite{Dung95} are widely used
and act as integral concepts in 
several advanced argumentation formalisms. 
They focus entirely on conflict resolution among arguments, treating the latter as abstract items without
logical structure. Hence, the only information available in AFs is the so-called attack relation that
determines whether an argument is in a certain conflict with another one.
As already outlined by Dung, AFs provide a 
formally simple basis to capture the essence
of different nonmonotonic formalisms. 

The subsequent sections and chapters are devoted to argumentation semantics which
play the flagship role in Dung's abstract argumentation theory. Almost
all of them are motivated by an easily understandable intuition of what
should be acceptable in the light of conflicts. However, although these
intuitions equip us with short and comprehensible formal definitions it
turned out that their intrinsic properties such as \textit{existence and uniqueness}, \textit{expressibility}, \textit{replaceability} and \textit{verifiability} are not that easily accessible. We will consider all mentioned properties w.r.t.\ almost all semantics available in the literature. In doing so we include two main axes: namely first,
the distinction between extension-based and labelling-based versions and
secondly, the distinction of different kind of argumentation frameworks
such as finite or unrestricted ones.

\section{Abstract Argumentation Theory} \label{sec:Dung}

Phan Minh Dung introduced argumentation frameworks (AFs) as directed graphs \cite[Definition~2]{Dung95}. This means, an AF $\F = (A,R)$ is simply a pair consisting of a set $A$, usually called \textit{set of arguments}, and a binary relation $R\subseteq A\times A$, so-called \textit{attack relation}. In order to judge the truth of sentences like ``For any AF $\F$ we have $\ldots$'' or ``There is no AF $\F$, s.t.\ $\ldots$'' we have to introduce a reference set $\m{U}$, so-called \textit{universe of arguments} and to require, that all possible AFs possess arguments of this set. This means, sentences as mentioned above always refer to AFs $\F = (A,R)$ with $A\subseteq\m{U}$. In the following we use $\m{F}_{\m{U}}$ as an abbreviation for the set of all AFs induced by \m{U}. In order to be able to consider AFs possessing an arbitrary finite number of arguments or even infinitely many we have to request that $\card{\m{U}} \geq \card{\N}$. No further conditions are imposed. For the rest of the thesis we assume that such a universe of arguments is given, i.e.\ we use the set \m{U} as an arbitrary but fixed parameter. For this reason, we often use $\m{F}$ instead of $\m{F}_{\m{U}}$.

\subsection{Important Syntactical Classes of AFs}
AFs can be classified by means of syntactical properties. In the following we list some features which will be frequently used throughout the thesis. Most of them are self-explanatory and/or already well-known from graph theory \cite{Die06}. Most of the research in abstract argumentation typically pertains to finite AFs.

\begin{definition} \label{def:classes}   
 An AF $\F 	= (A,R)$ is called 
 \begin{enumerate}
   \item \textit{empty} if $A = \emptyset$,
	 \item \textit{finite} if $\card{A}\in\N$,
	 \item \textit{finitary} if for any $a\in A$, $\card{\{b\in A \mid (b,a)\in R\}}\in\N$,
	 \item \textit{self-loop-free} if $R \cap id_A = \emptyset$,
	 \item \textit{odd-cycle-free} if there is no odd natural $n$, s.t.\ $a_0,...,a_n\in A$, $a_0=a_n$ and $(a_0,a_1),...,(a_{n-1},a_n)\in R$,
	 \item \textit{symmetric} if $R = R^{-1}$,
	 \item \textit{arbitrary} if $\F\in\m{F}$.
 \end{enumerate}
\end{definition}

\begin{example} The following AF is not empty, finite, finitary, self-loop-free, not odd-cycle-free, not symmetric and arbitrary.
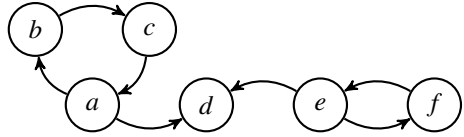
\begin{figure}[H] 
\centering
\begin{tikzpicture}

		\node (C) at (0.25,1.0) [circle, thick, draw, minimum size=0.7cm, ]{$c$};
    \node (B) at (-1.25,1.0) [circle, thick, draw, minimum size=0.7cm]{$b$};
    \node (A) at (-0.5,0.0) [circle, thick, draw, minimum size=0.7cm]{$a$};
    \node (D) at (1.0,0.0) [circle, thick, draw, minimum size=0.7cm]{$d$};
    \node (E) at (2.5,0.0) [circle, thick, draw, minimum size=0.7cm]{$e$};
    \node (F) at (4.0,0.0) [circle, thick, draw, minimum size=0.7cm]{$f$};

\draw[->,thick] (A) to [thick,bend left] (B);
\draw[->,thick] (B) to [thick,bend left] (C);
\draw[->,thick] (C) to [thick,bend left] (A);
\draw[->,thick] (A) to [thick,bend right] (D);

\draw[->,thick] (E) to [thick,bend right] (D);
\draw[->,thick] (E) to [thick,bend right] (F);
\draw[->,thick] (F) to [thick,bend right] (E);
\end{tikzpicture}

\label{fig:first}
\end{figure}
\end{example}

\subsection{Argumentation Semantics - The Flagship of Dung's Abstract Argumentation Theory} \label{sec:flagship}

In order to formalize the notions of \textit{existence} and \textit{uniqueness} in the context of abstract argumentation theory we have to clarify what we precisely mean by a \textit{semantics}. In the literature 
two main approaches to argumentation semantics can be found, namely so-called \textit{extension-based} and \textit{labelling-based} versions.  The main difference is that extension-based versions return a set of sets of arguments (so-called \textit{extensions}) for any given AF in contrast to a set of sets of $n$-tupels (so-called \textit{labellings}) as in case of labelling-based approaches. However, from a mathematical point of view both kinds of semantics are instances of Definition~\ref{def:semantics}. More precisely, extension-based versions are covered by $n = 1$ and labelling-based approaches can be obtained by setting $n \geq 2$. We use $\left(2^\m{U}\right)^n$ to denote the n-ary cartesian power of $2^\m{U}$, i.e.\ $\left(2^\m{U}\right)^n = \underbrace{2^\m{U}\times \cdots \times 2^\m{U}}_{n-\text{times}}$. 

\begin{definition} \label{def:semantics}
A semantics is a function $\sigma: \m{F}\rightarrow 2^{\left(2^\m{U}\right)^n}$ for some natural $n\in\N$, s.t.\linebreak $\F = (A,R) \mapsto \sigma(\F) \subseteq \left(2^A\right)^n$.
\end{definition}

In this chapter we are interested in definedness statuses w.r.t.\ finite, finitary and arbitrary frameworks. Besides conflict-free and admissible sets (abbreviated by $\cf$ and $\adm$) we consider a large number of mature semantics, namely naive, stage, stable, semi-stable, complete, preferred, grounded, ideal, eager semantics as well as the more exotic cf2 and stage2 semantics (abbreviated by
    $\nav,\stg,\stb,\semi,\com,\prf,\grd,\id,\eag, \cfzwei$ and $\stgzwei$
    respectively). In the following we introduce the extension-based versions of these semantics (indicated by $\Ext_{\sigma}$). Any considered semantics possesses a 3-valued labelling-based version (denoted as $\Lab_{\sigma}$). It is important to note that for all considered semantics we do not observe any differences between the definedness statuses of their labelling-based and extension-based versions. For the mature semantics this is due the fact that there is a one-to-one correspondence between $\sigma$-extensions and $\sigma$-labellings implying that $\card{\Ext_{\sigma}(\F)} = \card{\Lab_{\sigma}(\F)}$ for any AF $\F$ (for more details confer Section~\ref{sec:baspro}).

Before presenting the definitions we have to introduce some notational conventions. Given an AF $\F = (A,R)$ and a set $E\subseteq A$. We use $\E^+_{\F}$ or simply, $\E^+$ for \mbox{$\{b\mid (a,b)\in R, a\in \E\}$}. Moreover, $\E^{\oplus}_{\F}$ or simply, $\E^\oplus$ is called the \textit{range} of $E$ and stands for $\E^+\!\cup\E$. We say $a$ \textit{attacks} $b$ (in~$\F$) if $(a,b)\in R$. An argument $a$ is \textit{defended} by $E$ (in~$\AF$) if for each $b\in A$ with $(b,a)\in R$, $b$ is attacked by some $c\in E$. Finally, $\Gamma_\F: 2^A \to 2^A$ with $I\mapsto \{a\in A\mid a \text{ is defended by } I\}$ denotes the so-called \textit{characteristic function} (of $\F$) \cite{Dung95}. 
	
\begin{definition} \label{def:extsem}
     Let $\AF = (A,R)$ be an AF and $\E\subseteq A$.
 \begin{enumerate}
	\item \hspace{-2.5mm} $\E\in\Ext_{\cf}(\F)$ iff for no $a,b\in E$, $(a,b)\in R$,
	\item \hspace{-2.5mm} $\E\in\Ext_{\nav}(\AF)$ iff $\E\in\Ext_{\cf}(\AF)$ and for no $\I\in\Ext_{\cf}(\AF)$, $\E\subset\I$,
	\item \hspace{-2.5mm} $\E\in\Ext_{\stg}(\AF)$ iff $\E\in\Ext_{\cf}(\AF)$ and there is
	  no $\I\in\Ext_{\cf}(\AF)$, s.t.\ \mbox{$\E^\oplus\subset\I^\oplus$},
	\item \hspace{-2.5mm} $\E\in\Ext_{\stb}(\AF)$ iff $\E\in\Ext_{\cf}(\AF)$ and
	  $\E^\oplus=A$,
	\item \hspace{-2.5mm} $\E\in\Ext_{\adm}(\AF)$ iff $\E\in\Ext_{\cf}(\AF)$ and $\E$ defends all its \mbox{elements,}
	\item \hspace{-2.5mm} $\E\in\Ext_{\semi}(\AF)$ iff $\E\in\Ext_{\adm}(\AF)$ and
	  there is no $\I\in\Ext_{\adm}(\AF)$, s.t.\ \mbox{$\E^\oplus\subset\I^\oplus$},
	\item \hspace{-2.5mm} $\E\in\Ext_{\com}(\AF)$ iff $\E\in\Ext_{\adm}(\AF)$ and for any
	  $a\in A$ defended by $\E$~in~$\AF$, $a\in \E$,
	\item \hspace{-2.5mm} \mbox{$\E\in\Ext_{\prf}(\AF)$ iff $\E\in\Ext_{\adm}(\AF)$ and
	  for no $\I\in\Ext_{\com}(\AF)$, $\E\subset\I$,} 
		\item \hspace{-2.5mm} $\E\in\Ext_{\grd}(\AF)$ iff $E$ is the $\subseteq$-least fixpoint of $\Gamma_\F$,
	\item \hspace{-2.5mm} $\E\in\Ext_{\id}(\AF)$ iff $\E\in\Ext_{\adm}(\AF)$,
	  $\E\subseteq\bigcap\Ext_{\prf}(\AF)$ and there is no $\I\in\Ext_{\com}(\AF)$ satisfying
	  $\E\subset\I\subseteq\bigcap\Ext_{\prf}(\AF)$,
		\item \hspace{-2.5mm} $\E\in\Ext_{\eag}(\AF)$ iff $\E\in\Ext_{\adm}(\AF)$,
	  $\E\subseteq\bigcap\Ext_{\semi}(\AF)$ and there is no $\I\in\Ext_{\com}(\AF)$ satisfying
	  $\E\subset\I\subseteq\bigcap\Ext_{\semi}(\AF)$.
    \end{enumerate}
    \end{definition}
		
Finally, we introduce the recursively defined cf2 and stage2 semantics \cite{BarGG05,DvoG12}. 

\begin{definition} Let $\AF = (A,R)$ be an AF and $\E\subseteq A$. \label{def:recursive}
\begin{enumerate}
	\item $\E\in\Ext_{\cf2}(\AF)$ iff
	\begin{itemize}
		\item $\E\in\Ext_\nav(\F)$ if $\card{SCCs_{\F} = 1}$ and
		\item $\forall S\in SCCs_{\F} (E\cap S) \in \Ext_{\cf2}\left(\F|_{UP_\F(S,E)}\right)$,
	\end{itemize}	
		\item $\E\in\Ext_{\stg2}(\AF)$ iff
	\begin{itemize}
		\item $\E\in\Ext_\stg(\F)$ if $\card{SCCs_{\F} = 1}$ and
		\item $\forall S\in SCCs_{\F} (E\cap S) \in \Ext_{\stg2}\left(\F|_{UP_\F(S,E)}\right)$.
		\end{itemize}
	\end{enumerate}
Here $SCCs_{\F}$ denotes the set of all strongly connected components of $\F$, and for any $E,S\subseteq A$, 
$UP_\F(S,E) = \{a\in S\mid \nexists b\in E\sm S: (b,a)\in R\}$.   
\end{definition}

The following proposition summarizes well-known subset relations between the considered semantics. For two semantics $\sigma$, $\tau$ and a certain set of AFs~\m{C} we use $\sigma\subseteq_{\m{C}}\tau$ as a shorthand for $\sigma(\F)\subseteq\tau(\F)$ for any AF $\F\in\m{C}$. The presented relations hold for both extension-based as well as labelling-based versions of the considered semantics. In the interest of readability we present the relations graphically.

%
 %

%
%

    \begin{proposition}
      \label{Pro:semrel}
      For semantics $\sigma$ and $\tau$, $\sigma\subseteq_{\m{F}}\tau$ iff there is a
      path of solid arrows from $\sigma$~to~$\tau$ in Figure~\ref{fig:semrel}. A dotted arrow indicates that the corresponding subset relation is guaranteed for finite frameworks only.
			
      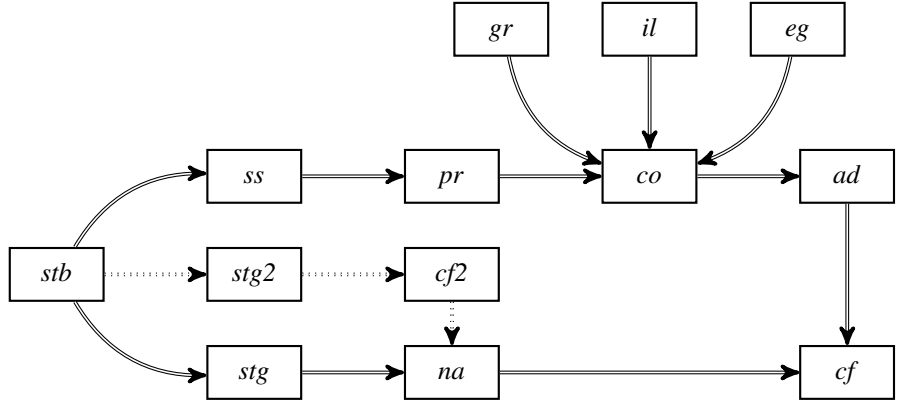
\begin{figure}[H]
	\centering
	\tikzstyle{sembox}=[rectangle,thick,draw,text width=1cm,text centered]
	
	\begin{tikzpicture}[scale=1.3]
	  \node (stb) at (0,0) [sembox]{\standardbox{\stb}};
	  \node (semi) at (2,1) [sembox]{\standardbox{\semi}};
	  \node (stage) at (2,-1) [sembox]{\standardbox{\stg}};
		 \node (stage2) at (2,0) [sembox]{\standardbox{\stgzwei}};
	  \node (pref) at (4,1) [sembox]{\standardbox{\prf}};
	  \node (naive) at (4,-1) [sembox]{\standardbox{\nav}};
		\node (cf2) at (4,0) [sembox]{\standardbox{\cfzwei}};
	  \node (comp) at (6,1) [sembox]{\standardbox{\com}};
	  \node (adm) at (8,1) [sembox]{\standardbox{\adm}};
	  \node (cf) at (8,-1) [sembox]{\standardbox{\cf}};
	  \node (grd) at (4.5,2.5) [sembox]{\standardbox{\grd}};
	  \node (ideal) at (6,2.5) [sembox]{\standardbox{\id}};
	  \node (eager) at (7.5,2.5) [sembox]{\standardbox{\eag}};
	  \draw[->]
	    (stb) edge[bend left,double] (semi)
	    (stb) edge[bend right,double] (stage)
	    (semi) edge[double] (pref)
			(stb) edge[double, dotted] (stage2)
			(stage2) edge[double, dotted] (cf2)
			(cf2) edge[double, dotted] (naive)
	    (stage) edge[double] (naive)
	    (pref) edge[double] (comp)
	    (comp) edge[double] (adm)
	    (adm) edge[double] (cf)
	    (naive) edge[double] (cf)
	    (grd) edge[double, bend right] (comp)
	    (ideal) edge[double] (comp)
	    (eager) edge[double, bend left] (comp)
	    ;
	\end{tikzpicture}
	\caption{Subset Relations between Semantics}
	\label{fig:semrel}
      \end{figure}
    \end{proposition}

Detailed proofs can be found in \cite[Proposition 2.7]{diss} as well as \cite[Section 3.1]{GagD14}. Note that the shorthand $\sigma\subseteq_{\m{C}}\tau$ requires that both semantics are total functions on \m{C} since a framework to which one of these semantics is undefined renders the subset shorthand undefined itself. The following simple example shows that Definition~\ref{def:recursive} does not always provide a definite answer on whether a certain candidate set is an $\cfzwei$-extension or $\stgzwei$-extension, respectively. This is due to the fact that the defined recursion does not terminate necessarily in case of non-finite AFs.\footnote{We mention that the inventors of both semantics considered finite AFs only \cite{BarGG05,DvoG12}. In case of finite AFs any recursion will terminate no matter which candidate set is considered.} Consequently, $\stgzwei$ and $\cfzwei$ are not total functions regarding arbitrary frameworks.

\begin{example}[Infinite Recursion \cite{BauS17}] 
	\label{ex:infrecursion}
	Consider the following AF \linebreak $\AF=(A\cup B,R)$ where
	\begin{itemize}
	  \item $A = \{a_i \mid i\in\N\}$, $B = \{b_i \mid i\in\N\}$ and
	  \item $R = \{(b_i,a_i), (a_{i+1},a_i), (a_i,b_{i+1}) \mid i\in\N\}$ 
   \end{itemize}

 \begin{center}
	  \begin{tikzpicture}[scale=1.05]
	    \path node[arg](b1){$a_1$}
			   +(-0.7,0) node (f){$\F\!:$}
	      
	       +(0,-1.6) node[arg](c1){$b_1$}
	      ++(1.3,0) node[arg](b2){$a_2$}
	      
	       +(0,-1.6) node[arg](c2){$b_2$}
	      ++(1.3,0) node[arg](b3){$a_3$}
	       
	       +(0,-1.6) node[arg](c3){$b_3$}
	      ++(1.3,0) node[arg](b4){$a_4$}
	      
	       +(0,-1.6) node[arg](c4){$b_4$}
	      ++(1.3,0) node[arg,opacity=.7](b5){$a_5$}
				
				+(0,-1.6) node[arg,opacity=.7](c5){$b_5$}
	      ++(1.3,0) node[arg,opacity=.4](b6){$a_6$}
	      
	       +(0,-1.6) node[arg,opacity=.4](c6){$b_6$}
	       +(0.9,0) node(x){\ldots}
				+(0.9,-1.6) node(y){\ldots}
	       ;
	    
	    \path [attack,<-]
	      (b1) edge (c1)
	      (b2) edge (c2)
	      (b3) edge (c3)
	      (b4) edge (c4)
	      (b5) edge[opacity=.7] (c5)
				(b6) edge[opacity=.4] (c6)
	      ;
	   
	    \path [attack]
	      (b1) edge (c2)
	      (b2) edge (c3)
	      (b3) edge (c4)
	      (b4) edge (c5)
				(b5) edge[opacity=.7] (c6)
	      ;
	    \path [attack]
	      (b2) edge (b1)
	      (b3) edge (b2)
	      (b4) edge (b3)
	      (b5) edge[opacity=.7] (b4)
				(b6) edge[opacity=.4] (b5)
	      ;
	   
	  \end{tikzpicture}
	\end{center}  
	
	Let $\sigma\in\{\cfzwei,\stgzwei\}$. We want to check whether the candidate set $E = \left\{b_i \mid i\in\N\right\}$ is a $\sigma$-extension. 
  Observe that the AF $\F$ possesses two SCCs, namely one consisting of the single argument $b_1$ and the other containing the remaining arguments, i.e.\ $S_1 = \{b_1\}$ and $S_2 = \left(A\cup B\right) \sm\{b_1\}$.
For $S_1$ we end up with the base case returning a positive answer. For $S_2$ we have to consider the AF $\F' = \F|_{UP_\F(S_2,E)} = \F|_{\left(A\cup B\right) \sm\{a_1,b_1\}}$ (since $a_1$ is attacked by $b_1\in E\sm S_2$)
 and the set $S' = E\cap S_2 = \{b_i \mid i\in\N, i\geq 2\}$. Obviously, determining whether $S'$ is an $\sigma$-extension
 w.r.t.\ $\F'$ is equivalent to decide whether $S$ is an $\sigma$-extension w.r.t.~$\F$.
This means, the consideration of the candidate set $E$ leads to infinite recursion. 
\end{example}

\subsection{Definedness Statuses of Argumentation Semantics}

We now introduce the two main definedness statuses of argumentation semantics which capture the notions of existence and uniqueness, namely so-called \textit{universal} and \textit{unique definedness}. Both versions are relativized to a certain set of AFs. If clear from context, unimportant or if $\m{C} = \m{F}$ we will not mention explicitly the considered set of AFs. 

\begin{definition} \label{def:uniuni}   
 Given a semantics $\sigma$ and a set $\m{C}$ of AFs. We say that $\sigma$ is
 \begin{enumerate}
	 \item \textit{universally defined w.r.t.\ \m{C}} if $\forall\F\in\m{C}$, $\card{\sigma(\F)}\geq 1$ and
	\item \textit{uniquely defined w.r.t.\ \m{C}} if $\forall\F\in\m{C}$, $\card{\sigma(\F)} = 1$.
\end{enumerate}
\end{definition}

Why are the introduced definedness statuses important? Let us consider an arbitrary logical formalism $\m{L}$ together with its semantics $\sigma_{\m{L}}$. One central question is whether the semantics provides any $\m{L}$-theory $T$ with a formal meaning, i.e.\ $\card{\sigma_{\m{L}}(T)} \geq 1$. A more demanding property than \textit{existence} is \textit{uniqueness}, i.e.\ $\card{\sigma_{\m{L}}(T)} = 1$ for any $\m{L}$-theory $T$. Clearly, these properties are interesting from several perspectives. For instance, in case of uniqueness, we observe a coincidence of sceptical and credulous reasoning modes. More precisely, if $\sigma_{\m{L}}(T) = \{E\}$, then $\bigcap\ \sigma_{\m{L}}(T) = \bigcup\ \sigma_{\m{L}}(T) = E$. Furthermore, if a theory $T$ is interpreted as \textit{meaningful} if and only if $\sigma_{\m{L}}(T)\neq \emptyset$, then existence might be a desired property. If the latter has to be neglected in the general case, then one further challenge is to identify sufficient properties of \m{L}-theories guaranteeing their meaningfulness. 

Let us come back to abstract argumentation frameworks \cite{Dung95}. Due to the practical nature of argumentation most work in the literature restricts itself to the case of finite AFs, i.e.\ any considered AF consists of finitely
many arguments and attacks only. For this class of AFs a proof or disproof of existence or uniqueness  is mostly straightforward. In the general infinite case however conducting such proofs is more intricate. It usually involves the proper use of set theoretic axioms, like the \textit{axiom of choice} or equivalent statements. Dung already proposed the existence of preferred extensions in the case of
    infinite argumentation frameworks. It has later on (e.g.\
    \cite{CamV10}) been pointed out that Dung has not been precise with
    respect to the use of principles. The existence of semi-stable extensions
    for finitary\footnote{An argument is called \textit{finitary} if it receives finitely many attacks only. Moreover, an AF is said to be \textit{finitary} if and only if it consists of finitary arguments only (cf.\ Definition~\ref{def:classes}).} argumentation frameworks was first shown by Weydert
    with the use of model-theoretic techniques \cite{weydert}. Later on, Baumann and Spanring presented a first comprehensive overview of results regarding existence and uniqueness for a whole bunch of semantics
    considered in the literature \cite{BauS15}. They provided
    complete or alternative proofs of already known results and contributed
 missing results for the infinite or finitary case. We mention two interesting results: Firstly, eager semantics is exceptional among the universally defined semantics since either there is exactly one or there
are infinitely many eager extensions. Secondly, stage semantics behaves similarly to semi-stable in the sense that extensions are guaranteed as long as finitary AFs are considered. A further step forward in the systematic analysis of argumentation semantics in the infinite case was presented in \cite{Spa16}. Spanring studied the relation between non-existence of extensions and the number of non-finitary arguments. It was shown that there are AFs where one single non-finitary argument causes a collapse\footnote{The term \textit{collapse} was firstly introduced in \cite{Spa16} and it refers to a semantics not providing any extension/labelling for a given AF. In \cite{BauU18} the authors studied how to repair such a semantical defect.} of semi-stable semantics. Interestingly, all known AFs which do not provide any stage extension possesses infinitely many non-finitary arguments. It is an open question whether this observation applies in general \cite[Conjecture~14]{Spa16}.

\section{Existence and Uniqueness w.r.t.\ Finite AFs}
As a matter of fact, in order to show that a certain semantics $\sigma$ is not universally defined w.r.t.\ a certain set~\m{C} it suffices to present an AF $\F\in\m{C}$, s.t.\ $\sigma(\F)~=~\emptyset$. Contrastingly, an affirmative answer w.r.t.\ universal definedness requires a proof involving all AFs in~\m{C}. Let us consider finite AFs first. It is well-known that stable semantics does not warrant the existence of
      extensions/labellings even in the case of finite AFs. Witnessing examples are given by odd-cycles (cf.\ Example~\ref{ex:nostb}). Interestingly, in case of finite AFs we have that being odd-cycle free is sufficient for warranting at least one stable extension/labelling.\footnote{This is due to the fact that firstly, in case of finite AFs, being odd-cycle free
      coincides with being \textit{limited controversial} \cite[Definition 32]{Dung95} and secondly, any limited controversial AFs warrants the existence of at least one stable extensions \cite[Corollary 36]{Dung95}.}
			
\begin{example} \label{ex:nostb} The following minimalistic AFs cause a collapse of stable semantics, i.e.\ $\stb(\F_1) = \stb(\F_3) = \emptyset$.
      
	\begin{tikzpicture}

		\node (A5) at (0,-2.5) [circle, minimum size=0.7cm, thick, draw, label = left:$\F_1:$] {$a_1$};
    
		\node (A1) at (2.8,-2.5) [circle, minimum size=0.7cm, thick, draw, label = left:$\F_3:$] {$a_1$};
    \node (B1) at (4,-2.5) [circle, minimum size=0.7cm,  thick, draw] {$a_2$};
		\node (C1) at (3.4,-1.5) [circle, minimum size=0.7cm,  thick, draw] {$a_3$};


\draw[->, thick] (A5) to [thick,loop,distance=0.5cm] (A5);
\draw[->, thick] (A1) to [thick,bend right] (B1);
\draw[->, thick] (B1) to [thick,bend right] (C1);
\draw[->, thick] (C1) to [thick,bend right] (A1);

\end{tikzpicture}

Observe that both frameworks do possess semi-stable, stage2 as well as stage extensions/labellings. The extensions are as follows: For any $\sigma\in\{\semi,\stgzwei,\stg\}$, $\tau\in\{\stgzwei,\stg\}$, $\Ext_{\sigma}(\F_1) = \{\emptyset\} = \Ext_{\semi}(\F_3)$ and $\Ext_{\tau}(\F_3) = \{\{a_1\},\{a_2\},\{a_3\}\}$. 

\end{example}

Let us consider now semi-stable semantics. Example~\ref{ex:nostb} shows that AFs may possess semi-stable extensions even in the absence of stable extensions. Are semi-stable extensions possibly guaranteed in case of finite AFs? Consider the following explanations about the existence of semi-stable extensions taken from \cite{semi06}: 

\begin{quote} For every argumentation framework there exists at least one semi-stable extension.
This is because there exists at least one complete extension, and a semi-stable extension
is simply a complete extension in which some property (the union of itself and the
arguments it defeats) is maximal.
\end{quote}

We would like to point out two issues. Firstly, the presented explanation should not be understood as: Since any semi-stable extension is a complete one and complete semantics is universally defined we conclude that semi-stable semantics is universally defined. Accepting this kind of (false) argumentation would imply the universal definedness of stable semantics since also any stable extension is a complete one. The second issue is that the presented explanation is not precise about \textit{why} it is guaranteed that the non-empty set of complete extensions possesses at least one range-maximal member. The following statement gives a more precise explanation \cite{CamCD12}:

\begin{quote} For every (finite) argumentation framework, there exists at least one semi-stable extension. This is
because there exists at least one complete extension (the grounded) and the fact that the argumentation
framework is finite implies that there exist at most a finite number of complete extensions. The semi-stable
extensions are then simply those complete extensions in which some property (its range) is
maximal.
\end{quote}

This means, the additional argument that we have to compare finitely many complete extensions only justifies the universal definedness of semi-stable extensions in case of finite AFs. Obviously, in case of infinite AFs we cannot expect to have finitely many complete extensions implying that this kind of argumentation is no longer valid for finitary as well as infinite AFs in general. 

In the rest of this section we want to argue why all considered semantics except the stable one are universally defined in case of finite AFs.\footnote{We mention that grounded, ideal and eager semantics are even uniquely defined w.r.t.\ finite AFs. This will be a by-product of Theorem~\ref{the:grduni}, Corollary~\ref{the:universalid} as well as Theorem~\ref{the:universaleager}.} Remember that many semantics are looking for certain $\subseteq$-maximal elements. The main advantage in case of finiteness is that it is simply impossible to have infinite $\subseteq$-chains which guarantees the existence of $\subseteq$-maximal elements. Consider the following more detailed explanations. Given a finite AF $\F = (A,R)$, i.e.\ $\card{A} = n \in\N$. Consequently, $1 \leq \card{2^A} = 2^n \in\N$. By definition of any extension-based semantics $\sigma$  we derive $0 \leq \card{\Ext_{\sigma}(\F)} \leq 2^n$ since $\Ext_{\sigma}(\F)\subseteq 2^A$ (cf.\ Definition~\ref{def:semantics}). This means, for any finite $\F$ and any semantics~$\sigma$ we have at least one candidate set for being a $\sigma$-extension (namely, the empty set) and at most finitely many $\sigma$-extensions. In any case, the empty set is conflict-free as well as admissible, i.e.\ $\card{\Ext_{\cf}(\F)}, \card{\Ext_{\adm}(\F)} \geq 1$. Furthermore, naive and preferred semantics are looking for $\subseteq$-maximal conflict-free or admissible sets, respectively. Since we have finitely many conflict-free as well as admissible sets only we derive the universal definedness of naive and preferred semantics in case of finite AFs. Combining $\Ext_{\prf}\subseteq\Ext_{\com}$ and $\card{\Ext_{\prf}(\F)} \geq 1$ yields the universal definedness of complete semantics in case of finite AFs.
Moreover, since $1 \leq \card{\Ext_{\cf}(\F)}, \card{\Ext_{\adm}(\F)} \leq 2^n$ is given we obtain the universal definedness of stage and semi-stable semantics in case of finite AFs because the existence of $\subseteq$-range-maximal is guaranteed. 
Let us consider ideal and eager semantics. Candidate sets of both semantics are admissible sets being in the intersection of all preferred or semi-stable extensions, respectively. Note that there is at least one admissible set satisfying this property, namely the empty one since definitely $\emptyset\subseteq \bigcap \Ext_{\prf}(\F) \subseteq \m{U}$ as well as $\emptyset\subseteq \bigcap \Ext_{\semi}(\F) \subseteq \m{U}$. This means, the sets of candidates are non-empty and finite which guarantees the existence of $\subseteq$-maximal elements implying the universal definedness of ideal and eager semantics in case of finite AFs. 
The grounded extension, i.e.\ the $\subseteq$-least fixpoint of the characteristic function $\Gamma_\AF$, is guaranteed due to the monotonicity of $\Gamma_\AF$ and the famous Knaster-Tarski theorem \cite{Tar55}.
Finally, even the more exotic stage2 as well as cf2 semantics are universally defined w.r.t.\ finite AFs. This can be seen as follows: Obviously, finitely many as well as initial SCCs are guaranteed due to finiteness. Consequently, one may start with computing stage/naive extension on these initial components and ``propagate'' the resulting extensions to the subsequent SCCs and so on. This procedure will definitely terminate and ends up with stage2/cf2 extensions. Apart from stable semantics we have argued that the extension-based versions of all considered semantics are universally defined w.r.t.\ finite AFs. In case of mature semantics, the result carry over to their labelling-based versions since any of these semantics possesses a one-to-one-correspondence between extensions and labellings. This property does not hold in case of admissible as well as conflict-free sets. However, since any admissible/conflict-free set induce at least one admissible/conflict-free labelling the result applies to their labelling versions too (cf.\ Section~\ref{sec:baspro}).

\section{Existence and Uniqueness w.r.t.\ Arbitrary AFs} \label{sec:exunrestrict}

\subsection{Non-well-defined Semantics}

In contrast to all other semantics available in the literature, cf2 as well as stage2 semantics were originally defined recursively. The recursive schema is based on the decomposition of AFs
along their strongly connected components (SCCs). Roughly speaking, the schema takes a base semantics $\sigma$ and proceeds along the induced partial ordering and evaluates the SCCs according to $\sigma$ while propagating relevant results to subsequent SCCs. This procedure defines a $\sigma\mathit{2}$ semantics.\footnote{Following this terminology we have to rename $\cfzwei$ semantics to $\nav\mathit{2}$ semantics since its base semantics is the naive semantics and not conflict-free sets.} Given so-called \textit{SCC-recursiveness} \cite{BarGG05} we have to face some difficulties in drawing conclusions with respect to infinite AFs. Firstly, arbitrary AFs need not to possess initial SCCs which is granted for finite AFs. This makes checking whether a certain set is an $\sigma\mathit{2}$-extension more complicated and in particular, especially due to the recursive definitions not that easy to handle. Secondly, even worse, even if an AF as well as subsequent subframeworks of it possess initial SCCs there is no guarantee that any recursion will stop in finitely many steps. More precisely, as shown in Example~\ref{ex:infrecursion} there might be candidate sets which lead to infinite recursion, i.e.\ the base case will never be considered. In \cite[Propositions 2.12 and 3.2]{GagD14} the authors considered alternative non-recursive definitions of cf2 as well as stage2 semantics in case of finite AFs. It is an open question whether these definitions overcome the problem of undefinedness for arbitrary frameworks.

\subsection{Collapsing Semantics}

Dealing with finite AFs is a common as well as attractive and reasonable
restriction, due to their computational nature. In the section before we have argued that apart from stable semantics all considered semantics are universally defined w.r.t.\ finite AFs. It is an important observation that warranting the existence of $\sigma$-extensions/labellings in case of finite AFs does not necessarily carry over
      to the infinite case, i.e.\ the semantics~$\sigma$ does not need to be
      universally defined w.r.t.\ arbitrary AFs. Take for instance semi-stable and stage semantics.
      To the best of our knowledge the first example showing that semi-stable
      as well as stage semantics does not guarantee extensions/labellings in case of
      non-finite AFs was given in \cite[Example 5.8.]{Ver03} and is picked up
      in the following example. 

      \begin{example}[Collapse of Stage and Semi-stable Semantics]
	\label{ex:noext}
	Consider the AF $\AF=(A\cup B\cup C,R)$ where
	\begin{itemize}
	  \item $A = \{a_i \mid i\in\N\}$, $B = \{b_i \mid i\in\N\}$, $C = \{c_i \mid i\in\N\}$ and
	  \item
	    $R = \{(a_i,b_i), (b_i,a_i), (b_i,c_i),
	      (c_i,c_i) \mid i\in\N\} \cup \{(b_i,b_j),
	      (b_i,c_j)\mid i,j\in\N, j<i\}$  
	    
	\end{itemize}

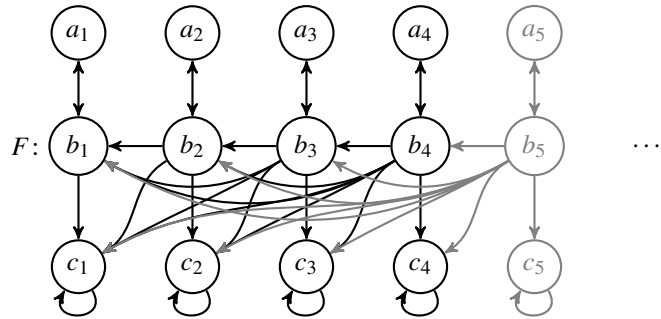
	  \begin{center}
	  \begin{tikzpicture}
	    \path node[arg](b1){$b_1$}
			   +(-0.7,0) node (f){$\F\!:$}
	       +(0,1.5) node[arg](a1){$a_1$}
	       +(0,-1.6) node[arg](c1){$c_1$}
	      ++(1.5,0) node[arg](b2){$b_2$}
	       +(0,1.5) node[arg](a2){$a_2$}
	       +(0,-1.6) node[arg](c2){$c_2$}
	      ++(1.5,0) node[arg](b3){$b_3$}
	       +(0,1.5) node[arg](a3){$a_3$}
	       +(0,-1.6) node[arg](c3){$c_3$}
	      ++(1.5,0) node[arg](b4){$b_4$}
	       +(0,1.5) node[arg](a4){$a_4$}
	       +(0,-1.6) node[arg](c4){$c_4$}
	      ++(1.5,0) node[arg,gray](b5){$b_5$}
	       +(0,1.5) node[arg,gray](a5){$a_5$}
	       +(0,-1.6) node[arg,gray](c5){$c_5$}
	       +(1.5,0) node(x){\ldots}
	       ;
	    \path [attack,loop,distance=0.5cm,out=-60,in=-120]
	      (c1) edge (c1)
	      (c2) edge (c2)
	      (c3) edge (c3)
	      (c4) edge (c4)
	      (c5) edge[gray] (c5)
	      ;
	    \path [attack,<-]
	      (c1) edge (b1)
	      (c2) edge (b2)
	      (c3) edge (b3)
	      (c4) edge (b4)
	      (c5) edge[gray] (b5)
	      ;
	    \path [attack,<->]
	      (b1) edge (a1)
	      (b2) edge (a2)
	      (b3) edge (a3)
	      (b4) edge (a4)
	      (b5) edge[gray] (a5)
	      ;
	    \path [attack,in=30,out=210]
	      (b2) edge (c1)
	      (b3) edge (c1)
	      (b3) edge (c2)
	      (b4) edge (c1)
	      (b4) edge (c2)
	      (b4) edge (c3)
	      (b5) edge[gray] (c1)
	      (b5) edge[gray] (c2)
	      (b5) edge[gray] (c3)
	      (b5) edge[gray] (c4)
	      ;
	    \path [attack]
	      (b2) edge (b1)
	      (b3) edge (b2)
	      (b4) edge (b3)
	      (b5) edge[gray] (b4)
	      ;
	    \path [attack,in=330,out=210]
	      (b3) edge (b1)
	      (b4) edge (b1)
	      (b4) edge (b2)
	      (b5) edge[gray] (b1)
	      (b5) edge[gray] (b2)
	      (b5) edge[gray] (b3)
	      ;
	  \end{tikzpicture}
	\end{center}  
  
	The set of preferred and naive extensions coincide, in particular
	$\Ext_{\prf}(\AF) = \Ext_{\nav}(\AF) = \{A\} \cup \left\{ E_i \mid i\in\N \right\}$
	where $E_i = (A\setminus\{a_i\}) \cup \{b_i\} $. Furthermore, none of
	these extensions is $\subseteq$-range-maximal since \linebreak
	$A^\oplus\subsetneq\E^\oplus_{i}\subsetneq\E^\oplus_{i+1}$ for any $i\in\N$. In
	consideration of $\semi\subseteq \prf$ and $\stg\subseteq \nav$ (cf.\
	Figure~\ref{fig:semrel}) we conclude that this framework possesses
	neither semi-stable nor stage extensions/labellings. 
      \end{example}

In Example~\ref{ex:infrecursion} we have seen that cf2 as well as stage2 semantics are not well-defined in general. This means, there are infinite AFs and candidate sets leading to an infinite recursion implying that there is no definite answer on whether such a set is an extension. However, the following example shows that even if for any candidate set a definitive decision is possible there need not to be an extension in contrast to finite~AFs. 

\begin{example}[Collapse of Cf2 and Stage2 Semantics]
	\label{ex:nostg}
	Taking into account the AF $\AF=(A\cup B\cup C,R)$ from
	Example~\ref{ex:noext}. Consider the AF $\G = \F|_B$, i.e.\ the restriction of $\F$ to $B$.

\begin{center}
	  \begin{tikzpicture}
	    \path 	node[arg](b1){$b_1$}
			+(-1,0) node(f){$\AG\!:$}
	      ++(1.5,0) node[arg](b2){$b_2$}
	      ++(1.5,0) node[arg](b3){$b_3$}
	      ++(1.5,0) node[arg](b4){$b_4$}
	      ++(1.5,0) node[arg,gray](b5){$b_5$}
	      ++(1.5,0) node[arg,lightgray](b6){$b_6$}
	      +(1.2,0) node(x){\ldots}
	      ;
	    \path [attack]
	      (b2) edge (b1)
	      (b3) edge (b2)
	      (b4) edge (b3)
	      (b5) edge[gray] (b4)
	      (b6) edge[lightgray] (b5)
	      ;
	    \path [attack,in=330,out=210]
	      (b6) edge[lightgray] (b1)
	      (b6) edge[lightgray] (b2)
	      (b6) edge[lightgray] (b3)
	      (b6) edge[lightgray] (b4)
	      (b5) edge[gray] (b1)
	      (b5) edge[gray] (b2)
	      (b5) edge[gray] (b3)
	      (b3) edge (b1)
	      (b4) edge (b1)
	      (b4) edge (b2)
	      ;
	  \end{tikzpicture}
	  \end{center}
		Let $\sigma \in \left\{\cfzwei,\stgzwei \right\}$. Obviously, any argument $b_i$ constitutes a SCC $\{b_i\}$ which is evaluated as
      $\{b_i\}$ by the base semantics of $\sigma$. Consequently, $\emptyset$
      cannot be a $\sigma$-extension. Furthermore, a singleton $\{b_j\}$ cannot
      be a $\sigma$-extension either. The $b_i$'s for $i>j$ are not affected by
      $\{b_j\}$ and thus, the evaluation of $\G|_{UP_\G(\{b_i\},\{b_j\})} = \G|_{\{b_i\}} = (\{b_i\},\emptyset)$ do not return $\emptyset$ as
      required. Finally, any set containing more than two arguments would rule
      out at least one of them and thus, cannot be a $\sigma$-extension. Hence, $\card{\Ext_{\sigma}(\AG)} = \card{\Lab_{\sigma}(\AG)} = 0$. 
 \end{example}

In Example~\ref{ex:noext} we have seen an AF $\F$ without any semi-stable and stage extensions/labellings. In \cite{BauS15} the authors studied the question of existence-dependency between both semantics in case of infinite AFs. More precisely, they studied whether it is possible that some AF does have semi-stable but no
      stage extensions or vice versa, there are stage but no
      semi-stable extensions. The following Example~\ref{ex:nosem} shows that stage extensions might exist even if semi-stable semantics collapses.\footnote{The AF $\G = \AF|_{B}$ depicted in Example~\ref{ex:nostg} witnesses the reverse case. It can be checked that $\Ext_{\semi}(\AG) = \{\emptyset\}$ and $\Ext_{\stg}(\AG) = \emptyset$ (cf.\ \cite[Example~2]{BauS15} for further explanations).}

      \begin{example}[No Semi-Stable but Stage Extensions/Labellings] 
	\label{ex:nosem} \mbox{} \\
Consider~again the AF $\F$ depicted in Example~\ref{ex:noext}. Using the components of $\F$ we define \linebreak $\G=(A\cup B\cup C\cup D\cup
	E,R\cup R')$  where 
	\begin{itemize}
	  \item $D = \{d_i \mid i\in\N\}$ and $E = \{e_i \mid i\in\N\}$ and
	  \item $R' = \{(a_i,d_i), (d_i,a_i), (b_i,d_i), (d_i,b_i), (d_i,c_i), (e_i,d_i), (e_i,e_i) \mid i\in\N\}$
	\end{itemize}

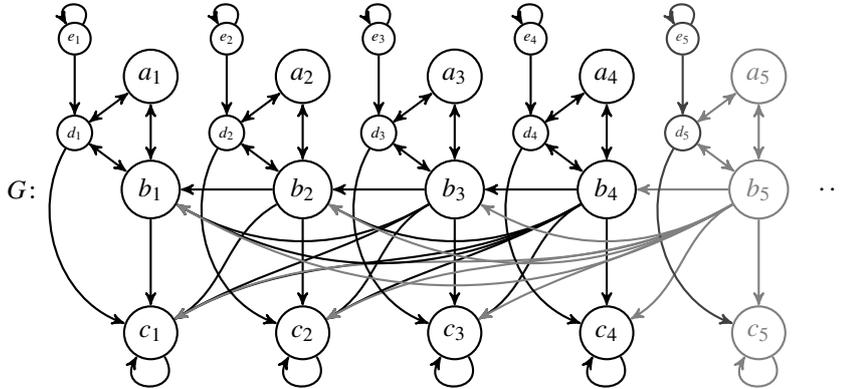
	\begin{center}
	  \begin{tikzpicture}
	    \path 	node[arg](b1){$b_1$}
			  +(-1.7,0) node(g){$\G\!:$}
	      +(-1,.75) node[argm](d1){$d_1$}
	      +(-1,2) node[argm](e1){$e_1$}
	      +(0,1.5) node[arg](a1){$a_1$}
	      +(0,-1.9) node[arg](c1){$c_1$}
	      ++(2,0) node[arg](b2){$b_2$}
	      +(-1,.75) node[argm](d2){$d_2$}
	      +(-1,2) node[argm](e2){$e_2$}
	      +(0,1.5) node[arg](a2){$a_2$}
	      +(0,-1.9) node[arg](c2){$c_2$}
	      ++(2,0) node[arg](b3){$b_3$}
	      +(-1,.75) node[argm](d3){$d_3$}
	      +(-1,2) node[argm](e3){$e_3$}
	      +(0,1.5) node[arg](a3){$a_3$}
	      +(0,-1.9) node[arg](c3){$c_3$}
	      ++(2,0) node[arg](b4){$b_4$}
	      +(-1,.75) node[argm](d4){$d_4$}
	      +(-1,2) node[argm](e4){$e_4$}
	      +(0,1.5) node[arg](a4){$a_4$}
	      +(0,-1.9) node[arg](c4){$c_4$}
	      ++(2,0) node[arg,gray](b5){$b_5$}
	      +(-1,.75) node[argm,darkgray](d5){$d_5$}
	      +(-1,2) node[argm,darkgray](e5){$e_5$}
	      +(0,1.5) node[arg,gray](a5){$a_5$}
	      +(0,-1.9) node[arg,gray](c5){$c_5$}
	      +(1,0) node(x){\ldots}
	      ;
	    \path [attack,loop,distance=0.4cm,out=60,in=120]
	      (e1) edge (e1)
	      (e2) edge (e2)
	      (e3) edge (e3)
	      (e4) edge (e4)
	      (e5) edge[darkgray] (e5)
	      ;
	    \path [attack,loop,distance=0.6cm,out=-60,in=-120]
	      (c1) edge (c1)
	      (c2) edge (c2)
	      (c3) edge (c3)
	      (c4) edge (c4)
	      (c5) edge[gray] (c5)
	      ;
	    \path [attack,->]
	      (e1) edge (d1)
	      (e2) edge (d2)
	      (e3) edge (d3)
	      (e4) edge (d4)
	      (e5) edge[darkgray] (d5)
	      ;
	    \path [attack,<-]
	      (c1) edge (b1)
	      (c2) edge (b2)
	      (c3) edge (b3)
	      (c4) edge (b4)
	      (c5) edge[gray] (b5)
	      ;
	    \path [attack,<->]
	      (b1) edge (d1)
	      (b2) edge (d2)
	      (b3) edge (d3)
	      (b4) edge (d4)
	      (b5) edge[gray] (d5)
	      ;
	    \path [attack,<->]
	      (a1) edge (d1)
	      (a2) edge (d2)
	      (a3) edge (d3)
	      (a4) edge (d4)
	      (a5) edge[gray] (d5)
	      ;
	    \path [attack,<->]
	      (b1) edge (a1)
	      (b2) edge (a2)
	      (b3) edge (a3)
	      (b4) edge (a4)
	      (b5) edge[gray] (a5)
	      ;
	    \path [attack,in=30,out=210]
	      (b2) edge (c1)
	      (b3) edge (c1)
	      (b3) edge (c2)
	      (b4) edge (c1)
	      (b4) edge (c2)
	      (b4) edge (c3)
	      (b5) edge[gray] (c1)
	      (b5) edge[gray] (c2)
	      (b5) edge[gray] (c3)
	      (b5) edge[gray] (c4)
	      ;
	    \path [attack]
	      (b2) edge (b1)
	      (b3) edge (b2)
	      (b4) edge (b3)
	      (b5) edge[gray] (b4)
	      ;
	    \path [attack,in=160,out=240]
	      (d1) edge (c1)
	      (d2) edge (c2)
	      (d3) edge (c3)
	      (d4) edge (c4)
	      (d5) edge[darkgray] (c5)
	      ;
	    \path [attack,in=330,out=210]
	      (b3) edge (b1)
	      (b4) edge (b1)
	      (b4) edge (b2)
	      (b5) edge[gray] (b1)
	      (b5) edge[gray] (b2)
	      (b5) edge[gray] (b3)
	      ;
	    \end{tikzpicture}
	  
	\end{center}

	In comparison to Example~\ref{ex:noext} we do not observe any changes
	as far as preferred and semi-stable semantics are concerned. In
	particular, $\Ext_{\prf}(\G)= \{A\} \cup \left\{ E_i \mid i\in\N \right\}$
	where $E_i = (A\setminus\{a_i\}) \cup \{b_i\} $ and again, none of
	these extensions is $\subseteq$-range-maximal. Hence, $\Ext_{\semi}(\G) =
	\emptyset$. Observe that we do have additional conflict-free as well as
	naive sets, especially the set~$D$. Since any $e\in E$ is
	self-defeating and unattacked and furthermore, $D^\oplus = A\cup B\cup C\cup
	D$ we conclude, $\Ext_{\stg}(\G) = \{D\}$. Due to the one-to-one correspondence the collapse or non-collapse transfer to their labelling-based versions.
      \end{example}

 \subsection{Universally Defined Semantics}
We now turn to semantics which are universally defined w.r.t.\ the whole class of AFs. The first non-trivial result in this line was already proven by Dung himself, namely the universal definedness of the extension-based version of preferred semantics \cite[Corollary 12]{Dung95}. He argued that the Fundamental Lemma (cf.\ \cite[Lemma 10]{Dung95}) immediately implies that the set of all admissible sets is a complete partial order which means that any $\subseteq$-chain possesses a least upper bound. Then (and this was not explicitly stated in \cite{Dung95}), due to the famous \textit{Zorn's lemma} \cite{zorn1935} the existence of $\subseteq$-maximal admissible sets, i.e.\ preferred extensions, is guaranteed. 

In order to get an idea how things work in the general case we illustrate some proofs in more detail. We will see that a proof of universal definedness w.r.t.\ arbitrary AFs is completely different to the argumentation in case of finite ones. In order to keep this section self-contained we start with Zorn's lemma and an equivalent version of it.

  \begin{lemma}[\cite{zorn1935}]
	\label{lem:zorn}
	Given a partially ordered set $(P,\leq)$. If any $\leq$-chain possesses
	an upper bound, then $(P,\leq)$ has a maximal element.
      \end{lemma}

  \begin{lemma}
	\label{lem:zorn2}
	Given a partially ordered set $(P,\leq)$. If any $\leq$-chain possesses
	an upper bound, then for any $p\in P$ there exists a maximal element
	$m\in P$, s.t.\ $p\leq m$.
      \end{lemma}

Having Lemma~\ref{lem:zorn2} at hand we may easily argue that any conflict-free/admissible set is bounded by a naive/preferred extension. 

      \begin{lemma}
	\label{lem:prna}
	Given $\AF = (A,R)$ and $E\subseteq A$,
	\begin{enumerate}
	  \item if $\E\in\Ext_{\cf}(\AF)$, then there exists $E'\in\Ext_{\nav}(\AF)$ s.t.\
	    $E\subseteq E'$ and
	  \item if $\E\in\Ext_{\adm}(\AF)$, then there exists $E'\in\Ext_{\prf}(\AF)$ s.t.\
	    $E\subseteq E'$.
	\end{enumerate}
      \end{lemma}

      \begin{proof}
	For $\AF = (A,R)$ we have the associated power set lattice
	$(2^A,\subseteq)$. Consider now the partially ordered fragments
	$\m{C} = (\Ext_{\cf}(\AF),\subseteq)$ as well as $\m{A} = (\Ext_{\adm}(\AF),\subseteq)$. In
	accordance with Lemma~\ref{lem:zorn2} the existence of naive and
	preferred supersets is guaranteed if any $\subseteq$-chain possesses an
	upper bound in $\m{C}$ or $\m{A}$, respectively. Given a
	$\subseteq$-chain $\m{E}\subseteq\Ext_{\cf}(\AF)$ or $\m{E}\subseteq\Ext_{\adm}(\AF)$,
	respectively. Consider now $\bar{E} = \bigcup \m{E}$. Obviously, $\bar{E}$
	is an upper bound of $\m{E}$, i.e.\ for any $E\in\m{E}$, $\E\subseteq\bar{E}$. It remains to show that $\bar{E}$ is conflict-free or admissible,
	respectively. Conflict-freeness is a finite condition. This means, if
	there were conflicting arguments $a,b\in\bar{E}$ there would have to be some conflict-free sets $E_a, E_b\in \bar{E}$, s.t.\ $a\in E_a$ and $b\in E_b$. Since \m{E} is a $\subseteq$-chain we have $E_a \subseteq E_b$ or $E_b \subseteq E_a$ which contradicts the conflict-freeness of at least one of them. Assume now $\bar{E}$ is not admissible.
	Consequently, there is some $a\in\bar{E}$ that is not defended by $\bar{E}$.
	Furthermore, there has to be an $E_a\in\m{E}$, s.t.\ $a\in E_a$ contradicting the admissibility of $E_a\in\Ext_{\adm}(\AF)$.
      \end{proof}

According to the last lemma, we may deduce the universal definedness of the extension-based versions of preferred as well as naive semantics as long as, for any AF $\F$, the existence of at least one conflict-free or admissible set is guaranteed. This is an easy task since the empty set is conflict-free as well as admissible even in the case of arbitrary AFs. Consequently, universal definedness of both extension-based semantics is given and the same applies to their labelling-based versions due to their one-to-one correspondence.

      \begin{theorem}
	\label{th:minor}
	Let $\sigma\in\{\prf,\nav\}$. The semantics $\sigma$ is universally defined.
      \end{theorem}

Remember that no matter which cardinality a  considered AF possesses, we have that any preferred extension/labelling is a complete extension/labelling (Proposition~\ref{Pro:semrel}). Thus, having the universal definedness of preferred semantics at hand we deduce that even complete semantics is universally defined w.r.t.\ the whole class of AFs . 

\begin{theorem}
	\label{th:minor2}
	The semantics $\com$ is universally defined.
      \end{theorem} 

Let us consider now eager and ideal semantics. An eager extension is defined as the $\subseteq$-maximal admissible set that is a subset of each semi-stable extension. This is very similar to the definition of an ideal extension where the role of semi-stable extensions is taken over by preferred ones. On a more abstract level, both semantics are instantiations of the following schema.

\begin{definition} \label{def:parsemantics} Let $\sigma$ be a semantics (so-called \textit{base semantics}). We define the $\sigma$-\textit{parametrized} semantics $\adm^\sigma$ as follows. For any AF $\F$, $$\Ext_{\adm^\sigma} = \max_{\subseteq}\left\{E\in\Ext_{\adm}(\AF) \left| E\subseteq
	\bigcap_{S\in\Ext_{\sigma}(\AF)}S\right.\right\}.$$
\end{definition}

These kind of semantics were firstly introduced in \cite{DvoDW11}. The authors studied general properties of these semantics in case of finite AFs with the additional restriction that the base semantics~$\sigma$ has to be universally defined. The following general theorem requires neither finiteness of AFs, nor any assumption on the base semantics.

      \begin{theorem}
	\label{th:minor3}
Any $\sigma$-\textit{parametrized} semantics is universally defined.
      \end{theorem}
	
	\begin{proof}
	Given an AF $\AF = (A,R)$ and a $\sigma$-parametrized semantics $\adm^\sigma$. Consider the following set
	$\Sigma = \left\{E\in\Ext_{\adm}(\AF) \mid E\subseteq
	\bigcap_{S\in\Ext_{\sigma}(\AF)}S \right\}$. Note that in the collapsing case, i.e.\ $\Ext_{\sigma}(\AF) = \emptyset$,
	we have: $\bigcap_{S\in\Ext_{\sigma}(\AF)}S = \{x\in\m{U}\mid \forall S\in\Ext_{\sigma}(\AF): x\in S\} = \m{U}$. However, in any case $\Sigma\neq\emptyset$ since for any $\AF$, $\emptyset\in\Ext_{\adm}(\AF)$ and obviously, $\emptyset\subseteq \bigcap_{S\in\Ext_{\sigma}(\AF)}S \subseteq \m{U}$. 
	In order to show that
	$\Ext_{\adm^\sigma}(\AF)\neq\emptyset$ it suffices to prove that $(\Sigma,\subseteq)$ possesses maximal elements. We will use Zorn's lemma. Given a
	$\subseteq$-chain $\m{E}\in2^\Sigma$. Consider now $\bar{E} = \bigcup \m{E}$. 
	Analogously to the proof of Lemma~\ref{lem:prna} we may easily show that $\bar{E}$ is conflict-free and even admissible.
	Moreover, since for any $E\in\m{E}$, $E\subseteq\bigcap_{S\in\sigma(\AF)}S$ we deduce $\bar{E}\subseteq
	\bigcap_{S\in\sigma(\AF)}S$ guaranteeing $\bar{E}\in\Sigma$. Now, applying Lemma~\ref{lem:zorn}, we deduce the existence of $\subseteq$-maximal elements in $\Sigma$, i.e.\ $\card{\Ext_{\adm^\sigma}(\F)}\geq 1$ concluding the proof. 
	\end{proof}

In particular, we obtain the result for the extension-based versions of eager and ideal semantics and thus, due to the one-to-one correspondence for both labelling-based versions too.

\begin{corollary}
	\label{the:minor4}
	Let $\sigma\in\{\eag,\id\}$. The semantics $\sigma$ is universally defined.
      \end{corollary} 
	
One obvious question is whether the statement above can be strengthened in the sense that both semantics are even uniquely defined w.r.t.\ the whole class of AFs. The following proposition, in particular the second item, shows that the unique definedness of eager semantics w.r.t.\ finite frameworks does not carry over to the general unrestricted case. 
			
			 \begin{proposition}
	\label{pro:eagernouni}
 For any $\AF$ we have: 
	\begin{enumerate}
	  \item $\semi(\AF) = \emptyset \Rightarrow \eag(\AF) = \prf(\AF)$
	    and
	  \item $\semi(\AF) = \emptyset \Rightarrow \card{\eag(\AF)}\geq \aleph_0 = \card{\N}$.
	    
	\end{enumerate}
      \end{proposition}
      \begin{proof}
			We show both assertions for the extension-based versions.
			\begin{enumerate}
				\item Given $\AF = (A,R)$ and let $\Ext_{\semi}(\AF) = \emptyset$. Hence, 
	$\bigcap_{S\in\Ext_{\semi}(\AF)}S~=~\m{U}$. Consequently we deduce $\Ext_{\semi}(\AF) = \max_{\subseteq}\left\{E\in\Ext_{\adm}(\AF) \left| E\subseteq\m{U}\!\right.\right\}.$ This means, $\Ext_{\semi}(\AF) = \Ext_{\prf}(\AF)$.\\
	\item We show the contrapositive. Assume $\card{\Ext_{\eag}(\AF)} = n$ for some finite cardinal
	$n\in\N$. Due to the first statement we derive, $\card{\Ext_{\prf}(\AF)} = n$. Since $\semi\subseteq\prf$ (cf.\
	Proposition~\ref{Pro:semrel}) we have finitely many candidates only. Furthermore, among these
	preferred extensions has to be at least one $\subseteq$-range-maximal set implying
	$\Ext_{\semi}(\AF)\neq\emptyset$.   
			\end{enumerate}
	\end{proof}
			
In a nutshell, if we observe a collapse of semi-stable semantics, then eager
      and preferred semantics coincide and moreover, we necessarily have
      infinitely many eager extensions/labellings. An AF witnessing such a behaviour can be found in Example~\ref{ex:noext}. 

\subsection{Uniquely Defined Semantics}

Although eager and ideal semantics are instances of $\sigma$-parametrized semantics we have shown the non-unique definedness (Proposition~\ref{pro:eagernouni}) for eager semantics only. This is no coincidence since preferred semantics, the base semantics of ideal semantics is universally defined in contrast to semi-stable semantics, the base semantics of the eager semantics. Moreover, the following theorem shows that any $\sigma$-parametrized semantics warrants the existence of exactly one extension if $\sigma$-extensions are conflict-free as well as guaranteed (\cite[Proposition 1]{DvoDW11}).
 
\begin{theorem}
	\label{the:parauni}
Given a $\sigma$-\textit{parametrized} semantics $\adm^\sigma$, s.t.\ $\sigma\subseteq\cf$ and $\sigma$ is universally defined w.r.t.\ a class \m{C}, then $\adm^\sigma$ is uniquely defined w.r.t.\ \m{C}.
      \end{theorem}
			
			\begin{proof}
			Given an AF $\F = (A,R)$. We already know $\card{\Ext_{\adm^\sigma}(\AF)} \geq 1$ (Theorem~\ref{th:minor3}). Hence, it suffices to show
	$\card{\Ext_{\adm^\sigma}(\AF)} \leq 1$. Suppose, to derive a contradiction, 
	that for some $I_1\neq I_2$ we have $I_1, I_2\in\Ext_{\adm^\sigma}(\AF)$.
	Consequently, by
	Definition~\ref{def:parsemantics}, $I_1, I_2\in\Ext_{\adm}(\AF)$ and
	$I_1,I_2\subseteq\bigcap_{S\in\Ext_{\sigma}(\AF)}S$ as well as neither
	$I_1\subseteq I_2$, nor $I_2\subseteq I_1$. Obviously, $I_1\cup
	I_2\subseteq\bigcap_{S\in\Ext_{\sigma}(\AF)}S$. Since $\Ext_{\sigma}(\AF)\neq \emptyset$ and $I_1$ as well as $I_2$ has to be subsets of any $\sigma$-extension (which are conflict-free by assumption) we deduce $I_1,I_2\in\Ext_{\cf}(\AF)$ and thus, $I_1\cup I_2\in\Ext_{\cf}(\AF)$. Furthermore, since both sets
	are admissible in $\F$ we derive  $I_1\cup I_2 \in\Ext_{\adm}(\AF)$
	contradicting the $\subseteq$-maximality of at least one of the sets $I_1$ and $I_2$.
      \end{proof}
			
  \begin{corollary}
	\label{the:universalid}
	The semantics $\id$ is uniquely defined.
      \end{corollary}
			
		A further prominent representative of uniquely defined semantics w.r.t.\ the whole class of AFs is the grounded semantics. Its unique definedness was already
      implicitly given in \cite{Dung95}. Unfortunately, this result was not
      explicitly stated in the paper. Nevertheless, in \cite[Theorem
      25]{Dung95} it was shown that firstly, the set of all complete extensions
      form a complete semi-lattice w.r.t. subset relation, i.e.\ the existence of a $\subseteq$-greatest lower
      bound for any non-empty subset $S$ is implied. Secondly, it was proven
      that the grounded extension is the $\subseteq$-least complete extension.
      Consequently, the existence of such a $\subseteq$-least extension is justified via setting $S = \Ext_{\com}(\AF)$ for any given $\AF$. Alternatively, one may stick to the original definition of the grounded extension, namely as $\subseteq$-least fixpoint of the characteristic function $\Gamma_\AF$ and argue that the monotonicity of $\Gamma_\AF$ as well as the Knaster-Tarski theorem \cite{Tar55} imply its existence.

		\begin{theorem}
	\label{the:grduni}
	The semantics $\grd$ is uniquely defined.
      \end{theorem}

  \section{Existence and Uniqueness w.r.t.\ Finitary AFs}
	
	Let us consider now finitary AFs, i.e.\ AFs where each argument receives finitely many attacks only. It was already observed by Dung itself that finitary AFs possess useful properties. More precisely, if an AF is finitary, then the characteristic function $\Gamma$ is not only monotonic, but even $\omega$-continuous \cite[Lemma 28]{Dung95} (which does not hold in case of arbitrary AFs \cite[Example~1]{BauS17}). This implies that the least fixed point of $\Gamma$, i.e.\ the unique grounded extension, can be ``computed'' in at most $\omega$ steps by iterating $\Gamma$ on the empty set (cf. \cite{Rud76} for more details).  A further advantage of finitary AFs is that for some semantics $\sigma$, the existence or even uniqueness of $\sigma$-extension is guaranteed which cannot be shown in general. 
	
Consider again the AF $\F$ depicted in Example~\ref{ex:noext}. In contrast to finite AFs where the existence of semi-stable as well stage extensions is guaranteed we observed a collapse of both semantics. Not that $\F$ is not finitary since, for example, the argument $b_1$ receives infinitely many attacks.  A positive answer in case of semi-stable
semantics, i.e.\ universal definedness w.r.t.\ finitary AFs was conjectured in \cite[Conjecture 1]{CamV10} and firstly proven by Emil Weydert in \cite[Theorem 5.1]{weydert}. Weydert proved his result in a
first order logic setup using generalized argumentation frameworks. Later on, Baumann and Spanring provided an alternative proof using transfinite induction. Moreover, they showed that even stage semantics warrants the existence of at least one extension in case of finitary AFs \cite[Theorem 14]{BauS15}. For detailed proofs we refer the reader to the mentioned scientific papers. 
	
\begin{theorem} \label{the:stgfin}
   Let $\sigma\in\{\semi,\stg\}$. The semantics $\sigma$ is universally defined w.r.t.\ finitary AFs. 
  \end{theorem}

Applying Theorem~\ref{the:parauni} we derive that exactly one eager extension/labelling is guaranteed as long as the AF in question is finitary. 

	\begin{theorem} \label{the:universaleager}
   The semantics $\eag$ is uniquely defined w.r.t.\ finitary AFs. 
  \end{theorem}

    \section{Summary of Results and Conclusion} \label{sec:sumcon}

      \newcommand{\ExtensionExists}{\boldsymbol{\exists}}
      \newcommand{\UniqueExtensionExists}{\ExtensionExists\boldsymbol{!}}
			
In this section we gave an overview on the question whether certain
semantics guarantee the existence or even unique determination of
extensions/labellings. We have seen that
these properties may vary from subclass to subclass. The following table gives a comprehensive overview over results presented
      in this section. The entry $''\ExtensionExists''$ ($''\UniqueExtensionExists''$) in
      row \textit{certain} and column $\sigma$ indicates that the semantics~$\sigma$ is universally (uniquely) defined w.r.t.\ the class of \textit{certain} frameworks. No entry reflects the situation that a \textit{certain} AF can be found which do not provide any $\sigma$-extension/labelling, i.e.\ $\sigma$ collapses. The two question marks represent open problems. Note that we already observed that cf2 as well as stage2 semantics are not well-defined in case of finitary as well arbitrary AFs. This means, there are infinite AFs and candidate sets leading to an infinite recursion implying that there is no definite answer on whether such a set is an extension (Example~\ref{ex:infrecursion}). Nevertheless, even if for any candidate set a definitive decision is possible there are infinite (but non-finitary) AFs where both semantics collapse (Example~\ref{ex:nostg}). In \cite[Conjecture 1]{BauS15} it is conjectured that this is impossible in case of finitary frameworks.

      \begin{table}[H]
	\centering
	\begin{tikzpicture}[scale=0.815] 
	  \tikzstyle{literal}=[text centered]
	  \draw[help lines] (1,0) grid (14,4);
	  \draw[help lines] (-1,0) -- (1,0);
	  \draw[help lines] (-1,1) -- (1,1);
	  \draw[help lines] (-1,2) -- (1,2);
	  \draw[help lines] (-1,3) -- (1,3);
	  \draw[help lines] (-1,0) -- (-1,3);
	  \node[literal] (S1) at (1.5,3.5)  {\standardbox{\stb}};
	  \node[literal] (S1) at (2.5,3.5)  {\standardbox{\semi}};
	  \node[literal] (S1) at (3.5,3.5)  {\standardbox{\stg}};
	  \node[literal] (S1) at (4.5,3.5)  {\standardbox{\cfzwei}};
	  \node[literal] (S1) at (5.5,3.5)  {\standardbox{\stgzwei}};
	  \node[literal] (S1) at (6.5,3.5)  {\standardbox{\prf}};
	  \node[literal] (S1) at (7.5,3.5)  {\standardbox{\adm}};
	  \node[literal] (S1) at (8.5,3.5)  {\standardbox{\com}};
	  \node[literal] (S1) at (9.5,3.5)  {\standardbox{\grd}};
	  \node[literal] (S1) at (10.5,3.5)  {\standardbox{\id}};
	  \node[literal] (S1) at (11.5,3.5)  {\standardbox{\eag}};
	  \node[literal] (S1) at (12.5,3.5)  {\standardbox{\nav}};
	  \node[literal] (S1) at (13.5,3.5)  {\standardbox{\cf}};
	  \node[literal] (S1) at (1.5,2.5)  {};
	  \node[literal] (S1) at (2.5,2.5)  {$\ExtensionExists$};
	  \node[literal] (S1) at (3.5,2.5)  {$\ExtensionExists$};
	  \node[literal] (S1) at (4.5,2.5)  {$\ExtensionExists$};
	  \node[literal] (S1) at (5.5,2.5)  {$\ExtensionExists$};
	  \node[literal] (S1) at (6.5,2.5)  {$\ExtensionExists$};
	  \node[literal] (S1) at (7.5,2.5)  {$\ExtensionExists$};
	  \node[literal] (S1) at (8.5,2.5)  {$\ExtensionExists$};
	  \node[literal] (S1) at (9.5,2.5)  {$\UniqueExtensionExists$};
	  \node[literal] (S1) at (10.5,2.5)  {$\UniqueExtensionExists$};
	  \node[literal] (S1) at (11.5,2.5)  {$\UniqueExtensionExists$};
	  \node[literal] (S1) at (12.5,2.5)  {$\ExtensionExists$};
	  \node[literal] (S1) at (13.5,2.5)  {$\ExtensionExists$};
	  \node[literal] (S1) at (1.5,1.5)  {};
	  \node[literal] (S1) at (2.5,1.5)  {$\ExtensionExists$};
	  \node[literal] (S1) at (3.5,1.5)  {$\ExtensionExists$};
	  \node[literal] (S1) at (4.5,1.5)  {$?$};
	  \node[literal] (S1) at (5.5,1.5)  {$?$};
	  \node[literal] (S1) at (6.5,1.5)  {$\ExtensionExists$};
	  \node[literal] (S1) at (7.5,1.5)  {$\ExtensionExists$};
	  \node[literal] (S1) at (8.5,1.5)  {$\ExtensionExists$};
	  \node[literal] (S1) at (9.5,1.5)  {$\UniqueExtensionExists$};
	  \node[literal] (S1) at (10.5,1.5)  {$\UniqueExtensionExists$};
	  \node[literal] (S1) at (11.5,1.5)  {$\UniqueExtensionExists$};
	  \node[literal] (S1) at (12.5,1.5)  {$\ExtensionExists$};
	  \node[literal] (S1) at (13.5,1.5)  {$\ExtensionExists$};
	  \node[literal] (S1) at (1.5,0.5)  {};
	  \node[literal] (S1) at (2.5,0.5)  {};
	  \node[literal] (S1) at (3.5,0.5)  {};
	  \node[literal] (S1) at (4.5,0.5)  {};
	  \node[literal] (S1) at (5.5,0.5)  {};
	  \node[literal] (S1) at (6.5,0.5)  {$\ExtensionExists$};
	  \node[literal] (S1) at (7.5,0.5)  {$\ExtensionExists$};
	  \node[literal] (S1) at (8.5,0.5)  {$\ExtensionExists$};
	  \node[literal] (S1) at (9.5,0.5)  {$\UniqueExtensionExists$};
	  \node[literal] (S1) at (10.5,0.5)  {$\UniqueExtensionExists$};
	  \node[literal] (S1) at (11.5,0.5)  {$\ExtensionExists$};
	  \node[literal] (S1) at (12.5,0.5)  {$\ExtensionExists$};
	  \node[literal] (S1) at (13.5,0.5)  {$\ExtensionExists$};
	  \node[literal] (S1) at (0,2.5)  {finite};
	  \node[literal] (S1) at (0,1.5)  {finitary};
	  \node[literal] (S1) at (0,0.5)  {arbitrary};
	\end{tikzpicture}
	\caption{Definedness Statuses of Semantics}
	\label{Eval:run}
      \end{table}

For a detailed complexity analysis of the associated decision problems, i.e. \textit{Given an AF $\AF$. Is $\card{\sigma(\AF)} \geq 1$ or even, $\card{\sigma(\AF)} = 1$?} we refer the reader to \cite{DvoD18}. The mentioned decisions problems are considered for finite AFs only since the input-length, i.e.\ the length of the formal encoding of an AF has to be finite (for finite representations of infinite AFs we refer the reader to \cite{BarCDG13}). Due to the table above some complexity results are immediately clear. For instance, the existence problem is trivial for all considered semantics except the stable one. An upper bound for the complexity of the uniqueness problem can be obtained via the complexity of the corresponding verification problem, i.e.\ \textit{Given an AF $\AF$ and a set $E$. Is $E\in\Ext_{\sigma}(\AF)$?}. More precisely, an algorithm which decides the uniqueness problem is the following two-step procedure: first, guessing a certain set $E$ non-deterministically and second, verifying whether this set is an $\sigma$-extension. 

As already mentioned, most of the literature concentrate on finite AFs for several reasons, especially due to their computational nature.  However, allowing an infinite number of arguments is essential in applications where upper bounds
on the number of available arguments cannot be established a priori, such as for example in dialogues \cite{BelGM15} or modeling
approaches including time or action sequences \cite{BauS12}. Moreover, even actual infinite AFs frequently occur in the instantiation-based context. More precisely, the semantics of so-called \textit{rule-based argumentation formalisms} (cf.\ Chapter 6 as well as \cite{Pra10}) is given via the evaluation of induced Dung-style AFs. In this context, even a finite set of rules may lead to an infinite set of arguments as observed in (cf.\ \cite{CamO14,Str15}).

In 2011, Baroni et al.\ wrote ``As a matter of fact, we
    are not aware of any systematic literature analysis of argumentation
    semantics properties in the infinite case.'' \cite[Section 4.4]{BarCG11}. Since then only few works have contributed to a better understanding of infinite AFs. In \cite{BarCDG13} the authors studied to which extent infinite AFs can be finitely represented via formal languages and considered several decision problems within this context. In \cite{BauS15} a detailed study of the central properties of existence
and uniqueness as presented in this chapter was given. Recently, the same authors addressed several central
issues like \textit{expressibility}, \textit{intertranslatability} or \textit{replaceability} (cf.\ Chapters~\ref{cha:realizability}~and~\ref{cha:replace}) in the general unrestricted case \cite{BauS17}.

\chapter{Expressibility of Argumentation Semantics}

\label{cha:realizability}
Given a certain logical formalism $\m{L}$ used as knowledge representation language or modelling tool in general. Depending on the application in mind, it might be interesting to know which kinds of model sets are actually expressible in~$\m{L}$?  More formally, if $\sigma_{\m{L}}$ denotes the semantics of $\m{L}$, we are interested in determining the set $\m{R}_{\m{L}}~=~\left\{\sigma_{\m{L}}\left(T\right) \mid T \text{ is an } \m{L}\text{-theory} \right\}$. This task, also known as \textit{realizability} or \textit{defineability}, highly depends on the considered formalism $\m{L}$. Clearly, potential necessary or sufficient properties for being in $\m{R}_{\m{L}}$, i.e.\ being $\sigma_{\m{L}}$-\textit{realizable}, may rule out a logic or make it perfectly appropriate for a certain application. For instance, it is well-known that in case of propositional logic any finite set of two-valued interpretations is realizable. This means, given such a finite set~\m{I}, we always find a set of formulae $T$, s.t.\ $\Mod(T) = \m{I}$.  Differently, in case of normal logic programs under stable model semantics we have that any finite candidate set is realizable if it forms a $\subseteq$-antichain, i.e.\ any two sets of the candidate set have to be incomparable with respect to the subset
relation. Remarkably, being such an $\subseteq$-antichain is not only necessary but even sufficient for realizability w.r.t.\ stable model semantics \cite{EitFPTW13,Str15}. One major application of realizability issues are dynamic evolvements of $\m{L}$-theories like in case of belief revision (cf.\ \cite{AGM,WilliamsA98,QiY08,DelgrandeP15,DelgrandeSTW08,DelgrandePW13,BauB15,DilHLRW15} for several knowledge representation formalisms). Roughly speaking, belief revision deals with the problem of integrating new pieces of information to a current knowledge base which is represented by a certain $\m{L}$-theory $T$. To this end, you are typically faced with the problem of modifying the given theory $T$ in such a way that the revised version $S$ satisfies $\sigma_{\m{L}}\left(S\right) = M$ for some model set $M$. Now, before trying to do this revision in a certain minimal way it is essential to know whether $M$ is realizable at all, i.e.\ $M\in\m{R}_{\m{L}}$. 

The first formal treatment of realizability issues w.r.t.\ extension-based argumentation semantics was recently given by Dunne et al.\ \cite{DunDLW13,DunDLW15}. They coined the term \textit{signature} for the set of all realizable sets of extensions. The authors provided simple criteria for several mature semantics deciding whether a set of extensions 
is contained in the corresponding signature. For instance, two obvious necessary conditions in case of preferred semantics (as well as many other semantics) is that a candidate set $\Ss$ has to be non-empty, due to universal definedness of preferred semantics and second, $\Ss$ has to be a $\subseteq$-antichain, also known as \textit{I-maximality criterion} \cite{BarG07}. However, these conditions are not sufficient 
implying that further requirements has to hold. In case of preferred semantics it turned out that adding the requirement of so-called \textit{conflict-sensitivity} indeed yield a set of characterizing properties. A $\subseteq$-antichain $\Ss$ is conflict-sensitive if for each pair of distinct sets $A$ and $B$ from $\Ss$
there are at least one $a\in A$ and one $b\in B$, s.t.\ $a$ and $b$ do not occur together 
in any set of~$\Ss$. This implies that there exists an AF $\F$ in which the set of its preferred extension coincides with $\Ss = \{\{a,b\},\{a,c\},\{b,d\},\{c,d\}\}$. Furthermore, since $\{a,b\}$ and $\{b,d\}$ are already contained in $\Ss$ it is impossible to realize the set $\Tt = \Ss\cup\{\{a,d\}\}$ under preferred semantics. From a practical point of view, such realizability insights can be used to limit the search space when enumerating preferred extensions. More precisely, applying the mentioned characterization result we obtain that not only $\{a,d\}$, but also any other set $A\subseteq\{a,b,c,d\}$ can not be a further preferred extension of a certain AF given that we already computed all sets contained in~$\Ss$. As a matter of fact, knowing that a certain set is realizable does not provide one automatically with a witnessing AF. Fortunately, there exist \textit{canonical frameworks} showing realizability in a constructive fashion as shown in \cite{DunDLW13,DunDLW15}.

Later on, restricted versions of realizability were considered, namely \textit{compact} as well as \textit{analytic realizability} in case of extension-based semantics \cite{compact,compact2,LinSW15,compactj}. Both versions are motivated by typical phenomena that
can be observed for several semantics. First, there potentially exist arguments in a given AF that do not appear in any
extension, so-called \textit{rejected} arguments. Second, most of the argumentation semantics possess the feature of allowing \textit{implicit conflicts}. An implicit conflict arises when two arguments 
are never jointly accepted although they do not attack each other. In order to understand in which way rejected arguments and implicit conflicts contribute to the expressive power of a certain semantics the notions of compact AFs as well as analytic AFs were introduced. The former kind disallows rejected arguments whereas the latter is free of implicit conflicts. It turned out that for many universally defined semantics the full range of expressiveness indeed relies on the use rejected arguments and implicit conflicts. This means, there are plenty of AFs which do not possess an equivalent AF which is in addition compact or analytic, respectively. 

Recently, a first study of extension-based realizability w.r.t.\ arbitrary frameworks was presented in \cite{BauS17}. The authors compared the expressive power of several mature semantics in the unrestricted setting. Interestingly, the results reveal an intimate connection between arbitrary and finitely compact AFs in terms of expressiveness. Nevertheless, an in-depth analysis of realizability in the unrestricted setting is still missing. For instance, necessary and sufficient properties for being realizable are not considered so far.

There are only few works which have dealt with labelling-based realizability in the context of Dung-style argumentation frameworks. Dyrkolbotn showed that, as long as additional arguments are allowed any finite set of labellings is \textit{realizable under projection} in case of preferred or semi-stable semantics \cite{Dyr14}. In order to realize a set of labellings $\Ss$ under projection it suffices to come up with an AF $\F$, s.t.\ its set of labellings modulo additional arguments coincide with~$\Ss$. The second work by Linsbichler et al. deals with the standard notion of realizability adapted to labelling-based semantics \cite{LinPS16}. The authors presented an algorithm which returns either ``No'' in case of non-realizability or a witnessing AF $\F$ in the positive case. Remarkably, the algorithm is not restricted to the formalism of abstract argumentation frameworks only. In fact, it can also be used to decide realizability in case of the more general abstract dialectical frameworks as well as various of its sub-classes like \textit{bipolar} ones \cite{ADF,ADF2,BauSt17}.

\section{Realizability and Signatures}

Let us start with the two central concepts of this section, namely \textit{realizability} as well as \textit{signature}. In a nutshell, we say that a certain set $\Ss$ is realizable under the semantics $\sigma$, if there is an AF $\F$ such that its set of $\sigma$-labellings/$\sigma$-extensions coincides with $\Ss$. Collecting all realizable sets defines the concept of a signature. In accordance with the existing literature the main part of this section is devoted to finite realizability for extension-based semantics, i.e.\ signatures which contain set of $\sigma$-extensions of finite AFs only. Realizability w.r.t.\ labelling-based semantics as well as the consideration of infinite AFs will be briefly outlined only. Consider the following general definition of realizability in the context of abstract argumentation.

\begin{definition}\label{def:realizable}
Given a semantics $\sigma: \m{F}\rightarrow 2^{\left(2^\m{U}\right)^n}$ and a set $\m{C}\subseteq\m{F}$. A set $\Ss\subseteq \left(2^\m{U}\right)^n$ is $\sigma$-\textit{realizable} w.r.t.\ \m{C} if there is an AF~$\F\in\m{C}$, s.t.\ $\sigma(\F)=\Ss$.
\end{definition}

\begin{definition} \label{def:signatures}
Given a semantics $\sigma$ and a set $\m{C}\subseteq\m{F}$. The $\sigma$-signature w.r.t.\ $\m{C}$ is defined as \linebreak $\Sigma_{\sigma}^\m{C} = \left\{\sigma(\F) \mid \F\in\m{C} \right\}$. 

\end{definition}
If clear from context or unimportant we simply speak of \textit{signatures} and write $\Sigma$ without mentioning a semantics $\sigma$ or set of AFs \m{C}. Similarly, we say that a certain set is \textit{realizable} instead of $\sigma$-realizable w.r.t.\ \m{C}. Please observe that both concepts are intimately connected via the following relation: for any set $\Ss$ we have, $\Ss$ is realizable if and only if $\Ss\in\Sigma$. Consequently, if $\Ss$ is not contained in $\Sigma$, then there is no framework whose extensions/labellings are exactly~$\Ss$. Hence, instead of searching for witnessing AFs (which might not exist) it is very attractive to find necessary as well as sufficient properties for the containment of a set $\Ss$ to a certain signature locally, i.e.\ by properties of $\Ss$ itself.

\section{Signatures w.r.t.\ Finite AFs}
We start with finite realizability. Instantiating Definitions \ref{def:realizable} and \ref{def:signatures} with $\m{C} = \{\F\in\m{F}\mid \F \text{ finite}\}$ formally capture the notions of realizability as well as signatures relativised to finite AFs. Consider the following definitions.

\begin{definition}\label{def:realizable_finite}
Given a semantics $\sigma: \m{F}\rightarrow 2^{\left(2^\m{U}\right)^n}$. A set $\Ss\subseteq \left(2^\m{U}\right)^n$ is \textit{finitely} $\sigma$-\textit{realizable} if there is an AF~$\F\in\{\F\in\m{F}\mid \F \text{ finite}\}$, s.t.\ $\sigma(\F)=\Ss$.
\end{definition}

\begin{definition} \label{def:signatures_finite}
Given a semantics $\sigma$. The finite $\sigma$-signature is defined as $\left\{\sigma(\F) \mid \F\in\m{F}, \F \text{ finite}\right\}$ abbreviated by $\Sigma_{\sigma}^f$. 
\end{definition}

We proceed with further notational shorthands (adjusted to the extension-based approach) which will be used throughout the whole section.

\begin{definition}[\cite{DunDLW15}] \label{def:extension-set}
Given $\Ss \subseteq 2^\m{U}$,
we use
\begin{itemize}
\item
$\Args_\Ss$ to denote $\bigcup_{S\in\Ss} S$ and $\Card{\Ss}$ for $\card{\Args_\Ss}$,
\item
$\Pairs_\Ss$ to denote $\{ (a,b) \mid 
\exists S\in \Ss: \{a,b\}\subseteq S\}$ and
\item $\dcl(\Ss)$ to denote (the so-called \textit{downward-closure})
$\{S' \subseteq S \mid S \in \Ss\}$ 
\end{itemize}
Furthermore, we say that $\Ss$ is an \textit{extension-set} if $\Card{\Ss}$ is a finite cardinal.
\end{definition}

In order to familiarize the reader with the introduced definitions we give the following example. 

\begin{example} Let $\Ss = \{\{a\},\{a,c\},\{a,b,d\}\}$. Then
\begin{itemize}
\item
$\Args_\Ss = \{a,b,c,d\}$ and $\Card{\Ss} = 4$. This means, $\Ss$ is an extension-set.
\item
$\Pairs_\Ss = \{(a,a),(b,b),(c,c),(d,d),(a,b),(a,c),(a,d),(b,d)\}\ \cup$\\ 
\mbox{} \hspace{+3.3em} $\{(b,a),(c,a),(d,a),(d,b)\}$
\item $\dcl(\Ss) = \{\emptyset,\{a\},\{b\},\{c\},\{d\},\{a,b\},\{a,c\},\{a,d\},\{b,d\},\{a,b,d\}\}$
\end{itemize}

Furthermore, since naive extensions are defined as $\subseteq$-maximal sets and obviously, $\{a\}\subset\{a,c\}$ we deduce that $\Ss$ is not $\nav$-realizable, i.e.\ $\Ss\notin\Sigma^f_{\Ext_{\nav}}$. Regarding complete semantics we obtain $\Ss\in\Sigma^f_{\Ext_{\com}}$ witnessed by the following AF $\F$.

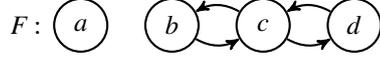
\vspace{3mm}
\begin{tikzpicture}
    
		\node (A1'') at (0,-2.5) [circle,minimum size=0.7cm, thick, draw, label = left:$\F:$]{$a$};
    \node (B1'') at (2.4,-2.5) [circle, minimum size=0.7cm, thick, draw]{$c$};
		\node (A2'') at (1.2,-2.5) [circle, minimum size=0.7cm, thick, draw] {$b$};
    \node (B2'') at (3.6,-2.5) [circle, minimum size=0.7cm,  thick, draw] {$d$};


\draw[->, thick] (B1'') to [thick,bend right] (A2'');
\draw[->, thick] (A2'') to [thick,bend right] (B1'');
\draw[->, thick] (B1'') to [thick,bend right] (B2'');
\draw[->, thick] (B2'') to [thick,bend right] (B1'');

\end{tikzpicture}

\end{example}

In the following we consider the signatures of the extension-based versions of stable, semi-stable, stage, naive, preferred, complete as well as grounded semantics \cite{DunDLW13,DunDLW15}. We provide a bunch of properties where certain subsets of them exactly matches the containment conditions for certain signatures. All properties can be decided by looking on the set in question only.

\subsection{Semantics based on Conflict-freeness} 
Our starting point are semantics based on conflict-free sets. Conflict-free sets by themselves inherited their conflict-freeness to any subset of them. More formally, the downward-closure does not vary the set of conflict-free sets for a given AF. A set possessing this property is called \textit{downward-closed}. Clearly, downward-closedness does not hold in case of admissible sets as well as any other reasonable semantics $\sigma$ where conflict-freeness is just one requirement among others for being a $\sigma$-extension. Take for instance naive semantics. Naive extension are defined as $\subseteq$-maximal conflict-free sets. Consequently, the set of all naive extensions is a $\subseteq$-antichain, i.e.\ any two naive extensions are \textit{incomparable} w.r.t.\ subset relation. This property also applies to many other semantics, such as stable and stage semantics as well as any uniquely defined semantics. However, although incomparability is a necessary condition for many considered semantics it is certainly not sufficient. Consider therefore the following example taken from \cite[Example 1]{DunDLW15}.

\begin{example}
\label{ex:ab_ac_bc}
Consider the $\subseteq$-antichain
$\Ss = \{\{a,b\},\{a,c\},\{b,c\}\}$
and a
semantics $\sigma$ which selects its reasonable positions among the conflict-free sets, i.e.\
$\Ext_{\sigma}(\F)\subseteq\Ext_{\cf}(\F)$ for any AF $\F$.
Now suppose there exists an AF $\F$ with
$\Ext_{\sigma}(\F) = \Ss$. Then 
$\F$ must not contain attacks between $a$ and $b$, $a$ and $c$, and respectively 
$b$ and $c$. This means, $\{a,b,c\}\in\Ext_{\cf}(\F)$. But then $\Ext_{\sigma}(\F)$ typically contains $\{a,b,c\}$. 
\end{example}

There are several ways to define the required property which excludes sets
like $\Ss$ from above. It turned out that in order to characterize conflict-free based semantics like stable, stage and naive semantics a rather strong condition is required, so-called \textit{tightness}. Roughly speaking, if an incomparable set is not tight, then there is a set
$S \in \Ss$ and an argument $a$ not belonging to $S$, s.t.\ for any $s \in S$ we find an other $S'\in \Ss$ with
$a$ and $s$ being members of it. The idea behind the notion of being tight is simply that
if an argument $a$ does not occur in some extension $S$
there must be a reason for that. The most simple reason 
one can think of is that there is a conflict between 
$a$ and some $s\in S$, i.e.\ $a$ and $s$ do not occur jointly in 
any extension-set of $\Ss$ or, in other words,
$(a,s) \notin \Pairs_\Ss$. In a way, this limits the multitude of incomparable elements of an extension-set.

We proceed with the formal definitions.

\begin{definition}[\cite{DunDLW13}]\label{def:tightetal}
Given $\Ss \subseteq 2^\m{U}$. We call $\Ss$ 
\begin{itemize}
\item 
\textit{downward-closed} if $\Ss=\dcl(\Ss)$,
\item
\textit{incomparable} if $\Ss$ is a $\subseteq$-antichain and
\item 
\textit{tight} if for all
$S \in \Ss$ and $a \in \Args_\Ss$
it holds that if
$S \cup \{a\} \notin \Ss$
then there exists an
$s \in S$
such that
$(a,s) \notin \Pairs_\Ss$.
\end{itemize}
\end{definition}

Please observe that for incomparable $\Ss$, the premise of the tightness condition, i.e.\ $S \cup \{a\} \notin \Ss$, is always fulfilled. However, tightness and incomparability are independent of each other, i.e.\ neither tightness implies incomparability or comparability, nor incomparability implies tightness or non-tightness.  
 
\begin{example} \label{ex:independent}
Consider again the
extension-set 
$\Ss = \{\{a,b\},\{a,c\},\{b,c\}\}$ from Example~\ref{ex:ab_ac_bc}.
The set $\Ss$ is incomparable but not tight which can be seen as follows. If setting $S = \{a,b\}$ we observe $S\cup \{c\}\notin\Ss$. Moreover, for any $s\in S$ we find an $S'\in \Ss$, s.t.\ $\{s,c\} = S'$ implying that $(s,c)\in\Pairs_\Ss$. More precisely, if $s=a$, then we have $S' = \{a,c\}$ and similarly, if $s=b$ we find $S' = \{b,c\}$.

Furthermore, it can be checked that $\Ss' = \{\{a,b\},\{a,c\},\{b,d\},\{c,d\}\}$ or $\Ss'' = \Ss \cup \{\{a,b,c\}\}$ are witnessing examples for incomparability and tightness or tightness and comparability, respectively.
\end{example}

Clearly, subsets of incomparable sets are incomparable. Such a kind of inheritance does not hold in case of tight sets (cf.\ $\Ss$ and $\Ss''$ as defined in Example~\ref{ex:independent}). Nevertheless, there are non-trivial tight subsets of any tight set. For instance, in any case the set of all $\subseteq$-maximal elements is tight. Furthermore, if a tight set is even incomparable, then any subset of it is tight too.

In the following we present the main statements only. However, in many cases we provide some short comments indicating how to prove the statement in question. For full proofs we refer the reader to the referenced papers.

\begin{lemma}[\cite{DunDLW15}]
\label{lemma:tight_props}
For a tight extension-set $\Ss \subseteq 2^\m{U}$ we have:
\begin{enumerate}
\item
the $\subseteq$-maximal elements in $\Ss$ form a tight set, and
\item
if $\Ss$ is incomparable then each $\Ss' \subseteq \Ss$ is tight.
\end{enumerate}
\end{lemma}

Note that the second statement of Lemma~\ref{lemma:tight_props} implies that
if the downward-closure of an incomparable extension-set $\Ss$ is tight,
then $\Ss$ itself has to be tight too. 

We proceed with a specific AF and check which properties apply to its different sets of extensions.

\begin{example} \label{ex:tight}
Consider the following AF $\F$.

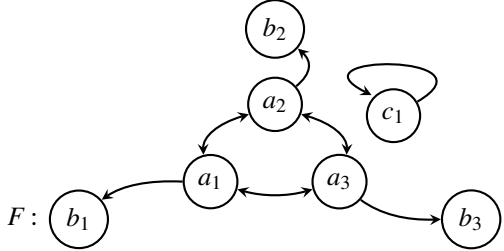
\begin{tikzpicture}[xscale=2,>=stealth]
      \path
			(30:0.9) node[arg] (c1) {$c_1$}
      (-45:.6) node[arg] (x1) {$a_3$}
			+( -30:1) node[arg] (y1) {$b_3$}
			( 90:.6) node[arg] (x2) {$a_2$}
      +( 90:1) node[arg] (y2) {$b_2$}
      (225:.6) node[arg] (x3) {$a_1$}
      +(210:1) node[arg] (y3) {$b_1$}
			+(201:1.3) node (F) {$\F:\ $}
      ;
      \path[<->,thick,bend right]
      (x1) edge (x2)
      (x2) edge (x3)
      (x3) edge (x1)
      ;
      \path[->,bend right, thick]
      (x1) edge (y1)
      (x2) edge (y2)
      (x3) edge (y3)
			;
			\path[->,loop, thick]
			(c1) edge (c1);
    \end{tikzpicture}

Since $c_1$ is self-defeating as well as unattacked we obtain $\Ext_{\stb}(\F) = \emptyset$. Furthermore, $\Ext_{\stg}(\F) = \{\{a_1,b_2,b_3\},
\{a_2,b_1,b_3\}, \{a_3,b_1,b_2\}\}$ and $\Ext_{\nav}(\F) = \Ext_{\stg}(\F) \cup \{\{b_1,b_2,b_3\}\}$. We observe,

\begin{enumerate}
	\item $\Ext_{\stb}(\F), \Ext_{\stg}(\F)$ as well as $\Ext_{\nav}(\F)$ are incomparable,
	\item $\Ext_{\stb}(\F), \Ext_{\stg}(\F)$ as well as $\Ext_{\nav}(\F)$ are tight and additionally, 
	\item $\dcl\left(\Ext_{\nav}(\F)\right)$ and $\dcl\left(\Ext_{\stb}(\F)\right)$ are tight and obviously, 
	\item $\Ext_{\stg}(\F)$ and $\Ext_{\nav}(\F)$ are non-empty.
\end{enumerate}

The first and the last items are not surprising since firstly, all considered semantics satisfy the I-maximality criterion which is just another name for incomparability and secondly, in Theorems~\ref{the:stgfin} and \ref{th:minor} we have already seen that stage extensions are guaranteed for finitary (hence, for finite) frameworks and naive semantics is even universally defined w.r.t.\ the whole class of AFs. This means, incomparability or non-emptiness of the mentioned sets of $\sigma$-extensions do not depend on the specific AF $\F$, but rather apply to any finite AF. Consequently, these properties represent necessary properties regarding realizability. The tightness statements of the second and third items can be checked in a straightforward manner. We now examine that $\dcl\left(\Ext_{\stg}(\F)\right)$ is non-tight. This can be seen as follows: Firstly, $\{b_2,b_3\}\in\dcl\left(\Ext_{\stg}(\F)\right)$. Now, for $b_1$ the premise of Definition~\ref{def:tightetal} is satisfied, i.e.\ $\{b_1,b_2,b_3\}\notin \dcl\left(\Ext_{\stg}(\F)\right)$.  Consequently, since $\{b_1,b_2\},\{b_1,b_3\}\in\dcl\left(\Ext_{\stg}(\F)\right)$ and therefore,
$(b_1,b_2),(b_1,b_3)\in\Pairs_{\dcl\left(\Ext_{\stg}(\F)\right)}$ we deduce the non-tightness of $\dcl\left(\Ext_{\stg}(\F)\right)$. This means, tightness of the downward-closure of a given set can not be a necessary criterion for belonging to the stage signature.

\end{example}

We now present the characterization theorems for conflict-free,
naive, stable as well as stage signatures. It is somehow surprising that only a few simple properties are sufficient to characterize these different signatures.

\begin{theorem}[\cite{DunDLW15}]
\label{the:charcf}
Given a set $\Ss \subseteq 2^\m{U}$, then
\begin{enumerate}
\item $\Ss\in \Sigma^f_{\Ext_{\cf}} \ToT \Ss \text{ is a non-empty, downward-closed, and tight extension-set,}$
\item $\Ss\in \Sigma^f_{\Ext_{\nav}} \ToT \Ss \text{ is a non-empty, incomparable extension-set and } \dcl{(\Ss)} \text{ is tight},$
\item $\Ss\in \Sigma^f_{\Ext_{\stb}} \ToT \Ss \text{ is a incomparable and tight extension-set,}$
\item $\Ss\in \Sigma^f_{\Ext_{\stg}} \ToT \Ss \text{ is a non-empty, incomparable and tight extension-set.}$
\end{enumerate}
\end{theorem}

We mention that a proof of the characterization theorem above requires two directions. Let us fix a certain semantics $\sigma\in\{\cf,\nav,\stb,\stg\}$. The first part is to show that for any finite AF $\F$, $\Ext_{\sigma}(\F)$ satisfies the mentioned properties. Now, for the second part, if a certain extension-set $\Ss$ satisfies the properties in question, then we have to find a finite AF $\F$, s.t.\ $\Ext_{\sigma}(\F) = \Ss$. 

Let us start with the first part. It suffices to consider tightness only since downward-closedness, non-emptiness and incomparability are clear (cf.\ some explanations given in Example~\ref{ex:tight}). It is easy to see that $\Ext_{\cf}(\F)$ is tight because if augmenting a conflict-free set $S$ with a non-conflicting argument $a$ yields a conflicting set, then obviously there has to be at least one element in $s\in S$, s.t.\ $\{a,s\}$ is conflicting. In order to prove that $\dcl\left(\Ext_{\nav}(\F)\right)$ is tight, it suffice to see that $\dcl\left(\Ext_{\nav}(\F)\right) = \Ext_{\cf}(\F)$. Consequently, applying Lemma~\ref{lemma:tight_props} we obtain the tightness of $\Ext_{\nav}(\F)$. Furthermore, with the same lemma, we get that 
every $\Ss \subseteq \Ext_{\nav}(\F)$ is tight. In consideration of $\stb\subseteq\stg\subseteq\nav$ (Proposition~\ref{Pro:semrel}) it follows that $\Ext_{\stb}(\F)$ as well as $\Ext_{\stg}(\F)$ are tight. 

In order to show that the mentioned properties are not only necessary but even sufficient we have to come up with witnessing AFs. Consider therefore the following prototype. 

\begin{definition}[\cite{DunDLW15}]\label{def_canonical1}
Given an extension-set $\Ss$, we define the \textit{canonical argumentation framework} 
for $\Ss$
as
$$
\F_\Ss^\cf = \big( \Args_\Ss, (\Args_\Ss \times \Args_\Ss) \setminus \Pairs_\Ss\big).
$$
\end{definition}
The idea behind 
the framework is simple: we draw a relation between two arguments iff they do not occur jointly in any set $S\in\Ss$.
Consequently, for any~$\Ss$, $F_\Ss^\cf$ is symmetric. Moreover, in any case, it is self-loop-free since $a\in\Args_\Ss$ implies $(a,a)\in\Pairs_\Ss$. Let us consider the following example.

\begin{example} \label{ex:canonical}
Let $\Ss=\{\{a_1,b_2,b_3\},\{a_2,b_1,b_3\},\{a_3,b_1,b_2\},\{b_1,b_2,b_3\}\}$ and consider the corresponding canonical framework $\F_\Ss^\cf$.

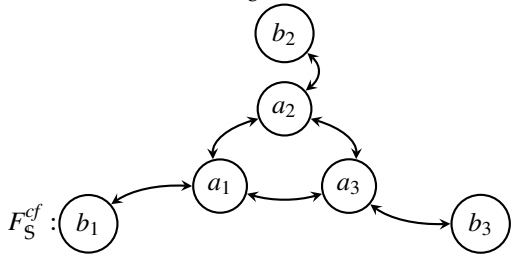
\begin{tikzpicture}[xscale=2,>=stealth]
      \path
		
      (-45:.6) node[arg] (x1) {$a_3$}
			+( -30:1) node[arg] (y1) {$b_3$}
			( 90:.6) node[arg] (x2) {$a_2$}
      +( 90:1) node[arg] (y2) {$b_2$}
      (225:.6) node[arg] (x3) {$a_1$}
      +(210:1) node[arg] (y3) {$b_1$}
			+(201:1.32) node (F) {$\F_\Ss^\cf:$}
      ;
      \path[<->,thick,bend right]
      (x1) edge (x2)
      (x2) edge (x3)
      (x3) edge (x1)
      ;
      \path[<->,bend right, thick]
      (x1) edge (y1)
      (x2) edge (y2)
      (x3) edge (y3)
			;
			
    \end{tikzpicture}

Please note that $\Ss$ is non-empty, incomparable as well as possesses a tight downward-closure (cf.\ Example~\ref{ex:tight}). Furthermore, $\F_\Ss^\cf$ realizes $\Ss$ under the naive semantics, i.e.\ $\Ext_{\nav}\!\left(\F_\Ss^\cf\right)\! = \Ss$. 
\end{example}

The following proposition shows that this is no coincidence.

\begin{proposition}[\cite{DunDLW15}]
\label{prop:naive_real}
For each 
non-empty, incomparable ex\-ten\-sion-set 
$\Ss$, where $\dcl(\Ss)$ is tight,
$\Ext_{\nav}\!\left(\F_\Ss^\cf\right)\! = \Ss$.
\end{proposition}

Moreover, the canonical framework can also be used as witnessing framework in case of conflict-free sets as stated in the following proposition. 

\begin{proposition}[\cite{DunDLW15}]
\label{prop:cf_real} 
For each non-empty, downward-closed and tight ex\-ten\-sion-set~$\Ss$, 
$\Ext_{\cf}\!\left(\F_\Ss^\cf\right)\!= \Ss$.
\end{proposition}

We proceed with stable and stage semantics. In Theorem~\ref{the:charcf} the only difference
between the characterizations of stable and stage signatures is the non-empty requirement for stage semantics. Remember that we are dealing with finite AFs and indeed in case of this restriction stable semantics is the only semantics which does not warrant the existence of extensions (cf.\ Table~\ref{Eval:run}).\footnote{For instance, $\F = (\{a\},\{(a,a)\})$ yields $\Ext_{\stb}(\F) = \emptyset$.} This means, stable semantics is the only semantics which may realize the empty extension-set (which is incomparable and tight too). The final step towards concluding Theorem~\ref{the:charcf} is to find witnessing frameworks for any non-empty, incomparable and tight extension-sets. At first we will show that the canonical framework does not do the job in case of these semantics. More precisely, given a non-empty, incomparable as well as tight extension-set $\Ss$, then the sets of stable as well as stage extensions of the canonical framework $\F_\Ss^\cf$ do not necessarily coincide with $\Ss$. 

\begin{example} \label{ex:canonical2}
Consider again Example~\ref{ex:canonical}. We define $\Tt = \Ss \sm \{\{b_1,b_2,b_3\}\}$. Please note that $\F_\Tt^\cf$ and $\F_\Ss^\cf$ are identical since $\Args_\Ss=\Args_\Tt$ and $\Pairs_\Ss=\Pairs_\Tt$. Furthermore, according to Example~\ref{ex:tight} we have that $\Tt$ is non-empty, incomparable and tight, but 
$\Ext_{\stb}\!\left(\F_\Tt^\cf\right) = \Ext_{\stg}\!\left(\F_\Tt^\cf\right) = \Ext_{\nav}\!\left(\F_\Ss^\cf\right)\! = \Ss \neq \Tt$. 
In order to get rid of the undesired stable as well as stage extension $E=\{b_1,b_2,b_3\}$ we may simply add a new self-defeating argument $e$ to $\F_\Ss^\cf$, s.t.\ $e$ is attacked by all other 
arguments excepting those stemming from $E$. The following framework $\F_\Tt^\stb$ illustrates this idea. Convince yourself that $\Ext_{\stb}\!\left(\F_\Tt^\stb\right) = \Ext_{\stg}\!\left(\F_\Tt^\stb\right) = \Tt$. 

\begin{tikzpicture}[xscale=2,>=stealth]
      \path
		    (-90:0.2) node[arg] (e) {$e$}
      (-45:1.1) node[arg] (x1) {$a_3$}
			+( -30:1) node[arg] (y1) {$b_3$}
			( 90:.6) node[arg] (x2) {$a_2$}
      +( 90:1) node[arg] (y2) {$b_2$}
      (225:1.1) node[arg] (x3) {$a_1$}
      +(210:1) node[arg] (y3) {$b_1$}
			+(201:1.37) node (F) {$\F_\Tt^\stb\!\!:$}
      ;
      \path[<->,thick,bend right]
      (x1) edge (x2)
      (x2) edge (x3)
      (x3) edge (x1)
      ;
      \path[<->,bend right, thick]
      (x1) edge (y1)
      (x2) edge (y2)
      (x3) edge (y3)
			;
			;
      \path[->,bend right, thick]
      (x1) edge (e)
      (x2) edge (e)
      (x3) edge (e)
			;
			\path[->,loop below, thick]
      (e) edge (e)
      ;
			
    \end{tikzpicture}

\end{example} 

The following definition generalizes the construction idea from above to arbitrary many undesired sets. The subsequent proposition states that we have indeed found witnessing examples for non-empty, incomparable and tight extension-sets as required for Theorem~\ref{the:charcf}.

\begin{definition}[\cite{DunDLW15}]
\label{def:st}
Given an extension-set $\Ss$
and
its canonical framework $F^\cf_\Ss=(A^\cf_\Ss , R^\cf_\Ss)$.
Let $\XC = \Ext_\stb\left(F^\cf_\Ss\right) \sm \Ss$ we define \\ $\F^\stb_\Ss = 
\left(A^\cf_\Ss \cup \{\bar E \mid E \in \XC \}, R^\cf_\Ss \cup \{ (\bar E,\bar E),(a,\bar E)\mid E \in \XC, a \in \Args_\Ss \setminus E\}\right).$

\end{definition}

\begin{proposition}[\cite{DunDLW15}]
\label{pro:stb_real}
For each non-empty, incomparable and tight extension-set $\Ss$, 
$\Ext_\stb\!\left(\F^\stb_\Ss\right) = \Ext_\stg\!\left(\F^\stb_\Ss\right) = \Ss$.
\end{proposition}

\subsection{Semantics based on Admissibility} 

Let us turn now to semantics based on admissible sets. In particular, we provide characterization theorems for the finite signatures w.r.t.\ admissible sets as well as preferred and semi-stable semantics. In contrast to semantics based on conflict-free sets where the notion of tightness played a decisive role (cf.\ Theorem~\ref{the:charcf}) we have to introduce a new concept, so-called \textit{conflict-sensitivity}. Conflict-sensitivity is a very basic property in the sense that it is fulfilled by almost all semantics $\sigma$ (or rather, their corresponding sets of $\sigma$-extensions) available in the literature. Furthermore, it is strictly weaker than tightness, i.e.\ tight extension-sets are always conflict-sensitive, but not necessarily vice versa. To explain the difference between these two notions let us consider the following example taken from \cite{DunDLW15}.

\begin{example}\label{example:pref}
Consider the following framework $\F$.

\begin{tikzpicture}[scale=1,>=stealth]
		\path 	node[arg](a'){$a'$}
		        +(-1,-1.5) node(f){$\F:$}
                        ++(0,-1.5) node[arg](b'){$b'$}
                        ++(1.5,1.5)node[arg](a){$a$}
			++(0,-1.5) node[arg](b){$b$}
			++(1.5,0) node[arg](d){$d$}
			++(0,1.5) node[arg](c){$c$}
			++(1.5,-0.75) node[arg](f){$f$} edge [thick,->,loop,distance=5mm,in=50, out=130,looseness=5] (f)
			++(1.5,0) node[arg](e){$e$}
			;
		\path [left,<->, thick]
			(a) edge (c)
			(b) edge (d)
			(c) edge (d)
                        (a) edge (a')
                        (b) edge (b')
			;
		\path [left,->, thick]
			(c) edge (f)
			(d) edge (f)
			(f) edge (e)
			;
		\path [left,->,loop,thick,distance=0.6cm,out=-145, in=145]
			(a') edge (a')
			(b') edge (b')
			;
\end{tikzpicture}

We have $\Ext_\prf(F) = \Ext_\semi(F)= \Ss = \{A,B,C\} = \{\{a,b\},\{a,d,e\},\{b,c,e\}\}$. First, observe that $\Ss$ is not tight. This can be seen as follows: Obviously, $A\cup \{e\}\notin\Ss$, but both $(a,e)$ and $(b,e)$ are contained in $\Pairs_\Ss$ since $\{a,e\}\subseteq B$ and $\{b,e\}\subseteq C$. This means, although $A\cup \{e\}$ is not a reasonable position w.r.t.\ preferred and semi-stable semantics we find witnessing extensions, namely $B$ and $C$, showing that any argument in $A$ is compatible with $e$, i.e.\ they can be accepted together. Please observe that this is not true for any two arguments in $A$ and $B$ or $A$ and $C$, respectively. For instance, $b,d\in A\cup B$, but $(b,d)\notin \Pairs_\Ss$ as well as $a,c\in A\cup C$, but $(a,c)\notin \Pairs_\Ss$. Furthermore, the same applies to $B$ and $C$, since $c,d\in B\cup C$ and $(c,d)\notin \Pairs_\Ss$. 
\end{example}

The following definition precisely formalizes the observed property of the AF~$\F$ presented in the example above. 

\begin{definition}[\cite{DunDLW15}]\label{def:admcl}
A set 
$\Ss \subseteq 2^\m{U}$ is called
\textit{conflict-sensitive}
if for each $A,B\in\Ss$ such that $A\cup B \notin \Ss$
it holds that $\exists a,b \in A\cup B : (a,b) \notin \Pairs_\Ss$. 
\end{definition}

As the name suggests, the property checks whether the absence of the union of any pair of extensions in an extension-set $\Ss$ is justified by a conflict indicated by $\Ss$. Note that 
for $a,b\in A$ (likewise $a,b\in B$), $(a,b)\in\Pairs_\Ss$ holds by definition.
Thus the property of conflict-sensitivity is determined by arguments $a\in A\setminus B$, 
$b\in B\setminus A$, for $A,B\in\Ss$. As already indicated tightness implies conflict-sensitivity as stated in the following lemma.

\begin{lemma}[\cite{DunDLW15}]
\label{lemma:tight->admclosed}
Every tight extension-set 
is also conflict-sensitive. 
\end{lemma}

Similarly to Lemma~\ref{lemma:tight_props} one may show that the set of all $\subseteq$-maximal elements of a conflict-sensitive set is conflict-sensitive too. Moreover, if the initial set is incomparable in addition, then even any subset of it is conflict-sensitive. Furthermore, in contrast to tight extension-sets it is possible to add the empty set to a conflict-sensitive set without loosing conflict-sensitivity.\footnote{Note that any one-element extension-set $\Ss\neq\{\emptyset\}$ is tight, whereas $\Ss\cup\{\emptyset\}$ is not.}
\pagebreak
\begin{lemma}[\cite{DunDLW15}]
\label{lemma:adm_props}
For a conflict-sensitive extension set $\Ss~\subseteq~2^\m{U}$,
\begin{enumerate}
\item
the $\subseteq$-maximal elements in $\Ss$ form a conflict-sensitive set, 
\item
if $\Ss$ is incomparable then each $\Ss' \subseteq \Ss$ is conflict-sensitive, and
\item
$\Ss \cup \{\emptyset\}$ is conflict-sensitive.
\end{enumerate}
\end{lemma}

Having conflict-sensitivity at hand, we are now ready to present characterization theorems for the signatures w.r.t.\ admissible sets as well as preferred and semi-stable semantics. Interestingly, it turns out that preferred and semi-stable semantics are equally expressive in case of finite AFs, i.e.\ $\Sigma^f_{\Ext_{\prf}} = \Sigma^f_{\Ext_{\semi}}$.

\begin{theorem}[\cite{DunDLW15}]
\label{the:charadm}
Given a set $\Ss \subseteq 2^\m{U}$, then
\begin{enumerate}
\item \hspace{-0.5em} $\Ss\in \Sigma^f_{\Ext_{\adm}}\! \ToT \Ss \text{ is a conflict-sensitive extension set containing }\emptyset,$
\item \hspace{-0.5em} $\Ss\in \Sigma^f_{\Ext_{\prf}}\! \ToT \Ss \text{ is a non-empty, incomparable and conflict-sensitive extension set,}$
\item \hspace{-0.5em} $\Ss\in \Sigma^f_{\Ext_{\semi}}\! \ToT \Ss \text{ is a non-empty, incomparable and conflict-sensitive extension set.}$

\end{enumerate}
\end{theorem}

Let us first argue that the mentioned properties are necessary conditions for being in the corresponding signature. For admissible sets it suffices to recall the following two facts: First, the empty set is admissible by definition; and second, if the union of two admissible sets is conflict-free, then the union is admissible too. In other words, if the union fails to be admissible, then there has to be a conflict proving the conflict-sensitivity of any set of admissible sets. Now, for preferred and semi-stable semantics. Non-emptiness is due to the already shown universal definedness of both semantics in case of finite AFs (cf.\ Table~\ref{Eval:run}). Moreover, incomparability is clear since both semantics satisfy the I-maximality criterion (cf.\ \cite{BarG07}) Finally, conflict-sensitivity of sets of admissible sets transfer to sets of preferred extensions via statement~1 of Lemma~\ref{lemma:adm_props} and therefore also to sets of semi-stable extensions via statement~2 of Lemma~\ref{lemma:adm_props} and the fact that $\semi\subseteq\prf$ (Proposition~\ref{Pro:semrel}).

In order to show that the mentioned properties are not only necessary but even sufficient we have to come up with witnessing AFs. In contrast to conflict-free based semantics we have to find AFs which encode the central notion of admissibility. Please note that the already introduced canonical frameworks $\F^\cf_\Ss$ as well as $\F^\stb_\Ss$  (cf.\ Definitions~\ref{def_canonical1}~and~\ref{def:st}) do not comply with the requirements. Consider therefore the following example.  

\begin{example} \label{ex:newcanneeded} Let us consider again the non-empty, incomparable as well as tight set $\Tt=$ \linebreak $\{\{a_1,b_2,b_3\},\{a_2,b_1,b_3\},\{a_3,b_1,b_2\}\}$ together with its corresponding canonical framework $\F_\Tt^\stb$ as presented in Example~\ref{ex:canonical2}. Due to Lemma~\ref{lemma:tight->admclosed} we have that any tight extension-set is even conflict-sensitive and thus, $\Tt$ satisfies the necessary requirements of Theorem~\ref{the:charadm}. Inspecting the canonical framework reveals that $\Ext_{\prf}\left(\F_\Tt^\stb\right) = \Tt\cup\{\{b_1,b_2,b_3\}\}\neq \Tt$. Although, $\Ext_{\semi}\left(\F_\Tt^\stb\right) = \Tt$ one may easily check that non-empty, incomparable as well as conflict-sensitive set $\Ss = \{\{a,b\},\{a,d,e\},\{b,c,e\}\}$ mentioned in Example~\ref{example:pref} shows that this equality does not hold in general. Likewise, one may prove that the framework $\F^\cf_\Ss$ is not appropriated as a witnessing prototype for semi-stable as well as preferred semantics.
\end{example}

It turned out that suitable canonical AFs can be built by means of so-called \textit{defense-formulae} as introduced in the following definition.

\begin{definition}[\cite{DunDLW15}]
\label{def_defense_formula}
Given an extension-set $\Ss$, the \textit{defense-formula} 
$\Def_a^\Ss$ of an argument $a\in\Args_\Ss$ in $\Ss$ is defined as:
$$
\Def_a^\Ss = \bigvee_{S \in \Ss, \text{s.t.} a \in S}  \bigwedge_{s \in S \setminus \{a\}} s.
$$
$\Def_a^\Ss$ given as (a logically equivalent) CNF
is called \textit{CNF-defense-formula} $\CDef_a^\Ss$ of $a$~in~$\Ss$.
\end{definition}

The main idea of the formula $\Def_a^\Ss$ is to describe the conditions for the argument $a$ being in an extension. 
Note that the variables coincide with the arguments. If $\Ss$ amounts to a set of admissible extensions, then each 
disjunct represents a set of arguments $A$ which allows $a$ to join in the sense that $A\cup\{a\}$ is a reasonable position w.r.t.\ admissible semantics. Put it differently, propositional models of $\Def_a^\Ss\wedge a$ represent (if considered as set of atoms) supersets of certain reasonable position. Please not that a defense-formula $\Def^\Ss_a$ is tautological if and only if $\{a\} \in \Ss$. We proceed with an example.

\begin{example}\label{ex_defense_formulas}
Consider again the non-empty, incomparable as well as conflict-sensitive set $\Ss =$ \linebreak $\{\{a,b\},\{a,d,e\},\{b,c,e\}\}$ stemming from Example~\ref{example:pref}. We obtain the following defense-formulae together with their corresponding CNF-defense-formulae (written in clause form).

\begin{itemize}
	\item $\Def^\Ss_a = b \vee (d\wedge e) \equiv (b \vee d) \wedge (b \vee e)$ and $\CDef^\Ss_a = \{\{b,d\},\{b,e\}\}$
	\item $\Def^\Ss_b = a \vee (c\wedge e) \equiv (a \vee c) \wedge (a \vee e)$ and $\CDef^\Ss_b = \{\{a,c\},\{a,e\}\}$
	\item $\Def^\Ss_c = b\wedge e$ and $\CDef^\Ss_c = \{\{b,e\}\}$
	\item $\Def^\Ss_d = a\wedge e$ and $\CDef^\Ss_d = \{\{a,e\}\}$
	\item $\Def^\Ss_e = (a\wedge d) \vee (b\wedge c) \equiv (a\vee b) \wedge (d\vee b) \wedge (a\vee c) \wedge (d\vee c)$ and $\CDef^\Ss_d = \{\{a,b\},\{a,c\},\{b,d\},\{c,d\}\}$
\end{itemize}

\end{example}

One simple idea for the realization of a certain set $\Ss$ under admissible semantics is the following two-step procedure. In the first step, we construct a framework $\F$ which maintains all elements of $\Ss$ as conflict-free sets. This can be done via the the canonical framework $\F^\cf_\Ss$. In the second step, we augment the initial framework $\F^\cf_\Ss$, s.t.\ only elements in $\Ss$ become admissible. The second step can be realized via adding a certain amount of additional arguments. More precisely, for any argument $a\in\Args_\Ss$ we add $n$ self-conflicting arguments $\alpha_{aC_1},...,\alpha_{aC_n}$ if $\card{\CDef^\Ss_a} = \card{\{C_1,...,C_n\}} = n$. Then, for any $i\in\{1,..,n\}$, $\alpha_{aC_i}$ attacks $a$ and is in turn attacked by any argument in $C_i$. Consider therefore the following example.

\begin{example}\label{ex_adm_construction}

Again consider the extension-set $\Ss = \{\{a,b\},\{a,d,e\},\{b,c,e\}\}$ and its corresponding CNF-defense-formulae as presented in Example~\ref{ex_defense_formulas}. In accordance with the above mentioned two-step procedure we obtain the dashed AF $\F^\cf_\Ss$ first. Then, in view of the CNF-defense-formulae we have to add 10 additional self-defeating arguments which attacks their corresponding argument. This intermediate step is depicted below.

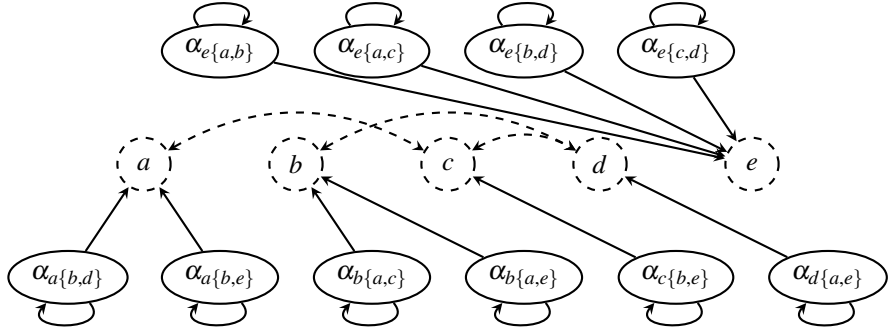
\begin{tikzpicture}[scale=1,>=stealth]
		\path 	node[arg,dashed](a){$a$}
			++(2.0,0) node[arg,dashed](b){$b$}
			++(2.0,0) node[arg,dashed](c){$c$}
			++(2.0,0) node[arg,dashed](d){$d$}
			++(2.0,0) node[arg,dashed](e){$e$}
			;
		\path 	(-1,-1.5)node[arg,ellipse,inner sep=2pt](alphaa1){$\alpha_{a\{b,d\}}$}
			++(2,0) node[arg,ellipse,inner sep=2pt](alphaa2){$\alpha_{a\{b,e\}}$}
			++(2,0) node[arg,ellipse,inner sep=2pt](alphab1){$\alpha_{b\{a,c\}}$}
			++(2,0) node[arg,ellipse,inner sep=2pt](alphab2){$\alpha_{b\{a,e\}}$}
			++(2,0) node[arg,ellipse,inner sep=2pt](alphac1){$\alpha_{c\{b,e\}}$}
			++(2,0) node[arg,ellipse,inner sep=2pt](alphad1){$\alpha_{d\{a,e\}}$}
			++(-2,3) node[arg,ellipse,inner sep=2pt](alphae1){$\alpha_{e\{c,d\}}$}
			++(-2,0) node[arg,ellipse,inner sep=2pt](alphae2){$\alpha_{e\{b,d\}}$}
			++(-2,0) node[arg,ellipse,inner sep=2pt](alphae3){$\alpha_{e\{a,c\}}$}
			++(-2,0) node[arg,ellipse,inner sep=2pt](alphae4){$\alpha_{e\{a,b\}}$}
			;
		\path [bend left,<->, thick, dashed]
			(a) edge (c)
			(c) edge (d)
			(b) edge (d)
			;
			
		\path [->, thick] 
			(alphaa1) edge (a)
			(alphaa2) edge (a)
			(alphab1) edge (b)
			(alphab2) edge (b)
			(alphac1) edge (c)
			(alphad1) edge (d)
			(alphae1) edge (e)
			(alphae2) edge (e)
			(alphae3) edge (e)
			(alphae4) edge (e)
			 ;	
		\path [->, thick] 
			
			 ;
		\draw[loop right,thick, distance=0.5cm,out=-50, in=-130,->] (alphaa1) edge (alphaa1);
		\draw[loop right,thick, distance=0.5cm,out=-50, in=-130,->] (alphaa2) edge (alphaa2);
		\draw[loop right,thick, distance=0.5cm,out=-50, in=-130,->] (alphab1) edge (alphab1);
		\draw[loop right,thick, distance=0.5cm,out=-50, in=-130,->] (alphab2) edge (alphab2);
		\draw[loop right,thick, distance=0.5cm,out=-50, in=-130,->] (alphac1) edge (alphac1);
		\draw[loop right,thick, distance=0.5cm,out=-50, in=-130,->] (alphad1) edge (alphad1);
		\draw[loop right,thick, distance=0.5cm,out=-230, in=-310,->] (alphae1) edge (alphae1);
		\draw[loop right,thick, distance=0.5cm,out=-230, in=-310,->] (alphae2) edge (alphae2);
		\draw[loop right,thick, distance=0.5cm,out=-230, in=-310,->] (alphae3) edge (alphae3);
		\draw[loop right,thick, distance=0.5cm,out=-230, in=-310,->] (alphae4) edge (alphae4);
\end{tikzpicture}

Let us consider the set $\{a,b\}\in\Ss$. In order for $\{a,b\}$ to be admissible we have to add counter-attacks for the arguments $\alpha_{a\{b,d\}}$, $\alpha_{a\{b,e\}}$, $\alpha_{b\{a,c\}}$ and $\alpha_{b\{a,e\}}$. For instance, $\alpha_{a\{b,d\}}$ is attacked by $b$ and $d$ and so forth. The following figure (built on top of the previous one) depicts resulting counter-attacks for the mentioned 4 arguments highlighted as densely dotted edges. For the sake of clarity we do not perform this construction for the remaining arguments.

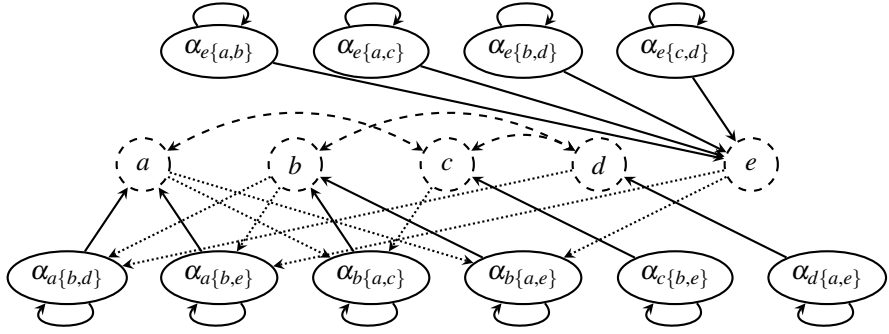
\begin{tikzpicture}[scale=1,>=stealth]
		\path 	node[arg,dashed](a){$a$}
			++(2.0,0) node[arg,dashed](b){$b$}
			++(2.0,0) node[arg,dashed](c){$c$}
			++(2.0,0) node[arg,dashed](d){$d$}
			++(2.0,0) node[arg,dashed](e){$e$}
			;
		\path 	(-1,-1.5)node[arg,ellipse,inner sep=2pt](alphaa1){$\alpha_{a\{b,d\}}$}
			++(2,0) node[arg,ellipse,inner sep=2pt](alphaa2){$\alpha_{a\{b,e\}}$}
			++(2,0) node[arg,ellipse,inner sep=2pt](alphab1){$\alpha_{b\{a,c\}}$}
			++(2,0) node[arg,ellipse,inner sep=2pt](alphab2){$\alpha_{b\{a,e\}}$}
			++(2,0) node[arg,ellipse,inner sep=2pt](alphac1){$\alpha_{c\{b,e\}}$}
			++(2,0) node[arg,ellipse,inner sep=2pt](alphad1){$\alpha_{d\{a,e\}}$}
			++(-2,3) node[arg,ellipse,inner sep=2pt](alphae1){$\alpha_{e\{c,d\}}$}
			++(-2,0) node[arg,ellipse,inner sep=2pt](alphae2){$\alpha_{e\{b,d\}}$}
			++(-2,0) node[arg,ellipse,inner sep=2pt](alphae3){$\alpha_{e\{a,c\}}$}
			++(-2,0) node[arg,ellipse,inner sep=2pt](alphae4){$\alpha_{e\{a,b\}}$}
			;
		\path [bend left,<->, thick, dashed]
			(a) edge (c)
			(c) edge (d)
			(b) edge (d)
			;
			
		\path [->, thick] 
			(alphaa1) edge (a)
			(alphaa2) edge (a)
			(alphab1) edge (b)
			(alphab2) edge (b)
			(alphac1) edge (c)
			(alphad1) edge (d)
			(alphae1) edge (e)
			(alphae2) edge (e)
			(alphae3) edge (e)
			(alphae4) edge (e)
			 ;	
		\path [->, thick,densely dotted] 
			(b) edge (alphaa1)
			(d) edge (alphaa1)
			(b) edge (alphaa2)
			(e) edge (alphaa2)
			(a) edge (alphab1)
			(c) edge (alphab1)
			(a) edge (alphab2)
			(e) edge (alphab2)
			
			 ;
		\draw[loop right,thick, distance=0.5cm,out=-50, in=-130,->] (alphaa1) edge (alphaa1);
		\draw[loop right,thick, distance=0.5cm,out=-50, in=-130,->] (alphaa2) edge (alphaa2);
		\draw[loop right,thick, distance=0.5cm,out=-50, in=-130,->] (alphab1) edge (alphab1);
		\draw[loop right,thick, distance=0.5cm,out=-50, in=-130,->] (alphab2) edge (alphab2);
		\draw[loop right,thick, distance=0.5cm,out=-50, in=-130,->] (alphac1) edge (alphac1);
		\draw[loop right,thick, distance=0.5cm,out=-50, in=-130,->] (alphad1) edge (alphad1);
		\draw[loop right,thick, distance=0.5cm,out=-230, in=-310,->] (alphae1) edge (alphae1);
		\draw[loop right,thick, distance=0.5cm,out=-230, in=-310,->] (alphae2) edge (alphae2);
		\draw[loop right,thick, distance=0.5cm,out=-230, in=-310,->] (alphae3) edge (alphae3);
		\draw[loop right,thick, distance=0.5cm,out=-230, in=-310,->] (alphae4) edge (alphae4);
\end{tikzpicture}

\end{example}

The following definition precisely formalizes the mentioned two-step procedure.

\begin{definition}[\cite{DunDLW15}]
\label{def_adm_constr}

Given an extension-set $\Ss$, the \textit{canonical defense-ar\-gu\-men\-ta\-tion-frame\-work} $\F^{\ddef}_\Ss = (A^\ddef_\Ss , R^\ddef_\Ss)$ 
extends the canonical AF
$\F_\Ss^\cf = (\Args_\Ss, R_\Ss^\cf)$ as follows:

\begin{align*}
A^\ddef_\Ss = &\ \Args_\Ss \cup \bigcup_{a \in \Args_\Ss} 
\left\{\alpha_{a\gamma} \mid \gamma \in \CDef_a^\Ss\right\} \mbox{, and}\\
R^\ddef_\Ss = &\ R^\cf_\Ss \cup \bigcup_{a \in \Args_\Ss} 
\left\{
(b,\alpha_{a\gamma}),
(\alpha_{a\gamma},\alpha_{a\gamma}), 
(\alpha_{a\gamma},a)
\mid \gamma \in \CDef_a^\Ss, b \in  \gamma\right\} \text{.}
\end{align*}

\end{definition}

The subsequent proposition shows that not only all elements in $\Ss$ become admissible in the constructed AF $F^{\ddef}_\Ss$, but rather that the set of admissible sets of $F^{\ddef}_\Ss$ exactly coincides with $\Ss$ given that $\Ss$ is conflict-sensitive as well as contains the empty set.

\begin{proposition}[\cite{DunDLW15}]
\label{theo_adm_signature}
For each conflict-sensitive extension set $\Ss$ where $\emptyset \in \Ss$, 
it holds that
$\Ext_{\adm}\left(\F^{\ddef}_\Ss\right) = \Ss$.
\end{proposition}

Interestingly, we may even use the canonical defense-AF to show that any non-empty, incomparable and conflict-sensitive ex\-ten\-sion-set $\Ss$ can be realized under the preferred semantics. This can be seen as follows: First, via Lemma~\ref{lemma:adm_props} we obtain the conflict-sensitivity of $\Ss\cup\{\emptyset\}$ since $\Ss$ is assumed to be conflict-sensitive. Consequently, using Proposition~\ref{theo_adm_signature} we obtain $\Ext_{\adm}\!\left(\F^{\ddef}_{\Ss\cup\{\emptyset\}}\right) = \Ss\cup\{\emptyset\}$. Since $\F^{\ddef}_{\Ss} = \F^{\ddef}_{\Ss\cup\{\emptyset\}} $ and due to the incomparability of $\Ss$, we have $\Ext_{\prf}\!\left(\F^{\ddef}_{\Ss}\right) = \Ss$ as stated in the following proposition. 

\begin{proposition}[\cite{DunDLW15}]
\label{theo_adm_signature}
For each non-empty, incomparable and conflict-sensitive ex\-ten\-sion-set $\Ss$, 
it holds that
$\Ext_{\prf}\!\left(\F^{\ddef}_\Ss\right) = \Ss$.
\end{proposition}

Furthermore, due to a translation result by Dvo{\v{r}}{\'{a}}k and Woltran we obtain that any non-empty, incomparable and conflict-sensitive extension-set $\Ss$ can be realized under semi-stable semantics too. More precisely, in \cite{DvoW11} it is shown that for any AF $\F$ exists an AF $\F'$, s.t. \linebreak $\Ext_{\prf}\!\left(\F\right) = \Ext_{\semi}\!\left(\F'\right)$.

\begin{proposition}[\cite{DunDLW15}]
\label{theo_sem_signature}
Each non-empty, incomparable and conflict-sensitive ex\-ten\-sion-set $\Ss$ is $\semi$-realizable.
\end{proposition}

\subsection{Uniquely Defined Semantics} 
Let us finally turn to grounded, ideal and eager semantics. Remember that all mentioned semantics warrants the existence of exactly one extension given that the frameworks in question are finite (cf.\ Table~\ref{Eval:run}). Furthermore, it is hardly surprising that this property is even sufficient for being in the corresponding signature, since any one-element extension-set $\Ss = \{E\}$ can be realized via $\F_{E} = (E,\emptyset)$. In particular, we obtain that all three semantics are equally expressive.

\begin{theorem}[\cite{DunSLW16}]
\label{the:charid}
Given a set $\Ss \subseteq 2^\m{U}$, then
\begin{enumerate}
\item $\Ss\in \Sigma^f_{\Ext_{\grd}} \ToT \Ss \text{ is an extension-set with} \card{\Ss} = 1,$
\item $\Ss\in \Sigma^f_{\Ext_{\id}}  \ToT \Ss \text{ is an extension-set with} \card{\Ss} = 1$ and
\item $\Ss\in \Sigma^f_{\Ext_{\eag}}  \ToT \Ss \text{ is an extension-set with} \card{\Ss} = 1.$

\end{enumerate}
\end{theorem}

\subsection{Summary of Results and Further Remarks}

In this subsection we provide a comprehensive overview of characterization results w.r.t.\ extension-based realizability in case of finite AFs. The following table collect and combine the results of the previous three 
 subsections. The table has to be interpreted as follows: Consider a certain column $\sigma$. Then, the entries ``$\times$'' in rows $r_1$,...,$r_n$ indicate that for any extension-set $\Ss$, $\Ss\in\Sigma^f_{\Ext_{\sigma}} \ToT r_1,...,r_n$. Moreover, an entry ``$\to$'' in row $r$ reflects the fact that the collection of the properties $r_1,...,r_n$ imply property $r$.

 \begin{table}[H]
	\centering
	\begin{tikzpicture}[scale=0.88] 
	  \tikzstyle{literal}=[text centered]
	  \draw[help lines] (1,-5) grid (11,4);
	  \draw[help lines] (-3,0) -- (1,0);
		\draw[help lines] (-3,1) -- (1,1);
	  \draw[help lines] (-3,2) -- (1,2);
	  \draw[help lines] (-3,3) -- (1,3);
		\draw[help lines] (-3,-1) -- (1,-1);
	  \draw[help lines] (-3,-2) -- (1,-2);
		\draw[help lines] (-3,-3) -- (1,-3);
	  \draw[help lines] (-3,-4) -- (1,-4);
		\draw[help lines] (-3,-5) -- (1,-5);
	  \draw[help lines] (-3,-5) -- (-3,3);

	  \node[literal] (S1) at (1.5,3.5)  {\standardbox{\cf}};
	  \node[literal] (S1) at (2.5,3.5)  {\standardbox{\nav}};
	  \node[literal] (S1) at (3.5,3.5)  {\standardbox{\stb}};
	  \node[literal] (S1) at (4.5,3.5)  {\standardbox{\stg}};
	  \node[literal] (S1) at (5.5,3.5)  {\standardbox{\adm}};
	  \node[literal] (S1) at (6.5,3.5)  {\standardbox{\prf}};
	  \node[literal] (S1) at (7.5,3.5)  {\standardbox{\semi}};
	  \node[literal] (S1) at (8.5,3.5)  {\standardbox{\grd}};
	  \node[literal] (S1) at (9.5,3.5)  {\standardbox{\id}};
	  \node[literal] (S1) at (10.5,3.5)  {\standardbox{\eag}};

	  \node[literal] (S1) at (1.5,2.5)  {$\times$};
	  \node[literal] (S1) at (2.5,2.5)  {$\times$};
	  \node[literal] (S1) at (3.5,2.5)  {};
	  \node[literal] (S1) at (4.5,2.5)  {$\times$};
	  \node[literal] (S1) at (5.5,2.5)  {$\to$};
	  \node[literal] (S1) at (6.5,2.5)  {$\times$};
	  \node[literal] (S1) at (7.5,2.5)  {$\times$};
	  \node[literal] (S1) at (8.5,2.5)  {$\to$};
	  \node[literal] (S1) at (9.5,2.5)  {$\to$};
	  \node[literal] (S1) at (10.5,2.5)  {$\to$};

	  \node[literal] (S1) at (1.5,1.5)  {$\to$};
	  \node[literal] (S1) at (2.5,1.5)  {};
	  \node[literal] (S1) at (3.5,1.5)  {};
	  \node[literal] (S1) at (4.5,1.5)  {};
	  \node[literal] (S1) at (5.5,1.5)  {$\times$};
	  \node[literal] (S1) at (6.5,1.5)  {};
	  \node[literal] (S1) at (7.5,1.5)  {};
	  \node[literal] (S1) at (8.5,1.5)  {};
	  \node[literal] (S1) at (9.5,1.5)  {};
	  \node[literal] (S1) at (10.5,1.5)  {};
	  \node[literal] (S1) at (11.5,1.5)  {};
	  \node[literal] (S1) at (12.5,1.5)  {};
	  \node[literal] (S1) at (13.5,1.5)  {};
	  \node[literal] (S1) at (1.5,0.5)  {};
	  \node[literal] (S1) at (2.5,0.5)  {};
	  \node[literal] (S1) at (3.5,0.5)  {};
	  \node[literal] (S1) at (4.5,0.5)  {};
	  \node[literal] (S1) at (5.5,0.5)  {};
	  \node[literal] (S1) at (6.5,0.5)  {};
	  \node[literal] (S1) at (7.5,0.5)  {};
	  \node[literal] (S1) at (8.5,0.5)  {$\times$};
	  \node[literal] (S1) at (9.5,0.5)  {$\times$};
	  \node[literal] (S1) at (10.5,0.5)  {$\times$};
		
		\node[literal] (S1) at (1.5,-0.5)  {};
	  \node[literal] (S1) at (2.5,-0.5)  {$\times$};
	  \node[literal] (S1) at (3.5,-0.5)  {};
	  \node[literal] (S1) at (4.5,-0.5)  {};
	  \node[literal] (S1) at (5.5,-0.5)  {};
	  \node[literal] (S1) at (6.5,-0.5)  {};
	  \node[literal] (S1) at (7.5,-0.5)  {};
	  \node[literal] (S1) at (8.5,-0.5)  {$\to$};
	  \node[literal] (S1) at (9.5,-0.5)  {$\to$};
	  \node[literal] (S1) at (10.5,-0.5)  {$\to$};
		
		\node[literal] (S1) at (1.5,-1.5)  {};
	  \node[literal] (S1) at (2.5,-1.5)  {$\times$};
	  \node[literal] (S1) at (3.5,-1.5)  {$\times$};
	  \node[literal] (S1) at (4.5,-1.5)  {$\times$};
	  \node[literal] (S1) at (5.5,-1.5)  {};
	  \node[literal] (S1) at (6.5,-1.5)  {$\times$};
	  \node[literal] (S1) at (7.5,-1.5)  {$\times$};
	  \node[literal] (S1) at (8.5,-1.5)  {$\to$};
	  \node[literal] (S1) at (9.5,-1.5)  {$\to$};
	  \node[literal] (S1) at (10.5,-1.5)  {$\to$};
		
		\node[literal] (S1) at (1.5,-2.5)  {$\times$};
	  \node[literal] (S1) at (2.5,-2.5)  {$\to$};
	  \node[literal] (S1) at (3.5,-2.5)  {$\times$};
	  \node[literal] (S1) at (4.5,-2.5)  {$\times$};
	  \node[literal] (S1) at (5.5,-2.5)  {};
	  \node[literal] (S1) at (6.5,-2.5)  {};
	  \node[literal] (S1) at (7.5,-2.5)  {};
	  \node[literal] (S1) at (8.5,-2.5)  {$\to$};
	  \node[literal] (S1) at (9.5,-2.5)  {$\to$};
	  \node[literal] (S1) at (10.5,-2.5)  {$\to$};
		
		\node[literal] (S1) at (1.5,-3.5)  {$\to$};
	  \node[literal] (S1) at (2.5,-3.5)  {$\to$};
	  \node[literal] (S1) at (3.5,-3.5)  {$\to$};
	  \node[literal] (S1) at (4.5,-3.5)  {$\to$};
	  \node[literal] (S1) at (5.5,-3.5)  {$\times$};
	  \node[literal] (S1) at (6.5,-3.5)  {$\times$};
	  \node[literal] (S1) at (7.5,-3.5)  {$\times$};
	  \node[literal] (S1) at (8.5,-3.5)  {$\to$};
	  \node[literal] (S1) at (9.5,-3.5)  {$\to$};
	  \node[literal] (S1) at (10.5,-3.5)  {$\to$};
		
		\node[literal] (S1) at (1.5,-4.5)  {$\times$};
	  \node[literal] (S1) at (2.5,-4.5)  {};
	  \node[literal] (S1) at (3.5,-4.5)  {};
	  \node[literal] (S1) at (4.5,-4.5)  {};
	  \node[literal] (S1) at (5.5,-4.5)  {};
	  \node[literal] (S1) at (6.5,-4.5)  {};
	  \node[literal] (S1) at (7.5,-4.5)  {};
	  \node[literal] (S1) at (8.5,-4.5)  {};
	  \node[literal] (S1) at (9.5,-4.5)  {};
	  \node[literal] (S1) at (10.5,-4.5)  {};

	  \node[literal] (S1) at (-1,2.5)  {$\Ss\neq\emptyset$};
	  \node[literal] (S1) at (-1,1.5)  {$\emptyset\in\Ss$};
	  \node[literal] (S1) at (-1,0.5)  {$\card{\Ss}=1$};
		\node[literal] (S1) at (-1,-0.5)  {$\dcl(\Ss)$ is tight};
		\node[literal] (S1) at (-1,-1.5)  {$\Ss$ is incomparable};
		\node[literal] (S1) at (-1,-2.5)  {$\Ss$ is tight};
	  \node[literal] (S1) at (-1,-3.5)  {$\Ss$ is conflict-sensitive};
		\node[literal] (S1) at (-1,-4.5)  {$\dcl(\Ss)=\Ss$};
	\end{tikzpicture}   
	\caption{Characterizing Properties for Realizable Extension-sets}
	\label{fig:sumcharacrel}
      \end{table}

Remember that the decision whether a certain extension-set $\Ss$ is realizable can not be done via brute force (i.e., enumerating AFs and checking whether their extensions coincide with $\Ss$) since there are no a priori bounds on the number of required arguments. Consequently, the results depicted in Table~\ref{fig:sumcharacrel} put us in a very good position since now, the question of realizability can be decided locally, i.e.\ by inspecting the set in question itself. Moreover, all mentioned properties can checked in polynomial time w.r.t.\ the size of the extensions as shown in \cite[Theorem 6]{DunDLW15}. For the majority of the properties tractability is immediately apparent. The only exception is tightness of the downward-closure of a given extension-set $\Ss$ since its size is not polynomially bounded in the size of $\Ss$ (cf.\ \cite[Proposition 12]{DunDLW15} for a way out of this problem).

By inspecting the respective properties as depicted in Table~\ref{fig:sumcharacrel}, we can immediately 
put the signatures of different semantics in relation to each other. The following theorem includes the signature w.r.t.\ complete semantics in addition. The reason why we did not included complete semantics in our considerations is simply that a precise characterization of the complete signature is still an open problem. Nevertheless, certain necessary properties are already found \cite[Proposition 4]{DunDLW15} justifying items 3 and 4 of the following theorem.

\begin{theorem}[\cite{DunDLW15}]\label{thm:rel}
The following relations hold 

\begin{enumerate}
	\item $\Sigma^f_{\Ext_\nav} \subset \Sigma^f_{\Ext_\stg} \subset \Sigma^f_{\Ext_\semi} = \Sigma^f_{\Ext_\prf}$,
	\item $\Sigma^f_{\Ext_\stb} = \Sigma^f_{\Ext_\stg} \cup \{\emptyset\}$,
	\item $\Sigma^f_{\Ext_\cf} \subset \Sigma^f_{\Ext_\adm} \subset \Sigma^f_{\Ext_\com}$,
	\item $\Sigma^f_{\Ext_\sigma} \subset \Sigma^f_{\Ext_\tau}$ where $\sigma\in\{\grd,\id,\eag\}$, $\tau\in\{\nav,\stb,\stg,\prf,\semi,\com\}$ and
	\item $\left\{\Ss \cup \{\emptyset\} \mid \Ss \in \Sigma^f_{\Ext_\prf}\right\} \subset \Sigma^f_{\Ext_\adm}$.
	
\end{enumerate}

\end{theorem}

The following Venn-diagram provides a compact overview of subset relations between the considered signatures. A bordered area represents a set of extension-sets. The outer ellipse $\m{ES} = \{\Ss\subseteq 2^\m{U} \mid \Ss \text{ is an extension set}\}$ stands for the set of all extension-sets over $\m{U}$. Clearly, all other signatures are subsets of $\m{ES}$ by definition. Furthermore, we use $\{\{\emptyset\}\}$ or $\{\emptyset\}$ the set consisting of the single extension-set $\{\emptyset\}$ (realizable by all considered semantics) or the set containing the empty extension-set (realizable by stable semantics only), respectively. 
The right side of Figure~\ref{fig:venn} shows signatures of semantics
providing only incomparable extension-sets. The intersection of these signatures with $\Sigma^f_{\Ext_\com}$ exactly coincides with $\Sigma^f_{\Ext_\grd}$ as well as $\Sigma^f_{\Ext_\id}$ and $\Sigma^f_{\Ext_\eag}$ which contain all extension-sets $\Ss$ with $|\Ss|=1$. Moreover, the only extension-set they have in common with the signatures of
conflict-free and admissible sets is the extension-set containing the empty extension. This fact causes the ``missing'' intersection in the middle of Figure~\ref{fig:venn}.

\def\firstell{(0.75,1) ellipse (1 and 0.85)}
\def\secondell{(-0.75,1) ellipse (1 and 0.85)}
\begin{figure}[H]
\centering
\begin{tikzpicture}[scale=1.25]

\fill[gray!30] (1.75,1) ellipse (2 and 1.15);
\fill[gray!50] (1.25,1) ellipse (1.5 and 1);
\fill[gray!70] \firstell;
\fill[gray!20] (-1.65,1) ellipse (2.1 and 1.15);
\fill[gray!40] (-1.25,1) ellipse (1.5 and 1);
\fill[gray!60] \secondell;
\fill[gray!50] (0,-0.5) circle (0.25);

\begin{scope}
  \clip \firstell;
  \fill[gray] \secondell;
\end{scope}

\draw[black,very thin] (0,1) ellipse (4 and 2);
\draw[black,very thin] (1.75,1) ellipse (2 and 1.15);
\draw[black,very thin] (1.25,1) ellipse (1.5 and 1);
\draw[black,very thin] (0.75,1) ellipse (1 and 0.85);
\draw[black,very thin] (-1.65,1) ellipse (2.1 and 1.15);
\draw[black,very thin] (-1.25,1) ellipse (1.5 and 1);
\draw[black,very thin] (-0.75,1) ellipse (1 and 0.85);
\draw[black,very thin] (0,-0.5) circle (0.25);

\draw (0,2.5) node {$\m{ES}$};
\draw (0.00,1) node {$\{\!\{\!\emptyset\!\}\!\}$};
\draw (0.95,1) node {$\Sigma^f_{\Ext_\nav}$};
\draw (2.2,1.35) node {$\Sigma^f_{\Ext_\stg}$};
\draw (2.2,1) node {$=$};
\draw[font=\scriptsize] (2.2,0.7) node {$\Sigma^f_{\Ext_\stb}\!\!\setminus\!\{\emptyset\}$};
\draw (3.2,1.35) node {$\Sigma^f_{\Ext_\prf}$};
\draw (3.2,1) node {$=$};
\draw (3.2,0.7) node {$\Sigma^f_{\Ext_\semi}$};
\draw (-0.85,1) node {$\Sigma^f_{\Ext_\cf}$};
\draw (-2.2,1) node {$\Sigma^f_{\Ext_\adm}$};
\draw (-3.2,1) node {$\Sigma^f_{\Ext_\com}$};
\draw (0,-0.5) node {$\{\emptyset\}$};
\end{tikzpicture}
\caption{Subset Relations between Finite Signatures}
\label{fig:venn}
\end{figure}
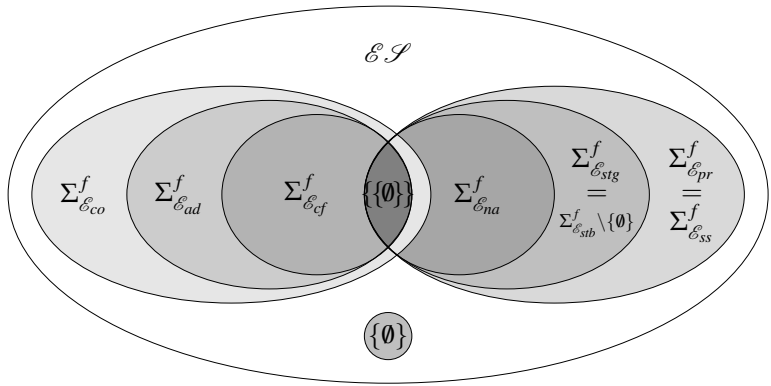

Finally, we want to mention that all considered finite signatures, apart from the complete signature, are closed under non-empty intersections. More precisely, if two finitely $\sigma$-realizable sets $\Ss$ and $\Tt$ possess a non-empty intersection, then $\Ss\cap\Tt$ is finitely $\sigma$-realizable too. This feature is mainly due to the fact that subsets of incomparable and tight as well as incomparable and conflict-sensitive sets maintain these properties (cf.\ Lemmas~\ref{lemma:tight_props} and \ref{lemma:adm_props}).

\begin{theorem}[\cite{DunDLW15}]
Let $\sigma \in \{\cf, \adm, \nav, \stb, \stg, \prf, \semi\}$. For any two finite AFs
$\F_1, \F_2$ exists an finite AF $\F$, s.t.\ 
$\Ext_{\sigma}(\F) = \Ext_{\sigma}(\F_1) \cap \Ext_{\sigma}(\F_2)$
given that $\Ext_{\sigma}(\F_1) \cap \Ext_{\sigma}(\F_2) \neq \emptyset$.\footnote{
The prerequisite of a non-empty intersection can be dropped in case of stable semantics.}
\end{theorem}

\section{Signatures w.r.t.\ Finite, Compact AFs} \label{sec:sigcompact}

So far we considered realizibility without any restriction (apart from finiteness) for witnessing AFs. This means, realizing AFs may contain \textit{rejected} arguments, i.e.\ arguments which do not appear in any extension. Rejected arguments are natural ingredients in typical argumentation scenarios and it is a priori completely unclear in
which ways rejected arguments contribute to the expressibility of a particular semantics. In order to have a handle for analyzing the effect of rejected arguments,
the class of \textit{compact} AFs and its induced signatures were introduced and studied \cite{compact,compact2,compactj}. An AF is compact with respect to a semantics
$\sigma$, if it does not contain rejected arguments, i.e.\ each of its arguments appears in at least one $\sigma$-extension. Now, the main question is whether it is possible to get rid of rejected arguments without changing the outcome? or, in other words: Under which circumstances can AFs be transformed into equivalent compact ones? Note that studying compactness is far from being an academic exercise since there
is a fundamental computational significance: When
searching for extensions, arguments span the search space, since extensions
are to be found among the subsets of the set of all arguments.
Hence the more arguments, the larger the search space. Compact
AFs are argument-minimal since none of the arguments can be
removed without changing the outcome, thus leading to a minimal
search space.

Let us first have a brief look on the naive semantics, which is defined
as $\subseteq$-maximal conflict-free sets: Here, it is rather easy to see that any AF
can be transformed into an equivalent compact AF by just removing all self-defeating
arguments. In other words, the same outcome (in terms of the naive
extensions) can be achieved by a simplified AF without rejected arguments. This means, naive semantics does not lose expressive power if we stick to compact AFs. However, it is not hard to find semantics where this coincidence does not hold implying that for such semantics the full range of expressiveness indeed relies on the concepts of rejected arguments. Consider therefore the following non-compact AF~$\F$.
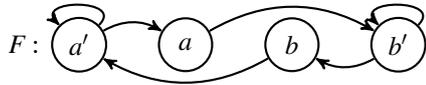
\vspace{2mm}

\begin{tikzpicture}

    \node (a') at (0,0) [circle, minimum size = 0.7cm, thick, draw, label = left:$\F:$] {$a'$};
    \node (a) at (1.4,0) [circle, minimum size = 0.7cm, thick, draw] {$a$};
    \node (b) at (2.8,0) [circle, minimum size=0.7cm,  thick, draw] {$b$};
		\node (b') at (4.2,0) [circle, minimum size=0.7cm,  thick, draw] {$b'$};


\draw[->,thick] (a') to [bend left,thick] (a);
\draw[->,thick] (b') to [bend left,thick] (b);
\draw[->,thick] (b') to [loop,thick, distance=0.5cm] (b');
\draw[->,thick] (a') to [loop,thick, distance=0.5cm] (a');
\draw[->,thick] (a) to [bend left,thick] (b');
\draw[->,thick] (b) to [bend left,thick] (a');

\end{tikzpicture}

Let us consider admissible sets. We obtain $\Ss = \Ext_{\adm}(\F) = \{\emptyset,\{a,b\}\}$. Obviously, any attempt of realizing $\Ss$ with a compact AF $\G = (\{a,b\},R)$ is doomed to failure since if $\{a,b\}$ is admissible in $\G$ we necessarily obtain the admissibility of $\{a\}$ as well as $\{b\}$ proving $\Ss \neq \Ext_{\adm}(\G)$. It was one main result in \cite{compact} to show that the finite, compact signatures w.r.t.\ stable, preferred, semi-stable, and
stage semantics are strict subsets of their corresponding finite signatures. This means, in case of those semantics, sticking to finite, compact AFs implies a loss of expressive power. 

\subsection{Central Definitions and Preliminary Observations}
\label{sec:compact_signatures}

In the following we formally introduce the central notions of \textit{compact argumentation frameworks}, \textit{compact realizibility} as well as \textit{compact signatures}. As already stated, the main idea behind compact AFs is the absence of
rejected arguments. For labelling-based semantics $\sigma$ (i.e., a semantics returning $n$-tuples) we assume that the first component of their associated $\sigma$-labellings are interpreted as \textit{acceptable sets of arguments} in analogy to $\sigma$-extensions in case of extension-based semantics. This means, if a certain argument occur in no first component of given $\sigma$-labellings we classify it as \textit{rejected}. For a given labelling $\L$ we use $\LI$ to refer to its first component. 

\begin{definition}
\label{def:caf}
Given a semantics $\sigma: \m{F}\rightarrow 2^{\left(2^\m{U}\right)^n}$. An AF $\F = (A,R)$ is compact for $\sigma$ (or simply, $\sigma$-compact) if 
$\Args_{\Ext_{\sigma}(\F)} = A$ (in case of $n = 1$) or $\Args_{\{\LI \mid \L\in\Lab_{\sigma}(\F)\}} = A$ (for $n \geq 2$), respectively.

\end{definition}

Although extension-based and labelling-based semantics are formally different semantics (according to Definition~\ref{def:semantics}) we often speak of the extension-based version or labelling-based version of a certain semantics. This can be formally justified for the considered semantics since there is a close relationship between both versions (cf.\ Facts~\ref{fact1} and \ref{fact2} for some formal relations). The following fact shows that for all considered semantics $\sigma$ there is no need to distinguish between $\sigma$-compactness w.r.t.\ the extension-based version of $\sigma$ and $\sigma$-compactness w.r.t.\ the labelling-based version of $\sigma$. As an aside, such a coincidence does not require a one-to-one correspondence between the extension-based and labelling-based version of a semantics $\sigma$. It suffices that any $\sigma$-extension induces a $\sigma$-labelling and vice versa in such a way that accepted arguments are preserved (cf.\ statements 1 and 2 of Fact~\ref{fact1}). 

\begin{fact} \label{fact:extlabcom} For any $\sigma\in\{\stb,\semi,\stg,\cfzwei,\stgzwei,\prf,\adm,\com,\grd,\id,\eag,\nav,\cf\}$ and any\footnote{Indeed, no finiteness restriction is required here.} AF $\F$ we have: $\F$ is compact for $\Ext_\sigma$ iff $\F$ is compact for $\Lab_\sigma$.
\end{fact}

In the following we use $\CAF_\sigma$ for AFs compact for $\sigma$. Moreover, we use $\CAF^f_\sigma$ to indicate that the considered frameworks are finite in addition. It is intuitively clear that there are AFs $\F$ being $\sigma$-compact without being $\tau$-compact for two different semantics $\sigma$ and $\tau$. The following example firstly presented in \cite[Figure~1]{compact} provides us with a witnessing framework.

\begin{example}
\label{ex:semistage_notssubseteq_pref}
Consider the following AF $\F$.\footnote{
The construct in the lower part of the figure
represents symmetric attacks between each pair of distinct arguments.
We will make use of this style in illustrations throughout the whole section.
}
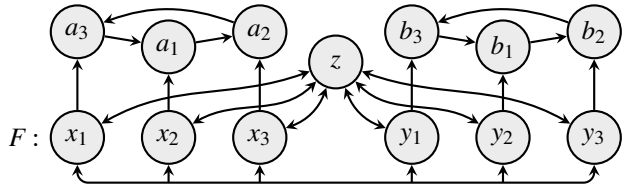
\vspace{2mm}

\begin{tikzpicture}[scale=1,>=stealth]
\centering
	\tikzstyle{arg}=[draw, thick, circle, fill=gray!15,inner sep=2pt]
		\path (0,0)     node[arg](a3){$a_3$}
			++(1.2,-0.2)	node[arg](a1){$a_1$}
			++(1.2,0.2)	node[arg](a2){$a_2$}
			++(2,0)	node[arg](b3){$b_3$}
			++(1.2,-0.2)	node[arg](b1){$b_1$}
			++(1.2,0.2)	node[arg](b2){$b_2$}
			;
		\path (0,-1.4)  node[arg](x1){$x_1$}
		  ++(-0.7,0) node(F){$\F:$}
			++(1.9,0) node[arg](x2){$x_2$}
			++(1.2,0) node[arg](x3){$x_3$}
			++(2,0) node[arg](y1){$y_1$}
            ++(1.2,0) node[arg](y2){$y_2$}
            ++(1.2,0) node[arg](y3){$y_3$}
			;
                \path (3.4,-0.4)node[arg,inner sep=3pt](z){$z$};
		\path [->, thick]
		    (x1) edge (a3)
			(x2) edge (a1)
			(x3) edge (a2)
			(y1) edge (b3)
			(y2) edge (b1)
			(y3) edge (b2)
            (a3) edge (a1)
            (a1) edge (a2)
            (b3) edge (b1)
            (b1) edge (b2)
            [bend right, out=-20, in=-160](a2) edge (a3)
            (b2) edge (b3)
			 ;
                \draw[<->,rounded corners=4pt, thick]
				(x1) -- (0,-2) -- (6.8,-2) -- (y3);
                \draw[->,thick] (1.2,-2) -- (x2);
                \draw[->,thick] (2.4,-2) -- (x3);
                \draw[->,thick] (4.4,-2) -- (y1);
                \draw[->,thick] (5.6,-2) -- (y2);
                \path[<->,thick,out=-160,in=30] (z) edge (x1);
                \path[<->,thick,out=-135,in=30] (z) edge (x2);
                \path[<->,thick,out=-110,in=20] (z) edge (x3);
                \path[<->,thick,out=-70,in=160] (z) edge (y1);
                \path[<->,thick,out=-45,in=150] (z) edge (y2);
                \path[<->,thick,out=-20,in=150] (z) edge (y3);                
\end{tikzpicture}

The preferred extensions of $\F$ are $\Ext_{\prf}(\F) = \{\{z\},$
$\{x_1,a_1\},$
$\{x_2,a_2\},$
$\{x_3,a_3\},$
$\{y_1,b_1\},$
$\{y_2,b_2\},$
$\{y_3,b_3\}\}$,
meaning that $\F$ is $\prf$-compact ($\F\in \CAF^f_\prf$)
since each argument occurs in at least one preferred extension.
On the other hand observe that
$\Ext_{\semi}(\F) = \Ext_{\prf}(\F) \sm \{\{z\}\}$ and
$\Ext_{\stg}(\F) = \{\{x_i,a_i,b_j\},\{y_i,b_i,a_j\} \mid 1\leq i,j \leq 3\}$,
i.e.\ $z$ is not contained in any semi-stable or stage extension.
Therefore $\F$ is neither compact for semi-stable nor compact for
stage semantics (i.e.\ $\F \notin \CAF^f_\semi$ and $\F\notin \CAF^f_\stg$).
\end{example}

How are the different sets of compact AFs related? We start with an easy observation.

\begin{lemma}[\cite{compactj}]
\label{lemma:subset}
For any two semantics $\sigma$ and $\tau$ such that
for each AF\ $F$,
for every $S \in \Ext_{\sigma}(\F)$
there is some $S' \in \Ext_{\tau}(\F)$
with $S \subseteq S'$,
we have $\CAF_\sigma \subseteq \CAF_\tau$.
\end{lemma}

Note that $\sigma\subseteq\tau$ is a special case of the premise of Lemma~\ref{lemma:subset}. Thus, $\CAF_\sigma \subseteq \CAF_\tau$, whenever $\sigma\subseteq\tau$ (see Figure~\ref{fig:semrel} for an overview). Strict subset relations have to be proven by providing a witnessing AF as presented in Example~\ref{ex:semistage_notssubseteq_pref}. Moreover, $\CAF_\prf = \CAF_\com = \CAF_\adm$ as well as $\CAF_\nav = \CAF_\cf$ is justified by Lemma~\ref{lem:prna} and the fact that $\prf\subseteq\com\subseteq\adm$ and $\nav\subseteq\cf$. Finally, in case of the uniquely defined grounded and ideal semantics we have, $\F = (A,R)$ is compact if and only if $R = \emptyset$. This in turn implies that $\F$ is compact for stable semantics. This means, $\CAF_\grd = \CAF_\id \subset \CAF_\stb$. Remember that eager semantics is uniquely defined w.r.t.\ finitary AFs only (Theorem~\ref{the:universaleager}, Example~\ref{ex:noext}). Consequently, we may conclude $\CAF^f_\grd = \CAF^f_\eag$ only. Although, the majority of the results do not require the finiteness restriction we present the following theorem in terms of finite AFs. Detailed
proofs for the relations between stable, semi-stable, preferred, stage and naive semantics
can be found in \cite[Theorem~2]{compactj}. 

\begin{theorem}
\label{thm:caf_relations}
The following relations hold:
\begin{enumerate}
\item $\CAF^f_\grd = \CAF^f_\id = \CAF^f_\eag$,
\item $\CAF^f_\prf = \CAF^f_\com = \CAF^f_\adm$,
\item $\CAF^f_\nav = \CAF^f_\cf$,
\item $\CAF^f_\grd \subset \CAF^f_\stb \subset \CAF^f_\semi \subset \CAF^f_\prf \subset \CAF^f_\nav$,
\item $\CAF^f_\stb \subset \CAF^f_\stg \subset \CAF^f_\nav$ and
\item $\CAF^f_\stg \not\subseteq \CAF^f_\sigma$ as well as $\CAF^f_\sigma \not\subseteq \CAF^f_\stg$ for any $\sigma \in \{\prf,\semi\}$.
\end{enumerate}
\end{theorem}
The following figure concisely summarizes all relations mentioned in the theorem above. Directed arrows between two boxes have to be interpreted as strict subset relations between the mentioned sets of compact AFs in these boxes. 

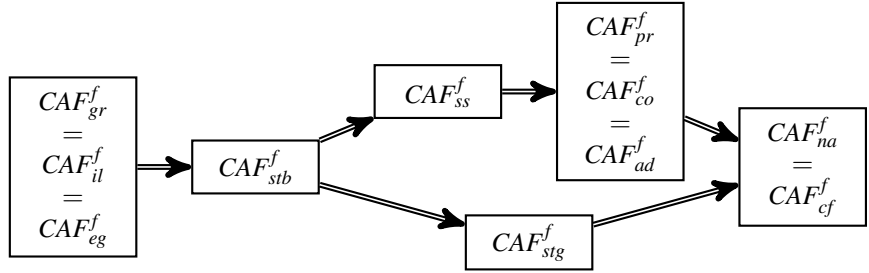
\begin{figure}[H]
\centering
\begin{tikzpicture}[scale=1.0]

\node (I) at (-1.6,1.5) [rectangle, thick, draw,text width = 1.43cm, text centered]{$\CAF^f_\grd$\\
$=$\\
$\CAF^f_\id$\\
$=$\\
$\CAF^f_\eag$};

\node (B) at (0.8,1.5) [rectangle, thick, draw,text width = 1.43cm, text centered]{$\CAF^f_\stb$};

\node (F) at (4.4,0.5) [rectangle, thick, draw,text width = 1.43cm, text centered]{$\CAF^f_\stg$};

\node (W) at (5.6,2.5) [rectangle, thick, draw,text width = 1.43cm, text centered]{$\CAF^f_\prf$\\
$=$\\
$\CAF^f_\com$\\
$=$\\
$\CAF^f_\adm$};

\node (X) at (3.2,2.5) [rectangle, thick, draw,text width = 1.43cm, text centered]{$\CAF^f_\semi$};

\node (O) at (8,1.5) [rectangle, thick, draw,text width = 1.43cm, text centered]{$\CAF^f_\nav$\\
$=$\\
$\CAF^f_\cf$};

\draw[->, thick, double] (I) to [thick,double] (B);

\draw[->, thick, double] (B) to [thick,double] (X);
\draw[->, thick, double] (X) to [thick,double] (W);
\draw[->, thick, double] (B) to [thick,double] (F);
\draw[->, thick, double] (F) to [thick,double] (O);
\draw[->, thick, double] (W) to [thick,double] (O);

\end{tikzpicture}
\caption{\label{fig:subcom}Subset Relations between Finite, Compact AFs}
\end{figure}

Instantiating Definitions \ref{def:realizable} and \ref{def:signatures} with $\m{C} = \CAF^f_\sigma$ formalize the notions of realizibility as well as signatures relativised to finite, compact AFs. Consider the following definitions.

\begin{definition}\label{def:realizable_finite_compact}
Given a semantics $\sigma: \m{F}\rightarrow 2^{\left(2^\m{U}\right)^n}$. A set $\Ss\subseteq \left(2^\m{U}\right)^n$ is \textit{finitely, compactly} $\sigma$-\textit{realizable} if there is an AF~$\F\in\CAF^f_\sigma$, s.t.\ $\sigma(\F)=\Ss$.
\end{definition}

\begin{definition} \label{def:signatures_finite_compact}
Given a semantics $\sigma$. The finite, compact $\sigma$-signature is defined as $\left\{\sigma(\F) \mid \F\in\CAF^f_\sigma \right\}$ abbreviated by $\Sigma_{\sigma}^{f,c}$. 
\end{definition}

It is clear that $\Sigma_{\sigma}^{c,f} \subseteq \Sigma^f_{\sigma}$ holds for any semantics $\sigma$, i.e.\ finite, compact realizibility implies finite realizibility. In the following we shed light on the question whether the mentioned subset relation is strict for a given semantics? In other words, we answer the question whether we indeed lose expressive power if sticking to compact AFs. 

\subsection{The Loss or Stability of Expressive Power}

Let us consider the uniquely defined grounded, ideal and eager semantics first. We already stated that a set $\Ss$ is realizable w.r.t.\ these semantics if and if only if $\Ss$ is an one-element extension-set if considering finite AFs (Theorem~\ref{the:charid}). Furthermore, it is immediate that an extension-set $\Ss = \{E\}$ can be compactly realized via $\F_{E} = (E,\emptyset)$. This means, these semantics do not lose expressive power if we restrict ourselves to compact AFs. Furthermore, the attentive reader may have noticed that the canonical argumentation framework $\F_\Ss^\cf$, which was used as a witnessing framework for conflict-free sets and naive semantics (cf.\ Definition~\ref{def_canonical1} as well as Propositions~\ref{prop:naive_real}~and~\ref{prop:cf_real}), does not involve further artificial arguments. Thus, it verifies finite, compact realizibility and shows that there is no expressive loss in case of conflict-free sets and naive semantics. For the other considered semantics, namely admissible, stable, stage, semi-stable, preferred as well as complete semantics we have to accept a strict weaker expressibility if we stick to compact AFs. In order to prove that in case of these semantics the full range of expressiveness indeed relies on the concept of rejected arguments we have to come up with witnessing extension-sets. Consider therefore the following example.

\begin{example} The extension-set $\Ss = \{\{a,b\},\{a,d,e\},\{b,c,e\}\}$ is realizable under preferred as well as semi-stable semantics (cf.\ Example~\ref{example:pref} for a realizing non-compact framework). Let $\sigma\in\{\prf,\semi\}$. Now suppose there exists an AF $\F=(\{a,b,c,d,e\},R)$, s.t.\ $\Ext_{\sigma}(\F)~=~\Ss$. Since $\{a,d,e\},\{b,c,e\}\in\Ss$ and $\sigma\subseteq\cf$ we conclude that there is no attack in $R$ involving $e$, i.e.\ $e$ is an isolated argument in $\F$. But then, $e$ is contained in each $\sigma$-extension of $\F$ contradicting $\{a,b\}\in\Ss$.
In Summary, $\Ss \in \Sigma_{\Ext_\sigma}^f \sm \Sigma_{\Ext_\sigma}^{f,c}$.
\end{example}

For further witnessing extension-sets we refer the reader to \cite[Propositions 35 and 57]{compactj} and proceed with the main theorem.

\begin{theorem}
\label{prop:compact_signatures}
It holds that
\begin{enumerate}
\item
$\Sigma^{f,c}_{\Ext_\sigma} = \Sigma^f_{\Ext_\sigma}$ for $\sigma \in \{\cf,\nav,\grd,\id,\eag\}$, and
\item
$\Sigma^{f,c}_{\Ext_\sigma} \subset \Sigma^f_{\Ext_\sigma}$ for $\sigma \in \{\adm,\stb,\stg,\semi,\prf,\com\}$.
\end{enumerate}
\end{theorem}

In both cases we may benefit of characterization theorems for finite signatures (cf.\ Theorems~\ref{the:charcf}, \ref{the:charadm}~and~\ref{the:charid}). If both signatures are identical (first item), then necessary and sufficient properties for being finitely $\sigma$-realizable immediately carry over to finite, compact $\sigma$-realizibility. If we observe a strict subset relation (second item), then we obtain at least necessary properties for being in the finite, compact $\sigma$-signature.

\begin{theorem}
\label{the:charcompactsig}
Given a set $\Ss \subseteq 2^{\m{U}}$, then
\begin{enumerate}

\item $\!\Ss\in \Sigma_{\Ext_{\cf}}^{f,c} \ToT \Ss\! \text{ is\! a non-empty, downward-closed and tight extension set,}$
\item $\!\Ss\in \Sigma_{\Ext_{\nav}}^{f,c} \ToT \Ss\! \text{ is\! a non-empty, incomparable extension set and } \dcl{(\Ss)} \text{ is tight},$
\item $\!\Ss\in \Sigma_{\Ext_{\grd}}^{f,c}  \ToT \Ss\! \text{ is\! an extension set with} \card{\Ss} = 1,$
\item $\!\Ss\in \Sigma_{\Ext_{\id}}^{f,c}   \ToT \Ss\! \text{ is\! an extension set with} \card{\Ss} = 1,$
\item $\!\Ss\in \Sigma_{\Ext_{\eag}}^{f,c}   \ToT \Ss\! \text{ is\! an extension set with} \card{\Ss} = 1$ and
\item $\!\Ss\in \Sigma_{\Ext_{\stb}}^{f,c}   \To \Ss\! \text{ is\! an incomparable and tight extension set,}$
\item $\!\Ss\in \Sigma_{\Ext_{\stg}}^{f,c}  \To \Ss\! \text{ is\! a non-empty, incomparable and tight extension set,}$
\item $\!\Ss\in \Sigma_{\Ext_{\adm}}^{f,c}  \To \Ss\! \text{ is\! a conflict-sensitive extension set containing }\emptyset,$
\item $\!\Ss\in \Sigma_{\Ext_{\prf}}^{f,c}  \To \Ss\! \text{ is\! a non-empty, incomparable and conflict-sensitive extension set,}$
\item $\!\Ss\in \Sigma_{\Ext_{\semi}}^{f,c}  \To \Ss\! \text{ is\! a non-empty, incomparable and conflict-sensitive extension set.}$

\end{enumerate}
\end{theorem}

\subsection{Comparing Finite, Compact Signatures and Final Remarks}

In the following we relate the finite, compact signatures of the semantics under consideration to each other.
Recall that for finite signatures it holds that $\Sigma^f_{\Ext_\nav}\ \subset \Sigma^f_{\Ext_\stg} = \left(\Sigma^f_{\Ext_\stb} \sm \{\emptyset\}\right) \subset \Sigma^f_{\Ext_\semi} = \Sigma^f_{\Ext_\prf}$ (cf.\ Figure~\ref{fig:venn}). This picture changes dramatically when considering the relationships between finite, compact signatures as depicted in Figure~\ref{fig:venn2} (incomparable semantics only) and formally stated in Theorem~\ref{thm:signatures_compact}. The dashed areas represent particular intersections for which the question of existence of extension-sets is still an open question. 


\begin{figure}[H]
  \centering
  \def\firstcircle{(0,0) circle (1.5cm)}
  \def\secondcircle{(0:2cm) circle (1.5cm)}
  \def\prfshape{(100:1.3cm) ellipse (2.1cm and 1.5cm)}
  \def\stgshape{(180:1cm) ellipse (2.3cm and 1.5cm)}
  \def\semishape{(10:1cm) ellipse (2.3cm and 1.5cm)}
  \def\stbshape{(8:.25cm) circle (1.04cm)}
  \def\navshape{(65:.6cm) ellipse (.8cm and .5cm)}
  \colorlet{circle edge}{gray!50}
  \colorlet{circle area}{gray!50}
  \tikzset{filled/.style={fill=circle area, draw=circle edge, thick,opacity=.5},
      outline/.style={draw=circle edge, thick}}
  \begin{tikzpicture}
    \fill[gray!10] \stgshape;
    \fill[gray!20] \semishape;
    \fill[gray!30] \prfshape;
    \fill[gray!60] \stbshape;
    \fill[gray!80] \navshape;
    \begin{scope}
      \clip \stgshape;
      \clip \semishape;
      \fill[pattern=horizontal lines, pattern color=gray!50,even odd rule]
	\semishape {} \stbshape;
    \end{scope}
    \begin{scope}
      \clip \stgshape;
      \clip \prfshape;
      \fill[pattern=vertical lines, pattern color=gray!50,even odd rule]
	\stgshape {} \stbshape; 
    \end{scope}
    \draw[outline] \prfshape node(prs) {};
        \draw[outline] \stgshape node(sgs) {};
        \draw[outline] \semishape node(sms) {};
        \draw[outline] \stbshape node(sbs) {};
        \draw[outline] \navshape node(nas) {$\Sigma_{\Ext_\nav}^{f,c}$};
        \node at (0.05cm,2cm)  {$\Sigma_{\Ext_\prf}^{f,c}$};
        \node at (-1.9cm,-0.5cm) {$\Sigma_{\Ext_\stg}^{f,c}$};
        \node at (2cm,-0.5cm) {$\Sigma_{\Ext_\semi}^{f,c}$};
        \node at (0.4cm,-0.5cm) {$\Sigma_{\Ext_\stb}^{f,c}$};
  \end{tikzpicture}
  \caption{Subset Relations between Finite, Compact Signatures}
  \label{fig:venn2}
\end{figure}
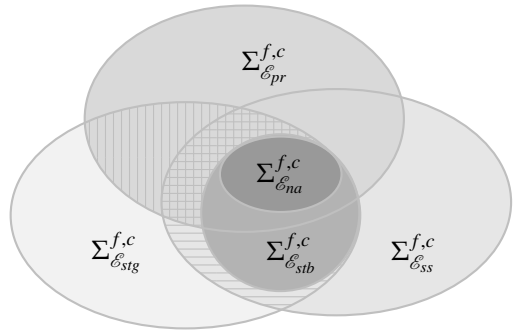

We proceed with an enumeration of relationships between finite, compact signature including further semantics like conflict-free and admissible sets as well as grounded, ideal, eager and complete semantics. For formal proofs we refer the interested reader to \cite[Theorem~36, Proposition~58]{compactj}.

\begin{theorem} The following relations hold:
\label{thm:signatures_compact}

\begin{enumerate}
\item $\Sigma_{\Ext_\sigma}^{f,c} \subset \Sigma_{\Ext_\nav}^{f,c} \subset \Sigma_{\Ext_\tau}^{f,c}$ for $\sigma\in\{\grd,\id,\eag\}$ and $\tau\in \{\stb,\stg,\semi,\prf\}$,
\item $\Sigma_{\Ext_\stb}^{f,c} \subset \Sigma_{\Ext_\sigma}^{f,c}$ for $\sigma \in \{\stg,\semi\}$,
\item $\Sigma_{\Ext_\cf}^{f,c} \subset \Sigma_{\Ext_\adm}^{f,c}$,
\item $\Sigma_{\Ext_\com}^{f,c} \sm \Sigma_{\Ext_\sigma}^{f,c} \neq\emptyset$ and $\Sigma_{\Ext_\sigma}^{f,c} \sm \Sigma_{\Ext_\com}^{f,c} \neq\emptyset$ for $\sigma\in\{\cf,\adm\}$,
\item $\Sigma_{\Ext_\prf}^{f,c} \setminus (\Sigma_{\Ext_\stb}^{f,c} \cup \Sigma_{\Ext_\semi}^{f,c} \cup \Sigma_{\Ext_\stg}^{f,c}) \neq \emptyset$,
\item $\Sigma_{\Ext_\stg}^{f,c} \setminus (\Sigma_{\Ext_\stb}^{f,c} \cup \Sigma_{\Ext_\prf}^{f,c} \cup \Sigma_{\Ext_\semi}^{f,c}) \neq \emptyset$,
\item $\Sigma_{\Ext_\stb}^{f,c} \setminus \Sigma_{\Ext_\prf}^{f,c} \neq \emptyset$,
\item $(\Sigma_{\Ext_\prf}^{f,c} \cap \Sigma^{f,c}_{\Ext_\semi}) \setminus (\Sigma_{\Ext_\stb}^{f,c} \cup \Sigma_{\Ext_\stg}^{f,c}) \neq \emptyset$ and
\item $\Sigma_{\Ext_\semi}^{f,c} \setminus (\Sigma_{\Ext_\stb}^{f,c} \cup \Sigma_{\Ext_\prf}^{f,c} \cup \Sigma_{\Ext_\stg}^{f,c}) \neq \emptyset$.
\end{enumerate}
\end{theorem}

Comparing the results on expressiveness of the considered semantics as stated in Theorems~\ref{thm:rel} and~\ref{thm:signatures_compact} we observe notable differences.
When allowing rejected arguments,
preferred and semi-stable semantics are equally expressive and at the same time
strictly more expressive than stable and stage semantics.
As we have seen, this property does not carry over to the compact setting (with the exceptions $\Sigma_{\Ext_\stb}^{f,c} \subset \Sigma_{\Ext_\semi}^{f,c}$ and $\Sigma_{\Ext_\stb}^{f,c} \subset \Sigma_{\Ext_\stg}^{f,c}$) where signatures become incomparable.

Finally, regarding the open issues represented as dashed areas in Figure~\ref{fig:venn2}. More precisely, it is an open problem whether there are extension-sets
lying in the intersection between $\Sigma_{\Ext_\prf}^{f,c}$ (resp.\ $\Sigma_{\Ext_\semi}^{f,c}$)
and $\Sigma_{\Ext_\stg}^{f,c}$ but outside of $\Sigma_{\Ext_\stb}^{f,c}$. In \cite{compactj} it is conjectured that such extension-sets do not exist. 

\begin{conjecture}[\cite{compactj}]
It holds that
$\Sigma_{\Ext_\prf}^{f,c} \cap \Sigma_{\Ext_\stg}^{f,c} \subset \Sigma_\stb^{f,c}$ and
$\Sigma_{\Ext_\semi}^{f,c} \cap \Sigma_{\Ext_\stg}^{f,c} = \Sigma_{\Ext_\stb}^{f,c}$.
\end{conjecture}

\section{Signatures w.r.t.\ Finite, Analytic AFs} \label{sec:analytic}

We now turn to a further phenomenon, so-called \textit{implicit conflicts} which can be frequently observed in typical argumentation scenarios. Consider therefore the following AF $\F$. 
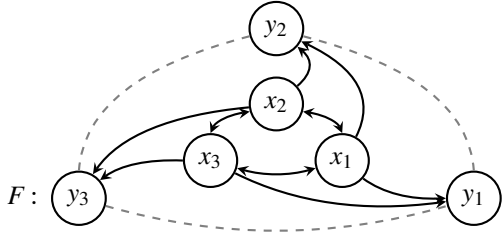
\vspace{2mm}

\begin{tikzpicture}[xscale=2,>=stealth]
      \path
      (-30:.5) node[arg] (x1) {$x_1$}
      +( -30:1) node[arg] (y1) {$y_1$}
			( 90:.5) node[arg] (x2) {$x_2$}
      +( 90:1) node[arg] (y2) {$y_2$}
      (210:.5) node[arg] (x3) {$x_3$}
      +(210:1) node[arg] (y3) {$y_3$}
			+(201:1.3) node (F) {$\F:\ $}
      ;
      \path[<->,thick,bend right]
      (x1) edge (x2)
      (x2) edge (x3)
      (x3) edge (x1)
      ;
      \path[->,bend right, thick]
      (x1) edge (y1)
      (x2) edge (y2)
      (x3) edge (y3)
			(x1) edge (y2)
      (x2) edge (y3)
      (x3) edge (y1)
      ;
			\path [-, thick, dashed, bend right=30, gray]
			(y1) edge (y2)
			(y3) edge (y1)
			(y2) edge (y3);
    \end{tikzpicture}
 
Let us consider stable semantics. Please note that any $x_i$ is jointly acceptable with one specific $y_j$. More precisely, $\Ext_{\stb}(\F)=\{\{x_1,y_3\},\{x_2,y_1\},\{x_3,y_2\}\}$ implying that we do not have any rejected arguments, i.e.\ $\F$ is stable compact. What can be said about the two pairs of arguments $x_1$ and $x_2$ as well as $y_1$ and $y_2$? First of all, both pairs represent a semantical conflict in $\F$ since neither of those pairs occur together in any stable extension. In case of $x_1$ and $x_2$, the conflict is even a syntactical one since both arguments attack each other in contrast to the pair consisting of $y_1$ and $y_2$. This difference leads to the distinction between syntactically underlined \textit{explicit conflicts} and syntactically unfounded \textit{implicit} ones (indicated by dashed lines). In order to understand how implicit conflicts contribute to the expressiveness of a certain semantics, the set of \textit{analytic} AFs and its induced signatures were introduced and studied \cite{LinSW15,compactj}. An analytic framework, i.e.\ a framework which is free of implicit conflicts maximizes the information on conflicts. One main question is: under which circumstances an arbitrary framework can be transformed into an equivalent analytic one? This question is interesting from a theoretical as well as practical
point of view. On the one hand, analytic frameworks are natural candidates for normal
forms of AFs, and on the other maximizing the number of explicit conflicts might help argumentation systems to evaluate AFs more efficiently.

Let us consider again the extension-set $\Ss = \{\{x_1,y_3\},\{x_2,y_1\},\{x_3,y_2\}\}$ stemming from the AF $\F$ depicted above. Replacing the dashed arrows with symmetric attacks in $\F$ shows that $\Ss$ can be analytically realized under stable semantics. Interestingly, this is no coincidence, since it was shown that in case of stable semantics any AF can be transformed into an equivalent analytical one. However, in general it is not that easy to make implicit conflicts explicit since there are frameworks where any suitable transformation requires the use of additional arguments as shown in \cite{LinSW15}.

\subsection{Central Definitions and Preliminary Observations}
\label{sec:analytic_afs}

In this section we consider the central notions of \textit{analytic argumentation frameworks}, \textit{analytic realizability} as well as \textit{analytic signatures}. In order to define analytic AF we have to differentiate between the concept of an attack
(as a syntactical element) and the concept of a conflict 
(with respect to the evaluation under a given semantics). More precisely, if two arguments cannot be accepted together, i.e.\ no reasonable position contain them jointly as elements, we say that these arguments are in conflict. If this conflict is syntactically underlined by an attack between them, we call this conflict explicit, otherwise implicit. Now, an analytic framework is an AF which simply does not contain any implicit conflicts. Consider the following definition.
 
\begin{definition}
\label{def:analytic}
Given a semantics $\sigma: \m{F}\rightarrow 2^{\left(2^\m{U}\right)^n}$, an AF $\F = (A,R)$ and two arguments $a,b\in A$. We say that

\begin{enumerate}
	\item $a$ and $b$ are in \textit{conflict for} $\sigma$ if $(a,b) \notin \Pairs_{\Ext_{\sigma}(\F)}$ (in case of $n = 1$) or $(a,b) \notin \Pairs_{\{\LI \mid \L\in\Lab_{\sigma}(\F)\}}$ (for $n \geq 2$), respectively,
	\item the conflict is \textit{explicit} w.r.t.~$\sigma$ if $(a,b)\!\in\! R$ or $(b,a)\!\in\! R$, otherwise \textit{implicit}, 
	\item the AF $\F$ is \textit{analytic} for $\sigma$ (or $\sigma$-analytic) if all conflicts are explicit.
\end{enumerate}
\end{definition}

Please notice that Definition~\ref{def:analytic} does not require $a$ and $b$ to be different arguments.
In particular, an argument that is not contained in any reasonable position is in conflict with itself. 
This conflict is explicit if the argument is self-attacking and implicit otherwise. Furthermore, for all considered semantics~$\sigma$ we observe there is no need to distinguish between $\sigma$-analyticality w.r.t.\ the extension-based version of $\sigma$ and $\sigma$-analyticality w.r.t.\ the labelling-based version of $\sigma$ (similarly as in case of $\sigma$-compactness as stated in Fact~\ref{fact:extlabcom}). Please note that this coincidence is justified for any semantics $\sigma$ whenever $\Ext_{\sigma}(\F) = \{\LI \mid \L\in\Lab_{\sigma}(\F)\}$ is guaranteed.

\begin{fact} \label{fact:extlabana} For any $\sigma\in\{\stb,\semi,\stg,\cfzwei,\stgzwei,\prf,\adm,\com,\grd,\id,\eag,\nav,\cf\}$ and any AF $\F$ we have: $\F$ is analytic for $\Ext_\sigma$  iff $\F$ is analytic for $\Lab_\sigma$.
\end{fact}

In the following we denote the set of all $\sigma$-analytic AFs as $\XAF_\sigma$. To indicate that the frameworks under consideration are finite we use $\XAF_\sigma^f$. We proceed with an example illustrating the new definitions. 

\begin{example}
  \label{ex:simple_ex}
  As a simple example consider the following AF $\F$ depicted below.

\begin{tikzpicture}[scale=1.5,>=stealth]
  \path
    node[arg,opacity=0]{$c$}
    ++(3,0) node[arg,opacity=0]{$d$}
    ;
		\path   node[arg](c){$c$}
		  ++(-0.5,0)	node(F){$\F:$}
			++(1.5,0)	node[arg](a){$a$}
                       ++(1,0)	node[arg](b){$b$}
                      ++(1,0)	node[arg](d){$d$}
			;
		\path [<->, thick] 
		        (a) edge (b)
			 ;	
		\path [->, thick]
		        (a) edge (c)
			(b) edge (d)
                       ;
		\path [-, thick, dashed, bend right=30, gray]
			(d) edge (c);
\end{tikzpicture}

	For $\sigma \in \{\stb,\prf,\semi,\stg\}$ we
  have $\Ext_{\sigma}(\F)=\{\{a,d\},\{b,c\}\}$.
  Observe that there is only one implicit conflict, namely the conflict between the arguments $c$ and $d$,
  denoted by a dashed line. Hence, $\F$ is not
  $\sigma$-analytic, i.e.\ $\F\notin\XAF_\sigma^f$.
  However, since $\Ext_{\nav}(\F)=\Ext_{\sigma}(\F)\cup\{\{c,d\}\}$ we have that
  $\F$ is $\nav$-analytic, i.e.\ $\F\in\XAF_\nav^f$.

\end{example}

As indicated in Example~\ref{ex:simple_ex} the sets of analytic AFs can
differ for different semantics. Just like in case of compact AFs (cf.\ Lemma~\ref{lemma:subset})
one may easily verify the following lemma which allows to obtain a plenty of subset relations between sets of analytic AFs.

\begin{lemma}[\cite{compactj}] \label{lemma:xaf_subset}

For any two semantics $\sigma$ and $\tau$ such that
for each AF\ $\F$,
for every $S \in \Ext_{\sigma}(\F)$
there is some $S' \in \Ext_{\tau}(\F)$
with $S \subseteq S'$,
we have $\XAF_\sigma \subseteq \XAF_\tau$.
  
 \end{lemma}
In line with the existing literature we restrict our considerations to finite AFs. Regarding universal (but not uniquely) defined semantics we obtain the same relations as in case of compact AFs (see explanations below Lemma~\ref{lemma:subset}). In any case we have $\XAF_\grd^f \subseteq \XAF_\id^f \subseteq \XAF_\eag^f$ since ideal semantics accepts more arguments than grounded semantics and eager semantics is even more credulous than ideal semantics. Furthermore, $\XAF_\eag^f\subseteq\XAF_\semi^f$ because the unique eager extension is contained in all semi-stable extension by definition and moreover, semi-stable semantics guarantees reasonable positions in case of finite AFs. Now, let $\F = (A,R)$ be analytic w.r.t.\ eager semantics and $\Ext_{\eag}(\F) = \{E\}$. We deduce that all arguments in $A\sm E$ have to be self-defeating. Consequently, its corresponding (conflict-free) base semantics (cf.\ Definition~\ref{def:parsemantics}) warrants exactly one extension for $\F$. More precisely, $\Ext_{\semi}(\F) = \{E\}$. Finally, due to the self-conflicting arguments and the admissibility of $E$ we obtain $\Ext_{\prf}(\F) = \{E\}$ and thus, $\Ext_{\id}(\F) = \{E\}$ showing that $\F$ is even analytic w.r.t.\ ideal semantics, i.e.\ $\XAF^f_\id = \XAF^f_\eag$. The AF $\F = (\{a,b\},\{(a,b),(b,a),(b,b)\})$ proves that a similar result in case of grounded and ideal semantics does not hold. Detailed proofs for the relations between stable, semi-stable, preferred, stage and naive semantics can be found in \cite[Theorem~4]{compactj}.

\begin{theorem}

\label{thm:XAFrelations}
The following relations hold: 
\begin{enumerate}
\item $\XAF^f_\grd \subset \XAF^f_\id = \XAF^f_\eag \subset \XAF^f_\semi $,
\item $\XAF^f_\prf = \XAF^f_\com = \XAF^f_\adm$,
\item $\XAF^f_\nav = \XAF^f_\cf$,
\item $\XAF^f_\stb \subset \XAF^f_\semi \subset \XAF^f_\prf \subset \XAF^f_\nav$,
\item $\XAF^f_\stb \subset \XAF^f_\stg \subset \XAF^f_\nav$,
\item $\XAF^f_\stg \not\subseteq \XAF^f_\sigma$ and $\XAF^f_\sigma \not\subseteq \XAF^f_\stg$ for any $\sigma \in \{\prf,\semi\}$,
\item $\XAF^f_\sigma \not\subseteq \XAF^f_\tau$ and $\XAF^f_\tau \not\subseteq \XAF^f_\sigma$ for any $\sigma\! \in\! \{\grd,\id,\eag\}$, $\tau\!\in\!\{\stb,\stg\}$.
\end{enumerate}
\end{theorem}

The following figure summarizes all relation in a compact way. Similarly to Figure~\ref{fig:subcom}, a directed arrow between two boxes has to be interpreted as strict subset relation between the mentioned sets of analytic AFs therein. 

\begin{figure}[H]
\centering
\begin{tikzpicture}[scale=1.0]

\node (G) at (-1.6,3.5) [rectangle, thick, draw,text width = 1.43cm, text centered]{$\XAF^f_\grd$};

\node (I) at (0.8,3.5) [rectangle, thick, draw,text width = 1.43cm, text centered]{$\XAF^f_\id$\\
$=$\\
$\XAF^f_\eag$};

\node (B) at (0.8,1.5) [rectangle, thick, draw,text width = 1.43cm, text centered]{$\XAF^f_\stb$};

\node (F) at (4.4,0.5) [rectangle, thick, draw,text width = 1.43cm, text centered]{$\XAF^f_\stg$};

\node (W) at (5.6,2.5) [rectangle, thick, draw,text width = 1.43cm, text centered]{$\XAF^f_\prf$\\
$=$\\
$\XAF^f_\com$\\
$=$\\
$\XAF^f_\adm$};

\node (X) at (3.2,2.5) [rectangle, thick, draw,text width = 1.43cm, text centered]{$\XAF^f_\semi$};

\node (O) at (8,1.5) [rectangle, thick, draw,text width = 1.43cm, text centered]{$\XAF^f_\nav$\\
$=$\\
$\XAF^f_\cf$};

\draw[->, thick, double] (G) to [thick,double] (I);

\draw[->, thick, double] (I) to [thick,double] (X);

\draw[->, thick, double] (B) to [thick,double] (X);
\draw[->, thick, double] (X) to [thick,double] (W);
\draw[->, thick, double] (B) to [thick,double] (F);
\draw[->, thick, double] (F) to [thick,double] (O);
\draw[->, thick, double] (W) to [thick,double] (O);

\end{tikzpicture}
\caption{\label{fig:subana}Subset Relations between Finite, Analytic AFs}
\end{figure}
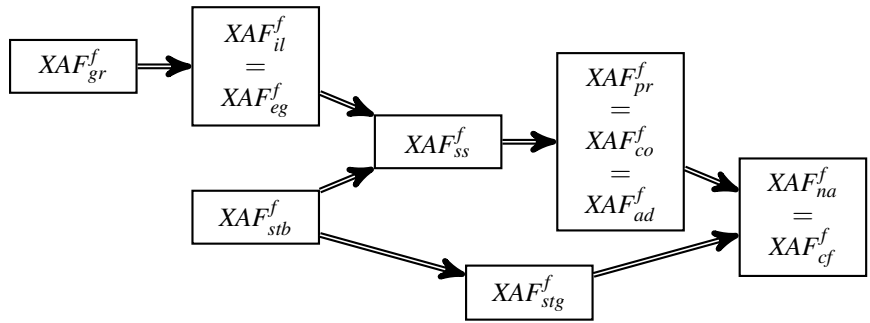 

At this point we want to mention that although Figures~\ref{fig:subcom}~and~\ref{fig:subana} look very similar we have that compactness and analyticality are sufficiently distinct properties. More precisely, apart from the uniquely defined semantics as well as naive semantics and conflict-free sets no subset relations between the sets of compact and analytic frameworks can be stated in general. Sticking to self-loop-free AFs allows one to draw further relations such as analyticality implies compactness for any considered semantics. The main reason for this general relation is that rejected arguments has to be self-defeating in case of analytic frameworks. A selection of proofs of relations listed below can be found in \cite[Proposition 5-8]{compactj}.

\begin{proposition}
\label{prop:relationXAFCAF}
Given an AF $\F$, then
\begin{enumerate}
	\item $\CAF^f_\sigma \subset \XAF^f_\sigma$ for $\sigma\in\{\grd,\id,\eag,\nav,\cf\}$, 
	\item $\CAF^f_\sigma \not\subseteq \XAF^f_\sigma$ and $\XAF^f_\sigma \not\subseteq \CAF^f_\sigma$ for $\sigma\in\{\adm,\stb,\semi,\prf,\stg,\com\}$.
\end{enumerate}
If $\F$ is self-loop-free in addition, then
\begin{enumerate}\setcounter{enumi}{2}
	\item $\F\in\XAF^f_\sigma$ and $\F \in \CAF^f_\sigma$ for $\sigma\in\{\nav,\cf\}$,
	\item $\F\in\XAF^f_\sigma \ToT \F \in \CAF^f_\sigma$ for $\sigma\in\{\grd,\id,\eag\}$ and
 \item $\F\in\XAF^f_\sigma \To \F\in\CAF^f_\sigma$ for $\sigma\in\{\adm,\stb,\semi,\prf,\stg,\com\}$.
\end{enumerate}

\end{proposition}

We now precisely formalize the notions of realizibility as well as signatures relativised to finite, analytic AFs. This can be formally done via instantiating Definitions \ref{def:realizable} and \ref{def:signatures} with $\m{C} = \XAF^f_\sigma$.

\begin{definition}\label{def:analyticrealizable}
Given a semantics $\sigma: \m{F}\rightarrow 2^{\left(2^\m{U}\right)^n}$. A set $\Ss\subseteq \left(2^\m{U}\right)^n$ is \textit{finitely, analytically} $\sigma$-\textit{realizable} if there is an AF~$\F\in\XAF^f_\sigma$, s.t.\ $\sigma(\F)=\Ss$.

\end{definition}

\begin{definition} \label{def:signatures_finite_analytic}
Given a semantics $\sigma$. The finite, analytic $\sigma$-signature is defined as $\left\{\sigma(\F) \mid \F\in\XAF^f_\sigma \right\}$ abbreviated by $\Sigma_{\sigma}^{f,x}$. 
\end{definition}

\subsection{The Loss or Stability of Expressive Power}

Clearly, every set in the finite, analytic signature of a semantics
is also contained in the finite signature. Remember that in case of compact AFs we do not lose any expressive power if considering the uniquely defined grounded, ideal and eager semantics as well as naive semantics and conflict-free sets (Theorem~\ref{prop:compact_signatures}). These equal expressiveness results carry over to analytic AFs and moreover, even stable and stage semantics may realize the same sets.
For instance, consider again the non-analytic AF $\F$ as introduced in Example~\ref{ex:simple_ex}. One may easily verify that adding an attack from $c$ to $d$ or vice versa yields an AF $\F'$ analytic for stable semantics which does not change the set of stable extensions. 
However, in general it is not that easy to make implicit conflicts explicit but it was shown that the use of additional arguments indeed allows one to turn any finite framework in an analytical one without changing the set of stable or stage extensions, respectively \cite[Proposition 28, Theorem 29]{compactj}. 
For the sake of completeness, we mention that it was an open question for a while, known as \textit{Explicit Conflict Conjecture} \cite{compact}, whether it is possible, under stable semantics, to translate a given AF into an equivalent analytic one without adding further arguments. In \cite{compactj} the conjecture was refuted for stable and even stage semantics. For the remaining semantics, i.e.\ admissible, semi-stable, preferred and complete semantics the conjecture does not hold either since in case of these semantics we even have that the finite, analytic signature is a strict subset of the corresponding finite one. This means, the full range of expressiveness indeed relies on the use of implicit conflicts. Consider the following example firstly presented in \cite[Example 6]{compactj}.

\begin{example} Take into account the AF $\F=(A,R)$ as depicted below.

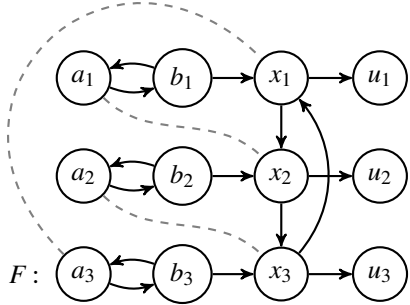
\vspace*{-0.8cm}
	\begin{tikzpicture}[scale = 1.3]
		\path
			(0,-1)node[arg,white]{$a_1$}
			++(-0.6,-2)node (F){$\F:$}
			(3,-3)node[arg,white]{$u_3$}
			;
		\foreach \x in{1,2,3}{
			\path (0,-\x)node[arg](a\x){$a_\x$}
				++(1,0)node[arg](b\x){$b_\x$}
				++(1,0)node[arg](x\x){$x_\x$}
				++(1,0)node[arg](u\x){$u_\x$}
			;
			\path [->, thick]
				(a\x) edge[bend right=20] (b\x)
				(b\x) edge[bend right=20] (a\x)
				(b\x) edge (x\x)
				(x\x) edge (u\x)
				;
		}
		\path [->, thick]
			(x1) edge (x2)
			(x2) edge (x3)
			(x3) edge[bend right=40] (x1)
			;
		\path [-, thick, dashed, gray]
			(a1) edge[out=-45,in=135] (x2)
			(a2) edge[out=-45,in=135] (x3)
			(a3) edge[out=135,in=135,looseness=2] (x1)
			;
\end{tikzpicture}

Formally, we have
  \begin{align*}
    A=&\left\{ a_i,b_i,x_i,u_i\mid i\in\{1,2,3\} \right\} \text{ and}\\
    R=&\left\{ (a_i,b_i),(b_i,a_i),(b_i,x_i),(x_i,u_i)\mid i\in\{1,2,3\} \right\}
\cup\left\{ (x_1,x_2),(x_2,x_3),(x_3,x_1) \right\}.
  \end{align*}

Regarding the extension-based version of preferred semantics we obtain the set $\Ss = \Ext_{\prf}(\F) = $ \linebreak $\{S_a,S_b,A_1,A_2,A_3,B_1,B_2,B_3\}$ with
\begin{align*}
    S_a&=\left\{ a_1,a_2,a_3 \right\}&
    S_b&=\left\{ b_1,b_2,b_3,u_1,u_2,u_3 \right\}\\
    A_1&=\left\{ a_2,a_3,b_1,x_2,u_1,u_3 \right\} &
    B_1&=\left\{ a_1,b_2,b_3,x_1,u_2,u_3 \right\}\\
    A_2&=\left\{ a_1,a_3,b_2,x_3,u_1,u_2 \right\}&
    B_2&=\left\{ a_2,b_1,b_3,x_2,u_1,u_3 \right\}\\
    A_3&=\left\{ a_1,a_2,b_3,x_1,u_2,u_3 \right\}&
    B_3&=\left\{ a_3,b_1,b_2,x_3,u_1,u_2 \right\}
  \end{align*}
	
We observe three implicit conflicts indicated by dashed lines. Consequently, $\F$ is not analytic w.r.t.\ preferred semantics. Moreover, we claim that $\Ss$ is not analytically $\prf$-realizable at all. For a contradiction we assume that there exists an AF $\G \in \XAF_\prf^f$, s.t.\ $\Ext_{\prf}(\G) = \Ss$.  We now investigate this hypothetical AF~$\G$. The main idea is to show that if the conflict between $a_1$ and $x_2$ is made explicit, then $\Ss\neq \Ext_{\prf}(\G)$. First, note that $\G$ contains at least all arguments in $A$ since $\Args_{\Ss} = A$. Due to $A_3$ and $B_3$ we deduce that
$S_a\cup\{u_2\}$ is conflict-free in $\G$. Furthermore, due to $A_1$, the admissibility of $S_a$ in $\G$ and the assumption that all conflicts has to be explicit, we infer that $a_1$ attacks $x_2$. Moreover, in consideration of $\Ss$, it is easy to see that $x_2$ is the only possible attacker of $u_2$ among $\Args_\Ss$. This implies that $S_a$ defends $u_2$ against all arguments in $\Args_{\Ss}$. Finally, any additional argument $z \notin \Args_\Ss$ in $\G$ must be attacked by
$S_a$ since $\G$ is analytic w.r.t.\ preferred semantics and $S_a$ must be admissible.
This causes $S_a \cup \{u_2\}$ to be admissible in $\G$ and
hence, $S_a$ cannot be preferred in $\G$.
In summary, any AF realizing $\Ss$ has to be non-analytic for preferred
semantics, i.e.\ $\Ss \in \Sigma_{\Ext_{\prf}}^f \sm \Sigma_{\Ext_{\prf}}^{f,x}$.
	
\end{example}

We proceed with the main theorem comparing finite signatures with their corresponding analytical ones.

\begin{theorem}[\cite{compactj}]
\label{prop:analytic_signatures}
It holds that
\begin{enumerate}
\item
$\Sigma^{f,x}_{\Ext_\sigma} = \Sigma_{\Ext_\sigma}^f$ for $\sigma \in \{\cf,\nav,\grd,\id,\eag,\stb,\stg\}$, and
\item
$\Sigma^{f,x}_{\Ext_\sigma} \subset \Sigma_{\Ext_\sigma}^f$ for $\sigma \in \{\adm,\semi,\prf,\com\}$.
\end{enumerate}
\end{theorem}

In the following we present characterization theorems for finite, analytic signatures or at least necessary properties for being finitely, analytically realizable. All results can be verified via combining the main theorem above as well as the already presented characterization theorems for finite signatures, namely Theorems~\ref{the:charcf},~\ref{the:charadm}~and~\ref{the:charid}. 

\begin{theorem}
\label{the:charanalyticsig}
Given a set $\Ss \subseteq 2^\m{U}$, then
\begin{enumerate}

\item $\Ss\in \Sigma_{\Ext_{\cf}}^{f,x} \ToT \Ss \text{ is a non-empty, downward-closed and tight extension set,}$
\item $\Ss\in \Sigma_{\Ext_{\nav}}^{f,x} \ToT \Ss \text{ is a non-empty, incomparable extension set and } \dcl{(\Ss)} \text{ is tight},$
\item $\Ss\in \Sigma_{\Ext_{\grd}}^{f,x}  \ToT \Ss \text{ is an extension set with} \card{\Ss} = 1,$
\item $\Ss\in \Sigma_{\Ext_{\id}}^{f,x}  \ToT \Ss \text{ is an extension set with} \card{\Ss} = 1,$
\item $\Ss\in \Sigma_{\Ext_{\eag}}^{f,x}  \ToT \Ss \text{ is an extension set with} \card{\Ss} = 1,$
\item $\Ss\in \Sigma_{\Ext_{\stb}}^{f,x}  \ToT \Ss \text{ is a incomparable and tight extension set,}$
\item $\Ss\in \Sigma_{\Ext_{\stg}}^{f,x}  \ToT \Ss \text{ is a non-empty, incomparable and tight extension set}$ and
\item $\Ss\in \Sigma_{\Ext_{\adm}}^{f,x}   \To \Ss \text{ is a conflict-sensitive extension set containing }\emptyset,$
\item $\Ss\in \Sigma_{\Ext_{\prf}}^{f,x}  \To \Ss \text{ is a non-empty, incomparable and conflict-sensitive extension set,}$
\item $\Ss\in \Sigma_{\Ext_{\semi}}^{f,x}  \To \Ss \text{ is a non-empty, incomparable and conflict-sensitive extension set.}$

\end{enumerate}
\end{theorem}

\subsection{Comparing Finite, Analytic Signatures and Final Remarks}

So far we have compared finite signatures and finite, analytic signatures
for the semantics under consideration. We have seen, for example, that preferred and semi-stable semantics can realize
strictly more when allowing the use of implicit conflicts, while this is not the case for stable and stage semantics.
In the following we relate the finite, analytic signatures of all considered semantics. Remember that we observed a considerable variety in the relations between incomparable semantics if sticking from finite to finite, compact signatures (cf.\ Figures~\ref{fig:venn}~and~\ref{fig:venn2}). However, in the analytic case we have slight differences only as illustrated in Figure~\ref{fig:venn_analytic} (for incomparable semantics)  and formally stated in Theorem~\ref{thm:analytic_signatures}. For instance, preferred and semi-stable signatures do not coincide anymore as shown by the following example taken from \cite[Figure 9, Proof of Theorem 34]{compactj}.

\begin{example} Consider the following AF $\F$ as depicted below.
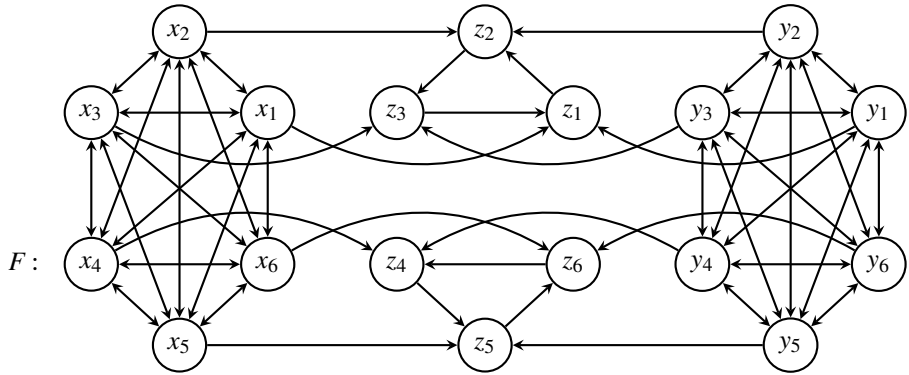
\begin{center}
  \begin{tikzpicture}[scale=1.34,>=stealth]
    \path
    +( 30:1) node[arg](x1){$x_1$}
    +( 90:1.3) node[arg](x2){$x_2$}
    +(150:1) node[arg](x3){$x_3$}
    ++(0,-.5)
		+(197:1.6) node(F){$\F:$}
    +(210:1) node[arg](x4){$x_4$}
		
    +(270:1.3) node[arg](x5){$x_5$}
    +(330:1) node[arg](x6){$x_6$}
    ++(3, .5)
    +( 30:1) node[arg](z1){$z_1$}
    +( 90:1.3) node[arg](z2){$z_2$}
    +(150:1) node[arg](z3){$z_3$}
    ++(0,-.5)
    +(210:1) node[arg](z4){$z_4$}
    +(270:1.3) node[arg](z5){$z_5$}
    +(330:1) node[arg](z6){$z_6$}
    ++(3, .5)
    +( 30:1) node[arg](y1){$y_1$}
    +( 90:1.3) node[arg](y2){$y_2$}
    +(150:1) node[arg](y3){$y_3$}
    ++(0,-.5)
    +(210:1) node[arg](y4){$y_4$}
    +(270:1.3) node[arg](y5){$y_5$}
    +(330:1) node[arg](y6){$y_6$}
    ;
    \path [->, thick]
    (x1) edge[bend right] (z1)
    (x2) edge (z2)
    (x3) edge[bend right] (z3)
    (x4) edge[bend left] (z4)
    (x5) edge (z5)
    (x6) edge[bend left] (z6)
    (y1) edge[bend left] (z1)
    (y2) edge (z2)
    (y3) edge[bend left] (z3)
    (y4) edge[bend right] (z4)
    (y5) edge (z5)
    (y6) edge[bend right] (z6)
    ;
    \path [->, thick]
    (z1) edge (z2)
    (z2) edge (z3)
    (z3) edge (z1)
    (z4) edge (z5)
    (z5) edge (z6)
    (z6) edge (z4)
    ;
    \path [<->,thick]
    (x1) edge (x2)
    (x2) edge (x3)
    (x1) edge (x3)
    (x4) edge (x5)
    (x5) edge (x6)
    (x4) edge (x6)
    (y1) edge (y2)
    (y2) edge (y3)
    (y1) edge (y3)
    (y4) edge (y5)
    (y5) edge (y6)
    (y4) edge (y6)
    ;
    \foreach \a in {1,2,3}{
      \foreach \b in {4,5,6}{
	\path [<->,thick]
	(x\a) edge (x\b)
	(y\a) edge (y\b)
	;
      }
    }
  \end{tikzpicture}
\end{center}  

The preferred extension of $\F$ can be compactly presented via a cyclic successor function$\ssucc$. More precisely, if $\ssucc(1)=2,\ssucc(2)=3,\ssucc(3)=1$ and  
  $\ssucc(4)=5,\ssucc(5)=6,\ssucc(6)=4$, then 
  $\Ext_{\prf}(\F) = \Ss = \Ss_0\cup\Ss_1\cup\Ss_2$ with
  \begin{align*}
    \Ss_0&=\left\{\left\{x_i,y_j,z_{\ssucc(i)},z_{\ssucc(j)}\right\}\!\mid i\in\left\{1,2,3\right\}\!,j\in\left\{4,5,6\right\}\!\textit{\! or }i\in\left\{4,5,6\right\}\!,j\in\left\{1,2,3\right\}\right\}\!,\\
    \Ss_1&=\left\{\left\{x_i,y_i,z_{\ssucc(i)}\right\}\!\mid i\in\left\{1,2,3,4,5,6\right\}\right\} \text{ and}\\
    \Ss_2&=\left\{\left\{x_i,y_{\ssucc(i)},z_{\ssucc(\ssucc(i))}\right\},\left\{x_{\ssucc(i)},y_i,z_{\ssucc(\ssucc(i))}\right\}\!\mid i\in\left\{1,2,3,4,5,6\right\}\right\}.
  \end{align*}
This means, $\F$ is $\prf$-analytic and therefore, $\Ss\in\Sigma^{f,x}_{\Ext_\prf}$. We show now that $\Ss~\notin~\Sigma^{f,x}_{\Ext_\semi}$. Assume that there is some $\G = (B,S) \in\XAF_\semi^f$ with $\Ext_\prf(\G)~=~\Ss$.
  We take a look at $\Ss_1$ and more specifically
  $\{x_1,y_1,z_2\}\in~\Ss_1$. Now we need an explicit conflict between $x_1$ and
  $x_4$, but in the selected set only $x_1$ can possibly defend against this
  attack, hence $(x_1,x_4)\in S$. The same argument works for $x_1$ and $x_3$
  as well as $z_2$ and~$z_3$, meaning that also $(x_1,x_3),(z_2,z_3)\in S$.
  For symmetry reasons
  $\{(x_i,x_j),(x_j,x_i),(y_i,y_j),(y_j,y_i)\mid
  i\in\{1,2,3\},j\in\{4,5,6\}\}\subseteq S$ and $\{(x_{\ssucc(i)},x_i),(z_i,z_{\ssucc(i)})\mid
  i\in\{1,2\dots 6\}\}\subseteq S$.

  We take a look at $\Ss_2$ and more specifically $\{x_1,y_2,z_3\}\in\Ss_2$. As
  there should be an explicit conflict between $x_1$ and $x_2$ with only $x_1$
  possibly defending this extension against $x_2$ we need $(x_1,x_2)\in S$.
  Further as in this set only $y_2$ and $z_3$ can possibly attack $z_2$ we have
  the set $\{y_2,z_3\}$ attacking $z_2$. For symmetry reasons
  $\{(x_i,x_{\ssucc(i)}),(y_i,y_{\ssucc(i)})\mid i\in\{1,2\dots 6\}\}\subseteq S$ and
  each set $\{x_i,z_{\ssucc(i)}\},\{y_i,z_{\ssucc(i)}\}$ for $i\in\{1,2\dots 6\}$ attacks~$z_i$.

  Finally we take a look at $\Ss_0$ and specifically the set
  $I = \{x_1,y_4,z_2,z_5\}~\in~\Ss_0$. Since $I$ necessarily is an admissible
  extension in an analytic AF we have that $I$ attacks all rejected
  arguments. By the above observations we now have that $I$ even attacks all
  arguments not being member of $I$ in $\G$, which means that $I$ is a stable
  extension and stable semantics and semi-stable semantics thus coincide on
  $\G$. But then, with $J = \{x_1,y_1,z_2\}\in\Ss_1$ not being in conflict with
  for instance $z_4$ we have that $J$ can not be a stable or semi-stable
  extension in $\G$ concluding
  $\Ss\notin\Sigma^{f,x}_{\Ext_\semi}$. 

\end{example}

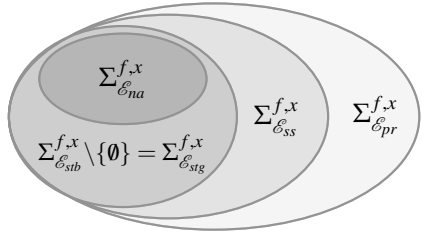
\begin{figure}[H]
  \centering
  \def\firstcircle{(0,0) circle (1.5cm)}
  \def\secondcircle{(0:2cm) circle (1.5cm)}
  \def\prfshape{(1.2cm,0) ellipse (2.7cm and 1.5cm)}
  \def\semishape{(0.6cm,0) ellipse (2.1cm and 1.35cm)}
  \def\stbshape{(0,0) ellipse (1.5cm and 1.2cm)}
  \def\navshape{(0,0.5) ellipse (1.1cm and 0.6cm)}
 \colorlet{circle edge}{gray!85}
  \colorlet{circle area}{gray!10}
  \tikzset{filled/.style={fill=circle area, draw=circle edge, thick,opacity=.5},
      outline/.style={draw=circle edge, thick}}
  \begin{tikzpicture}
    
    \fill[filled,fill=gray!15] \prfshape;
		\fill[filled,fill=gray!35] \semishape;
    \fill[filled,fill=gray!55] \stbshape;
    \fill[filled,fill=gray!75] \navshape;

    \draw[outline] \prfshape node(prs) {};
    \draw[outline] \semishape node(sms) {};
    \draw[outline] \stbshape node(sbs) {};
    \draw[outline] \navshape node (nas) {$\Sigma_{\Ext_\nav}^{f,x}$};
    \node[anchor=west] at (2.9cm,0) {$\Sigma_{\Ext_{\prf}}^{f,x}$};
    \node[anchor=west] at (1.6cm,0) {$\Sigma_{\Ext_{\semi}}^{f,x}$};
	\node at (0,-0.5cm) {\small $\Sigma_{\Ext_{\stb}}^{f,x}{\setminus}\{\emptyset\} = \Sigma_{\Ext_{\stg}}^{f,x}$};
  \end{tikzpicture}
  \caption{Subset Relations between Finite, Analytic Signatures}
  \label{fig:venn_analytic}
\end{figure}

\begin{theorem}[\cite{DunDLW15}]\label{thm:analytic_signatures}

The following relations hold: 

\begin{enumerate}
	\item $\Sigma_{\Ext_\sigma}^{f,x} \subset \Sigma_{\Ext_\nav}^{f,x} \subset \Sigma_{\Ext_\stg}^{f,x} \subset \Sigma_{\Ext_\semi}^{f,x} \subset \Sigma_{\Ext_\prf}^{f,x}$ for $\sigma\in\{\grd,\id,\eag\}$,
	\item $\Sigma_{\Ext_\stb}^{f,x} = \Sigma_{\Ext_\stg}^{f,x} \cup \{\emptyset\}$,
	\item $\Sigma_{\Ext_\cf}^{f,x} \subset \Sigma_{\Ext_\adm}^{f,x}$ and
	\item $\Sigma_{\Ext_\com}^{f,x} \sm \Sigma_{\Ext_\sigma}^{f,x} \neq\emptyset$ and $\Sigma_{\Ext_\sigma}^{f,x} \sm \Sigma_{\Ext_\com}^{f,x} \neq\emptyset$ for $\sigma\in\{\cf,\adm\}$.
	
\end{enumerate}

\end{theorem}

\section{Remarks on Unrestricted AFs and Intertranslatability} \label{sec:sigunrestrict}
Recently, some first results regarding expressibility w.r.t.\ unrestricted frameworks were presented in \cite{BauS17}. Remember that the set of unrestricted frameworks, abbreviated by \m{F}, contains all AFs $\F = (A,R)$, s.t.\ $A\subseteq\m{U}$. This means, \m{F} contains finite as well as infinite AFs and especially, AFs possessing all available arguments. It is obvious that signatures w.r.t.\ unrestricted frameworks contain more realizable sets then their finite counterparts since finite AFs may realize finite as well as finitely many extensions only. The following definition formally captures all considered types of signatures (cf.\ Definitions~\ref{def:signatures_finite},~\ref{def:signatures_finite_compact}~and~\ref{def:signatures_finite_analytic}) without any finite assumption.

\begin{definition} 
Given a semantics $\sigma$. We call the set $S$ the 
\begin{enumerate}
	\item (\textit{unrestricted}) $\sigma$-\textit{signature} if $S = \left\{\sigma(\F) \mid \F\in\m{F}\right\}$ abbreviated by $\Sigma_{\sigma}$,
	\item \textit{compact} $\sigma$-\textit{signature} if $S = \left\{\sigma(\F) \mid \F\in\CAF\right\}$ abbreviated by $\Sigma_{\sigma}^c$ and
	\item \textit{analytic} $\sigma$-\textit{signature} if $S = \left\{\sigma(\F) \mid \F\in\XAF\right\}$ abbreviated by $\Sigma_{\sigma}^x$.
\end{enumerate}
\end{definition}

In \cite{BauS17} the authors were interested in a comparison of the expressive power of several mature semantics in the unrestricted setting. The following result shows that the relation between unrestricted signatures is intimately connected to their relation in case of finite, compact signatures. More precisely, non-empty relative complements in case of finite, compact signatures between two semantics carry over to their unrestricted versions. The main reason for this relation is the fact that unrestricted frameworks may contain any available argument of the universe \m{U}.  

\begin{theorem}[\cite{BauS17}]
  \label{thm:sigtausig}
  Given two semantics 
  $\sigma$ and $\tau$ with $\sigma,\tau\in$\linebreak $\{\nav,\stb,\stg,\semi,\prf,\com,\grd,\id,\eag,\cf,\adm\}$ we
  have:

	\begin{enumerate}
		\item If
  $\Sigma_{\Ext_\sigma}^{f,c}\setminus\Sigma_{\Ext_\tau}^{f,c}\neq\emptyset$, then 
 $\Sigma_{\Ext_\sigma}^c\setminus\Sigma_{\Ext_\tau}^c\neq\emptyset$ and 
\item If
  $\Sigma_{\Ext_\sigma}^c\setminus\Sigma_{\Ext_\tau}^c\neq\emptyset$, then 
 $\Sigma_{\Ext_\sigma}\setminus\Sigma_{\Ext_\tau}\neq\emptyset$. 
	\end{enumerate}
	
\end{theorem}

The following example illustrates the main proof idea.

\begin{example}
Let $\m{E}\in\Sigma_{\Ext_\prf}^{f,c} \sm \Sigma_{\Ext_\stb}^{f,c}$ (cf.\ Figure~\ref{fig:venn2}) and $\F = (A,R)$ a witnessing framework. This means, $\F$ is finite, $\Ext_\prf(\F)=\m{E}$ and $\prf$-compact, i.e.\ $\bigcup\m{E} = A$. Consider now $\AH=(\m{U},R)$. Obviously,  
  $\m{E}' = \Ext_\prf(\AH) = \{E\cup(\m{U}\setminus A):E\in\m{E}\}$ and $\bigcup\m{E}' = \m{U}$ showing the $\sigma$-compactness of $\AH$. In particular, $\m{E}'\in\Sigma_{\Ext_\prf}^{c}$. Note that any $\stb$-realization of $\m{E}'$ has to be compact too since there
  are no additional arguments available. Assume $\m{E}'\in\Sigma_{\Ext_\stb}^{c}$, i.e.\ there is an AF $\AG' = (\m{U},R')$, s.t.\ $\Ext_\stb(\AG') = \m{E}'$.
  We observe that due to conflict-freeness there can not be attacks in $\G'$
  between arguments from $A$ and $\m{U}\setminus A$ nor
  between any of the arguments from $\m{U}\setminus A$. Consequently, $\G=(A,R')$ is finite, $\Ext_\stb(\G)=\m{E}$ and $\stb$-compact implying that $\m{E}\in\Sigma_{\Ext_\stb}^{f,c}$ in 
  contradiction to the initial assumption.
	
\end{example}

Now we are prepared for a comparison in case of unrestricted frameworks. Ignoring the superscripts in Figure~\ref{fig:venn2} provides you with a graphical representation for selected semantics.

\begin{theorem}
  \label{thm:unrstrsig}
  For unrestricted signatures the following hold:

  \begin{enumerate}
    \item $\{\{E\} \mid E\subseteq\m{U}\} = \Sigma_{\Ext_\sigma} \subset \Sigma_{\Ext_\nav} \subset \Sigma_{\Ext_\tau}$ for $\sigma\in\{\grd,\id\}$, $\tau \in \{\stb,\stg,\semi,\prf\}$,
    \item $\Sigma_{\Ext_\eag} \subset \Sigma_{\Ext_\prf}$,
    \item $\Sigma_{\Ext_\stb} \subset \Sigma_{\Ext_\sigma}$ for $\sigma \in \{\stg,\semi\}$,
    \item $\Sigma_{\Ext_\prf} \setminus (\Sigma_{\Ext_\stb} \cup \Sigma_{\Ext_\semi} \cup \Sigma_{\Ext_\stg}) \neq \emptyset$,
    \item $\Sigma_{\Ext_\stg} \setminus (\Sigma_{\Ext_\stb} \cup \Sigma_{\Ext_\prf} \cup \Sigma_{\Ext_\semi}) \neq \emptyset$,
    \item $\Sigma_{\Ext_\stb} \setminus \Sigma_{\Ext_\prf} \neq \emptyset$,
    \item $\Sigma_{\Ext_\semi} \setminus (\Sigma_{\Ext_\stb} \cup \Sigma_{\Ext_\prf} \cup \Sigma_{\Ext_\stg}) \neq \emptyset$,
    \item $\Sigma_{\Ext_\com} \sm \Sigma_{\Ext_\sigma} \neq\emptyset$ and $\Sigma_{\Ext_\sigma} \sm \Sigma_{\Ext_\com} \neq\emptyset$ for $\sigma\in\{\cf,\adm\}$,
    \item$\Sigma_{\Ext_\cf}\subset\Sigma_{\Ext_\adm}$.
  \end{enumerate}
\end{theorem}

Finally, we briefly consider the  closely related topic of \textit{intertranslatability}. Intertranslatability revolves around the idea of mapping one semantics to another. One main motivation for studying this issue is the possibility to reuse a solver for one semantics for another \cite{DvoW11}. The main tool for this endeavour are functions mapping AFs to AFs, so-called \textit{translations} formally defined as follows.

\begin{definition}\cite{DvoW11}
  Given two semantics $\sigma,\tau$. A function $f:\m{F}\to\m{F}$ is
  called an \textit{exact translation} for $\sigma\rightarrow\tau$, if
  $\sigma(\F)=\tau(f(\F))$ for any AF $\F$. It is called a \textit{faithful translation} if
  for any AF $\F$ first $\lvert\sigma(\F)\rvert=\lvert\tau(f(\F))\rvert$ and
  second $\sigma(\F)=\{S\cap A(\F)\mid S\in\tau(f(\F))\}$.
\end{definition}

Please note that for some semantics there are no exact translations available due to reasons inherent to those semantics. For instance, preferred semantics satisfies \textit{I-maximality}, i.e.\  for any AF $\F$, $\Ext_\prf(\F)$ forms a $\subseteq$-antichain \cite{BarG07}. This implies that an exact translation $\Ext_\adm\rightarrow\Ext_\prf$ can not exist since for $\F = (\{a\},\emptyset)$ we observe $\{\emptyset,\{a\}\} = \Ext_\adm(\F)$. Sticking to faithful translations provides us with a positive answer if we consider finite AFs only \cite[Translation 3.1.85]{Spanring12}. Interestingly, the considered translation does not serve in the general unrestricted case and interestingly, it was shown that a
search for a suitable translation will never succeed (cf.\ \cite[Example~6]{BauS17}).

The following theorem (a generalization of the finite version from~\cite[Section~6.1]{DvoS16}) establishes a close relation between realizability and intertranslatability as promised, namely: if $\tau$ is not less expressive than $\sigma$, then $\sigma$ can be exactly translated to $\tau$ and vice versa. 

\begin{theorem}[\cite{BauS17}] \label{the:bridgesig}
  Given semantics $\sigma$, $\tau$. We have: $\Sigma_\sigma\subseteq\Sigma_\tau$
  if and only if there is an exact translation for $\sigma\rightarrow\tau$.
\end{theorem}

The following Figure~\ref{fig:intertrans} illustrates translational (im)possibilities in an eye-catching way. 
Figure~\ref{fig:intertrans.restricted} summarizes known results regarding faithful translations 
 in the finite case~\cite{DvoW11,Spanring12,DvoS16}, augmented with obvious
insights for unique status semantics $\id$ and $\eag$.
For semantics $\sigma,\tau$, encirclement in the same component indicates
bidirectional translations.
An arrow from $\sigma$ to $\tau$ means directional translations.
If there is no directed path (for instance for $\nav$ to $\cf$, or for $\cf$ to
$\grd$), then there is no translation. Figure~\ref{fig:intertrans.unrestricted} features the same visualization for
unrestricted AFs. 
Dropping the finiteness restriction has some further consequences for the
considered semantics, namely exact and faithful intertranslatability coincide. 
It is an open question whether both forms of translations are essentially the same in the general unrestricted setting.
In consideration of Theorem~\ref{the:bridgesig} we may interpret
Figure~\ref{fig:intertrans.unrestricted} as a comparison of the expressiveness
of the considered semantics. 
That is, $\Sigma_\sigma\subset\Sigma_\tau$ if and only if there is a directed path from $\sigma$ to~$\tau$.

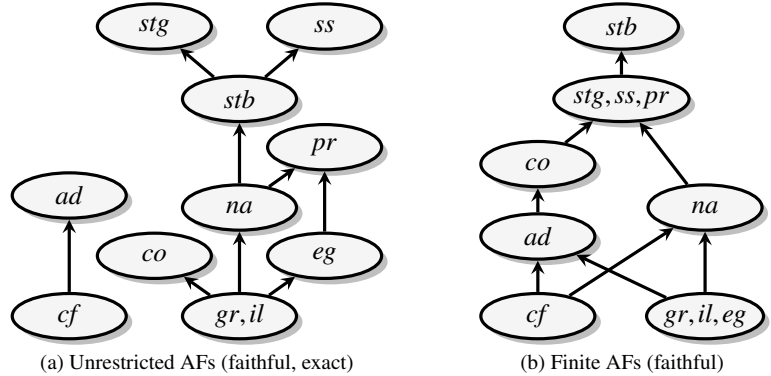
\begin{figure}[h]
  \centering
\subfloat[Unrestricted AFs (faithful, exact)]{
\label{fig:intertrans.unrestricted}
\begin{tikzpicture}[>=stealth,xscale=.8,yscale=.8]
    \path
      (0,0)
      ++(0,0) node[box](cf){$\cf$}
      ++(90:2) node[box](ad){$\adm$}
      ;
    \path
      (2.8,0)   node[box](gr){$\grd,\id$}
      +(-1.4, 1)node[box](co){$\com$}
      ++(1.4, 1)node[box](eg){$\eag$}
       +( 0,1.8)node[box](pr){$\prf$}
      (2.8,0)    
      ++( 0,1.8)node[box](na){$\nav$}
      ++( 0,1.8)node[box](st){$\stb$}
      +(-1.4,1.2)node[box](sg){$\stg$}
      +( 1.4,1.2)node[box](sm){$\semi$}
      ;
    \tikzstyle{trans}=[left,->,very thick]
    \path [trans]
      (cf) edge (ad)
      (gr) edge (eg)
           edge (co)
	   edge (na)
      (eg) edge (pr)
      (na) edge (pr)
           edge (st)
      (st) edge (sg)
           edge (sm)
      ;
\end{tikzpicture}
} \hspace{1cm}
\subfloat[Finite AFs (faithful)]{
\label{fig:intertrans.restricted}
\begin{tikzpicture}[>=stealth,yscale=.48,xscale=1.1]
  \path 
    (2,0)node[box](ground){$\grd,\id,\eag$}
    (2,3)node[box](naive){$\nav$}
    (0,0)node[box](cf){$\cf$}
    (0,2.2)node[box](ad){$\adm$}
    (0,4.2)node[box](co){$\com$}
    (1,6)node[box](others){$\stg,\semi,\prf$}
    (1,8)node[box](stable){$\stable$};
    ;
  \tikzstyle{trans}=[left,->,very thick]
  \path [trans]
    (ground) edge (naive)
    (cf) edge (naive)
    (cf) edge (ad)
    (ad) edge (co)
    (ground) edge (ad)
    (co) edge (others)
    (naive) edge (others)
    (others) edge (stable);
\end{tikzpicture}
}

\caption{Translational (Im)Possibilities}
\label{fig:intertrans}
\end{figure}

As a final note, in contrast to the unrestricted setting Baumann and Spanring observed that for slightly restricted AFs $\F = (A,R)$, s.t.\ 
$\card{A} \leq \card{\m{U}\setminus A}$ it is possible to provide exact and efficiently computable translations from preferred to semi-stable semantics
via $f(\F) = \F' = (A',R')$ with $A'=A\cup\{a'\!\mid\! a\in A\}$ and
$R' = R \cup \{(a,a'),(a',a')\!\mid\! a\in A\}$. It is an interesting question whether this restriction allows for similar translational possibilities as in case of finite AFs. 

\section{Realizability and Signatures for Labelling-based Versions}

Although any considered semantics $\sigma$ possesses an extension-based version (indicated by $\Ext_\sigma$) as well as a closely related 3-valued labelling-based version (indicated by $\Lab_\sigma$) we formally have that both versions are different semantics (or more precisely, functions) in the sense of Definition~\ref{def:semantics}. This formal difference has some impact on realizability as well as signatures. Let us consider realizability in the realm of finiteness. As a matter of fact, for any considered 3-valued labelling-based version $\Lab_\sigma$ we have: if $\F = (A,R)$ and $\L = (\LI,\LO,\LU) \in \Lab_\sigma(\F)$, then $A = \LI \cup \LO \cup \LU$. This means, $\sigma$-labellings assign a status to any argument in $\F$. Now, in case of finite AFs we know that potentially realizable sets of labellings have to involve finitely many arguments only. Moreover, these finitely many arguments precisely determine the set $A$ of witnessing AFs $\F = (A,R)$.\footnote{This is exactly the point which does not carry over to finite realizability in case of extension-based semantics (cf.\ statement 2 of Theorem~\ref{prop:compact_signatures}).} Consider therefore the following example. 

\begin{example} Consider the following set of 3-valued labellings  $\Ss = \left\{(\{a\},\emptyset,\{b,c\}),(\{a,b\},\{c\},\emptyset)\right\}$. Is $\Ss$ $\com$-realizable? Since $\{a\} \cup \emptyset \cup \{b,c\} = \{a,b,c\}$ we deduce that candidates have to be members of $\m{C} = \{\F = (A,R) \mid A = \{a,b,c\}\}$. Note that $\card{\m{C}} = 2^{\card{\{a,b,c\}}^2} = 2^{9} = 512$. Clearly, this is a huge number, but it is a finite one. Consequently, the question of realizability can be decided by computing the $\sigma$-labellings of all AFs in $\m{C}$. Of course, any intelligent search algorithm would involve further information like $\{a,b\}$ has to be conflict-free in a witnessing AF. Such an observation would decrease the number of candidates to $2^5 = 32$. However, in both cases one would find the unique witnessing framework $\F$, i.e.\ $\Lab_{\com}(\F) = \Ss$, as depicted below.
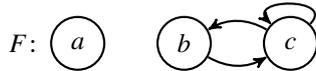
\vspace{1mm}

\begin{tikzpicture}

    \node (A1'') at (0,-2.5) [circle,minimum size=0.7cm, thick, draw, label = left:$\F\!:$]{$a$};
    \node (B1'') at (1.4,-2.5) [circle, minimum size=0.7cm, thick, draw]{$b$};
		\node (C1'') at (2.8,-2.5) [circle, minimum size=0.7cm, thick, draw]{$c$};


\draw[->,thick] (C1'') to [loop,thick, distance=0.5cm] (C1'');
\draw[->, thick] (B1'') to [thick,bend right] (C1'');
\draw[->, thick] (C1'') to [thick,bend right] (B1'');

\end{tikzpicture} 

\end{example}

The example above shows that the search space can be very large even in case of small numbers of arguments. Consequently, locally verifiable necessary as well as sufficient properties for realizability just like in case of extension-based semantics are of high interest too. To the best of our knowledge only two papers have dealt with labelling-based realizability in the context of AFs. The first study was presented by Dyrkolbotn \cite{Dyr14}. The author showed that, as long as additional arguments are allowed any finite set of labellings is realizable under preferred and semi-stable semantics. It is important to emphasize that Dyrkolbotn uses a more relaxed notion of realizibility, namely \textit{realizibility under projection} (cf.\ Definition~\ref{def:realizableprojection}). The other work \cite{LinPS16} deals with the standard notion of finite realizability (Definition~\ref{def:realizable_finite}). The authors presented an algorithm which returns either ``No'' in case of non-realizibility or a witnessing AF $\F$ in the positive case. The algorithm is not purely a guess-and-check method since it also includes a propagation step where certain necessary properties of witnessing AFs are processed. Remarkably, the algorithm is not restricted to the formalism of abstract argumentation frameworks only. In fact, it can also be used to decide realizability in case of the more general abstract dialectical frameworks as well as various of its sub-classes \cite{ADF,ADF2}.

\subsection{Preliminary Results for Labelling-based Signatures}

In the following we shed light on general relations between the labelling-based and extension-based signatures of the considered semantics. Fortunately, due to former characterization results we will even achieve characterizing or at least necessary properties for finite realizability regarding labelling-based versions. We proceed with the definition of an \textit{labelling-set} which is the $n$-valued analogon (for $n \geq 2$) to an extension-set as introduced in Definition~\ref{def:extension-set}. A labelling-set is a finite set of $n$-tuples which are dealing with the same set of arguments and moreover, any $n$-tuple assigns exactly one status to each argument in question. 

\begin{definition}
Given $\Ss \subseteq \left(2^\m{U}\right)^n$. $\Args_\Ss$ denotes $\bigcup_{\L = (\L_1,...,\L_n)\in\Ss} \bigcup_{i=1}^n \L_i$ and $\Card{\Ss}$ stands for $\card{\Args_\Ss}$. We say that $\Ss$ is a \textit{labelling-set} if 

\begin{enumerate}
	\item $\Card{\Ss}$ is a finite cardinal,
	\item for any $\L = (\L_1,\ldots,\L_n) \in\Ss$, $\Args_\Ss = \bigcup_{i=1}^n \L_i$ and
	\item for any $\L = (\L_1,\ldots,\L_n) \in\Ss$, $\L_1,\ldots,\L_n$ are pairwise disjoint.
	
\end{enumerate}
 \end{definition}

 The following proposition establishes a connection between extension-based and labelling-based realizibility for any considered semantics. Roughly speaking, it states that labelling-based realizability requires extension-based realizability of the corresponding sets of in-labelled arguments. For a $3$-tuple $\L = (\L_1,\L_2,\L_3)$ we also write $(\LI,\LO,\LU)$ as usual.

\begin{proposition} \label{pro:signaturerel} Given a set of 3-tuples $\Ss \subseteq \left(2^\m{U}\right)^3$. For any semantics \\ $\sigma\in\{\stb,\semi,\stg,\cfzwei,\stgzwei,\prf,\adm,\com,\grd,\id,\eag,\nav,\cf\}$ we have,

\begin{enumerate}
	\item $\Ss\in\Sigma_{\Lab_{\sigma}} \To \{\LI \mid \L\in\Ss\}\in\Sigma_{\Ext_{\sigma}}$ \hfill{(unrestricted realizability)}
	\item $\Ss\in\Sigma_{\Lab_{\sigma}}^c \To \{\LI \mid \L\in\Ss\}\in\Sigma_{\Ext_{\sigma}}^c$ \hfill{(compact realizability)}
	\item $\Ss\in\Sigma_{\Lab_{\sigma}}^x \To \{\LI \mid \L\in\Ss\}\in\Sigma_{\Ext_{\sigma}}^x$ \hfill{(analytic realizability)}
	\item $\Ss\in\Sigma_{\Lab_{\sigma}}^f \To \{\LI \mid \L\in\Ss\}\in\Sigma_{\Ext_{\sigma}}^f$ \hfill{(finite realizability)}
	\item $\Ss\in\Sigma_{\Lab_{\sigma}}^{f,c} \To \{\LI \mid \L\in\Ss\}\in\Sigma_{\Ext_{\sigma}}^{f,c}$ \hfill{(finite, compact realizability)}
	\item $\Ss\in\Sigma_{\Lab_{\sigma}}^{f,x} \To \{\LI \mid \L\in\Ss\}\in\Sigma_{\Ext_{\sigma}}^{f,x}$ \hfill{(finite, analytic realizability)}
\end{enumerate}

\end{proposition}

Please note that the implications above are justified for any semantics $\sigma$ whenever the different versions of it satisfy $\Ext_{\sigma}(\F) = \{\LI \mid \L\in\Lab_{\sigma}(\F)\}$ for any relevant AF $\F$. In the former sections we already presented characterization theorems or at least necessary properties for being finitely realizable regarding extension-based versions (cf.\ Theorems~\ref{the:charcf},~\ref{the:charadm}~and~\ref{the:charid}). Combining these results with the proposition above yields the following necessary properties for finite realizability in the labelling-based case. Note that the mentioned implications apply to finite, compact as well as finite, analytic signatures too since $\Sigma_{\Lab_{\sigma}}^{f,c} \subseteq \Sigma_{\Lab_{\sigma}}^f$ as well as $\Sigma_{\Lab_{\sigma}}^{f,x} \subseteq \Sigma_{\Lab_{\sigma}}^f$ by definition. In case of grounded, ideal and eager semantics we have that being an one-element labelling-set is necessary and even sufficient for being finitely realizable. One may easily verify that the only-if-directions of these semantics are justified by the witnessing framework $\F_\L = (\LI\cup\LO\cup\LU, \{(i,o)\mid i\in\LI,o\in\LO\} \cup \{(u,u)\mid u\in\LU\})$ given that $\Ss = \{\L\}$. 

\begin{theorem}
\label{the:charlabellingbased}
Given a set of 3-tuples $\Ss \subseteq \left(2^\m{U}\right)^3$, then
\begin{enumerate}
\item $\Ss\in \Sigma_{\Lab_{\cf}}^f \To \{\LI\mid \L\in\Ss\} \text{ is a non-empty, downward-closed and}$ tight extension-set,
\item $\Ss\in \Sigma_{\Lab_{\nav}}^f \To \{\LI\mid \L\in\Ss\}\text{ is a non-empty, incomparable extension-set}$ and $\dcl{(\Ss)} \text{ is tight},$
\item $\Ss\in \Sigma_{\Lab_{\grd}}^f \ToT \Ss \text{ is a labelling-set with} \card{\Ss} = 1,$
\item $\Ss\in \Sigma_{\Lab_{\id}}^f  \ToT \Ss \text{ is a labelling-set with} \card{\Ss} = 1,$ 
\item $\Ss\in \Sigma_{\Lab_{\eag}}^f \ToT \Ss \text{ is a labelling-set with} \card{\Ss} = 1,$ 
\item $\Ss\in \Sigma_{\Lab_{\stb}}^f  \To \{\LI\mid \L\in\Ss\} \text{ is a incomparable and tight extension-set,}$
\item $\Ss\in \Sigma_{\Lab_{\stg}}^f  \To \{\LI\mid \L\in\Ss\} \text{ is a non-empty, incomparable and}$ tight extension-set,
\item $\Ss\in \Sigma_{\Lab_{\adm}}^f  \To \{\LI\mid \L\in\Ss\} \text{ is a conflict-sensitive extension set containing }\emptyset,$
\item $\Ss\in \Sigma_{\Lab_{\prf}}^f  \To \{\LI\mid \L\in\Ss\} \text{ is a non-empty, incomparable and}$ conflict-sensitive extension-set,
\item $\Ss\in \Sigma_{\Lab_{\semi}}^f  \To \{\LI\mid \L\in\Ss\} \text{ is a non-empty, incomparable and}$conflict-sensitive extension-set.

\end{enumerate}
\end{theorem}

\subsection{Realizibility under Projection}
\label{sec:sjur}
We turn now to \textit{realizability under projection} which was firstly considered in \cite{Dyr14}. In order to realize a set of labellings $\Ss$ under projection it suffices to come up with an AF $\F$, s.t.\ its set of labellings restricted to the relevant arguments coincide with~$\Ss$. Consider therefore the following illustrating example.

\begin{example} \label{ex:labelprojection} Given $\Ss = \{(\{a\},\{b\},\emptyset),(\{b\},\{a\},\emptyset),(\emptyset,\{a,b\},\emptyset)\}$. We observe that the corresponding set of sets of in-labelled arguments $\Ss^I = \{\emptyset,\{a\},\{b\}\}$ violates incomparability. Thus, applying statement~9 of Theorem~\ref{the:charlabellingbased} we derive that $\Ss$ is not finitely $\prf$-realizable. Consider now the following AF~$\F$.

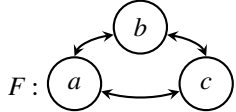
\begin{tikzpicture}[xscale=2,>=stealth]
      \path
      (-30:.5) node[arg] (x1) {$c$}
     
			( 90:.5) node[arg] (x2) {$b$}
     
      (210:.5) node[arg] (x3) {$a$}
      
			(201:0.8) node (F) {$\F:\ $}
      ;
      \path[<->,thick,bend right]
      (x1) edge (x2)
      (x2) edge (x3)
      (x3) edge (x1)
      ;
    \end{tikzpicture}
		
We obtain $\Lab_{\prf}(\F) = \{(\{a\},\{b,c\},\emptyset),(\{b\},\{a,c\},\emptyset),(\{c\},\{a,b\},\emptyset)\}$. Now, if we restrict any labelling $\L = (\LI,\LO,\LU) \in\Lab_{\prf}(\F)$ to the arguments $a$ and $b$, i.e.\ $\L|_{\{a,b\}} = (\LI\cap\{a,b\},\LO\cap\{a,b\},\LU\cap\{a,b\})$ we obtain exactly all labellings in $\Ss$. In this sense, $\Ss$ is $\prf$-realizable under projection. 
\end{example}  

We proceed with the formal definitions. For the sake of completeness we introduce realizability under projection and its corresponding signatures w.r.t.\ any kind of semantics as defined in Definition~\ref{def:semantics}. 

\begin{definition}\label{def:realizableprojection}
Given a semantics $\sigma: \m{F}\rightarrow 2^{\left(2^\m{U}\right)^n}$. A set $\Ss\subseteq \left(2^\m{U}\right)^n$ is $\sigma$-\textit{realizable under projection} if there is an AF $\F$, s.t.\ $\sigma(F)|_{\Args_\Ss} = \{E|_{\Args_\Ss}\mid E\in\Ext_{\sigma}(\F)\} = \Ss$ (in case of $n = 1$) or $\sigma(\F)|_{\Args_\Ss} = \{\L|_{\Args_\Ss}\mid \L\in\Lab_{\sigma}(\F)\} = \Ss$ (for $n \geq 2$), respectively.
\end{definition}

\begin{definition} \label{def:signaturesprojection}
Given a semantics $\sigma$. The unrestricted as well as finite $\sigma$-projection-signatures are defined as follows: 

\begin{enumerate}
	\item $\Sigma_{\sigma}^p = \left\{\sigma(\F)|_{B} \mid \F = (A,R)\in\m{F}, B\subseteq A \right\}$ and
	\item $\Sigma_{\sigma}^{f,p} = \left\{\sigma(\F)|_{B} \mid \F = (A,R)\in\m{F}, \F \text{ is finite, } B\subseteq A \right\}$
	
\end{enumerate}

\end{definition}

Analogously to Proposition~\ref{pro:signaturerel} we state the following relation between labelling-based and extension-based versions of the considered semantics.

\begin{proposition} \label{pro:signaturerelpro} Given a set of 3-tuples $\Ss \subseteq \left(2^\m{U}\right)^3$. For any semantics \\ $\sigma\in\{\stb,\semi,\stg,\cfzwei,\stgzwei,\prf,\adm,\com,\grd,\id,\eag,\nav,\cf\}$ we have,
\begin{enumerate}
	\item $\Ss\in\Sigma_{\Lab_{\sigma}}^p \To \{\LI \mid \L\in\Ss\}\in\Sigma_{\Ext_{\sigma}}^p$ \hfill{(unrestricted realizability under projection)}
	\item $\Ss\in\Sigma_{\Lab_{\sigma}}^{f,p} \To \{\LI \mid \L\in\Ss\}\in\Sigma_{\Ext_{\sigma}}^{f,p}$ \hfill{(finite realizability under projection)}
\end{enumerate}

\end{proposition}

As a matter of fact, any projection signature is a superset of the corresponding signature. The following question then arises naturally: how much more sets can be generated if we stick to realizability under projection? For instance, we have already seen that even comparable sets are realizable under projection by semantics satisfying incomparability (Example~\ref{ex:labelprojection}). It was the main result in \cite[Theorem 3.1]{Dyr14} that in case of semi-stable and preferred semantics indeed any $3$-valued labelling-set is finitely realizable under projection. The proof relies on two basic constructions. The first step \textit{generates} an AF, consisting of so-called \textit{circuits}, s.t.\ its set of preferred as well as semi-stable labellings restricted to the relevant arguments contains any possible labelling. The second construction \textit{eliminates} undesired labellings step by step. Combining this realizability result with statement 2 of Proposition~\ref{pro:signaturerelpro} yields the following theorem. 

\begin{theorem} Let $\sigma\in\{\prf,\semi\}$. We have,

\begin{enumerate}
\item $\Sigma_{\Lab_{\sigma}}^{f,p} = \left\{\Ss\subseteq \left(2^{\m{U}}\right)^3 \mid \Ss \text{ is a labelling-set} \right\}$ and
	\item $\Sigma_{\Ext_{\sigma}}^{f,p} = \left\{\Ss\subseteq 2^{\m{U}} \mid \Ss \text{ is an extension-set} \right\}$.
	
\end{enumerate} 
\end{theorem}

\section{Final Remarks and Conclusion} \label{sec:sumcon2}

We have dealt with different forms of realizability in the context of abstract argumentation frameworks. In accordance with the existing literature the main part of
this section was devoted to finite realizability for extension-based semantics. However, for any semantics $\sigma$ we may state the following general subset relations depicted as Venn-diagram.

\begin{figure}[H]
  \centering
  \def\firstcircle{(0,0) circle (1.5cm)}
  \def\secondcircle{(0:2cm) circle (1.5cm)}
  \def\xshape{(-0.6cm,0) ellipse (2.7cm and 1.5cm)}
	\def\prfshape{(0.6,0) ellipse (2.7cm and 1.5cm)}
  \def\stgshape{(0cm,0) ellipse (4.2cm and 2.2cm)}
  \def\semishape{(0,0) ellipse (2.1cm and 1.35cm)}
  \def\stbshape{(0,-0.5) ellipse (1.1cm and 0.6cm)}
  \def\navshape{(0,0.5) ellipse (1.1cm and 0.6cm)}
  \colorlet{circle edge}{gray!50}
  \colorlet{circle area}{gray!50}
  \tikzset{filled/.style={fill=circle area, draw=circle edge, thick,opacity=.5},
      outline/.style={draw=circle edge, thick}}
  \begin{tikzpicture}
    \fill[filled,fill=gray!10] \stgshape;
		\fill[filled,fill=gray!25] \prfshape;
		\fill[filled,gray!40] \xshape;
    \fill[filled,fill=gray!55] \semishape;
    \fill[filled,gray!10] \navshape;
    \fill[filled,gray!25] \stbshape;

    \draw[outline] \prfshape node(prs) {};
    \draw[outline] \stgshape node(sgs) {};
    \draw[outline] \semishape node(sms) {};
    \draw[outline] \stbshape node(sbs) {};
		\draw[outline] \xshape node(sbs) {};
    \draw[outline] \navshape node (nas) {$\Sigma_{\sigma}^{f,c}$};
    \node[anchor=west] at (2.2cm,0) {$\Sigma_{\sigma}^{f,p}$};
    \node[anchor=west] at (1.2cm,0) {$\Sigma_{\sigma}^f$};
		\node[anchor=west] at (-3cm,0) {$\Sigma_{\sigma}$};
		\node[anchor=west] at (-4.05cm,0) {$\Sigma_{\sigma}^p$};
	\node at (0,-0.5cm) {$\Sigma_{\sigma}^{f,x}$};
  \end{tikzpicture}
  \caption{Subset Relations between Different Kinds of Signatures}
  \label{fig:venn_analytic}
\end{figure}
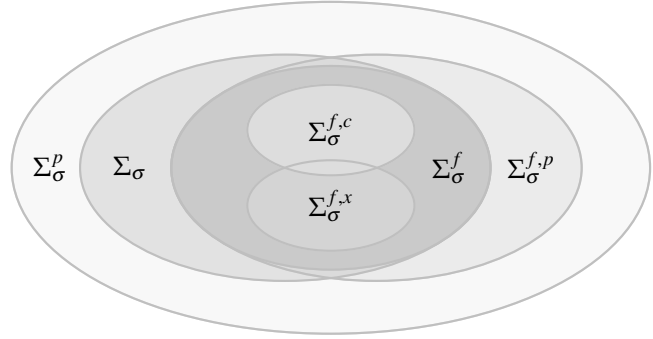

In case of the extension-based versions of naive, grounded, ideal, eager, stable, stage, preferred and semi-stable semantics as well as conflict-free and admissible sets we provided exact characterizations for their corresponding general signatures. We have seen that for some semantics we do not lose any expressive power if sticking to compact or analytic AFs, i.e.\ $\Sigma_\sigma^f = \Sigma_\sigma^{f,c}$ or $\Sigma_\sigma^f = \Sigma_\sigma^{f,x}$, respectively. However, for certain prominent semantics, e.g.\ preferred semantics we have that the expressive power indeed relies on the use of rejected arguments or implicit conflicts. For such semantics, it remains an open problem to present exact characterizations for finite, compact or finite, analytic realizability, respectively. 
In case of labelling-based versions of semi-stable and preferred semantics we have seen that any labelling-set is realizable under projection. In \cite{Dyr14} it was already noted that this equality does not hold for any semantics. For instance, the empty labelling is admissible for any AF $\F$. Hence, in case of admissible semantics, no labelling-set is realizable under projection if it fails to include the empty labelling. 

Finally, let us mention some computational issues not considered so far. It can be said that on the one hand, the classes of finite, compact and finite, analytic 
provide computational benefits both in practice and in terms of theoretical 
worst-case analysis. On the other hand testing for membership in one of
the classes is, for most of the semantics, of rather high complexity and thus 
these classes cannot be directly used to improve systems. We refer the interested reader to \cite{compactj} for more details. Moreover, in general, given an extension-set $\Ss$, deciding whether $\Ss$ is compactly realizable
is a hard problem, that is, by definition of the decision problem there are no good reasons to believe that we can do
any better than guessing a compact AF and checking whether its extension-set
coincides with~$\Ss$. Nevertheless, for some semantics we have seen that finite, compact realizability can be characterized locally, i.e.\ by properties of $\Ss$ itself (as shown in  Theorem~\ref{the:charcompactsig}). In this case, finite, compact realizability can be checked in polynomial time as for standard finite realizability \cite[Theorem 6]{DunDLW15}. Moreover, in \cite{compactj} a huge number of shortcuts
to detect non-compactness are provided. By shortcut we mean a property of the given
extension-set $\Ss$ that is easily computable (preferably in polynomial time) which (sometimes)
provides us with a definitive answer to the decision problem. These
shortcuts are related to numerical aspects of argumentation frameworks like results concerning maximal number of extensions \cite{characteristic}.

\chapter{Replaceability of Subframeworks} \label{cha:replace}

Given a certain logical formalism $\m{L}$ and two syntactically different $\m{L}$-theories $T_1$ and $T_2$. One central question is whether, and if so, how to decide that these $\m{L}$-theories represent the same information? Of course, in order to answer this question we have to clarify what we exactly mean by sharing the same information first. Note that there is neither a uniquely determined, nor a certain preferred interpretation by the formalism $\m{L}$ itself. For instance, equating information with possessing the same semantics yields to the well-known notion of \textit{ordinary} or \textit{standard equivalence}. This means, assuming that $\sigma_{\m{L}}$ is the semantics of $\m{L}$ we might answer that $T_1$ and $T_2$ are equivalent if and only if $\sigma_{\m{L}}(T_1) = \sigma_{\m{L}}(T_2)$. A more demanding interpretation of sharing the same information is to require that $T_1$ and $T_2$ are semantically indistinguishable even if further $\m{L}$-theories $T$ are added to both simultaneously. More formally, we may state: $T_1$ and $T_2$ are considered to be equivalent if and only if $\sigma_{\m{L}}(T_1 \cup T) = \sigma_{\m{L}}(T_2 \cup T)$ for any theory~$T$. This notion is known as \textit{strong equivalence} and is of high interest for any logical formalism since it allows one to locally replace, and thus give rise for simplification, parts of a given theory without changing the semantics of the latter. In contrast to classical logics where standard and strong equivalence coincide (cf.\ Chapter~\ref{cha:intersect}), it is possible to find ordinary but not strongly equivalent objects for any nonmonotonic formalism available in the literature. Consequently, much effort has been devoted to characterizing strong equivalence for nonmonotonic formalisms such as logic programs~\cite{DBLP:journals/tocl/LifschitzPV01}, causal theories \cite{DBLP:conf/lpnmr/Turner04}, default logic \cite{Turner01} as well as nonmonotonic logics in general~\cite{DBLP:journals/amai/Truszczynski06}. 

In \cite{strong} the authors introduced the notion of strong equivalence for abstract AFs. They provided a series of characterization theorems for deciding strong equivalence of two AFs with respect to several semantics. In view of the fact that strong equivalence is defined semantically it is the main and quite surprisingly insight that being strongly equivalent can be decided syntactically. More precisely, they introduced the notion of a \textit{kernel} of an AF $\AF$ which is (informally speaking) a subgraph of $\AF$ where certain attacks are deleted and showed that syntactical identity of suitably chosen kernels characterizes strong equivalence w.r.t.\ the considered semantics. Strong equivalence is, as its name suggests, a very (and often unnecessarily to) strong notion of equivalence if dynamic evolvements are considered. 
In many argumentation scenarios the type of modification which may potentially occur
can be anticipated and furthermore, more importantly, does not range over \textit{arbitrary expansions} as required for strong equivalence. 
Let us consider the instantiation-based context where AFs are built from an underlying knowledge base. Here, we typically observe that older arguments and their corresponding attacks survive and only new arguments which may interact with the previous ones arise given that a new
piece of information is added to the underlying knowledge base. This type of dynamic evolvement is a so-called \textit{normal expansion} and its corresponding equivalence notion were firstly studied in \cite{normal}. 
Over the last five years several equivalence notions taking into account specific types of evolvements reflecting the nature of various argumentation scenarios were defined and characterized. The considered dynamic scenarios range from the most general form, so-called
\textit{updates} \cite{Bau14} where arguments and attacks can be deleted and added to different types of \textit{expansions} \cite{strong,normal,Bau10} and \textit{deletions} \cite{Bau14} where arguments and/or attacks are allowed to be added or deleted in a certain way only. 

Into the year 2015 all characterization theorems were stated in terms of extension-based semantics. Recently, Baumann presents their labelling-based counterparts and showed that, although labelling-based semantics contain more information then there extension-based counterpart, there is a majority
of equivalence relations where labelling-based and extension-based versions coincide \cite{Bau16}. Even more recently, a first consideration of strong equivalence regarding unrestricted frameworks were presented in \cite{BauS17}. It turned out that there are no characterizational differences compared to the finite case as long as the AFs in question are \textit{jointly expandable}, i.e.\ that the existence of fresh arguments is guaranteed.

Another approach somehow complementary to the ones mentioned before is presented in \cite{inputoutput} where sharing the same information is interpreted as possessing the same Input/Output behavior. Roughly speaking, the main idea is to consider an argumentation framework as a kind of black box which receives some input from the external world (i.e, a set of external arguments) via incoming attacks and produces an output to the external world via outgoing attacks. Such an interacting module is called an \textit{argumentation multipole}. Two multipoles connected with the same external world are considered as \textit{Input/Output equivalent} if the effects, i.e.\ the produced labellings for external arguments are the same for any reasonable input-labelling. This notion yields the possibility of replacing a
multipole with another one embedded in a larger framework without affecting the
labellings of the unmodified part of the initial framework. The interested reader is referred to \cite{BarGL18} for further information. In the following we shed light on equivalence notions induced by certain dynamic scenarios.

\section{Dynamic Scenarios and Corresponding Equivalence Notions}

There are two main classes of dynamic scenarios, namely \textit{expansions} and \textit{deletions}. Both of them can be further divided in \textit{normal} and \textit{local} versions. These scenarios are motivated by real-world argumentation as well as instantiation-based argumentation \cite{cameva}. For instance, let us consider the dynamics of a discussion or dispute illustrated by the following citation \cite{BesH09}:

\begin{quote}
How does argumentation usually take place? Argumentation starts
when an initial argument is put forward, making some claim. An objection
is raised, in the form of a counterargument. The latter is addressed
in turn, eventually giving rise to a counter-counterargument,
if any. And so on.
\end{quote}

This means, in order to strengthen the own point of view or to rebut the opponents arguments it is natural that one tries to come up with \textit{stronger} arguments, i.e.\ new arguments which are not attacked by the former arguments. This type of dynamics is formally captured by so-called \textit{strong expansions} \cite{Bau10}. The formal counterpart of it, so-called \textit{weak expansions} \cite{Bau10}, where the new arguments do not attack (but may be attacked by) the old ones seem to be more an academic exercise than a task with practical relevance with regard to real-world argumentation.\footnote{We mention that they do play a decisive role w.r.t.\ computational issues, so-called  \textit{splitting methods} (cf.\ \cite{split,tafa,split2}).} Let us turn to instantiation-based argumentation where arguments and attacks stem from an underlying knowledge base (cf.\ Chapter 6 for detailed information as well as Figure 2 in Chapter 4 for an illustration). What happens on the abstract level if a new piece of information is added? It turns out that in almost all deductive
argumentation systems older arguments and their corresponding attacks survive
and only new arguments which may interact with the previous ones arise. This type of dynamic evolvement is formally captured by so-called \textit{normal expansions}. \textit{Local expansions} in contrast, i.e.\ expansions where new attacks are added only correspond to re-instantiations if we change to a less restrictive notion of attack (cf.\ \cite{hunterlink} for different attack notions).

We start with the definition of the different types of expansions together with some introducing examples.

\begin{definition}[\cite{Bau10}] \label{def:expansion}
An AF $\G$ is an \textit{expansion} of AF $\F = (A,R)$ (for short, $\F\preceq_E \G$) iff $\G = (A\ \dot{\cup}\ B, R\ \dot{\cup}\ S)$ for some (maybe empty) sets $B$ and $S$, s.t.\ $A\cap B = R\cap S = \emptyset$. An expansion is called

\begin{enumerate}
    \item  \textit{normal} ($\F\preceq_{N}\!\G$) iff\ $\forall ab\ \left((a,b)\in S \to a\in B \vee b\in B\right)$,
    \item \textit{strong}~($\F\preceq_S\!\G$)~iff~$\F\preceq_{N}\!\G \text{ and }\forall ab\ \left((a,b)\in S \rightarrow \neg(a\in A \wedge b\in B)\right)$,
		\item \textit{weak}~($\F\preceq_W\!\G$)~iff~$\F\preceq_{N}\!\G \text{ and } \forall ab\ \left((a,b)\in S \rightarrow \neg(a\in B \wedge b\in A)\right)$,
		\item \textit{local} ($\F\preceq_{L}\!\G$) iff $B = \emptyset$.
\end{enumerate}
\end{definition}
For short, being a normal expansion means that new attacks
must involve at least one new argument in contrast to local expansions where new attacks involve old arguments only. Moreover, strong and weak expansions are normal and their names refer to properties of the additional arguments, namely arguments which are never attacked by former arguments (so-called \textit{strong} arguments) and arguments which do not attack former arguments (so-called \textit{weak} arguments). 

Observe that any arbitrary expansion can be splitted up in a normal and a local part. This can be nicely seen in the following example.

\begin{example} The AF $\F$ is the initial framework. An arbitrary, normal, strong, weak or local expansion of it are given by $\F_E$, $\F_N$, $\F_S$, $\F_W$ and $\F_L$, respectively. Grey-highlighted arguments or attacks represent added information. 
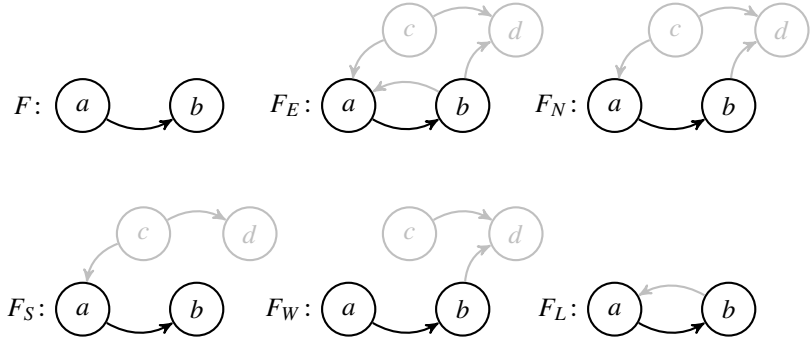
\begin{figure}[H]
\centering
\begin{tikzpicture}
    \node (A1) at (0,0) [circle, minimum size=0.7cm, thick, draw, label = left:$\F\!:$]{$a$};
    \node (B1) at (1.5,0) [circle, minimum size=0.7cm, thick, draw]{$b$};
		
		 \node (A1'') at (3.5,0) [circle,minimum size=0.7cm, thick, draw, label = left:$\F_E\!:$]{$a$};
    \node (B1'') at (5,0) [circle, minimum size=0.7cm, thick, draw]{$b$};
		\node (C1'') at (4.3,1) [circle, minimum size=0.7cm, thick, draw,gray!50]{$c$};
    \node (D1'') at (5.7,1) [circle, minimum size=0.7cm, thick, draw,gray!50]{$d$};

		\node (A2'') at (7,0) [circle, minimum size=0.7cm, thick, draw, label = left:$\F_N\!:$] {$a$};
    \node (B2'') at (8.5,0) [circle, minimum size=0.7cm,  thick, draw] {$b$};
    \node (C2'') at (7.8,1) [circle, minimum size=0.7cm, thick, draw,gray!50]{$c$};
		\node (D2'') at (9.2,1) [circle, minimum size=0.7cm, thick, draw,gray!50]{$d$};


		\node (A3) at (0,-2.7) [circle, minimum size=0.7cm, thick, draw, label = left:$\F_S\!:$] {$a$};
    \node (B3) at (1.5,-2.7) [circle, minimum size=0.7cm, thick, draw] {$b$};
    \node (C3) at (0.8,-1.7) [circle, minimum size=0.7cm, thick, draw,gray!50]{$c$};
		\node (D3) at (2.2,-1.7) [circle, minimum size=0.7cm, thick, draw,gray!50]{$d$};

		\node (A2) at (3.5,-2.7) [circle, minimum size=0.7cm, thick, draw, label = left:$\F_W\!:$] {$a$};
    \node (B2) at (5,-2.7) [circle,minimum size=0.7cm , thick, draw] {$b$};
    \node (C2) at (4.3,-1.7) [circle, minimum size=0.7cm, thick, draw,gray!50]{$c$};
		\node (D2) at (5.7,-1.7) [circle, minimum size=0.7cm, thick, draw,gray!50]{$d$};
		
		\node (A3'') at (7,-2.7) [circle, minimum size=0.7cm, thick, draw, label = left:$\F_L\!:$] {$a$};
    \node (B3'') at (8.5,-2.7) [circle, minimum size=0.7cm, thick, draw] {$b$};
		\node (X'') at (9.2,-1.7) [circle, minimum size=0.7cm, thick] {};

\draw[->,thick] (A1'') to [thick,bend right] (B1'');
\draw[->,gray!50,thick] (C1'') to [thick,bend right,gray!50] (A1'');
\draw[->,gray!50,thick] (B1'') to [thick,bend left,gray!50] (D1'');
\draw[->,gray!50,thick] (B1'') to [thick,bend right,gray!50] (A1'');
\draw[->,gray!50,thick] (C1'') to [thick,bend left,gray!50] (D1'');

\draw[->,thick] (A2'') to [thick,bend right] (B2'');
\draw[->,thick,gray!50] (C2'') to [thick,bend right,gray!50] (A2'');
\draw[->,thick,gray!50] (B2'') to [thick,bend left,gray!50] (D2'');
\draw[->,thick,gray!50] (C2'') to [thick,bend left,gray!50] (D2'');

\draw[->,thick] (A3'') to [thick,bend right] (B3'');

\draw[->,thick,gray!50] (B3'') to [thick,bend right,gray!50] (A3'');

\draw[->,thick] (A1) to [thick,bend right] (B1);

\draw[->,thick] (A2) to [thick,bend right] (B2);
\draw[->,thick,gray!50] (B2) to [thick,bend left,gray!50] (D2);
\draw[->,thick,gray!50] (C2) to [thick,bend left,gray!50] (D2);

\draw[->,thick] (A3) to [thick,bend right] (B3);
\draw[->,thick,gray!50] (C3) to [thick,bend right,gray!50] (A3);

\draw[->,thick,gray!50] (C3) to [thick,bend left,gray!50] (D3);

\end{tikzpicture}
\caption{Different Kinds of Expansions}
\label{fig:expansion}
\end{figure}

\end{example}

In 2014 the natural counter-parts (or more precisely, inverse operations) to arbitrary, normal and local expansions, so-called \textit{deletions} were introduced \cite{Bau14}. Furthermore, the most general form of a dynamic scenario (where expansion and deletion can be combined) a so-called \textit{update} were considered too. Analogously to expansions, any arbitrary deletion can be splitted in a normal and a local part. This means, a \textit{normal deletion} retract arguments and their corresponding attacks. \textit{Local deletions} in contrast delete attacks only.\footnote{We mention that \textit{strong} as well as \textit{weak deletions} are not introduced/considered so far. They could be easily defined as inverse operations of their expansion counterparts. Before doing so, it would be interesting to identify real-world situations or instantiation-based dynamics were such kind of evolvements naturally occur.} 
The main motivation behind these notions stems from instantiation-based context. More precisely, a normal deletion on the abstract level correspond to deleting information of a given knowledge base. Changing to a more restrictive notion of attack correspond to a local deletion and a combination of both of them give rise to an arbitrary deletion on the abstract level. We proceed with the formal definitions as well as introductory examples.

\begin{definition}[\cite{Bau14}] \label{def:update} 
Given an AF $\F = (A,R)$, a set of arguments $B$ and a set of attacks~$S$ as well as a further AF $\H$.  The AF 
$$\G = \left(\F\sm[B,S]\right)\dcup\H \eqdef  \left((A,R\sm S)|_{A\sm B}\right)\dcup\H$$
is called an \textit{update} of $\F$ (for short, $\F\asymp_U \G$). An update is called a
\begin{enumerate}
\item \textit{deletion} ($\F\succeq_{D}\G$) iff $\H = (\emptyset,\emptyset)$,
    \item \textit{normal deletion} ($\F\succeq_{ND}\G$)
    iff ($\F\succeq_{D}\G$) and $S = \emptyset$,
    \item \textit{local deletion} ($\F\succeq_{LD}\G$)
    iff $\F\succeq_{D}\G$ and $B = \emptyset$.
\end{enumerate}
\end{definition}

Let us take a closer look at the definition of $\G = \left(\F\sm[B,S]\right)\dcup\H$. The AF $\H$ plays the role of added information, i.e.\ it contains new arguments and attacks. Consequently, for all kind of deletions we have $\H = (\emptyset,\emptyset)$ which leaves us with $\G = \F\sm[B,S]$. The set $B$ contains arguments which have to deleted. Since attacks depend on arguments we have to delete the attacks which involve arguments from $B$ too. This operation is formally captured by the restriction of $\F$ to $A\sm B$. Furthermore, the set $S$ contains particular attacks which have to be deleted. This means, the pair $\pair$ does not necessarily have to be an AF. Therefore we use $\pair$ instead of $(B,S)$. If clear from context we use $B$ and $S$ instead of $[B,\emptyset]$ or $[\emptyset,S]$, i.e.\ we simply write $\F\sm B$ as well as $\F\sm S$ for normal or local deletions, respectively.

\begin{example} The AF $\F$ represents the initial situation. An update as well as arbitrary, normal or local deletion of it are given by $\F_U$, $\F_D$, $\F_{ND}$ and $\F_{LD}$. Grey-highlighted arguments or attacks represent added information in contrast to dotted arguments and attacks which represent deleted objects.\footnote{This convention will be used throughout the whole chapter.} More formally, in accordance with Definition~\ref{def:update} we have that $\F_U = \left(\F\sm[B,S]\right)\dcup\H$, $\F_D = \F\sm[B,S]$, $\F_{ND} = \F\sm B$, $\F_{LD} = \F\sm S$ where the set of arguments $B = \{c\}$, the set of attacks $S = \{(b,a)\}$ and the AF $\H = (\{b,d,e,f\},\{(d,b),(e,f),(f,d)\})$.

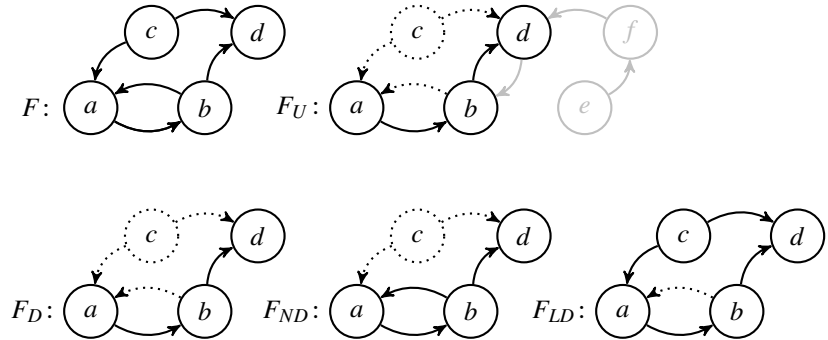
\begin{figure}[H]
\centering
\begin{tikzpicture}

		 \node (A1'') at (0,0) [circle,minimum size=0.7cm, thick, draw, label = left:$\F\!:$]{$a$};
    \node (B1'') at (1.5,0) [circle, minimum size=0.7cm, thick, draw]{$b$};
		\node (C1'') at (0.8,1) [circle, minimum size=0.7cm, thick, draw]{$c$};
    \node (D1'') at (2.2,1) [circle, minimum size=0.7cm, thick, draw]{$d$};

		\node (A2'') at (3.5,0) [circle, minimum size=0.7cm, thick, draw, label = left:$\F_U\!:$] {$a$};
    \node (B2'') at (5,0) [circle, minimum size=0.7cm,  thick, draw] {$b$};
    \node (C2'') at (4.3,1) [circle, minimum size=0.7cm, thick, draw,dotted]{$c$};
		\node (D2'') at (5.7,1) [circle, minimum size=0.7cm, thick, draw]{$d$};
		\node (E2'') at (6.5,0) [circle, minimum size=0.7cm,  thick, draw,gray!50] {$e$};
		\node (F2'') at (7.1,1) [circle, minimum size=0.7cm, thick, draw,gray!50]{$f$};


		\node (A3) at (0,-2.7) [circle, minimum size=0.7cm, thick, draw, label = left:$\F_D\!:$] {$a$};
    \node (B3) at (1.5,-2.7) [circle, minimum size=0.7cm, thick, draw] {$b$};
    \node (C3) at (0.8,-1.7) [circle, minimum size=0.7cm, thick, draw, dotted]{$c$};
		\node (D3) at (2.2,-1.7) [circle, minimum size=0.7cm, thick, draw]{$d$};

		\node (A2) at (3.5,-2.7) [circle, minimum size=0.7cm, thick, draw, label = left:$\F_{ND}\!:$] {$a$};
    \node (B2) at (5,-2.7) [circle,minimum size=0.7cm , thick, draw] {$b$};
    \node (C2) at (4.3,-1.7) [circle, minimum size=0.7cm, thick, draw, dotted]{$c$};
		\node (D2) at (5.7,-1.7) [circle, minimum size=0.7cm, thick, draw]{$d$};
		
		\node (A3'') at (7,-2.7) [circle, minimum size=0.7cm, thick, draw, label = left:$\F_{LD}\!:$] {$a$};
    \node (B3'') at (8.5,-2.7) [circle, minimum size=0.7cm, thick, draw] {$b$};
		\node (C3'') at (7.8,-1.7) [circle, minimum size=0.7cm, thick, draw] {$c$};
    \node (D3'') at (9.3,-1.7) [circle, minimum size=0.7cm, thick, draw] {$d$};
		\node (X'') at (9.2,-1.7) [circle, minimum size=0.7cm, thick] {};

\draw[->,thick] (A1'') to [thick,bend right] (B1'');
\draw[->,thick] (C1'') to [thick,bend right,gray!50] (A1'');
\draw[->,thick] (B1'') to [thick,bend left,gray!50] (D1'');
\draw[->,thick] (B1'') to [thick,bend right,gray!50] (A1'');
\draw[->,thick] (C1'') to [thick,bend left,gray!50] (D1'');

\draw[->,thick] (A2'') to [thick,bend right] (B2'');
\draw[->,thick, dotted] (C2'') to [thick,bend right, dotted] (A2'');
\draw[->,thick] (B2'') to [thick,bend left, gray!50] (D2'');
\draw[->,thick, dotted] (C2'') to [thick,bend left,dotted] (D2'');
\draw[->,thick,  dotted] (B2'') to [thick,bend right, dotted] (A2'');

\draw[->,thick, gray!50] (E2'') to [thick,bend right,red] (F2'');
\draw[->,thick, gray!50] (F2'') to [thick,bend right,red] (D2'');
\draw[->,thick, gray!50] (D2'') to [thick,bend left,red] (B2'');

\draw[->,thick] (A3'') to [thick,bend right] (B3'');
\draw[->,thick, dotted] (B3'') to [thick,bend right,red] (A3'');
\draw[->,thick] (C3'') to [thick,bend right] (A3'');
\draw[->,thick] (C3'') to [thick,bend left] (D3'');
\draw[->,thick] (B3'') to [thick,bend left] (D3'');

\draw[->,thick] (A1) to [thick,bend right] (B1);

\draw[->,thick] (A2) to [thick,bend right] (B2);
\draw[->,thick] (B2) to [thick,bend left] (D2);
\draw[->,thick] (B2) to [thick,bend right] (A2);
\draw[->,thick, dotted] (C2) to [thick,bend left,red] (D2);
\draw[->,thick, dotted] (C2) to [thick,bend right,red] (A2);

\draw[->,thick] (A3) to [thick,bend right] (B3);
\draw[->,thick] (B3) to [thick,bend left] (D3);
\draw[->,thick, dotted] (B3) to [thick,bend right] (A3);
\draw[->,thick, dotted] (C3) to [thick,bend right] (A3);

\draw[->,thick, dotted] (C3) to [thick,bend left,red] (D3);

\end{tikzpicture}
\caption{An Update and Different Kinds of Deletions}
\label{fig:update}
\end{figure}

\end{example}

We now turn to the corresponding equivalence notions (cf.\
\cite[Section 3.8]{BauSt15} for chronological
order). Two AFs $\F$ and $\G$ are said to be \textit{ordinarily equivalent} w.r.t.\ a 
semantics $\sigma$ if they possess the same $\sigma$-extensions/labellings.
In this case, we say that $\F$ and $\G$ possess the same \textit{explicit} information.
In contrast, sharing the same \textit{implicit} information, i.e.\ being semantically indistinguishable w.r.t.\ any suitable future
scenario is a much more demanding property which allows to replace $\F$ and $\G$ by each other without loss of semantical information. 

\begin{example} Consider the following AFs $\F$ and $\G$. We have $\m{E}_{\prf}(\F) = \m{E}_{\prf}(\G) = \{\{a\}\}$. This means, $\F$ and $\G$ possess the same explicit information w.r.t.\ preferred semantics or in other words, they are ordinarily equivalent.

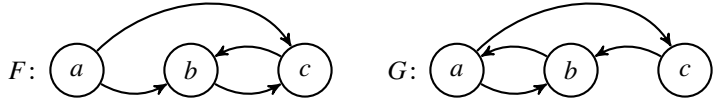
\begin{tikzpicture}

    \node (A1'') at (0,-2.5) [circle,minimum size=0.7cm, thick, draw, label = left:$\F\!:$]{$a$};
    \node (B1'') at (1.5,-2.5) [circle, minimum size=0.7cm, thick, draw]{$b$};
		\node (C1'') at (3,-2.5) [circle, minimum size=0.7cm, thick, draw]{$c$};

		\node (A2'') at (5,-2.5) [circle, minimum size=0.7cm, thick, draw, label = left:$\G\!:$] {$a$};
    \node (B2'') at (6.5,-2.5) [circle, minimum size=0.7cm,  thick, draw] {$b$};
    \node (C2'') at (8,-2.5) [circle, minimum size=0.7cm, thick, draw]{$c$};

\draw[->, thick] (A1'') to [thick,bend right] (B1'');
\draw[->, thick] (A1'') to [thick,bend left,out=44,in=120] (C1'');

\draw[->, thick] (B1'') to [thick,bend right] (C1'');
\draw[->, thick] (C1'') to [thick,bend right] (B1'');

\draw[->, thick] (A2'') to [thick,bend right] (B2'');
\draw[->, thick] (A2'') to [thick,bend left,out=44,in=120] (C2'');

\draw[->, thick] (C2'') to [thick,bend right] (B2'');
\draw[->, thick] (B2'') to [thick,bend right] (A2'');

\end{tikzpicture}

Assume that expansions as well deletions are the dynamic scenarios of interest. This means, we ask whether the AFs $\F$ and $\G$ even possess the same implicit information w.r.t.\ expansions or deletions, respectively? In order to give a negative answer one has to come up with one single dynamic scenario were the revised versions possess different preferred extensions. A positive answer in contrast is a statement about infinitely many dynamic scenarios (even in case of finite AFs). In this example, we give a negative answer for both modification types.

In case of expansions, we conjoin to both the AF $\H = (\{a,b\},\{(b,a)\})$. Consider the resulting frameworks below. We have $\m{E}_{\prf}(\F\dcup\H) = \{\{a\},\{b\}\}$ and since $\G\dcup\H = \G$ we obtain $\m{E}_{\prf}(\G\dcup\H) = \{\{a\}\}$ without re-computing.
 
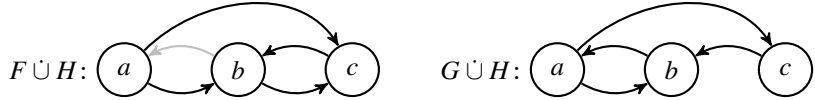
\begin{tikzpicture}

    \node (A1'') at (0,-2.5) [circle,minimum size=0.7cm, thick, draw, label = left:$\F\dcup\H\!:$]{$a$};
    \node (B1'') at (1.5,-2.5) [circle, minimum size=0.7cm, thick, draw]{$b$};
		\node (C1'') at (3,-2.5) [circle, minimum size=0.7cm, thick, draw]{$c$};

		\node (A2'') at (5.7,-2.5) [circle, minimum size=0.7cm, thick, draw, label = left:$\G\dcup\H\!:$] {$a$};
    \node (B2'') at (7.2,-2.5) [circle, minimum size=0.7cm,  thick, draw] {$b$};
    \node (C2'') at (8.7,-2.5) [circle, minimum size=0.7cm, thick, draw]{$c$};

\draw[->, thick] (A1'') to [thick,bend right] (B1'');
\draw[->, thick] (A1'') to [thick,bend left,out=44,in=120] (C1'');

\draw[->, thick] (B1'') to [thick,bend right] (C1'');
\draw[->, thick, gray!50] (B1'') to [thick,bend right] (A1'');
\draw[->, thick] (C1'') to [thick,bend right] (B1'');

\draw[->, thick] (A2'') to [thick,bend right] (B2'');
\draw[->, thick] (A2'') to [thick,bend left,out=44,in=120] (C2'');

\draw[->, thick] (C2'') to [thick,bend right] (B2'');
\draw[->, thick] (B2'') to [thick,bend right] (A2'');

\end{tikzpicture}

To reveal the inherent difference between $\F$ and $\G$ in case of deletions we may retract with the argument $c$. Consider the resulting (normal) deletions $\F\sm\{c\}$ and $\G\sm\{c\}$ of $\F$ or $\G$, respectively. Now, $\{b\}$ becomes a preferred extension in $\F\sm\{c\}$ but still not in $\G\sm\{c\}$.

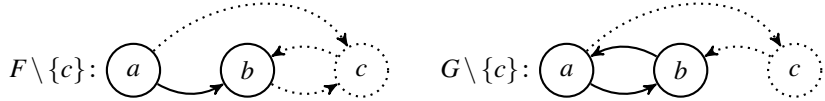
\begin{tikzpicture}

    \node (A1'') at (0,-2.5) [circle,minimum size=0.7cm, thick, draw, label = left:$\F\sm\{c\}\!:$]{$a$};
    \node (B1'') at (1.5,-2.5) [circle, minimum size=0.7cm, thick, draw]{$b$};
		\node (C1'') at (3,-2.5) [circle, minimum size=0.7cm, thick, draw, dotted]{$c$};

		\node (A2'') at (5.7,-2.5) [circle, minimum size=0.7cm, thick, draw, label = left:$\G\sm\{c\}\!:$] {$a$};
    \node (B2'') at (7.2,-2.5) [circle, minimum size=0.7cm,  thick, draw] {$b$};
    \node (C2'') at (8.7,-2.5) [circle, minimum size=0.7cm, thick, draw, dotted]{$c$};

\draw[->, thick] (A1'') to [thick,bend right] (B1'');
\draw[->, thick, dotted] (A1'') to [thick,bend left,out=44,in=120] (C1'');

\draw[->, thick, dotted] (B1'') to [thick,bend right] (C1'');
\draw[->, thick, dotted] (C1'') to [thick,bend right] (B1'');

\draw[->, thick] (A2'') to [thick,bend right] (B2'');
\draw[->, thick, dotted] (A2'') to [thick,bend left,out=44,in=120] (C2'');

\draw[->, thick, dotted] (C2'') to [thick,bend right] (B2'');
\draw[->, thick] (B2'') to [thick,bend right] (A2'');

\end{tikzpicture}
\end{example}

We now formally define what we precisely mean by possessing the same implicit information. 
As already stated, the first paper in this line of work was \cite{strong} engaged with
characterizing \textit{strong equivalence}. For the sake of clarity and
comprehensibility we use the term \textit{expansion equivalence} since strong equivalence \cite[Definition~2]{strong}
corresponds to semantical indistinguishability w.r.t.\ arbitrary expansions. 

Before presenting the definitions let us introduce some further notational convention which will be extensively used throughout the whole chapter. For a given AF $\G = (B,S)$, we use $A(\G) = B$, $R(\G) = S$, $L(\G) = \{a\in B\mid (a,a)\in S\}$ and $NL(\G) = B\sm L(\G)$.

\begin{definition} \label{def:equivalence} Given a semantics $\sigma$. Two AFs $\AF$ and $\AG$ are

\begin{enumerate}
\item \textit{ordinarily equivalent w.r.t.}\ $\sigma$ (for short, $\AF\equiv^{\sigma}\!\AG$) iff $\sigma(\AF)=\sigma(\AG)$,

\item \textit{expansion equivalent w.r.t.}\ $\sigma$ (for short, $\AF\equiv_{E}^{\sigma}\!\AG$) iff
	for each AF $\AH$ we have, $\AF\dcup \AH \equiv^{\sigma}\! \AG\dcup \AH$,
	
	\item \textit{normal expansion equivalent w.r.t.}\ $\sigma$ (for short, $\AF\equiv_{N}^{\sigma}\!\AG$) iff
	for each AF~$\AH$, such that $\AF \preceq_N\! \AF\dcup \AH$ and $\AG \preceq_N\! \AG\dcup \AH$ we have, $\AF\dcup \AH \equiv^{\sigma}\! \AG\dcup \AH$,
	
	\item \textit{strong expansion equivalent w.r.t.}\ $\sigma$ (for short, $\AF\equiv_{S}^{\sigma}\!\AG$) iff
	for each AF~$\AH$, such that $\AF \preceq_S\! \AF\dcup \AH$ and $\AG \preceq_S\! \AG\dcup \AH$ we have, $\AF\dcup \AH \equiv^{\sigma}\! \AG\dcup \AH$,
	
	\item \textit{weak expansion equivalent w.r.t.}\ $\sigma$ (for short, $\AF\equiv_{W}^{\sigma}\!\AG$) iff
	for each AF~$\AH$, such that $\AF \preceq_W\! \AF\dcup \AH$ and $\AG \preceq_W\! \AG\dcup \AH$ we have, $\AF\dcup \AH \equiv^{\sigma}\! \AG\dcup \AH$,
	
	\item \textit{local expansion equivalent}\footnote{Note that a suitable AF $\AH$ is not necessarily a local expansion of $\AF$ and $\AG$ in the sense of Definition~\ref{def:expansion}. Nevertheless, we may loosely speak about local expansions.} \textit{w.r.t.}\ $\sigma$ (for short, $\AF\equiv_{L}^{\sigma}\!\AG$) iff
	for each AF~$\AH$ with $A(\AH)\subseteq A(\AF\dcup \AG)$ we have, $\AF\dcup \AH \equiv^{\sigma}\! \AG\dcup \AH$.
	
\item \textit{update equivalent w.r.t.}\ $\sigma$ (for short, $\AF\equiv^{\sigma}_{U}\!\AG$) iff for any pair $\pair$ and any AF $\AH$ we have, $\left(\AF\sm[B,S]\right)\dcup\AH\equiv^{\sigma}\!\left(\AG\sm[B,S]\right)\dcup\AH$,

\item \textit{deletion equivalent w.r.t.}\ $\sigma$ (for short, $\AF\equiv^{\sigma}_{D}\!\AG$) iff for any pair $\pair$ we have, $\AF\sm[B,S]\equiv^{\sigma}\!\AG\sm[B,S]$,

\item \textit{normal deletion equivalent w.r.t.}\ $\sigma$ (for short, $\AF\equiv^{\sigma}_{ND}\!\AG$) iff for any set of arguments $B$ we have, $\AF\sm B\equiv^{\sigma}\!\AG\sm B$,

\item \textit{local deletion equivalent w.r.t.}\ $\sigma$ (for short, $\AF\equiv^{\sigma}_{LD}\!\AG$) iff for any set of attacks $S$ we have,\linebreak $\AF\sm S\equiv^{\sigma}\!\AG\sm S$,
\end{enumerate}
\end{definition}

Remember that there are several relations between the considered dynamic scenarios. For instance, in accordance with Definitions~\ref{def:expansion} and \ref{def:update}, any normal expansion (deletion) is an arbitrary expansion (deletion). Furthermore, in the light of Definition~\ref{def:equivalence}, we certainly affirm that expansion equivalence is much more demanding then local expansion equivalence. In other words, local expansion equivalence of two AFs is an immediate and unavoidable consequence of being expansion equivalent. Finally, any considered equivalence notion is at least as demanding then ordinary equivalence.\footnote{The empty framework $(\emptyset,\emptyset)$ as well as the empty pair $[\emptyset,\emptyset]$ justifies this assertion for any type of expansions or deletions, respectively.}  Please note that these relations do not depend on certain properties of a considered semantics. Consequently, Figure~\ref{fig:prelimrel} gives a preliminary overview for such interrelations (arising from the definitions) between the introduced equivalence notions for any possible semantics. For reasons, which will become clearer later, we also consider the identity relation. For two equivalence notion $\Phi$ and $\Psi$ we have $\Phi\subseteq\Psi$ iff there is a link from $\Phi$ to $\Psi$. 

\begin{figure}
\centering
\begin{tikzpicture}[scale=1.0]

\node (I) at (-1.6,1.5) [rectangle, thick, draw,text width = 1.3cm, text centered]{\textbf{identity}\\
\footnotesize{relation}};
\node (A) at (0,0) [rectangle, thick, draw,text width = 1.39cm, text centered]{\textbf{update}\\
\footnotesize{equivalence}};
\node (B) at (2.2,1.8) [rectangle, thick, draw,text width = 1.65cm, text centered]{\textbf{expansion}\\
\footnotesize{equivalence}};
\node (C) at (2.2,-1.8) [rectangle, thick, draw,text width = 1.43cm, text centered]{\textbf{deletion}\\
\footnotesize{equivalence}};
\node (D) at (4.8,-2.6) [rectangle, thick, draw,text width = 1.43cm, text centered]{\textbf{local}\\
\textbf{deletion}\\
\footnotesize{equivalence}};
\node (E) at (4.8,-1) [rectangle, thick, draw,text width = 1.43cm, text centered]{\textbf{normal}\\
\textbf{deletion}\\
\footnotesize{equivalence}};
\node (F) at (4.8,1) [rectangle, thick, draw,text width = 1.65cm, text centered]{\textbf{local}\\
\textbf{expansion}\\
\footnotesize{equivalence}};
\node (G) at (4.8,2.6) [rectangle, thick, draw,text width = 1.65cm, text centered]{\textbf{normal}\\
\textbf{expansion}\\
\footnotesize{equivalence}};
\node (H) at (7.7,3.4) [rectangle, thick, draw,text width = 1.65cm, text centered]{\textbf{strong}\\
\textbf{expansion}\\
\footnotesize{equivalence}};
\node (W) at (7.7,1.8) [rectangle, thick, draw,text width = 1.65cm, text centered]{\textbf{weak}\\
\textbf{expansion}\\
\footnotesize{equivalence}};
\node (O) at (8.6,-1.5) [rectangle, thick, draw,text width = 1.43cm, text centered]{\textbf{ordinary}\\
\footnotesize{equivalence}};

\draw[->, thick, double] (I) to [thick,double, bend right] (A);
\draw[->, thick, double] (A) to [thick,double, bend left] (B);
\draw[->, thick, double] (A) to [thick,double, bend right] (C);
\draw[->, thick, double] (C) to [thick,double, bend right] (D);
\draw[->, thick, double] (C) to [thick,double, bend left] (E);
\draw[->, thick, double] (B) to [thick,double, bend right] (F);
\draw[->, thick, double] (B) to [thick,double, bend left] (G);
\draw[->, thick, double] (G) to [thick,double, out=30, in= 170] (H);
\draw[->, thick, double] (G) to [thick,double, out=330, in = 190] (W);
\draw[->, thick, double] (D) to [thick,double, out=330, in= 200] (O);
\draw[->, thick, double] (E) to [thick,double, out=30, in= 180] (O);
\draw[->, thick, double] (H) to [thick,double, out=320, in= 80] (O);
\draw[->, thick, double] (W) to [thick,double, out=240, in= 120] (O);

\end{tikzpicture}
\caption{\label{fig:prelimrel}Preliminary Subset Relations between Equivalence Notions}
\end{figure}
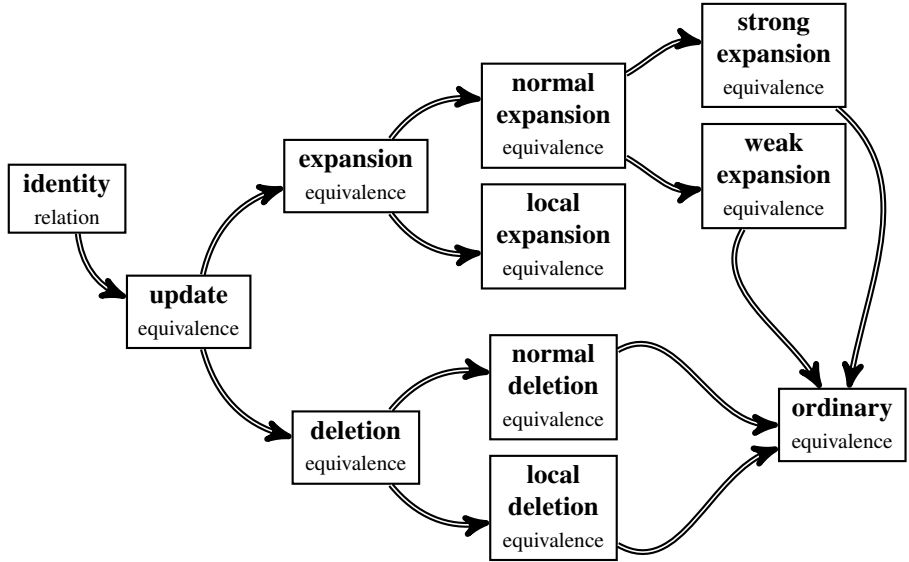

In the remainder of this section we shed light on the question of \textit{how} to determine whether two AFs are equivalent w.r.t.\ certain scenarios? As a by-product of these characterization results we will see that for many semantics the preliminary relations between the introduced equivalence notions depicted above can be delineated in a much more compact way. The majority of the presented characterization results is devoted to finite AFs as well as extension-based semantics. We will see that there are some differences if sticking to unrestricted frameworks or the corresponding labelling-based versions.

\section{Characterization Theorems for Extension-based Semantics}

\subsection{The Central Notion of Expansion Equivalence}

In order to get an idea of how to find a characterization we start with some reflections. For this purpose we consider the most restrictive semantics, namely the stable one as well as the most prominent type of equivalence, namely expansion equivalence. What are necessary features of expansion equivalence w.r.t.\ stable semantics, i.e.\ which properties are implied if two AFs $\F$ and $\G$ are expansion equivalent? In consideration of Figure~\ref{fig:prelimrel} we deduce their ordinary equivalence, i.e.\ $\m{E}_{\stb}(\F) = \m{E}_{\stb}(\G)$. Note that possessing the same set of extensions neither imply sharing the same arguments nor sharing the same self-loops as shown in the following example.

\begin{example} \label{ex:nonequi} Consider the AFs $\F$, $\G$ and $\H$. Each two of them are ordinarily equivalent since $\m{E}_{\stb}(\F) = \m{E}_{\stb}(\G) = \m{E}_{\stb}(\H) = \{\{a\}\}$.

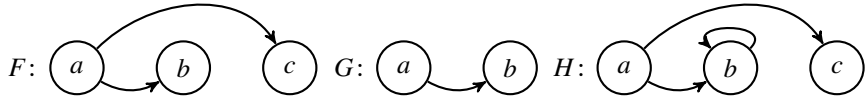
\begin{tikzpicture}

    \node (A1'') at (0,-2.5) [circle,minimum size=0.7cm, thick, draw, label = left:$\F\!:$]{$a$};
    \node (B1'') at (1.4,-2.5) [circle, minimum size=0.7cm, thick, draw]{$b$};
		\node (C1'') at (2.8,-2.5) [circle, minimum size=0.7cm, thick, draw]{$c$};

		\node (A2'') at (4.3,-2.5) [circle, minimum size=0.7cm, thick, draw, label = left:$\G\!:$] {$a$};
    \node (B2'') at (5.7,-2.5) [circle, minimum size=0.7cm,  thick, draw] {$b$};
		
		\node (A3'') at (7.2,-2.5) [circle, minimum size=0.7cm, thick, draw, label = left:$\H\!:$] {$a$};
    \node (B3'') at (8.6,-2.5) [circle, minimum size=0.7cm,  thick, draw] {$b$};
		\node (C3'') at (10,-2.5) [circle, minimum size=0.7cm,  thick, draw] {$c$};

\draw[->, thick] (A1'') to [thick,bend right] (B1'');
\draw[->, thick] (A1'') to [thick,bend left,out=44,in=120] (C1'');

\draw[->, thick] (A2'') to [thick,bend right] (B2'');

\draw[->, thick] (A3'') to [thick,bend right] (B3'');
\draw[->, thick] (A3'') to [thick,bend left,out=44,in=120] (C3'');
\draw[->, thick] (B3'') to [thick,loop,distance=0.5cm] (B3'');

\end{tikzpicture}

The AFs $\I_1 = (\{c\},\emptyset)$ and $\I_2 = (\{a,b,c\},\{(b,a),(b,c)\})$ witness that neither $\F$ and $\G$, nor $\F$ and $\H$ are expansion equivalent w.r.t.\ stable semantics. Convince yourself that $\m{E}_{\stb}(\F\dcup\I_1) = \{\{a\}\} \neq \{\{a,c\}\} = \m{E}_{\stb}(\G\dcup\I_1)$ and $\m{E}_{\stb}(\F\dcup\I_2) = \{\{a\},\{b\}\} \neq \{\{a\}\} = \m{E}_{\stb}(\G\dcup\I_2)$.

\end{example}

Restricting ourselves to finite AFs, it is not difficult to see that in case of expansion equivalence w.r.t.\ stable semantics the observed relation between non-sharing the same arguments/loops and non-equivalence does hold in general. In other words, possessing the same arguments as well as possessing the same loops are indeed necessary conditions for being expansion equivalent in the finite setting.

Let us summarize our observations in the following fact.

\begin{fact} \label{fact:expansionequi} Given two finite AFs $\F$ and $\G$. If $\AF\equiv^{\m{E}_{\stb}}_{E}\!\AG$, then
\begin{enumerate}
	\item $\m{E}_{\stb}(\F) = \m{E}_{\stb}(\G)$,
	\item $A(\F) = A(\G)$ and
	\item $L(\F) = L(\G)$.
\end{enumerate}
\end{fact}

As already stated in Figure~\ref{fig:prelimrel}, being identical (i.e.\ $A(\F) = A(\G)$ and $R(\F) = R(\G)$) is sufficient for being expansion equivalent. Combining this undeniable fact together with the second and third items of Fact~\ref{fact:expansionequi} encourages one to search for syntactical properties sufficient as well as necessary for being expansion equivalent. In order to guarantee the first item of Fact~\ref{fact:expansionequi} we have to
identify attacks which do not contribute anything when computing stable extensions. Moreover, these attacks which do not affect the evaluation of a given AF $\F$ have to be \textit{redundant}, no matter how $\F$ is extended. Remember that being a stable extension can be simply verified by checking whether the set in question is conflict-free and possesses a full range.\footnote{The topic of \textit{verifiability} of argumentation semantics $\sigma$ will be considered in Chapter~\ref{cha:ver}. The main question is which (minimal amount of) information on top of conflict-free sets is exactly needed to determine whether a certain set is a $\sigma$-extension.} This means, good candidates for ``useless'' attacks w.r.t.\ stable semantics should fulfill the following two properties: firstly, having or not having such an attack does not change the status of a set from being conflict-free to conflicting or vice versa and secondly, having or not having such an attack does not affect the range of a conflict-free set. Certainly, an attack $(a,b)$ stemming from a self-defeating argument $a$ does not change the conflict status of a certain set $E$. This can be seen as follows: If $a\in E$, then $E$ was conflicting as well as remains conflicting after deleting or adding $(a,b)$. Furthermore, if $a\notin E$, then $E$ might be conflicting or not. In either case the conflict status of $E$ does not change if (a,b) is added or removed since $\{a,b\}\nsubseteq E$. Finally, such an attack $(a,b)$ might have an influence on the range of conflicting sets but it definitely has not in case of conflict-free sets since $a\notin E$ can not be questioned. 

\begin{example} \label{ex:stableind} Consider the following AF $\F$. We have, $\m{E}_{\stb}(\F) = \{\{a\}\}$.

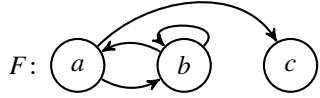
\begin{tikzpicture}

    \node (A1'') at (0,-2.5) [circle,minimum size=0.7cm, thick, draw, label = left:$\F\!:$]{$a$};
    \node (B1'') at (1.4,-2.5) [circle, minimum size=0.7cm, thick, draw]{$b$};
		\node (C1'') at (2.8,-2.5) [circle, minimum size=0.7cm, thick, draw]{$c$};
    
\draw[->, thick] (A1'') to [thick,bend right] (B1'');
\draw[->, thick] (A1'') to [thick,bend left,out=44,in=120] (C1'');
\draw[->, thick] (B1'') to [thick,loop,distance=0.5cm] (B1'');
\draw[->, thick] (B1'') to [thick,bend right] (A1'');

\end{tikzpicture}

According to our considerations above adding or deleting an attack stemming from the self-defeating argument $b$ does not change the semantics. Consider therefore the following three possible ``manipulations''.

\begin{tikzpicture}
    \node (A1'') at (0,-2.5) [circle,minimum size=0.7cm, thick, draw, label = left:$\G_1\!:$]{$a$};
    \node (B1'') at (1.2,-2.5) [circle, minimum size=0.7cm, thick, draw]{$b$};
		\node (C1'') at (2.4,-2.5) [circle, minimum size=0.7cm, thick, draw]{$c$};

		\node (A2'') at (3.9,-2.5) [circle, minimum size=0.7cm, thick, draw, label = left:$\G_2\!:$] {$a$};
    \node (B2'') at (5.1,-2.5) [circle, minimum size=0.7cm,  thick, draw] {$b$};
		\node (C2'') at (6.3,-2.5) [circle, minimum size=0.7cm,  thick, draw] {$c$};
		
		\node (A3'') at (7.8,-2.5) [circle, minimum size=0.7cm, thick, draw, label = left:$\G_3\!:$] {$a$};
    \node (B3'') at (9,-2.5) [circle, minimum size=0.7cm,  thick, draw] {$b$};
		\node (C3'') at (10.2,-2.5) [circle, minimum size=0.7cm,  thick, draw] {$c$};

\draw[->, thick] (A1'') to [thick,bend right] (B1'');
\draw[->, thick] (A1'') to [thick,bend left,out=44,in=120] (C1'');
\draw[->, thick, dotted] (B1'') to [thick,bend right] (A1'');
\draw[->, thick] (B1'') to [thick,loop,distance=0.5cm] (B1'');

\draw[->, thick] (A2'') to [thick,bend right] (B2'');
\draw[->, thick] (A2'') to [thick,bend left,out=44,in=120] (C2'');
\draw[->, thick, dotted] (B2'') to [thick,bend right] (A2'');
\draw[->, thick, gray!50] (B2'') to [thick,bend right] (C2'');
\draw[->, thick] (B2'') to [thick,loop,distance=0.5cm] (B2'');

\draw[->, thick] (A3'') to [thick,bend right] (B3'');
\draw[->, thick] (A3'') to [thick,bend left,out=44,in=120] (C3'');
\draw[->, thick] (B3'') to [thick,bend right] (A3'');
\draw[->, thick, gray!50] (B3'') to [thick,bend right] (C3'');
\draw[->, thick] (B3'') to [thick,loop,distance=0.5cm] (B3'');

\end{tikzpicture}

Indeed, $\m{E}_{\stb}(\F) = \m{E}_{\stb}(\G_1) = \m{E}_{\stb}(\G_2) = \m{E}_{\stb}(\G_3) = \{\{a\}\}$ support our claims for the static case. We encourage the reader to try to do the impossible, namely semantically distinguish the AFs $\F$ and its manipulations by an arbitrary expansion.

\end{example}

It was the main result in \cite{strong} that expansion equivalence can be indeed decided by looking at the syntax only. The authors introduced so-called \textit{kernels} which are simply functions mapping each AF $\F$ to its redundancy-free version. This means, the kernel of an AF $\F$ does not possess any redundant attack. Put it differently, for any surviving attack exist at least one dynamic scenario were deleting this attack would cause a semantical difference. We proceed with the formal definition of the very first kernels already introduced in \cite{strong}. We sometimes call them \textit{classical}.

\begin{definition} \label{def:kernel}
Let $\sigma\in\{\stb,\adm,\grd,\com\}$. The $\sigma$-kernel $\k(\sigma): \m{F} \to \m{F}$ with $\k(\sigma)(\F) = \AF^{\k(\sigma)} = \left(A,R^{\k(\sigma)}\right)$ for a given AF $\AF = (A,R)$ is defined as:
 \begin{align*}
R^{\k(\stb)} =&\ R \sm \left\{(a, b) \mid a \neq b, (a, a) \in R\right\}\!,\\
R^{\k(\adm)} =&\ R \sm \{(a, b) \mid a \neq b, (a, a) \in R, \{(b,a),(b,b)\}\cap R \neq\emptyset\},\\
R^{\k(\grd)} =&\ R \sm \{(a, b) \mid a \neq b, (b, b) \in R, \{(a,a),(b,a)\}\cap R \neq\emptyset\},\\
R^{\k(\com)} =&\ R\sm \{(a, b) \mid a \neq b, (a,a),(b,b)\in R \}.
\end{align*}
\end{definition}

In order to get an idea of how the classical kernels work we proceed with an example. 

\begin{example} Consider again the AF $\G_3$ depicted in Example~\ref{ex:stableind}. We apply now all classical kernels. 

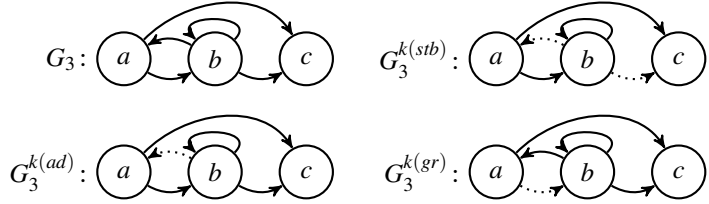
\begin{tikzpicture}

		\node (A2'') at (4.9,-2.5) [circle, minimum size=0.7cm, thick, draw, label = left:$\G_3^{\k(\stb)}\!:$] {$a$};
    \node (B2'') at (6.1,-2.5) [circle, minimum size=0.7cm,  thick, draw] {$b$};
		\node (C2'') at (7.3,-2.5) [circle, minimum size=0.7cm,  thick, draw] {$c$};
		
		\node (A3'') at (0,-2.5) [circle, minimum size=0.7cm, thick, draw, label = left:$\G_3\!:$] {$a$};
    \node (B3'') at (1.2,-2.5) [circle, minimum size=0.7cm,  thick, draw] {$b$};
		\node (C3'') at (2.4,-2.5) [circle, minimum size=0.7cm,  thick, draw] {$c$};
		
		\node (A2') at (4.9,-4) [circle, minimum size=0.7cm, thick, draw, label = left:$\G_3^{\k(\grd)}\!:$] {$a$};
    \node (B2') at (6.1,-4) [circle, minimum size=0.7cm,  thick, draw] {$b$};
		\node (C2') at (7.3,-4) [circle, minimum size=0.7cm,  thick, draw] {$c$};
		
		\node (A3') at (0,-4) [circle, minimum size=0.7cm, thick, draw, label = left:$\G_3^{\k(\adm)}\!:$] {$a$};
    \node (B3') at (1.2,-4) [circle, minimum size=0.7cm,  thick, draw] {$b$};
		\node (C3') at (2.4,-4) [circle, minimum size=0.7cm,  thick, draw] {$c$};


\draw[->, thick] (A2'') to [thick,bend right] (B2'');
\draw[->, thick] (A2'') to [thick,bend left,out=44,in=120] (C2'');

\draw[->, thick] (B2'') to [thick,loop,distance=0.5cm] (B2'');
\draw[->, thick, dotted] (B2'') to [thick,bend right] (A2'');
\draw[->, thick, dotted] (B2'') to [thick,bend right] (C2'');

\draw[->, thick] (A3'') to [thick,bend right] (B3'');
\draw[->, thick] (A3'') to [thick,bend left,out=44,in=120] (C3'');
\draw[->, thick] (B3'') to [thick,bend right] (A3'');
\draw[->, thick] (B3'') to [thick,bend right] (C3'');
\draw[->, thick] (B3'') to [thick,loop,distance=0.5cm] (B3'');

\draw[->, thick, dotted] (A2') to [thick,bend right] (B2');
\draw[->, thick] (A2') to [thick,bend left,out=44,in=120] (C2');

\draw[->, thick] (B2') to [thick,loop,distance=0.5cm] (B2');
\draw[->, thick] (B2') to [thick,bend right] (A2');
\draw[->, thick] (B2') to [thick,bend right] (C2');

\draw[->, thick] (A3') to [thick,bend right] (B3');
\draw[->, thick] (A3') to [thick,bend left,out=44,in=120] (C3');
\draw[->, thick, dotted] (B3') to [thick,bend right] (A3');
\draw[->, thick] (B3') to [thick,bend right] (C3');
\draw[->, thick] (B3') to [thick,loop,distance=0.5cm] (B3');

\end{tikzpicture}

The stable kernel deletes all attacks $(a,b)$ stemming from a self-defeating argument~$a$. A deletion of $(a,b)$ in case of the grounded kernel additionally requires that $a$ is counter-attacked by $b$ or $b$ is self-defeating or both. Interchanging $a$ and $b$ yields the condition for deletion in case of the grounded kernel. Finally, $\G_3^{\k(\com)} = \G_3$ since deleting an attack $(a,b)$ w.r.t.\ the complete kernel requires that both arguments $a$ and $b$ are self-defeating.
\end{example}  

Before turning to characterization theorems, we collect some useful properties of the introduced kernels. The following fact contains intrinsic properties of the classical kernels.\footnote{Although most of the properties are immediately clear even in case of unrestricted frameworks we will state all of them for finite AFs only as done in the existing literature. The same applies to Fact~\ref{fact:kernel2}. Some results regarding unrestricted frameworks can be found in Section~\ref{chap:equiunrestr}.} More precisely, any classical kernel $\k$ is \textit{node-preserving} and \textit{loop-preserving}, i.e.\ the sets of arguments and self-defeating arguments do not change if applying $\k$. Moreover, in the absence of self-loops, each AF coincides
with its classical kernels. Furthermore, the decision whether
an attack $(a, b)$ has to be deleted does not depend on further arguments than
$a$ and~$b$. Put differently, the reason of being redundant is \textit{context-free}, i.e.\ it stems from
the arguments themselves. The last two properties claim that equality of kernels is \textit{robust} w.r.t.\ further compositions as well as deleting arguments and corresponding attacks. For a given AF $\F = (B,S)$ we use $A(\F), R(\F)$ and $L(\F)$ to refer to its arguments, attacks and self-defeating arguments, i.e.\ $A(\F) = B$, $R(\F) = S$ and $L(\F) = \{a\in A(\F)\mid (a,a)\in R(\F)\}$.

\begin{fact}[cf.\ \cite{strong,Bau14}] \label{fact:kernel} Given $\k\in\{\k(\stb),\k(\adm),\k(\grd),\k(\com)\}$. For any finite AF $\AF$ we have:

\begin{enumerate}
  \item $A(\AF) = A\left(\AF^\k\right)$, \hfill{(node-preserving)}
	\item $L(\AF) = L\left(\AF^\k\right)$, \hfill{(loop-preserving)}
	\item $L(\AF) = \emptyset\ \To\ \AF = \AF^\k$ and \hfill{(sufficient condition for identity)}
	\item $(a,b)\in R\left(\F^\k\right) \ToT (a,b)\in R\left(\left(\F|_{\{a,b\}}\right)^\k\right).$ \hfill{(context-freeness)}
\end{enumerate}
 Furthermore, for finite AFs $\AF$ and $\AG$ we have:
\begin{enumerate}  \setcounter{enumi}{4}
	\item If $\AF^\k = \AG^\k$, then $(\AF\dcup\AH)^\k = (\AG\dcup\AH)^\k$ for any finite AF $\AH$ and  \hfill{($\dcup$-robustness)}
	\item If $\AF^\k = \AG^\k$, then $(\AF\sm B)^\k = (\AG\sm B)^\k$ for any finite set of args $B$.   \hfill{($\sm$-robustness)}
	
\end{enumerate}

\end{fact}

We proceed with extrinsic properties, i.e.\ features of kernels in presence of semantics. More precisely, stable, admissible, grounded and complete semantics are insensitive w.r.t.\ the application of their corresponding classical $\sigma$-kernel, i.e.\ the set of $\sigma$-extensions remains unchanged. Furthermore, the admissible kernel neither effects semi-stable, eager, preferred and ideal semantics. Similarly
in case of stable kernel and stage semantics.
\pagebreak
\begin{fact}[\cite{strong,GagW13}] \label{fact:kerneldecisive} For any finite AF $\AF$ we have:

\begin{enumerate}
  \item $\Ext_{\sigma}(\AF) = \Ext_{\sigma}\!\left(\AF^{\k(\sigma)}\right)$ for $\sigma\in\{\stb,\adm,\grd,\com\}$,
	\item $\Ext_{\sigma}(\AF) = \Ext_{\sigma}\!\left(\AF^{\k(\adm)}\right)$ for $\sigma\in\{\semi,\eag,\prf,\id\}$ and
	\item $\Ext_{\stg}(\AF) = \Ext_{\stg}\!\left(\AF^{\k(\stb)}\right)$.
\end{enumerate}

\end{fact}

As already mentioned, kernels play a decisive role in deciding expansion equivalence. In general, we say that an equivalence notion $\equiv$ \textit{is characterizable through} $\k$ or simply, $\k$ \textit{is a characterizing kernel\linebreak (of $\equiv$)} if for any two AFs $\F$ and $\G$, $\F\equiv\G$ iff $\F^\k= \G^\k$. This means, proving whether two frameworks are equivalent can be done by simply checking whether the corresponding kernels are identical. Note that all classical kernels can be efficiently constructed from a given AF. The following main theorem states that for all nine considered semantics $\sigma$ there is a certain classical kernel $\k$, s.t.\ expansion equivalence w.r.t.~$\sigma$ is characterizable through~$\k$ in the finite setting. This is a very remarkable result since expansion equivalence is defined semantically. For instance, two finite AFs $\F$ and $\G$ are expansion equivalent w.r.t.\ stable semantics if and only if the associated stable kernels $\F^{\k(\stb)}$ and $\G^{\k(\stb)}$ are syntactically equal. Observe that there is no need to introduce further kernels since one single kernel may serve for different semantics.

\begin{theorem}{\cite{strong,GagW13}} \label{the:strong} For finite AFs $F$,$G$ we have:
\begin{enumerate}
	\item $\AF\equiv^{\Ext_\sigma}_E\! \AG \ToT \AF^{\k(\sigma)} = \AG^{\k(\sigma)}$ for any $\sigma\in\{\stb,\adm,\com,\grd\}$,
	\item  $\AF\equiv^{\Ext_\sigma}_E\! \AG \ToT \AF^{\k(\adm)} = \AG^{\k(\adm)}$ for any $\sigma\in\{\prf,\id,\semi,\eag\}$ and
	\item $\AF\equiv^{\Ext_\stg}_E\! \AG \ToT \AF^{\k(\stb)} = \AG^{\k(\stb)}$.
\end{enumerate}
\end{theorem}

Having Theorem~\ref{the:strong} at hand we can now formally verify that all AFs depicted in Example~\ref{ex:stableind} are expansion equivalent w.r.t.\ stable semantics. This means, the recommended search for arbitrary expansions revealing semantical difference between them will never succeed. As an aside, one might get the impression that the syntactical characterization presented in Theorem~\ref{the:strong} is somehow unique. This is not true. Consider therefore the equivalence class $[\F]^{\Ext_\stb}_{E} = \{\G\mid \AF\equiv^{\Ext_\stb}_E\! \AG \}$ induced by $\F$. Mathematically speaking, the stable kernel $\F^{\k(\stb)}$ represents the least (w.r.t.\ subgraph-relation) element in $[\F]^{\Ext_\stb}_{E}$. It is not difficult to prove that $[\F]^{\Ext_\stb}_{E}$ even possesses a greatest element, namely $\F^{\k'(\stb)} = \left(A(\F), R(\F) \cup \{(a, b) \mid a \neq b, (a, a) \in R(\F)\}\right)$, i.e.\ the framework resulting from $\F$ by adding (instead of deleting) all redundant attacks. In case of finite AFs it can be shown with reasonable effort that expansion equivalence w.r.t.\ stable semantics is characterizable through $\k'(\stb)$ too. In the same manner, all other semantics considered in Theorem~\ref{the:strong} possess alternative ``greatest elements'' characterizations. We will see that the so-called \textit{naive kernel} (compare Definition~\ref{def:kernel2}) provides such a kind of characterization for naive semantics. The reason for this ``choice'' is simply that the induced equivalence classes do not necessarily possess a least element in case of naive semantics. 

Finally, let us turn to the more exotic cf2 as well as stage2 semantics which are defined via a recursive schema based on the decomposition of AFs along their strongly connected components (SCCs). These semantics are exceptional regarding expansion equivalence since in contrast to all other semantics considered in this section we have that even attacks between two self-attacking arguments are \textit{meaningful}. This means, the presence or absence of such attacks may change the outcome of an AF. Moreover, it turned out that any attack is non-redundant. In summary, expansion equivalence coincides with syntactical identity or more formally, for any finite AF $\F$, $\card{[\F]^{\Ext_\cfzwei}_{E}} = \card{[\F]^{\Ext_\stgzwei}_{E}} = \card{\{\F\}} = 1$.  

\begin{theorem}{\cite{GagW13,GagD14}} \label{the:strongSCC} Given $\sigma\in\{\cfzwei,\stgzwei\}$. For finite AFs $F$ and $G$ we have,
$$\AF\equiv^{\Ext_\sigma}_E \AG \ToT \AF = \AG.$$
\end{theorem}

\subsection{Further Equivalence Notions Characterizable through Kernels} \label{sec:furthequi}
Let us turn to the remaining equivalence notions? Are there similar syntax-based characterization results?

\subsubsection{Weaker Notions of Expansion Equivalence}

Let us consider less demanding notions than expansion equivalence, e.g.\ normal and local expansion equivalence. In consideration of Definition~\ref{def:expansion} we do not have good reasons to believe that two AFs could be semantically distinguished by normal or local expansions, given that we only have a witnessing arbitrary expansion showing their non-equivalence. 
It was one surprising result in this line of research, that for many semantics, expansion equivalence coincide with definitorially weaker notions of it. This implies that weaker notions than expansion equivalence can be characterized by classical kernels too. The first results in this respect were already given in \cite[Theorem 8]{strong}. The authors showed that for some semantics expansion equivalence and local expansion equivalence coincide if considering finite AFs. 
It is worthwhile to gain a thorough understanding of this relation since it actually means that if there is an arbitrary expansion which semantically distinguish two finite AFs, than there has to be a local expansion doing likewise. Later it was shown that even normal expansion equivalence coincides with expansion equivalence for a whole bunch of semantics \cite{normal}. Interestingly, in contrast to local expansion equivalence, there are (to the best of our knowledge) no semantics together with witnessing AFs known which show that this coincidence does not hold in general.

\begin{example} \label{Ex:kernelad} Consider the following AFs $\F$ and $\G$. According to Theorem~\ref{the:strong} they are not expansion equivalent w.r.t.\ preferred semantics since $\F^{\k(\adm)} = \F \neq \G = \G^{\k(\adm)}$.

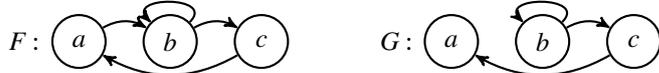
\begin{tikzpicture}

    \node (a) at (0,0) [circle, minimum size = 0.7cm, thick, draw, label = left:$\F:$] {$a$};
    \node (b) at (1.2,0) [circle, minimum size = 0.7cm, thick, draw] {$b$};
    \node (c) at (2.4,0) [circle, minimum size=0.7cm,  thick, draw] {$c$};
    
    \node (a'') at (4.9,0) [circle, minimum size = 0.7cm, thick, draw, label = left:$\G:$] {$a$};
    \node (b'') at (6.1,0) [circle, minimum size = 0.7cm, thick, draw] {$b$};
    \node (c'') at (7.3,0) [circle, minimum size = 0.7cm, thick, draw] {$c$};


\draw[->,thick] (a) to [bend left,thick] (b);
\draw[->,thick] (b) to [loop,thick, distance=0.5cm] (b);
\draw[->,thick] (b) to [bend left,thick] (c);
\draw[->,thick] (c) to [bend left,thick] (a);

\draw[->,thick] (b'') to [loop,thick, distance=0.5cm] (b'');
\draw[->,thick] (b'') to [bend left,thick] (c'');
\draw[->,thick] (c'') to [bend left,thick] (a'');

\end{tikzpicture}

As already stated (up to now) normal expansion equivalence coincides with expansion equivalence for any considered semantics. One possible scenario which makes the predicted different behaviour explicit is the following.

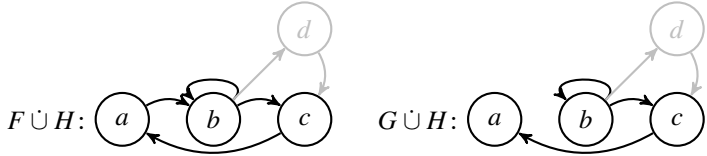
\begin{tikzpicture}

    \node (a) at (0,0) [circle, minimum size = 0.7cm, thick, draw, label = left:$\F\dcup\H\!:$] {$a$};
    \node (b) at (1.2,0) [circle, minimum size = 0.7cm, thick, draw] {$b$};
    \node (c) at (2.4,0) [circle, minimum size = 0.7cm, thick, draw] {$c$};
		\node (d) at (2.4,1.2) [circle, minimum size = 0.7cm, thick, draw, gray!50] {$d$};
    
    \node (a'') at (4.9,0) [circle, minimum size = 0.7cm, thick, draw, label = left:$\G\dcup\H\!:$] {$a$};
    \node (b'') at (6.1,0) [circle, minimum size = 0.7cm, thick, draw] {$b$};
    \node (c'') at (7.3,0) [circle, minimum size = 0.7cm, thick, draw] {$c$};
		\node (d'') at (7.3,1.2) [circle, minimum size = 0.7cm, thick, draw, gray!50] {$d$};


\draw[->,thick] (a) to [bend left,thick] (b);
\draw[->,thick] (b) to [loop,thick, distance=0.5cm] (b);
\draw[->,thick] (b) to [bend left,thick] (c);
\draw[->,thick] (c) to [bend left,thick] (a);
\draw[->,thick, gray!50] (b) to [thick] (d);
\draw[->,thick, gray!50] (d) to [bend left,thick ] (c);

\draw[->,thick] (b'') to [loop,thick, distance=0.5cm] (b'');
\draw[->,thick] (b'') to [bend left,thick] (c'');
\draw[->,thick] (c'') to [bend left,thick] (a'');
\draw[->,thick, gray!50] (b'') to [thick] (d'');
\draw[->,thick, gray!50] (d'') to [bend left,thick] (c'');

\end{tikzpicture}

Formally, we define $\H = (\{b,c,d\},\{(b,d),(d,c)\})$ and we obtain $\{\{a,d\}\} = \Ext_{\prf}(\F\dcup\H)\neq \{\emptyset\} = \Ext_{\prf}(\G\dcup \H)$. We encourage the reader to try to find a witnessing example showing that $\F$ and $\G$ are not local expansion equivalent w.r.t.\ preferred semantics. Due to Theorem~\ref{the:strong2} there has to be at least one distinguishing local expansion.
\end{example} 

How do the semantics behave in case of strong expansion equivalence? Remember, a special feature of strong expansions is that a former attack between old arguments will never become a counterattack to an added attack. In this sense, former attacks do not play a role with respect to being a potential defender of an added argument. Hence, in contrast to arbitrary expansions where such attacks might be relevant, we may delete them without changing the behavior with respect to further evaluations. To make this point clearer consider again the AF $\F\dcup\H$ depicted in Example~\ref{Ex:kernelad}. Note that the already existing attack $(a,b)$ in $\F$ becomes a defending attack of the newly added argument~$d$. This means, such attacks in fact play an important role with respect to further evaluation in case of arbitrary expansions. It was one main result in \cite{normal} that for some semantics attacks like $(a,b)$ in $\F$ are indeed redundant w.r.t.\ strong expansions. Even more surprising, strong expansion equivalence is characterizable through kernels. Therefore, more involved kernel definitions, so-called \textit{$\sigma$-*-kernels} had to be introduced. These kernels allow more deletions than their classical counterparts for expansion equivalence. In contrast to them, $\sigma$-*-kernels are \textit{context-sensitive}, i.e.\ the question whether an attack $(a,b)$ is redundant can not be answered by considering the arguments $a$ and $b$ only~\cite{Bau14}. 

The first three kernels presented in the definition below were firstly introduced in \cite{normal} with the objective to characterize strong expansion equivalence with respect to certain semantics. For the sake of completeness we also present the so-called \textit{$\stg$-*-kernel} as well as \textit{$\nav$-kernel} \cite{equisurvey,BauLW16}.\footnote{As an aside, we use the supplement ``*'', whenever the kernel in question is non-classical and expansion equivalence  is already characterized by another kernel.}

\begin{definition} \label{def:kernel2}
Let $\sigma\in\{\adm,\grd,\com,\stg\}$. The $\sigma$-*-kernel $\k^*(\sigma): \m{F} \to \m{F}$ with $\k^*(\sigma)(\F) = \AF^{\k^*(\sigma)} = \left(A,R^{\k^*(\sigma)}\right)$ for a given AF $\AF = (A,R)$ is defined as:
\begin{align*}
R^{\k^*(\adm)}=&\  R\ \sm \{(a, b) \mid a \neq b, \left((a, a) \in R \wedge \{(b,a),(b,b)\}\cap\ R \neq\emptyset\right)\\
&\vee\left((b,b) \in R \wedge \forall c\ ( (b,c)\in R \to \{(a,c),(c,a),(c,c),(c,b)\}\cap R \neq\emptyset)\right) \},\\
R^{\k^*(\grd)}=&\  R\sm \{(a, b) \mid a \neq b, \left((b, b) \in R \wedge \{(a,a),(b,a)\}\cap\ R\neq\emptyset\right)\\ &\vee\left((b,b) \in R \wedge \forall c\ ( (b,c)\in R \to \{(a,c),(c,a),(c,c)\}\cap R \neq\emptyset)\right) \},\\
R^{\k^*(\com)}=&\  R\sm \{(a, b) \mid a \neq b, \left((a, a),(b,b) \in R \right) \vee \left((b,b) \in R \wedge (b,a)\notin R\right.\\
& \left.\wedge \forall c\ ( (b,c)\in R \to \{(a,c),(c,a),(c,c),(c,b)\}\cap R \neq\emptyset)\right) \},\\
R^{\k^{*}(\stg)} =&\ R\ \sm \{(a, b) \mid a \neq b, (a, a) \in R \vee \forall c\ ( c\neq a \to (c,c)\in R)\}\\
R^{\k(\nav)} =&\ R\cup \{(a, b) \mid a \neq b, \{(a,a),(b,a),(b,b)\}\cap R\neq\emptyset \}.
\end{align*}
The latter represents the so-called \textit{$\nav$-kernel} $\AF^{\k(\nav)} = \left(A,R^{\k(\nav)}\right)$.
\end{definition}

For an illustrating example we refer the reader to Example~\ref{ex:admstar}. Analogously to Fact~\ref{fact:kernel} we collect some properties of the newly introduced kernels. The first three properties are immediately clear by definition.\footnote{The AF $\F = (\{a,b\},\{(a,b)\})$ shows that the naive kernel has to be excluded from item~3 of Fact~\ref{fact:kernel2} since $\F^{\k(\nav)} = (\{a,b\},\{(a,b),(b,a)\})\neq\F$.} The robustness w.r.t.\ deletions and corresponding attacks is less obvious but it is already shown for all considered kernels (except the $\stg$-*-kernel) in case of finite AFs (cf.\ \cite[Theorems 6 and 14]{Bau14}).

\begin{fact} \label{fact:kernel2} Given $\k\in\{\k^*(\adm),\k^*(\grd),\k^*(\com),\k^*(\stg),\k(\nav)\}$ and $\k^*\in\{\k^*(\adm),\k^*(\grd),\k^*(\com),\k^*(\stg)\}$. For two finite AFs $\AF$ and $\AG$ we have:

\begin{enumerate}
  \item $A(\AF) = A\left(\AF^\k\right)$, \hfill{(node-preserving)}
	\item $L(\AF) = L\left(\AF^\k\right)$, \hfill{(loop-preserving)}
	\item $L(\AF) = \emptyset\ \To\ \AF = \AF^{\k^*}$ and \hfill{(sufficient condition for identity)}
	\item If $\AF^\k = \AG^\k$, then $(\AF\sm B)^\k = (\AG\sm B)^\k$ for any finite set of arguments $B$.\hfill{($\sm$-robustness)}
	\end{enumerate}
\end{fact}

Let us consider the $\adm$-*-kernel (which, as we shall
see, characterizes strong expansion equivalence for preferred semantics) in more detail. Consider the first disjunct. This first condition is exactly the same as in case of the $\adm$-kernel (compare Definition~\ref{def:kernel}), i.e.\ an attack $(a,b)$ has to be deleted if $a$ is self-attacking and at least one of the attacks $(b,a)$ or $(b,b)$ exist. The second disjunct provides one with further options to delete an attack $(a,b)$, namely if $b$ is self-defeating and furthermore, for all arguments $c$ which are attacked by $b$ at least one of the following conditions has to be fulfilled: 
	  
		\begin{enumerate}
			\item $a$ attacks $c$,
			\item $c$ attacks $a$,
			\item $c$ attacks $c$,
			\item $c$ attacks $b$.
		\end{enumerate}

The motivation for the second disjunct is the following: At first observe that $b$ cannot be an element of any conflict-free set. Consequently, in case of strong expansions the attack $(a,b)$ may only be relevant with respect to the defense of $c$. In the first three cases this relevance becomes unimportant since $\{a,c\}$ is conflicting. In the fourth case the redundancy of $(a,b)$ with respect to the defense of $c$ is given by the fact that $c$ already defends itself against $b$. Please note that the consideration of $c=a$ or $c=b$ is not excluded by Definition~\ref{def:kernel2}. The following frameworks exemplify different cases.

\begin{example} \label{ex:admstar} The following graphs show six frameworks and their corresponding $\adm$-*-kernels. The dotted attacks represent initial attacks which have to be deleted if applying the $\adm$-*-kernel.

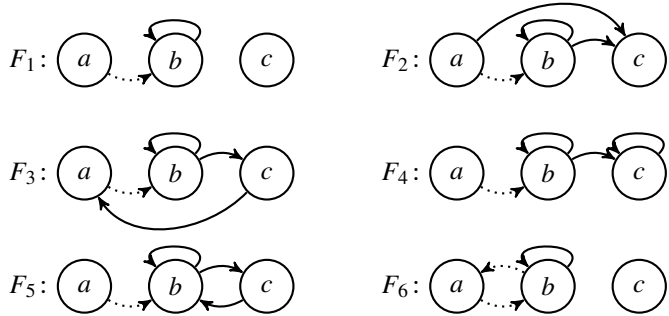
\begin{tikzpicture}

		\node (A5) at (0,-2.5) [circle, minimum size=0.7cm, thick, draw, label = left:$\F_1\!:$] {$a$};
    \node (B5) at (1.2,-2.5) [circle, minimum size=0.7cm,  thick, draw] {$b$};
		\node (C5) at (2.4,-2.5) [circle, minimum size=0.7cm,  thick, draw] {$c$};
		
		\node (A1) at (4.9,-2.5) [circle, minimum size=0.7cm, thick, draw, label = left:$\F_2\!:$] {$a$};
    \node (B1) at (6.1,-2.5) [circle, minimum size=0.7cm,  thick, draw] {$b$};
		\node (C1) at (7.3,-2.5) [circle, minimum size=0.7cm,  thick, draw] {$c$};
		
		\node (A2) at (0,-4) [circle, minimum size=0.7cm, thick, draw, label = left:$\F_3\!:$] {$a$};
    \node (B2) at (1.2,-4) [circle, minimum size=0.7cm,  thick, draw] {$b$};
		\node (C2) at (2.4,-4) [circle, minimum size=0.7cm,  thick, draw] {$c$};
		
		\node (A3) at (4.9,-4) [circle, minimum size=0.7cm, thick, draw, label = left:$\F_4\!:$] {$a$};
    \node (B3) at (6.1,-4) [circle, minimum size=0.7cm,  thick, draw] {$b$};
		\node (C3) at (7.3,-4) [circle, minimum size=0.7cm,  thick, draw] {$c$};
		
		\node (A4) at (0,-5.5) [circle, minimum size=0.7cm, thick, draw, label = left:$\F_5\!:$] {$a$};
    \node (B4) at (1.2,-5.5) [circle, minimum size=0.7cm,  thick, draw] {$b$};
		\node (C4) at (2.4,-5.5) [circle, minimum size=0.7cm,  thick, draw] {$c$};

		\node (A6) at (4.9,-5.5) [circle, minimum size=0.7cm, thick, draw, label = left:$\F_6\!:$] {$a$};
    \node (B6) at (6.1,-5.5) [circle, minimum size=0.7cm,  thick, draw] {$b$};
		\node (C6) at (7.3,-5.5) [circle, minimum size=0.7cm,  thick, draw] {$c$};


\draw[->, thick, dotted] (A1) to [thick,bend right] (B1);
\draw[->, thick, dotted] (A2) to [thick,bend right] (B2);
\draw[->, thick, dotted] (A3) to [thick,bend right] (B3);
\draw[->, thick, dotted] (A4) to [thick,bend right] (B4);
\draw[->, thick, dotted] (A5) to [thick,bend right] (B5);
\draw[->, thick, dotted] (A6) to [thick,bend right] (B6);
\draw[->, thick, dotted] (B6) to [thick,bend right] (A6);

\draw[->, thick] (B1) to [thick,loop,distance=0.5cm] (B1);
\draw[->, thick] (B2) to [thick,loop,distance=0.5cm] (B2);
\draw[->, thick] (B3) to [thick,loop,distance=0.5cm] (B3);
\draw[->, thick] (B4) to [thick,loop,distance=0.5cm] (B4);
\draw[->, thick] (B5) to [thick,loop,distance=0.5cm] (B5);
\draw[->, thick] (B6) to [thick,loop,distance=0.5cm] (B6);

\draw[->, thick] (B1) to [thick,bend left] (C1);
\draw[->, thick] (B2) to [thick,bend left] (C2);
\draw[->, thick] (B3) to [thick,bend left] (C3);
\draw[->, thick] (B4) to [thick,bend left] (C4);

\draw[->, thick] (A1) to [thick,bend left,out=44,in=120] (C1);
\draw[->, thick] (C2) to [thick,bend left,out=44,in=120] (A2);
\draw[->, thick] (C3) to [thick,loop,distance=0.5cm] (C3);
\draw[->, thick] (C4) to [thick,bend left] (B4);

\end{tikzpicture}

Consider the formal description of $R^{\k^*(\adm)}$ as given in Definition~\ref{def:kernel2}. The AF $F_1$ is somehow the base case since the only argument $c$, s.t.\ $(b,c)\in R(\F_1)$ is~$b$ itself. Since $(b,b) \in R(\F_1)$ we deduce that the considered intersection is non-empty and thus, the deletion of $(a,b)$ is justified. The subsequent four frameworks $\F_2$, $\F_3$, $\F_4$ and $\F_5$ are the base case plus one further argument $c$ different from $a$ and $b$, s.t.\ for any $i\in\{2,3,4,5\}$, $(b,c)\in R(\F_i)$. The last framework $\F_6$ illustrates the case $b$ counterattacks $a$. Note that the reason to delete $(a,b)$ is somehow self-referential since (additionally to the base case) it is justified by $(a,b)\in R(\F_6)$. Due to the first disjunct (i.e.\ just like in case of the classical $\adm$-kernel) even the attack $(b,a)$ has to be deleted.

\end{example}  

We proceed with further characterization theorems.\footnote{Please note that the results in case of cf2 and stage2 semantics have never been published before.} An comprehensive overview of equivalence notion and their characterizing kernels in case of finite AFs and extension-based semantics is presented in Figure~\ref{fig:extrel}. 
\pagebreak
\begin{theorem}{\cite{strong,normal,equisurvey,BauLW16a}} \label{the:strong2} For finite AFs $F$ and $G$ we have the following coincidences.
\begin{enumerate}
	\item $\AF\equiv^{\Ext_\sigma}_E\! \AG \ToT \AF\equiv^{\Ext_\sigma}_N\! \AG $ for $\sigma\in\{\stg,\stb,\semi,\eag,\adm,\prf,\id,\grd,\com,\nav,\cfzwei,\stgzwei\}$,
	\item $\AF\equiv^{\Ext_\sigma}_E\! \AG \ToT \AF\equiv^{\Ext_\sigma}_L\! \AG $ for $\sigma\in\{\semi,\eag,\adm,\prf,\id,\nav\}$ and
	\item $\AF\equiv^{\Ext_\sigma}_E\! \AG \ToT \AF\equiv^{\Ext_\sigma}_S\! \AG $ for $\sigma\in\{\stg,\stb,\semi,\eag,\nav\}$.
\end{enumerate}
Furthermore, for any two finite AFs $F$ and $G$ we have the following non-classical characterizations.
\begin{enumerate} \setcounter{enumi}{3}
	\item  $\AF\equiv^{\Ext_\stg}_L\! \AG \ToT \AF^{\k^*(\stg)} = \AG^{\k^*(\stg)}$,
	\item $\AF\equiv^{\Ext_\sigma}_S\! \AG \ToT \AF^{\k^*(\adm)} = \AG^{\k^*(\adm)}$ for $\sigma\in\{\adm,\prf,\id\}$, 
	\item $\AF\equiv^{\Ext_\sigma}_S\! \AG \ToT \AF^{\k^*(\sigma)} = \AG^{\k^*(\sigma)}$ for $\sigma\in\{\com,\grd\}$ and 
	\item  $\AF\equiv^{\Ext_\nav}_E\! \AG \ToT \AF^{\k(\nav)} = \AG^{\k(\nav)}$.
\end{enumerate}
\end{theorem}

At this point we want to highlight a very surprising relation. Remember that normal expansion equivalence and normal deletion equivalence are completely unrelated in the general picture (cf.\ Figure~\ref{fig:prelimrel}). The observation that the characterizing kernels (including the identity map in case of cf2 and stage2 semantics) of normal expansion equivalence w.r.t.\ all considered semantics in this section satisfy $\sm$-robustness (cf.\ Facts~\ref{fact:kernel} and \ref{fact:kernel2}) reveals that normal expansion equivalence implies normal deletion equivalence for these semantics.

\begin{corollary} {\label{cor:normal}} Given $\sigma\in\{\stg,\stb,\semi,\eag,\adm,\prf,\id,\grd,\com,\nav,\cfzwei,\stgzwei\}$ and two finite AFs $\F$and $\G$. We have: $\F\equiv^{\Ext_\sigma}_{N}\!\G \To \F\equiv^{\Ext_\sigma}_{ND}\!\G.$ 
\end{corollary}

The attentive reader may have noticed that we do not have characterized local expansion
equivalence w.r.t.\ stable, complete as well as grounded
extension-based semantics. We mention that all three equivalence notions are already characterized but the characterization theorems are not purely kernel-based (cf.\ \cite[Theorems~9,10,11]{strong}). Furthermore, it can be checked that none of the kernels presented in Definitions~\ref{def:kernel} and~\ref{def:kernel2} serve as a characterizing kernel. Consider therefore the following example~\cite[Example~15]{strong}.

\begin{example} 
The AFs $\F$ and $\G$ are local expansion equivalent w.r.t.\ stable semantics. This can be seen as follows. Given an AF $\H$, s.t.\ $A(\H) \subseteq\{a, b\}$. If $(a, b) \in R(\H)$ and $(a,a) \notin R(\H)$,
we obtain $\Ext_{\stb}(\F\dcup\H) = \Ext_{\stb}(\G\dcup\H) = \{\{a\}\}$. Otherwise, $\Ext_{\stb}(\F\dcup\H) = \Ext_{\stb}(\G\dcup\H) = \emptyset$.

\begin{tikzpicture}

    \node (a) at (0,0) [circle, minimum size = 0.7cm, thick, draw, label = left:$\F:$] {$a$};
    \node (b) at (1.4,0) [circle, minimum size = 0.7cm, thick, draw] {$b$};

    \node (b'') at (4.0,0) [circle, minimum size = 0.7cm, thick, draw, label = left:$\G:$] {$b$};


\draw[->,thick] (b) to [bend left,thick] (a);
\draw[->,thick] (b) to [loop,thick, distance=0.5cm] (b);

\draw[->,thick] (b'') to [loop,thick, distance=0.5cm] (b'');

\end{tikzpicture}

Remember that all introduced kernels are node-preserving (Facts \ref{fact:kernel} and \ref{fact:kernel2}). Consequently, none of them may serve as a characterizing kernel for local expansion equivalence w.r.t.\ stable semantics.
\end{example}

We mention that weak expansion equivalence is already characterized in case of stable semantics \cite[Proposition 3]{normal} as well as admissible, preferred and complete semantics \cite[Theorem 1]{zoo2}. All characterization results are not kernel-based. For instance, two AFs are weak expansion equivalent w.r.t.\ stable semantics iff both do not possess stable extensions at all or if they share the same arguments and at the same time possess the same stable extensions. Consequently, $\F = (\{a\},\{(a,a)\})$ and $\G = (\{a,b,c\},\{(a,b),(b,c),(c,a)\})$ are weak expansion equivalent w.r.t.\ stable semantics. Both frameworks witness that any potential characterizing kernel $\k$ is necessarily neither node- nor loop-preserving. 

As a final note, we are not aware of any study of weaker notions of expansion equivalence in case of cf2 as well as stage2 semantics. 

\subsubsection{Notions of Deletion Equivalence and Update Equivalence}

We start with local deletion equivalence. Remember that local deletion equivalent AFs cannot be semantically distinguished by deleting a certain set of attacks in both simultaneously. How ``strong'' is this notion? Are there redundant attacks or even redundant arguments? 

\begin{example} \label{ref:local} Consider the following AFs $\F$, $\G$ and $\H$.

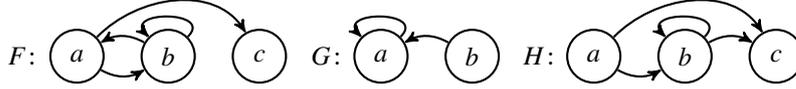
\begin{tikzpicture}

    \node (A1'') at (0,-2.5) [circle,minimum size=0.7cm, thick, draw, label = left:$\F\!:$]{$a$};
    \node (B1'') at (1.2,-2.5) [circle, minimum size=0.7cm, thick, draw]{$b$};
		\node (C1'') at (2.4,-2.5) [circle, minimum size=0.7cm, thick, draw]{$c$};

		\node (A2'') at (4.0,-2.5) [circle, minimum size=0.7cm, thick, draw, label = left:$\G\!:$] {$a$};
    \node (B2'') at (5.2,-2.5) [circle, minimum size=0.7cm,  thick, draw] {$b$};
		
		\node (A3'') at (6.8,-2.5) [circle, minimum size=0.7cm, thick, draw, label = left:$\H\!:$] {$a$};
    \node (B3'') at (8,-2.5) [circle, minimum size=0.7cm,  thick, draw] {$b$};
		\node (C3'') at (9.2,-2.5) [circle, minimum size=0.7cm,  thick, draw] {$c$};

\draw[->, thick] (A1'') to [thick,bend right] (B1'');
\draw[->, thick] (A1'') to [thick,bend left,out=44,in=120] (C1'');
\draw[->, thick] (B1'') to [thick,loop,distance=0.5cm] (B1'');
\draw[->, thick] (B1'') to [thick,bend right] (A1'');

\draw[->, thick] (A2'') to [thick,loop,distance=0.5cm] (A2'');
\draw[->, thick] (B2'') to [thick,bend right] (A2'');

\draw[->, thick] (A3'') to [thick,bend right] (B3'');
\draw[->, thick] (A3'') to [thick,bend left,out=44,in=120] (C3'');
\draw[->, thick] (B3'') to [thick,loop,distance=0.5cm] (B3'');
\draw[->, thick] (B3'') to [thick,bend left] (C3'');

\end{tikzpicture}

The AFs $\F$ and $\G$ do not possess the same arguments. Let us delete all occurring attacks, i.e.\ $S_A = R(\F) \cup R(\G)$. We obtain the following local deletions where $\{a,b,c\}\in\Ext_{\sigma}(\F\sm S_A)\sm \Ext_{\sigma}(\G\sm S_A)$ for all semantics $\sigma$ considered in this section.  

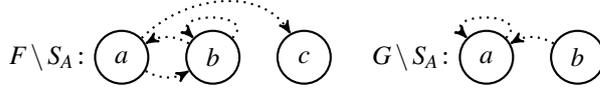
\begin{tikzpicture}

    \node (A1'') at (0,-2.5) [circle,minimum size=0.7cm, thick, draw, label = left:$\F\sm S_A\!:$]{$a$};
    \node (B1'') at (1.2,-2.5) [circle, minimum size=0.7cm, thick, draw]{$b$};
		\node (C1'') at (2.4,-2.5) [circle, minimum size=0.7cm, thick, draw]{$c$};

		\node (A2'') at (4.8,-2.5) [circle, minimum size=0.7cm, thick, draw, label = left:$\G\sm S_A\!:$] {$a$};
    \node (B2'') at (6.0,-2.5) [circle, minimum size=0.7cm,  thick, draw] {$b$};

\draw[->, thick, dotted] (A1'') to [thick,bend right] (B1'');
\draw[->, thick, dotted] (A1'') to [thick,bend left,out=44,in=120] (C1'');
\draw[->, thick, dotted] (B1'') to [thick,loop,distance=0.5cm] (B1'');
\draw[->, thick, dotted] (B1'') to [thick,bend right] (A1'');

\draw[->, thick, dotted] (A2'') to [thick,loop,distance=0.5cm] (A2'');
\draw[->, thick, dotted] (B2'') to [thick,bend right] (A2'');

\end{tikzpicture}

The AFs $\F$ and $\H$ possess the same arguments but differ in their attack-relation, e.g. $(b,c)\in R(\H) \sm R(\F)$. This difference can be made more explicit if defining $S_R = \left(R(\F) \cup R(\H)\right)\sm \{(b,c)\}$. Consider the resulting local deletions.     

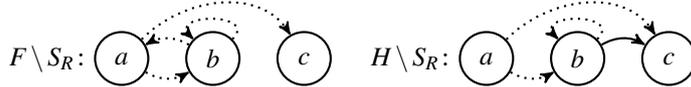
\begin{tikzpicture}

    \node (A1'') at (0,-2.5) [circle,minimum size=0.7cm, thick, draw, label = left:$\F\sm S_R\!:$]{$a$};
    \node (B1'') at (1.2,-2.5) [circle, minimum size=0.7cm, thick, draw]{$b$};
		\node (C1'') at (2.4,-2.5) [circle, minimum size=0.7cm, thick, draw]{$c$};

		\node (A3'') at (4.8,-2.5) [circle, minimum size=0.7cm, thick, draw, label = left:$\H\sm S_R\!:$] {$a$};
    \node (B3'') at (6,-2.5) [circle, minimum size=0.7cm,  thick, draw] {$b$};
		\node (C3'') at (7.2,-2.5) [circle, minimum size=0.7cm,  thick, draw] {$c$};

\draw[->, thick, dotted] (A1'') to [thick,bend right] (B1'');
\draw[->, thick, dotted] (A1'') to [thick,bend left,out=44,in=120] (C1'');
\draw[->, thick, dotted] (B1'') to [thick,loop,distance=0.5cm] (B1'');
\draw[->, thick, dotted] (B1'') to [thick,bend right] (A1'');

\draw[->, thick, dotted] (A3'') to [thick,bend right] (B3'');
\draw[->, thick] (B3'') to [thick,bend left] (C3'');
\draw[->, thick, dotted] (A3'') to [thick,bend left,out=44,in=120] (C3'');
\draw[->, thick, dotted] (B3'') to [thick,loop,distance=0.5cm] (B3'');

\end{tikzpicture}

Once again we have $\{a,b,c\}\in\Ext_{\sigma}(\F\sm S_R)$ for all known semantics $\sigma$ and $\{a,b,c\}\notin \Ext_{\sigma}(\H\sm S_R)$ if assuming conflict-freeness of the considered semantics.

\end{example}

The observations above indicate that there is not much space for redundancy in case of local expansion equivalence and indeed, it was one main result in \cite{Bau14} that local expansion equivalence collapse to identity for all semantics considered in this section. Moreover, instead of proving this one by one for any semantics the author followed the line in \cite{BarG07} and provide abstract criteria guaranteeing the coincidence with syntactical identity. These criteria are very weak requirements, namely \textit{conflict-freeness} (\m{CF}) and the principle of \textit{isolate-inclusion} (\m{II}). The latter is fulfilled by a semantics $\sigma$ iff for any AF $\F$, the set of all isolated arguments is contained in at least one $\sigma$-extension. Observe that any considered semantics apart from stable semantics satisfy \m{II}.\footnote{A counter-example in case of stable semantics is given by $\F = (\{a,b\},\{(b,b)\})$. Obviously, $a$ is isolated but $\Ext_{\stb}(\F) = \emptyset$. Nevertheless, local expansion equivalence in case of stable semantics collapse to identity too.} 

\begin{theorem}[\cite{Bau14}] \label{the:identity} Given a semantics $\sigma$ satisfying \m{CF} and \m{II}. For two finite AFs $\F$ and $\G$ we have: 
$$\F\equiv^{\Ext_\sigma}_{LD}\G \ToT \F = \G.$$ 
\end{theorem}

Since being identical implies local deletion equivalence we deduce that all equivalence notion ``inbetween'' them collapse to identity too (cf.\ Figure~\ref{fig:prelimrel}). 

\begin{proposition} \label{the:identity} Given a semantics $\sigma$ satisfying \m{CF} and \m{II}. For two finite AFs $\F$ and $\G$ we have: 
$$\F\equiv^{\Ext_\sigma}_{U}\G \ToT \F\equiv^{\Ext_\sigma}_{D}\G \ToT \F = \G.$$ 
\end{proposition}

This means, for semantics satisfying conflict-freeness and isolate-inclusion any argument/attack may play a crucial role with respect to further evaluations if updates, deletions or
local deletions are considered. Note that the results may apply to future semantics. In order to refine the general picture (as depicted in Figure~\ref{fig:prelimrel}) for the semantics considered in this section we state the following relations.\footnote{The results in case of cf2 and stage2 semantics have never been published before. It can be checked that both semantics satisfy the preconditions of Theorem~\ref{the:identity}.} 
\pagebreak
 \begin{corollary} {\label{cor:relDE} Let $\sigma\in\{\stg,\stb,\semi,\eag,\adm,\prf,\id,\grd,\com,\nav,\cfzwei,\stgzwei\}$. For two finite AFs $\F$ ,$\G$ we have:

\begin{enumerate}
\item $\F\equiv^{\Ext_\sigma}_{U}\G \ToT\F\equiv^{\Ext_\sigma}_{D}\G \ToT\F\equiv^{\Ext_\sigma}_{LD}\G \ToT \F = \G$, \hfill{($\k=\iden$)}
  \item $\F\equiv^{\Ext_\sigma}_{D}\G \To \F\equiv^{\Ext_\sigma}_{E}\G,$ \hfill{(deletion vs. expansion)}
	\item $\F\equiv^{\Ext_\sigma}_{LD}\G \To \F\equiv^{\Ext_\sigma}_{L}\G.$ \hfill{(local versions)}
	\end{enumerate}}
\end{corollary}

\subsection{The Exceptional Case of Normal Deletion Equivalence} \label{sec:exceptdel}

Normal deletion equivalence, where the retraction of
arguments and corresponding attacks is considered, is exceptional in several regards. Firstly,
the characterization theorems for admissible, complete and grounded
semantics partially rely on $\sigma$-*-kernels. Remember that these kernels were originally introduced to characterize strong expansion (cf.\ Theorem~\ref{the:strong2}). Secondly, normal deletion equivalent AFs do not even have to share the same arguments and thus give space for simplifications. 

\begin{example}\label{ex:normaldeletion} Consider the following AFs $\F$ and $\G$. We have $\Ext_{\adm}(\F) = \Ext_{\adm}(\G) = \{\emptyset,\{a\}\}$. Even more, for any set of arguments $B$, $\Ext_{\adm}(\F\sm B) = \Ext_{\adm}(\G\sm B)$ showing their normal deletion equivalence, i.e.\ $\F\equiv^{\Ext_\adm}_{ND}\G$.

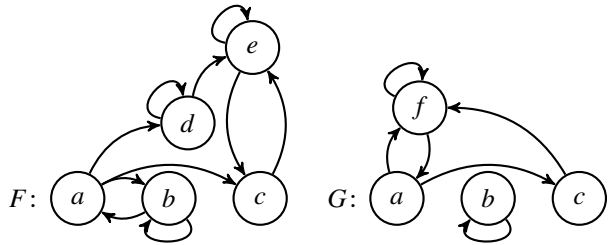
\begin{tikzpicture}

        \node (a1) at (0,0) [circle, thick, draw, minimum size = 0.7cm, label = left:$\F\!:$] { $a$};
        \node (a2) at (1.2,0) [circle, thick, draw, minimum size = 0.7cm] { $b$};
        \node (a3) at (2.4,0) [circle, thick, draw, minimum size = 0.7cm] { $c$};
				
				\node (a4) at (1.45,1) [circle, thick, draw, minimum size = 0.7cm] { $d$};
        \node (a5) at (2.3,2) [circle, thick, draw, minimum size = 0.7cm] { $e$};
        
        \node (a1'') at (4.2,0) [circle, thick, draw, draw, minimum size = 0.7cm, label = left:$\G\!:$] { $a$};
        \node (a2'') at (5.4,0) [circle, thick, draw, minimum size = 0.7cm] {$b$};
				\node (a3'') at (6.6,0) [circle, thick, draw, minimum size = 0.7cm] { $c$};
				\node (a4'') at (4.5,1.2) [circle, thick, draw, minimum size = 0.7cm] { $f$};

\draw[->, thick] (a2) to [thick,bend left] (a1);

\draw[->, thick] (a1) to [thick,bend left] (a3);
\draw[->, thick] (a1) to [thick,bend left] (a2);
\draw[->, thick] (a2) to [thick, loop, distance=5mm,in=230,out=310,looseness=5] (a2);

\draw[->, thick] (a1) to [thick,bend left] (a4);
\draw[->, thick] (a5) to [thick,bend right] (a3);

\draw[->, thick] (a3) to [thick,bend right] (a5);
\draw[->, thick] (a4) to [thick,bend left] (a5);

\draw[->, thick] (a4) to [thick, loop, distance=5mm,in=100,out=170,looseness=5] (a4);
\draw[->, thick] (a5) to [thick, loop, distance=5mm,in=100,out=170,looseness=5] (a5);

\draw[->, thick] (a2'') to [thick, loop, distance=5mm,in=230,out=310,looseness=5] (a2'');
\draw[->, thick] (a1'') to [thick,bend left] (a3'');
\draw[->, thick] (a1'') to [thick,bend left] (a4'');
\draw[->, thick] (a3'') to [thick,bend right] (a4'');
\draw[->, thick] (a4'') to [thick,bend left] (a1'');
\draw[->, thick] (a4'') to [thick, loop, distance=5mm,in=85,out=155,looseness=5] (a4'');

\end{tikzpicture}

Observe that the non-shared arguments $d$, $e$ and $f$ do not play a role for the evaluation w.r.t.\ admissible semantics since firstly, they are self-defeating and thus cannot be part of an admissible set; and secondly, if they attack a non-looping argument shared by both arguments, e.g.\ $e$ attacks $c$ in $\F$ or $f$ attacks $a$ in $\G$, then they are counter-attacked by the same argument, i.e.\ $c$ attacks $e$ in $\F$ and $a$ attacks $f$ in $\G$. Consequently, they cannot influence potential admissible sets being a subset of $\{a,b\}$. Finally, let us consider the $\adm$-*-kernel of both frameworks (cf.\ Example~\ref{ex:admstar} and the comments above for more details).

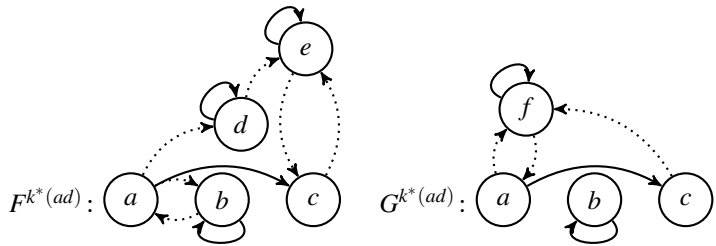
\begin{tikzpicture}

        \node (a1) at (0,0) [circle, thick, draw, minimum size = 0.7cm, label = left:$\F^{\k^*(\adm)}\!:$] { $a$};
        \node (a2) at (1.2,0) [circle, thick, draw, minimum size = 0.7cm] { $b$};
        \node (a3) at (2.4,0) [circle, thick, draw, minimum size = 0.7cm] { $c$};
				
				\node (a4) at (1.45,1) [circle, thick, draw, minimum size = 0.7cm] { $d$};
        \node (a5) at (2.3,2) [circle, thick, draw, minimum size = 0.7cm] { $e$};
        
        \node (a1'') at (4.9,0) [circle, thick, draw, draw, minimum size = 0.7cm, label = left:$\G^{\k^*(\adm)}\!:$] { $a$};
        \node (a2'') at (6.1,0) [circle, thick, draw, minimum size = 0.7cm] {$b$};
				\node (a3'') at (7.3,0) [circle, thick, draw, minimum size = 0.7cm] { $c$};
				\node (a4'') at (5.2,1.2) [circle, thick, draw, minimum size = 0.7cm] { $f$};

\draw[->, thick, dotted] (a2) to [thick,bend left] (a1);

\draw[->, thick] (a1) to [thick,bend left] (a3);
\draw[->, thick, dotted] (a1) to [thick,bend left] (a2);
\draw[->, thick] (a2) to [thick, loop, distance=5mm,in=230,out=310,looseness=5] (a2);

\draw[->, thick, dotted] (a1) to [thick,bend left] (a4);
\draw[->, thick, dotted] (a5) to [thick,bend right] (a3);

\draw[->, thick, dotted] (a3) to [thick,bend right] (a5);
\draw[->, thick, dotted] (a4) to [thick,bend left] (a5);

\draw[->, thick] (a4) to [thick, loop, distance=5mm,in=100,out=170,looseness=5] (a4);
\draw[->, thick] (a5) to [thick, loop, distance=5mm,in=100,out=170,looseness=5] (a5);

\draw[->, thick] (a2'') to [thick, loop, distance=5mm,in=230,out=310,looseness=5] (a2'');
\draw[->, thick] (a1'') to [thick,bend left] (a3'');
\draw[->, thick, dotted] (a1'') to [thick,bend left] (a4'');
\draw[->, thick, dotted] (a3'') to [thick,bend right] (a4'');
\draw[->, thick, dotted] (a4'') to [thick,bend left] (a1'');
\draw[->, thick] (a4'') to [thick, loop, distance=5mm,in=85,out=155,looseness=5] (a4'');

\end{tikzpicture}

Obviously, $\F$ and $\G$ do not possess the same kernels but note that their restrictions to the shared arguments
do, i.e.\ $\left(\F|_{\{a,b,c\}}\right)^{\k^*(\adm)} = \left(\G|_{\{a,b,c\}}\right)^{\k^*(\adm)}$. 
\end{example}

It turned out that the issues raised in Example~\ref{ex:normaldeletion} are essential to characterize normal deletion equivalence w.r.t.\ admissible semantics. In case of complete and grounded semantics slightly different conditions have to be fulfilled, namely w.r.t.\ the non-shared arguments we have ``it is forbidden to be attacked'' instead of  ``counter-attack if attacked'' like in case of admissible semantics and furthermore, instead of the $\adm$-*-kernel the corresponding $\sigma$-*-kernels are used. Consider therefore the following definition and the characterization theorem. We use $\Delta$ to denote the symmetric difference, i.e.\linebreak $A \Delta A' = A\sm A' \cup A'\sm A$. Moreover, $NL(\F) = A(\F)\sm L(\F)$, i.e. $NL(\F)$ contains all arguments of $\F$ which are not self-defeating.

\begin{definition} 
Given $\F = (A,R)$ and $\G = (A',R')$ and let $\sigma\in\{\com,\grd\}$.

\begin{enumerate}
	\item\! $\lpaf \ToT_{def} L\left(\F\dcup\G|_{A\Delta A'}\right) = A\Delta A',$ \hfill{(``non-shared args are self-defeating'')}
	\item\!  $\attadaf\ToT_{def}\forall b\in A\sm A'\ \forall a\in NL(\F|_{A\cap A'}): \left((b,a)\!\in\!R \to (a,b)\!\in\!R\right)\\ \wedge \forall b\in A'\sm A\ \forall a\in NL(\G|_{A\cap A'}): \left((b,a)\!\in\!R' \to (a,b)\!\in\!R'\right)$,\hfill{(``counter-attack if attacked'')}
	\item\! $\attsigmaaf\ToT_{def}\forall b\in A\sm A'\ \forall a\in NL(\F|_{A\cap A'}): (b,a)\notin\!R\\ \wedge \forall b\in A'\sm A\ \forall a\in NL(\G|_{A\cap A'}): (b,a)\notin R'$. \hfill{(``it is forbidden to be attacked'')}
\end{enumerate}
\end{definition} 

\begin{theorem}[\cite{Bau14}] \label{the:normal} Let $\sigma\in\{\adm,\com,\grd\}$. Given two finite AFs $\F=(A,R)$ and $\G=(A',R')$ and let $I = A\cap A'$,
$$\F\equiv^{\Ext_\sigma}_{ND}\!\G \ToT \lpaf,\ \attsigmaaf, \left(\F|_{I}\right)^{k^*(\sigma)} = \left(\G|_{I}\right)^{k^*(\sigma)}.$$
\end{theorem}

In contrast to admissible, complete and grounded semantics where normal deletion equivalence is indeed weaker than normal expansion equivalence we observe that these notions coincide in case of stable semantics. This means, normal deletion equivalence w.r.t.\ stable semantics is characterized by the classical stable kernel too. 

The following theorem corrects the corresponding result in \cite[Theorem~10]{Bau14} which
did not take into account that an empty framework possess a stable extension, namely the empty one.\footnote{We mention that Theorem~10 in \cite{Bau14} hold, given that resulting AFs have to be non-empty. The claimed normal deletion equivalence of the AFs $\F$ and $\G$ depicted in \cite[Example~4]{Bau14} can be disproved by setting $B = \{a,b,c,f\}$. } 

\begin{theorem} For finite AFs $\F$ and $\G$ we have:
$$\F\equiv^{\Ext_\stb}_{ND}\!\G \ToT \F^{\k(\stb)} = \G^{\k(\stb)}.$$
\end{theorem}

\begin{proof} ($\To$) We show the contrapositive, i.e.\ $\F^{\k(\stb)} \neq \G^{\k(\stb)} \To \F\not\equiv^{\Ext_\stb}_{ND}\!\G$.\\
$1^{st}$ case: Assume $A\left(\F^{\k(\stb)}\right) \neq A\left(\G^{\k(\stb)}\right)$ and w.l.o.g. let $a\in A\left(\F^{\k(\stb)}\right) \sm A\left(\G^{\k(\stb)}\right)$. Since the stable kernel is node-preserving (Fact~\ref{fact:kernel}) we obtain $\G\sm B = (\emptyset,\emptyset)$ and $\F\sm B \in\left\{(\{a\},\emptyset),(\{a\},\{(a,a)\})\right\}$ if $B = \left(A(\F)\cup A(\G)\right) \sm \{a\}$. In either case, $\emptyset\in\Ext_{\stb}(\G)\sm \Ext_{\stb}(\F)$ since $\Ext_{\stb}(\F)\in\left\{\emptyset,\{\{a\}\}\right\}$. From now on we assume $A\left(\F^{\k(\stb)}\right) = A\left(\G^{\k(\stb)}\right)$.\\
$2^{nd}$ case: Consider $R\left(\F^{\k(\stb)}\right) \neq R\left(\G^{\k(\stb)}\right)$ and w.l.o.g. let $(a,b)\in R\left(\F^{\k(\stb)}\right) \sm R\left(\G^{\k(\stb)}\right)$. Let $a=b$. Remember that the stable kernel is loop-preserving (Fact~\ref{fact:kernel}). Therefore, $(a,a)\in R\left(\F\right) \sm R\left(\G\right)$. We obtain $\G\sm B = (\{a\},\emptyset)$ and $\F\sm B = (\{a\},\{(a,a)\})$ if $B = \left(A(\F)\cup A(\G)\right) \sm \{a\}$. Hence, $\emptyset = \Ext_{\stb}(\F)\neq \Ext_{\stb}(\G) = \{\{a\}\}$. From now on we assume $L\left(\F^{\k(\stb)}\right) = L\left(\G^{\k(\stb)}\right)$. Consider now $a\neq b$. Consequently, $(a,b)\in R(\F)$ and $(a,a)\notin R(\F)$. Hence, $(a,a)\notin R(\G)$ and furthermore, $(a,b)\notin R(\G)$. Define $B = \left(A(\F)\cup A(\G)\right) \sm \{a,b\}$. In any case, $\{a\}\in \Ext_{\stb}(\F\sm B)\sm \Ext_{\stb}(\G\sm B)$ concluding the if-direction.\\
($\oT$) Given $\F^{\k(\stb)} = \G^{\k(\stb)}$. Applying Theorems~\ref{the:strong} and \ref{the:strong2} one after the other yields $\F\equiv^{\Ext_\stb}_{E}\!\G$ and then $\F\equiv^{\Ext_\stb}_{N}\!\G$. Finally, Corollary~\ref{cor:normal} justifies $\F\equiv^{\Ext_\stb}_{ND}\!\G$ concluding the proof. 
\end{proof}

\subsection{Characterization Theorems in Case of Self-loop-free AFs}

We already observed that apart from naive kernel any mentioned kernel $\k$ does not change anything if the considered AF $\F$ is self-loop-free, i.e.\ $\F = \F^\k$ (cf.\ Facts~\ref{fact:kernel} and \ref{fact:kernel2}). Consequently, any equivalence relation characterizable through such a kernel collapses to identity if we restrict ourselves to self-loop-free AFs. This is stated in the following theorem.

\begin{theorem} Given any binary relation $\equiv\ \subseteq \m{F}\times\m{F}$ characterizable through $\k$ where \linebreak $\k\in\{\k(\stb),\k(\adm),\k(\grd),\k(\com),\k^*(\adm),\k^*(\grd),\k^*(\com),\k^*(\stg)\}$. For any self-loop-free AFs $\F$ and $\G$, 
$$\F\equiv \G \ToT \F = \G.$$
\end{theorem}

We will refrain from listing all combinations of semantics and equivalence notions characterizable through a kernel mentioned in the theorem above. Please confer Figures~\ref{fig:extrel} and \ref{fig:labelrel} for compact overviews.
For all such combinations, self-loop-free AFs are redundancy-free, i.e.\ all
attacks as well as arguments may play a crucial role w.r.t.\ further evaluations and thus, there is no space for simplification. In the introductory part of this section we
noted that many equivalence notions, e.g. normal and local expansion equivalence are motivated by the instantiation-based context where AFs are built from an underlying knowledge base. However, we want to mention that there are
some formalisms like classical logic-based argumentation where self-attacking arguments
do not occur \cite[Theorem 4.13]{hunterlink}, while for other systems, e.g.
ASPIC self-defeating arguments indeed may arise \cite[Section~7]{Pra10}. 

\subsection{Summary of Results and Conclusion}

In the presented results the notion of a kernel played a crucial role. Indeed, kernels are interesting from
several perspectives: First, they allow to decide the corresponding notion of equivalence by a
simple check for topological (i.e.\ syntactical) equality. Moreover, all kernels we have obtained
so far can be efficiently constructed from a given argumentation framework. This means, if a certain equivalence notion is characterizable through such a kernel, then we have tractability of the associated decision problem.

The following Figure~\ref{fig:extrel} provides a comprehensive overview of the state of the art in case of extension-based semantics. The entry ``$\k$'' in row $M$ and column $\sigma$ indicates that $\equiv_M^{\Ext_{\sigma}}$ is characterizable through $\k$. The abbreviation ``$\iden$'' stands for identity map and the question mark represents an open problem. Further abbreviations like ``$\mathit{L}$'' and ``$\attsigma$'' refer to additional conditions relevant in case of normal deletion equivalence (cf.\ Theorem~\ref{the:normal}). The entry ``$[m,n]$'' indicates three facts. First, the characterization problem is already solved in Theorem/Proposition $n$ in $m$.\footnote{For $m$ we use the following assignments: $1 = $ \cite{split}, $2 = $\cite{zoo2}, $3 = $\cite{zoo} and $4 = $\cite{strong} } Second, the characterization result is not (purely) kernel-based and third, it can be checked that none of the introduced kernels serve as a characterization.

\begin{figure}[t]
\begin{center}
\begin{tikzpicture}
\tikzstyle{literal}=[text centered, text width=15mm, minimum height=5mm]
\draw[help lines] (0,-4) grid (13,6) ;

\node[literal] (S1) at (1.5,5.5)  {$\textcolor{white}{d}\stg\textcolor{white}{p}$};
\node[literal] (S1) at (2.5,5.5)  {$\textcolor{white}{d}\stb\textcolor{white}{p}$};
\node[literal] (S1) at (3.5,5.5)  {$\textcolor{white}{d}\semi\textcolor{white}{p}$};
\node[literal] (S1) at (4.5,5.5)  {$\textcolor{white}{d}\eag\textcolor{white}{p}$};
\node[literal] (S1) at (5.5,5.5)  {$\textcolor{white}{d}\adm\textcolor{white}{p}$};
\node[literal] (S1) at (6.5,5.5)  {$\textcolor{white}{d}\prf\textcolor{white}{p}$};
\node[literal] (S1) at (7.5,5.5)  {$\textcolor{white}{d}\id\textcolor{white}{p}$};
\node[literal] (S1) at (8.5,5.5)  {$\textcolor{white}{d}\grd\textcolor{white}{p}$};
\node[literal] (S1) at (9.5,5.5)  {$\textcolor{white}{d}\com\textcolor{white}{p}$};
\node[literal] (S1) at (10.5,5.5)  {$\textcolor{white}{d}\nav\textcolor{white}{p}$};
\node[literal] (S1) at (11.5,5.5)  {$\textcolor{white}{d}\cfzwei\textcolor{white}{p}$};
\node[literal] (S1) at (12.5,5.5)  {$\textcolor{white}{d}\stgzwei\textcolor{white}{p}$};

\node[literal] (S1) at (1.5,4.5)  {\scriptsize{$ ? $}};
\node[literal] (S1) at (2.5,4.5)  {\scriptsize{[1,3]}};
\node[literal] (S1) at (3.5,4.5)  {\scriptsize{$ ? $}};
\node[literal] (S1) at (4.5,4.5)  {\scriptsize{$ ? $}};
\node[literal] (S1) at (5.5,4.5)  {\scriptsize{[2,1]}};
\node[literal] (S1) at (6.5,4.5)  {\scriptsize{[3,1]}};
\node[literal] (S1) at (7.5,4.5)  {\scriptsize{$ ? $}};
\node[literal] (S1) at (8.5,4.5)  {\scriptsize{$ ? $}};
\node[literal] (S1) at (9.5,4.5)  {\scriptsize{[2,1]}};
\node[literal] (S1) at (10.5,4.5)  {\scriptsize{$ ? $}};
\node[literal] (S1) at (11.5,4.5)  {\scriptsize{$ ? $}};
\node[literal] (S1) at (12.5,4.5)  {\scriptsize{$ ? $}};

\node[literal] (S1) at (1.5,3.5)  {\scriptsize{$ \k^*(\stg)$}};
\node[literal] (S1) at (2.5,3.5)  {\scriptsize{[4,9]}};
\node[literal] (S1) at (3.5,3.5)  {\scriptsize{$ \k(\adm)$}};
\node[literal] (S1) at (4.5,3.5)  {\scriptsize{$ \k(\adm)$}};
\node[literal] (S1) at (5.5,3.5)  {\scriptsize{$ \k(\adm)$}};
\node[literal] (S1) at (6.5,3.5)  {\scriptsize{$ \k(\adm)$}};
\node[literal] (S1) at (7.5,3.5)  {\scriptsize{$ \k(\adm)$}};
\node[literal] (S1) at (8.5,3.5)  {\scriptsize{[4,10]}};
\node[literal] (S1) at (9.5,3.5)  {\scriptsize{[4,11]}};
\node[literal] (S1) at (10.5,3.5)  {\scriptsize{$ \k(\nav)$}};
\node[literal] (S1) at (11.5,3.5)  {\scriptsize{$ ? $}};
\node[literal] (S1) at (12.5,3.5)  {\scriptsize{$ ? $}};

\node[literal] (S1) at (1.5,2.5)  {\scriptsize{$ \k(\stb)$}};
\node[literal] (S1) at (2.5,2.5)  {\scriptsize{$ \k(\stb)$}};
\node[literal] (S1) at (3.5,2.5)  {\scriptsize{$ \k(\adm)$}};
\node[literal] (S1) at (4.5,2.5)  {\scriptsize{$ \k(\adm)$}};
\node[literal] (S1) at (5.5,2.5)  {\scriptsize{$ \k(\adm)$}};
\node[literal] (S1) at (6.5,2.5)  {\scriptsize{$ \k(\adm)$}};
\node[literal] (S1) at (7.5,2.5)  {\scriptsize{$ \k(\adm)$}};
\node[literal] (S1) at (8.5,2.5)  {\scriptsize{$ \k(\grd)$}};
\node[literal] (S1) at (9.5,2.5)  {\scriptsize{$ \k(\com)$}};
\node[literal] (S1) at (10.5,2.5)  {\scriptsize{$\k(\nav)$}};
\node[literal] (S1) at (11.5,2.5)  {\scriptsize{$ \iden $}};
\node[literal] (S1) at (12.5,2.5)  {\scriptsize{$ \iden $}};

\node[literal] (S1) at (1.5,1.5)  {\scriptsize{$ \k(\stb)$}};
\node[literal] (S1) at (2.5,1.5)  {\scriptsize{$ \k(\stb)$}};
\node[literal] (S1) at (3.5,1.5)  {\scriptsize{$ \k(\adm)$}};
\node[literal] (S1) at (4.5,1.5)  {\scriptsize{$ \k(\adm)$}};
\node[literal] (S1) at (5.5,1.5)  {\scriptsize{$ \k(\adm)$}};
\node[literal] (S1) at (6.5,1.5)  {\scriptsize{$ \k(\adm)$}};
\node[literal] (S1) at (7.5,1.5)  {\scriptsize{$ \k(\adm)$}};
\node[literal] (S1) at (8.5,1.5)  {\scriptsize{$ \k(\grd)$}};
\node[literal] (S1) at (9.5,1.5)  {\scriptsize{$ \k(\com)$}};
\node[literal] (S1) at (10.5,1.5)  {\scriptsize{$ \k(\nav)$}};
\node[literal] (S1) at (11.5,1.5)  {\scriptsize{$ \iden $}};
\node[literal] (S1) at (12.5,1.5)  {\scriptsize{$ \iden $}};

\node[literal] (S1) at (1.5,0.5)  {\scriptsize{$ \k(\stb)$}};
\node[literal] (S1) at (2.5,0.5)  {\scriptsize{$ \k(\stb)$}};
\node[literal] (S1) at (3.5,0.5)  {\scriptsize{$ \k(\adm)$}};
\node[literal] (S1) at (4.5,0.5)  {\scriptsize{$ \k(\adm)$}};
\node[literal] (S1) at (5.5,0.5)  {\scriptsize{$ \k^*(\adm)$}};
\node[literal] (S1) at (6.5,0.5)  {\scriptsize{$ \k^*(\adm)$}};
\node[literal] (S1) at (7.5,0.5)  {\scriptsize{$ \k^*(\adm)$}};
\node[literal] (S1) at (8.5,0.5)  {\scriptsize{$ \k^*(\grd)$}};
\node[literal] (S1) at (9.5,0.5)  {\scriptsize{$ \k^*(\com)$}};
\node[literal] (S1) at (10.5,0.5)  {\scriptsize{$ \k(\nav)$}};
\node[literal] (S1) at (11.5,0.5)  {\scriptsize{$ ? $}};
\node[literal] (S1) at (12.5,0.5)  {\scriptsize{$ ? $}};

\node[literal] (S1) at (1.5,-0.5)  {\scriptsize{$ ? $}};
\node[literal] (S1) at (2.5,-0.5)  {\scriptsize{$ \k(\stb)  $}};
\node[literal] (S1) at (3.5,-0.5)  {\scriptsize{$ ? $}};
\node[literal] (S1) at (4.5,-0.5)  {\scriptsize{$ ? $}};
\node[literal] (S1) at (5.5,-0.5)  {\scriptsize{$ \k^*(\adm)  $}\\\scriptsize{$L,Att^{\adm}$}};
\node[literal] (S1) at (6.5,-0.5)  {\scriptsize{$ ? $}};
\node[literal] (S1) at (7.5,-0.5)  {\scriptsize{$ ? $}};
\node[literal] (S1) at (8.5,-0.5)  {\scriptsize{$ \k^*(\grd)  $}\\\scriptsize{$L,Att^{\grd}$}};
\node[literal] (S1) at (9.5,-0.5)  {\scriptsize{$ \k^*(\com)  $}\\\scriptsize{$L,Att^{\com}$}};
\node[literal] (S1) at (10.5,-0.5)  {\scriptsize{$ ? $}};
\node[literal] (S1) at (11.5,-0.5)  {\scriptsize{$ ? $}};
\node[literal] (S1) at (12.5,-0.5)  {\scriptsize{$ ? $}};

\node[literal] (S1) at (1.5,-1.5)  {\scriptsize{$ \iden $}};
\node[literal] (S1) at (2.5,-1.5)  {\scriptsize{$ \iden $}};
\node[literal] (S1) at (3.5,-1.5)  {\scriptsize{$ \iden $}};
\node[literal] (S1) at (4.5,-1.5)  {\scriptsize{$ \iden $}};
\node[literal] (S1) at (5.5,-1.5)  {\scriptsize{$ \iden $}};
\node[literal] (S1) at (6.5,-1.5)  {\scriptsize{$ \iden $}};
\node[literal] (S1) at (7.5,-1.5)  {\scriptsize{$ \iden $}};
\node[literal] (S1) at (8.5,-1.5)  {\scriptsize{$ \iden $}};
\node[literal] (S1) at (9.5,-1.5)  {\scriptsize{$ \iden $}};
\node[literal] (S1) at (10.5,-1.5)  {\scriptsize{$ \iden $}};
\node[literal] (S1) at (11.5,-1.5)  {\scriptsize{$ \iden $}};
\node[literal] (S1) at (12.5,-1.5)  {\scriptsize{$ \iden $}};

\node[literal] (S1) at (1.5,-2.5)  {\scriptsize{$ \iden $}};
\node[literal] (S1) at (2.5,-2.5)  {\scriptsize{$ \iden $}};
\node[literal] (S1) at (3.5,-2.5)  {\scriptsize{$ \iden $}};
\node[literal] (S1) at (4.5,-2.5)  {\scriptsize{$ \iden $}};
\node[literal] (S1) at (5.5,-2.5)  {\scriptsize{$ \iden $}};
\node[literal] (S1) at (6.5,-2.5)  {\scriptsize{$ \iden $}};
\node[literal] (S1) at (7.5,-2.5)  {\scriptsize{$ \iden $}};
\node[literal] (S1) at (8.5,-2.5)  {\scriptsize{$ \iden $}};
\node[literal] (S1) at (9.5,-2.5)  {\scriptsize{$ \iden $}};
\node[literal] (S1) at (10.5,-2.5)  {\scriptsize{$ \iden $}};
\node[literal] (S1) at (11.5,-2.5)  {\scriptsize{$ \iden $}};
\node[literal] (S1) at (12.5,-2.5)  {\scriptsize{$ \iden $}};

\node[literal] (S1) at (1.5,-3.5)  {\scriptsize{$ \iden $}};
\node[literal] (S1) at (2.5,-3.5)  {\scriptsize{$ \iden $}};
\node[literal] (S1) at (3.5,-3.5)  {\scriptsize{$ \iden $}};
\node[literal] (S1) at (4.5,-3.5)  {\scriptsize{$ \iden $}};
\node[literal] (S1) at (5.5,-3.5)  {\scriptsize{$ \iden $}};
\node[literal] (S1) at (6.5,-3.5)  {\scriptsize{$ \iden $}};
\node[literal] (S1) at (7.5,-3.5)  {\scriptsize{$ \iden $}};
\node[literal] (S1) at (8.5,-3.5)  {\scriptsize{$ \iden $}};
\node[literal] (S1) at (9.5,-3.5)  {\scriptsize{$ \iden $}};
\node[literal] (S1) at (10.5,-3.5)  {\scriptsize{$ \iden $}};
\node[literal] (S1) at (11.5,-3.5)  {\scriptsize{$ \iden $}};
\node[literal] (S1) at (12.5,-3.5)  {\scriptsize{$ \iden $}};

\node[literal] (S1) at (0.5,4.5)  {W};
\node[literal] (S1) at (0.5,3.5)  {L};
\node[literal] (S1) at (0.5,2.5)  {E};
\node[literal] (S1) at (0.5,1.5)  {N};
\node[literal] (S1) at (0.5,0.5)  {S};
\node[literal] (S1) at (0.5,-0.5)  {ND};
\node[literal] (S1) at (0.5,-1.5)  {D};
\node[literal] (S1) at (0.5,-2.5)  {LD};
\node[literal] (S1) at (0.5,-3.5)  {U};

\end{tikzpicture}
\end{center}
\caption{\label{fig:extrel}Extension-based Characterizations for Finite AFs}

\end{figure}

Remember that any arbitrary expansion (deletion) can be split into a normal and local
part. So one natural conjecture is that normal and local
expansion (deletion) equivalence jointly imply expansion (deletion) equivalence. Using the results
presented in this section we can not only verify the
addressed conjecture but even give a significantly stronger result. In fact, the
main and quite surprisingly relations for the considered semantics can be briefly
and concisely stated in the following two equations, namely ``normal expansion equivalence = expansion equivalence''  and ``local deletion equivalence = deletion equivalence''. 
\pagebreak
The fact that different notions of equivalence might or might not coincide is interesting from a conceptual point of view. To illustrate this let us have a look 
at normal and strong expansion equivalence. Recall that 
normal expansions add new arguments and possibly new attacks which involve at least one of the fresh arguments, while strong expansions (a subclass of normal expansions) restrict the
possible attacks between the new arguments and the old ones to a single direction. In dynamic settings, both concepts can be justified in the sense that 
new arguments might be raised but this will not influence the relation between already existing 
arguments. For strong expansions, only strong arguments will be raised, i.e.\ arguments which cannot be attacked by existing ones. The corresponding equivalence notions now check whether two AFs are
``equally robust'' to such new arguments, and indeed, normal expansion equivalence always
implies strong expansion equivalence but the other direction is only true for some of the
semantics, namely stage, stable, semi-stable, eager and naive semantics. One interpretation is that when 
two AFs are not normal expansion equivalent, then this can be made explicit by only posing
strong arguments (not attacked by existing ones), while for the other semantics this 
is not the case. For this particular example, it seems that the notion of admissibility
which is more ``explicit'' in the admissible, preferred, ideal, grounded and complete semantics is responsible for
the fact that frameworks might be strong expansion equivalent but not normal expansion 
equivalent. 

In Figure~\ref{fig:prelimrel} we presented preliminary relations between several notions of equivalence
which hold for any semantics. The refinement depicted in Figure~\ref{fig:rel} applies to any extension-based semantics considered in this section. 

\begin{figure}[H]
\centering
\begin{tikzpicture}[scale=1.0]

\node (I) at (-1.6,1.5) [rectangle, thick, draw,text width = 1.45cm, text centered]{\textbf{identity}\\
$=$\\
\textbf{update}\\
$=$\\
\textbf{deletion}\\
$=$\\
\textbf{local}\\
\textbf{deletion}\\
\footnotesize{equivalence}};

\node (B) at (1.2,1.5) [rectangle, thick, draw,text width = 1.65cm, text centered]{\textbf{expansion}\\
$=$\\
\textbf{normal}\\
\textbf{expansion}\\
\footnotesize{equivalence}};

\node (E) at (4.4,-1.2) [rectangle, thick, draw,text width = 1.65cm, text centered]{\textbf{normal}\\
\textbf{deletion}\\
\footnotesize{equivalence}};
\node (F) at (4.4,0.6) [rectangle, thick, draw,text width = 1.65cm, text centered]{\textbf{local}\\
\textbf{expansion}\\
\footnotesize{equivalence}};

\node (H) at (4.4,4.4) [rectangle, thick, draw,text width = 1.65cm, text centered]{\textbf{strong}\\
\textbf{expansion}\\
\footnotesize{equivalence}};
\node (W) at (4.4,2.6) [rectangle, thick, draw,text width = 1.65cm, text centered]{\textbf{weak}\\
\textbf{expansion}\\
\footnotesize{equivalence}};
\node (O) at (7.8,1.5) [rectangle, thick, draw,text width = 1.45cm, text centered]{\textbf{ordinary}\\
\footnotesize{equivalence}};

\draw[->, thick, double] (I) to [thick,double] (B);

\draw[->, thick, double] (B) to [thick,double, bend right] (F);
\draw[->, thick, double] (B) to [thick,double, bend right] (E);
\draw[->, thick, double] (B) to [thick,double, bend left] (H);
\draw[->, thick, double] (B) to [thick,double, bend left] (W);

\draw[->, thick, double] (E) to [thick, bend right,double] (O);
\draw[->, thick, double] (F) to [thick, bend right,double] (O);
\draw[->, thick, double] (H) to [thick, bend left,double] (O);
\draw[->, thick, double] (W) to [thick, bend left,double] (O);

\end{tikzpicture}
\caption{\label{fig:rel}Relations for $\sigma\in\{\stg,\stb,\semi,\eag,\adm,\prf,\id,\grd,\com,\nav,\cfzwei,\stgzwei\}$ - Extension-based Versions and Finite AFs}
\end{figure}

Finally, we present the overall picture for the most prominent semantics, namely the stable one. Interestingly, in contrast to Figure~\ref{fig:rel} all equivalence notions are comparable, i.e.\ they are totally ordered w.r.t.\ $\subseteq$. Comprehensive overviews for single semantics can be found in \cite[Section 5.5.2]{diss} or \cite{zoo2}. The latter also contains a comparison to different notions of \textit{minimal change equivalence} which are related to the so-called \textit{minimal change problem} considered \cite{minimal,BauB13}. As an aside, very recently the authors of \cite{BauDLS17} introduced so-called \textit{$C$-relativized equivalence} that subsumes ordinary and expansions equivalence as its extreme corner cases. The set $C$ represents so-called \textit{core} arguments
which will not be directly touched by the possible
expansions. This means, for any set $C$ we obtain a further intermediate notion between expansion and ordinary equivalence. However, due to its recency further relations are not studied so far.

\begin{figure}[H]
\centering
\begin{tikzpicture}[scale=1.0]

\node (I) at (-1.6,1.5) [rectangle, thick, draw,text width = 1.4cm, text centered]{\textbf{identity}\\
$=$\\
\textbf{update}\\
$=$\\
\textbf{deletion}\\
$=$\\
\textbf{local}\\
\textbf{deletion}\\
\footnotesize{equivalence}};

\node (B) at (1,1.5) [rectangle, thick, draw,text width = 1.65cm, text centered]{\textbf{expansion}\\
$=$\\
\textbf{normal}\\
\textbf{expansion}\\
$=$\\
\textbf{strong}\\
\textbf{expansion}\\
$=$\\
\textbf{normal}\\
\textbf{deletion}\\
\footnotesize{equivalence}};

\node (F) at (3.6,1.5) [rectangle, thick, draw,text width = 1.65cm, text centered]{\textbf{local}\\
\textbf{expansion}\\
\footnotesize{equivalence}};

\node (W) at (6.2,1.5) [rectangle, thick, draw,text width = 1.65cm, text centered]{\textbf{weak}\\
\textbf{expansion}\\
\footnotesize{equivalence}};
\node (O) at (8.8,1.5) [rectangle, thick, draw,text width = 1.43cm, text centered]{\textbf{ordinary}\\
\footnotesize{equivalence}};

\draw[->, thick, double] (I) to [thick,double] (B);

\draw[->, thick, double] (B) to [thick,double] (F);
\draw[->, thick, double] (F) to [thick,double] (W);

\draw[->, thick, double] (W) to [thick,double] (O);

\end{tikzpicture}
\caption{\label{fig:relstb}Stable Semantics - Extension-based Version and Finite AFs}
\end{figure}
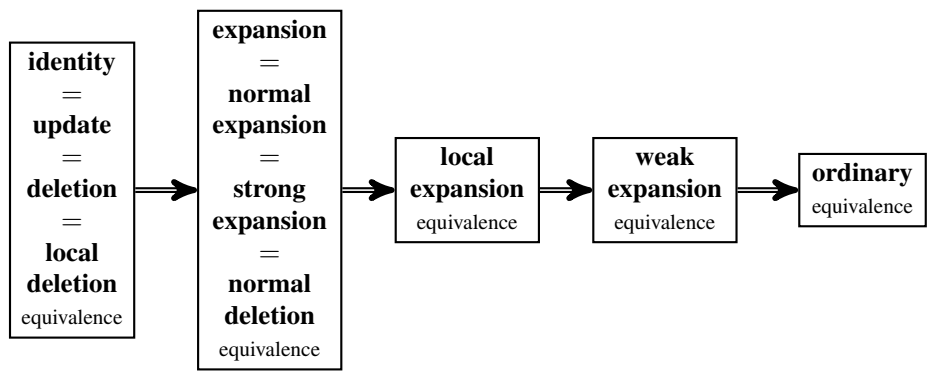

\section{Equivalence in the Light of Unrestricted Frameworks}
\label{chap:equiunrestr}
Recently, a first study of several abstract properties in the unrestricted setting were presented in \cite{BauS17}. The main result regarding expansion equivalence can be summarized as follows: All characterization results carry over to the unrestricted setting as long as the AFs in question are \textit{jointly expandable} (w.r.t.\ $\m{U}$). Consider therefore the following definition and the corresponding characterization theorem.

\begin{definition} $\F$ and $\G$ are jointly expandable if $\m{U}\sm (A(\F) \cup A(\G)) \neq\emptyset$. \label{def:jointexp}
\end{definition}

\begin{theorem}{\cite{BauS17}} \label{the:expequi_joinexp} For jointly expandable AFs $\F$ and $\G$ we have:
\begin{enumerate}
	\item $\AF\equiv^{\Ext_\sigma}_E\! \AG \ToT \AF^{\k(\sigma)} = \AG^{\k(\sigma)}$ for any $\sigma\in\{\stb,\adm,\com,\grd,\nav\}$,
	\item  $\AF\equiv^{\Ext_\sigma}_E\! \AG \ToT \AF^{\k(\adm)} = \AG^{\k(\adm)}$ for any $\sigma\in\{\prf,\id,\semi,\eag\}$ and
	\item $\AF\equiv^{\Ext_\stg}_E\! \AG \ToT \AF^{\k(\stb)} = \AG^{\k(\stb)}$.
\end{enumerate}
\end{theorem}

The main proof strategies are straightforward extensions of those presented in \cite{strong}. However, finiteness assumptions are often used implicitly and one has to pay attention whether a certain reasoning step (e.g.\ subset relation between semantics, definedness statuses of semantics, finitely many extensions etc.) carry over to the infinite setting.

Interestingly, in case of the admissible as well as naive kernel we may even drop the restriction of joint expandability as stated in the following theorem.

\begin{theorem}{\cite{BauS17}} \label{the:expequi_unrestricted} For unrestricted AFs $F$ and~$G$ we have:
\begin{enumerate}
	\item $\AF\equiv^{\Ext_\nav}_E\! \AG \ToT \AF^{\k(\nav)} = \AG^{\k(\nav)}$ and
	\item  $\AF\equiv^{\Ext_\sigma}_E\! \AG \ToT \AF^{\k(\adm)} = \AG^{\k(\adm)}$ for any $\sigma\in\{\adm,\prf,\id,\semi,\eag\}$.
	\end{enumerate}
\end{theorem}

The following two examples taken from \cite{BauS17} show that this assertion does not hold for all kernels considered in this section. The main reason for this different behaviour is that for some semantics it plays a decisive role whether AFs can be expanded by  ``fresh'' arguments which is not given for unrestricted frameworks in general but guaranteed for jointly expandable AFs .

\begin{example} \label{ex:non-reg} Given $c\!\in\!\m{U}$ and define the following two AFs, $\F = (\m{U}\sm\!\{c\},$ $\{(a,a)\mid a\in\m{U}\sm\{c\}\})$ and $\G = (\m{U}, \{(a,a)\mid a\in\m{U}\sm\{c\}\})$. For any $\AH$ we observe $\Ext_\stb(\F\dcup\AH) = \Ext_\stb(\G\dcup\AH)$. In particular, 

\mbox{$\Ext_\stb(\F\dcup\AH) =      
\begin{cases}
    \{\{c\}\}, & \text{if}\ \{(c,a)\mid a\in \m{U}\sm\{c\}\}\subseteq R(\AH) \text{and}\ (c,c)\notin R(\AH)\\
   \emptyset, & \text{otherwise} 
    \end{cases}  $}
		
\hspace{-1em} Consequently, $\F\!\equiv^{\Ext_\stb}_E\!\G$ although $A(\F) \neq A(\G)$ (and thus, $\F^{\k(\stb)}\neq\G^{\k(\stb)}$). 
\end{example}

\begin{example} Consider the AFs $\F\! =\! (\m{U},\! \{(a,a) \mid a\in\m{U}\})$ and $\G = (\m{U}, \{(a,b)\mid a,b\in\m{U}, a\neq b)$.  Applying the grounded kernel does not change anything for either framework, i.e. $\F^{\k(\grd)}\! =\! \F$ and $\G\! =\! \G^{\k(\grd)}$. Due to the absence of unattacked arguments we deduce $\Ext_\grd(\F\dcup\AH) = \Ext_\grd(\G\dcup\AH) = \{\emptyset\}$ for any AF $\AH$. Consequently, $\F\! \equiv^{\Ext_\grd}_E\! \G$ although	$\F^{\k(\grd)} \neq \G^{\k(\grd)}$.
\end{example}

\pagebreak
\section{Characterization Theorems for Labelling-Based Semantics} \label{sec:equivlabel}

We now return to the finite setting and consider the second main approach used for evaluating argumentation scenarios, namely labelling-based semantics. As a matter of fact, the labelling-based versions of all considered semantics provides one with more information than their extension-based counter-parts. More precisely, the defined 3-valued labellings assign a status to any argument of the considered AF $\F$, i.e.\ in addition to the information which arguments are \textit{accepted} we also have labels for the remaining arguments indicating that they are either \textit{rejected} or \textit{undecided} with respect to $\F$. 
It is
well known that many semantics establish a one-to-one
correspondence between their extension-based and
labelling-based versions (for more details see the subsequent Section~\ref{sec:baspro}). This means, any labelling is associated
with exactly one extension and vice versa. It is not immediately apparent whether this property guarantees
that there is a coincidence of the extension-based and labelling-based equivalence notions. In \cite{Bau16} a negative answer was given. The main reason for the invalidity is that
AFs may possess the same extensions without sharing the
same arguments which is impossible in case of labellings since any argument has to be labelled.
Furthermore, even sharing the same arguments does not ensure the validity of the converse direction. Consider therefore the following example.

\begin{example} \label{ex:lab} Consider the AFs $\F$ and $\G$ as depicted below. Although both frameworks possess the same unique preferred extension, they do not share the same preferred labellings. More precisely, $\Ext_{\prf}(\F) = \Ext_{\prf}(\G) = \{\{a\}\}$ but $\{(\{a\},\{b\},\emptyset)\} = \Lab_{\prf}(\F) \neq \Lab_{\prf}(\G) = \{(\{a\},\emptyset,\{b\})\}$.

\begin{tikzpicture}

		\node (A5) at (0,-2.5) [circle, minimum size=0.7cm, thick, draw, label = left:$\F:$] {$a$};
    \node (B5) at (1.2,-2.5) [circle, minimum size=0.7cm,  thick, draw] {$b$};

		\node (A1) at (3.6,-2.5) [circle, minimum size=0.7cm, thick, draw, label = left:$\G:$] {$a$};
    \node (B1) at (4.8,-2.5) [circle, minimum size=0.7cm,  thick, draw] {$b$};


\draw[->, thick] (A5) to [thick,bend right] (B5);

\draw[->, thick] (B1) to [thick,loop,distance=0.5cm] (B1);

\draw[->, thick] (B5) to [thick,loop,distance=0.5cm] (B5);

\end{tikzpicture}

Moreover, observe that $\F^{\k^*(\adm)} = \G = \G^{\k^*(\adm)}$. Consequently, both frameworks are even strong expansion equivalent w.r.t.\ preferred extension-based semantics (Theorem~\ref{the:strong2}). This means, equivalence notions may differ considerably if considered under the extension-based or labelling-based approach. 

\end{example}

In contrast to extension-based semantics where characterization results are spread over a high number of publications there is only one reference, namely \cite{Bau16} concerned with labelling-based semantics. The author considered 8 different equivalence notions w.r.t.\ 8 prominent labelling-based semantics in the finite setting. In effect, similarly to extension-based semantics, almost all labelling-based equivalence notions can be decided
syntactically. Differently from the extension-based approach we observe a much more homogeneous picture. For instance, there is no need for the more sophisticated $\sigma$-*-kernels as we will see.

\subsection{Basic Properties and a Fundamental Relation} \label{sec:baspro}
Before turning to the main results we start with some preliminary facts relating $\sigma$-extensions and $\sigma$-labellings. In the following we restrict ourselves to the semantics considered in \cite{Bau16}. For any 3-valued labelling $\L = (\L_1,\L_2,\L_3)$ we use $\L = (\LI,\LO,\LU)$ as usual.

\begin{fact} \label{fact1} Given a finite AF $\AF = (A,R)$ and $E\subseteq A$. We write $E^{\Lab}$ for $(E,E^+,A\sm E^\oplus))$. For all $\sigma\in\{\stb,\semi,\eag,\adm,\prf,\id,\grd,\com\}$ we have,
\begin{enumerate}
  \item If $\L\in\Lab_{\sigma}(\AF)$, then $\LI\in\Ext_{\sigma}(\AF)$, \hfill{(extension induced by labelling)} 
	\item If $\E\in\Ext_{\sigma}(\AF)$, then $\E^{\Lab}\in\Lab_{\sigma}(\AF)$ and \hfill{(labelling induced by extension)} 
	\item Obviously, $(E^{\Lab})^I = E$. \hfill{($I\circ \Lab = \iden$)} 
	\end{enumerate}
\end{fact}

We point out that the first two properties mentioned in Fact~\ref{fact1} do not ensure that there is a one-to-one correspondence between $\sigma$-labellings and $\sigma$-extensions. This desirable feature (which would indeed justify the terms $\sigma$-labellings and $\sigma$-extensions) is given if additionally, labellings are uniquely determined by their in-labelled arguments.  

\begin{fact} \label{fact2} Given a finite AF $\AF = (A,R)$ and a set $E\subseteq A$. For all semantics $\sigma\in\{\stb,\semi,\eag,\prf,\id,\grd,\com\}$ we have,
\begin{enumerate}
  \item For any $\L,\M\in\Lab_{\sigma}(\AF), \LI = \MI$ iff $\L = \M$, \hfill{(uniquely determined by in-labels)}
	\item Given $\L\in\Lab_{\sigma}(\AF)$, then $(\LI)^{\Lab} = \L$ and \hfill{($\Lab\circ I = \iden$)} 
	\item $\card{\Lab_{\sigma}(\AF)} = \card{\Ext_{\sigma}(\AF)}$. \hfill{(same cardinality)} 
\end{enumerate}
\end{fact}

As an aside, we mention that (although not immediately apparent) the first two items of Fact~\ref{fact2} are equivalent independently of any semantics definition. Please note that admissible labellings are excluded from Fact~\ref{fact2}. The AF $\F$ depicted in Example~\ref{ex:lab} shows that this is no coincidence. It possesses two admissible labellings associated with one admissible extension. More precisely, the admissible labellings $(\{a\},\{b\},\emptyset)$ as well as $(\{a\},\emptyset,\{b\})$ refer to the same admissible extension $\{a\}$. 
		
We proceed with a general relation between labelling-based and extension-based versions of certain equivalence notion. More precisely, for any considered semantics and any equivalence notion presented in Definition~\ref{def:equivalence} we have that being equivalent w.r.t.\ labellings implies being equivalent w.r.t.\ extensions. The main reason for this fundamental relation is the following lemma stating that possessing the same labellings implies sharing the same extensions. We mention that this property is already guaranteed if the semantics $\sigma$ in question satisfies that any $\sigma$-extension induces an $\sigma$-labelling and vice versa (cf.\ statements 1 and 2 of Fact~\ref{fact1}).

\begin{lemma}[\cite{Bau16}] \label{pro:labext} Given two finite AFs $\AF$,$\AG$. For any $\sigma\in\{\stb,\semi,\eag,\adm,\prf,\id,\grd,\com\}$ we have,
$$\Lab_{\sigma}(\AF) = \Lab_{\sigma}(\AG) \To \Ext_{\sigma}(\AF) = \Ext_{\sigma}(\AG).$$
\end{lemma}
\begin{proof} Reductio ad absurdum. Assume $\Ext_{\sigma}(\AF) \neq \Ext_{\sigma}(\AG)$. Then, w.l.o.g. exists $\E\in\Ext_{\sigma}(\AF) \sm\Ext_{\sigma}(\AG)$. Consequently, $\E^{\Lab}\in\Lab_{\sigma}(\AF)$ (item 2 of Fact~\ref{fact1}). Thus, $\E^{\Lab}\in\Lab_{\sigma}(\AG)$ (assumption). Hence, $(\E^{\Lab})^I\in\Ext_{\sigma}(\AG)$ (Item 1 of Fact~\ref{fact1}). Furthermore, $(\E^{\Lab})^I = E\in\Ext_{\sigma}(\AG)$ (item 3 of Fact~\ref{fact1}). Contradiction! 
\end{proof}

We now present the fundamental relation between labelling-based and extension-based equivalence notion. 

\begin{theorem}[\cite{Bau16}] \label{the:mainrelation} Given two finite AFs $\AF$ and $\AG$. For any $\sigma\in\{\stb,\semi,\eag,\adm,\prf,\id,\grd,\com\}$ and any $M\in\{W,L,E,N,S,ND,D,LD,U\}$ we have,

$$\AF\equiv_{M}^{\Lab_{\sigma}}\!\AG \To \AF\equiv_{M}^{\Ext_{\sigma}}\!\AG.$$

\end{theorem}
\begin{proof} We show the contrapositive. Assume $\AF\not\equiv_{M}^{\Ext_{\sigma}}\!\AG$. This means, there is a certain scenario $S$ according to $M$, s.t.\ $\Ext_{\sigma}(S(\AF)) \neq \Ext_{\sigma}(S(\AG))$.\footnote{For instance, in case of expansion equivalence (i.e.\ M = E) a scenario $S$ is simply the union with a further AF $\AH$, i.e.\ $S(\AF) = \AF\dcup \AH$ and $S(\AG) = \AG\dcup \AH$.} Consequently, $\Lab_{\sigma}(S(\AF)) \neq \Lab_{\sigma}(S(\AG))$ (Lemma~\ref{pro:labext}) proving $\AF\not\equiv_{M}^{\Lab_{\sigma}}\!\AG$. 
\end{proof}

In Example~\ref{ex:lab} we have seen that the converse direction does not hold in general. Nevertheless, there is huge number of equivalence notions where labelling-based and extension-based versions do indeed coincide (cf.\ Figure~\ref{fig:labelrel} for an overview).


\subsection{Coincidences of Extension-based and \mbox{Labelling-based} Versions}

Remember that the identity relation is the finest equivalence relation. Furthermore, it is already shown that deletion, local deletion as well as update equivalence w.r.t.\ $\Ext_{\sigma}$ collapse to identity (see Figure~\ref{fig:rel}). Consequently, applying the fundamental relation stated in Theorem~\ref{the:mainrelation} we obtain the identical characterization results w.r.t.\ labelling-based semantics.

\begin{theorem}[\cite{Bau16}] \label{the:del} For finite AFs $\AF$ and $\AG$, a scenario $M\in\{D,LD,U\}$ and a semantics $\sigma\in\{\stb,\semi,\eag,\adm,\prf,\id,\grd,\com\}$ we have,
 $$\AF\equiv_{M}^{\Lab_{\sigma}}\AG \iff \AF=\AG.$$
\end{theorem}

Analogously to extension-based semantics (cf.\ Fact~\ref{fact:kerneldecisive}) we have that there are combinations of kernels and semantics $\sigma$, s.t.\ the application of a kernel does not vary the set of $\sigma$-labellings.

\begin{fact} \label{lem:somekernel} For any finite AF $\AF$, 

\begin{enumerate}
	\item $\Lab_{\sigma}(\AF) = \Lab_{\sigma}\left(\AF^{\k(\sigma)}\right)$ for $\sigma\in\{\com,\stb,\grd\}$ and
	\item $\Lab_{\tau}(\AF) = \Lab_{\tau}\left(\AF^{\k(\adm)}\right)$ for $\tau\in\{\semi,\eag,\prf,\id\}$.
\end{enumerate}
 
\end{fact}

The fact above is the decisive property which allows one to carry over further kernel-based characterization results for extension-based semantics to their labelling-based version. In order to show this result it was necessary to find a condition for equality of two complete labellings of different AFs. Remember that two complete labellings of the same framework are identical if and only if they possess the same in-labelled arguments (Fact~\ref{fact2}). In case of different AFs we have to require additionally that both frameworks share the same arguments and the same range w.r.t.\ the set of in-labelled arguments.

\begin{fact} \label{pro:equalitylabel} Given two finite  AFs $\AF$ and $\AG$ as well as $\L\in\Lab_{\com}(\AF)$ and \linebreak $\M\in\Lab_{\com}(\AG)$. We have $\L = \M$ iff simultaneously $A(\AF) = A(\AG)$, $\LI = \MI$ and $R^+_{\AF}(\LI) = R^+_{\AG}(\MI)$. 
\end{fact}

Please observe that admissible labellings do not fulfill Fact~\ref{pro:equalitylabel}. Consider for instance again the AF $\F$ depicted in Example~\ref{ex:lab} and its two admissible labellings $(\{a\},\{b\},\emptyset)$ and $(\{a\},\emptyset,\{b\})$.


We proceed with the main coincidence theorem. It stipulates that several expansion equivalence relations as well as weaker notions do not distinguish between their labelling-based and extension-based version. This means, kernel-based characterization results (depicted in Figure~\ref{fig:prelimrel}) carry over to labelling-based semantics. Similarly to extension-based semantics we present an overview of characterizing kernels at the end of this section (cf.\ Figure~\ref{fig:labelrel}).

\begin{theorem}[\cite{Bau16}] \label{the:coincidence} Given finite AFs $\AF$ and $\AG$. We have, 
\begin{enumerate}
	\item $\AF\equiv_{M}^{\Ext_{\sigma}}\AG \iff \AF\equiv_{M}^{\Lab_{\sigma}}\AG$ for $\sigma\in\{\stb,\semi,\eag,\prf,\id,\grd,\com\}, M\in\{E,N\}$,
	\item $\AF\equiv_{L}^{\Ext_{\sigma}}\AG \iff \AF\equiv_{L}^{\Lab_{\sigma}}\AG$ for $\sigma\in\{\semi,\eag,\prf,\id\}$ and 
	\item $\AF\equiv_{S}^{\Ext_{\sigma}}\AG \iff \AF\equiv_{S}^{\Lab_{\sigma}}\AG$ for $\sigma\in\{\stb,\semi,\eag\}$.
\end{enumerate}
\end{theorem}

\subsection{\mbox{Non-Coincidence of Extension-based and Labelling-based Versions}}
We now leave the realm of uniformity of extension-based and labelling-based characterizations. This section is divided into three parts. We start with characterization theorems for admissible labellings. In particular, we will see that the admissible kernel (originally introduced to characterize equivalence notions w.r.t.\ admissible extension-based semantics) does not serve as characterizing kernel for admissible labellings. We then proceed with strong expansion equivalence w.r.t.\ labellings. We will see that the remaining notions are characterizable via traditional kernels instead of $\sigma$-*-kernels. In the third part we consider normal deletion equivalence w.r.t.\ labelling-based semantics. In contrast to their extension-based versions where many notions has defied any attempt of solving, we present characterization theorems based on traditional kernels for all eight considered semantics.

\subsubsection{Expansion Equivalence w.r.t.\ Admissible Labellings}

Expansion equivalence as well as its local, normal and strong versions w.r.t.\ admissible extensions are characterizable through the admissible kernel. The following example shows that this assertion does not hold in case of admissible labellings.

\begin{example}	\label{ex:admlab}	The following two AFs possess the same admissible kernels, namely $\AF^{\k(\adm)} = \AG^{\k(\adm)} = \AF$. Consequently, applying characterization theorems for extension-based semantics we obtain $\AF\equiv^{\Ext_{\adm}}_M\AG$ for $M\in\{L,E,N\}$ (cf.\ Figure~\ref{fig:extrel}).  

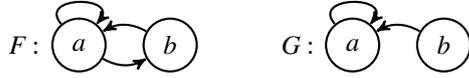
\begin{tikzpicture}
    
		\node (A1'') at (0,-2.5) [circle,minimum size=0.7cm, thick, draw, label = left:$\AF:$]{$a$};
    \node (B1'') at (1.2,-2.5) [circle, minimum size=0.7cm, thick, draw]{$b$};
		\node (A2'') at (3.6,-2.5) [circle, minimum size=0.7cm, thick, draw, label = left:$\AG:$] {$a$};
    \node (B2'') at (4.8,-2.5) [circle, minimum size=0.7cm,  thick, draw] {$b$};

\draw[->, thick] (A1'') to [thick,bend right] (B1'');
\draw[->, thick] (B1'') to [thick,bend right] (A1'');
\draw[->, thick] (B2'') to [thick,bend right] (A2'');
\draw[->, thick] (A1'') to [thick, loop, distance=5mm,in=50, out=130,looseness=5] (A1'');
\draw[->, thick] (A2'') to [thick, loop, distance=5mm,in=50, out=130,looseness=5] (A2'');

\end{tikzpicture} 

Observe that $(\{b\},\emptyset,\{a\})\in\Lab_{\adm}(\AG)\sm\Lab_{\adm}(\AF)$ because the argument $a$ cannot be undecided in $\AF$ since it attacks the in-labelled argument $b$. Thus $\AF\not\equiv^{\Lab_{\adm}}_M\AG$ for $M\in\{L,E,N,S\}$.
\end{example}

Let us assume that the equivalence notions considered in the example above are characterizable through a certain kernel $\k$. Due to the fundamental relation (Theorem~\ref{the:mainrelation}) and the characterization results w.r.t.\ admissible extensions (Figure~\ref{fig:extrel}), we already know that the kernel $\k$ has to satisfy the following implication: \mbox{$\F^\k = \G^\k \To \F^{\k(\adm)} = \G^{\k(\adm)}$} for any two AFs $\F$ and $\G$. This means, we are looking for a weaker kernel than the admissible one in the sense that first, everything which is redundant w.r.t.~$\k$ has to be redundant w.r.t.\ the admissible kernel too; and second, an attack from $a$ to $b$ has to survive even if $a$ is self-defeating and $b$ counterattacks $a$. One candidate for $\k$ is the complete kernel since redundancy w.r.t.\ the complete kernel implies redundancy w.r.t.\ to the admissible one, and furthermore, it deletes an attack between two arguments if and only if both are self-defeating. And indeed, it was shown that expansion equivalence as well as its local, normal and strong variant w.r.t.\ admissible labellings are characterizable through the complete kernel as stated by the following theorem. 


\begin{theorem}[\cite{Bau16}] \label{the:adcomplete} Given finite AFs $\AF$ and $\AG$. We have,
$$\AF\equiv_{M}^{\Lab_{\adm}}\AG \iff \AF^{\k(\com)} = \AG^{\k(\com)} \text{ with } M\in\{L,E,N,S\}.$$
\end{theorem}

\subsubsection{Strong Expansion Equivalence for Preferred, Ideal, Grounded and Complete Labellings}
In this subsection we present characterization theorems for strong expansion equivalence w.r.t.\ labelling-based preferred, ideal, grounded and complete semantics. Remember that in case of strong expansions a former attack between old arguments will never become a counterattack to an added attack. Consequently, in contrast to arbitrary expansions former attacks do not play a role with respect to being a potential defender of an added argument. The context-sensitive $\sigma$-*-kernels took these considerations into account and allow for more deletions than their classical counterparts. 

\begin{example} According to Definition~\ref{def:kernel2} we have, $\AF^{\k^*(\sigma)} = \AG^{\k^*(\sigma)}$ for any semantics $\sigma\in\{\adm,\grd,\com\}$. More precisely, the attacks $(a,b)$ in $\F$ as well as $(c,b)$ in $\G$ are redundant w.r.t.\ all three $\sigma$-*-kernels. This means, in consideration of Figure~\ref{fig:extrel} both frameworks are strong expansion equivalent w.r.t.\ the extension-based versions of preferred, ideal, grounded and complete semantics.

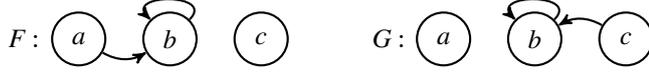
\begin{tikzpicture}

		\node (A5) at (0,-2.5) [circle, minimum size=0.7cm, thick, draw, label = left:$\F:$] {$a$};
    \node (B5) at (1.2,-2.5) [circle, minimum size=0.7cm,  thick, draw] {$b$};
		\node (C5) at (2.4,-2.5) [circle, minimum size=0.7cm,  thick, draw] {$c$};
		
		\node (A1) at (4.8,-2.5) [circle, minimum size=0.7cm, thick, draw, label = left:$\G:$] {$a$};
    \node (B1) at (6,-2.5) [circle, minimum size=0.7cm,  thick, draw] {$b$};
		\node (C1) at (7.2,-2.5) [circle, minimum size=0.7cm,  thick, draw] {$c$};


\draw[->, thick] (A5) to [thick,bend right] (B5);
\draw[->, thick] (C1) to [thick,bend right] (B1);

\draw[->, thick] (B1) to [thick,loop,distance=0.5cm] (B1);

\draw[->, thick] (B5) to [thick,loop,distance=0.5cm] (B5);

\end{tikzpicture}

Consider the following dynamic scenario where a stronger argument than the former ones is added. Formally, we conjoin the AF $\H = (\{c,d\},\{(d,c)\})$ to both frameworks $\F$ and $\G$.

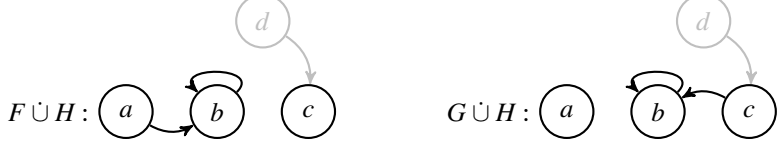
\begin{tikzpicture}

		\node (A5) at (0,-2.5) [circle, minimum size=0.7cm, thick, draw, label = left:$\F\dcup\H:$] {$a$};
    \node (B5) at (1.2,-2.5) [circle, minimum size=0.7cm,  thick, draw] {$b$};
		\node (C5) at (2.4,-2.5) [circle, minimum size=0.7cm,  thick, draw] {$c$};
		\node (D5) at (1.8,-1.3) [circle, minimum size=0.7cm,  thick, gray!50, draw] {$d$};
		
		\node (A1) at (5.8,-2.5) [circle, minimum size=0.7cm, thick, draw, label = left:$\G\dcup\H:$] {$a$};
    \node (B1) at (7,-2.5) [circle, minimum size=0.7cm,  thick, draw] {$b$};
		\node (C1) at (8.2,-2.5) [circle, minimum size=0.7cm,  thick, draw] {$c$};
		\node (D1) at (7.6,-1.3) [circle, minimum size=0.7cm,  thick, gray!50, draw] {$d$};


\draw[->, thick] (A5) to [thick,bend right] (B5);
\draw[->, thick] (C1) to [thick,bend right] (B1);

\draw[->, thick, gray!50] (D5) to [thick,bend left] (C5);
\draw[->, thick, gray!50] (D1) to [thick,bend left] (C1);

\draw[->, thick] (B1) to [thick,loop,distance=0.5cm] (B1);

\draw[->, thick] (B5) to [thick,loop,distance=0.5cm] (B5);

\end{tikzpicture}

Note that both frameworks has to possess the same $\sigma$-extension since \mbox{$\AG\equiv^{\Ext_{\sigma}}_S\!\AH$} for $\sigma\in\{\prf,\id,\grd,\com\}$ is already ensured. Furthermore, we observe $(\{a,d\},\{b,c\},\emptyset)\in\Lab_{\sigma}(\AF\dcup\H)\sm\Lab_{\sigma}(\AG\dcup\H)$ since $b$ cannot be out-labelled in $\AG\dcup\AH$ because there is no in-labelled attacker. This means, $\AF\not\equiv^{\Lab_{\sigma}}_S\AG$ for $\sigma\in\{\prf,\id,\grd,\com\}$. 
 
\end{example}

Analogously to the previous section let us assume that strong expansion equivalence w.r.t.\ the considered labelling-based semantics are characterizable through a certain kernel $\k$. We immediately obtain, \mbox{$\F^\k\!=\!\G^\k\To\F^{\k^*(\sigma)}\!=\!\G^{\k^*(\sigma)}$} for any two AFs $\F$ and $\G$. Possible candidates are the classical counterparts of the $\sigma$-*-kernels and indeed it was shown that these kernels guarantee the desired outcome. This means, in case of strong expansion equivalence w.r.t.\ preferred, ideal, grounded and complete semantics we have that the labelling-based version is characterizable through a classical $\sigma$-kernel if and only if the extension-based version is characterizable through the corresponding $\sigma$-*-kernel. 



\begin{theorem}[\cite{Bau16}] \label{the:strongexpansion} Given finite AFs $\AF$ and $\AG$. We have,
\begin{enumerate}
	\item $\AF\equiv_{S}^{\Lab_{\sigma}}\AG \ToT \AF^{\k(\adm)} = \AG^{\k(\adm)}$ for $\sigma\in\{\prf,\id\}$,
	\item $\AF\equiv_{S}^{\Lab_{\grd}}\AG \ToT \AF^{\k(\grd)} = \AG^{\k(\grd)}$ and
	\item $\AF\equiv_{S}^{\Lab_{\com}}\AG \ToT \AF^{\k(\com)} = \AG^{\k(\com)}$.
	\end{enumerate}
\end{theorem}

\subsubsection{Normal Deletion Equivalence}

Characterizing normal deletion equivalence in case of extension-based semantics is exceptional in several regards. Remember that normal deletions retract arguments and their corresponding attacks. Firstly, only a few characterization
results are achieved (cf.\ Figure~\ref{fig:extrel}). Furthermore, apart from stable semantics, none of the characterization results is purely kernel-based, i.e.\ beside the equality of kernels on certain parts of the frameworks further loop- as well as attack-conditions have to be satisfied. Finally, quite surprisingly, normal deletion equivalent AFs do not even have to share the same arguments enabling equivalence classes with an infinite number of elements. Being equivalent w.r.t.\ labellings and possessing different arguments at the same time is impossible in case of labellings since any argument has to be labelled. It turned out that any considered labelling-based semantics is characterizable through traditional kernels and thus, do not share any of the features mentioned above. Consider the following main theorem.

\begin{theorem}[\cite{Bau16}] \label{the:normaldeletion} Given finite AFs $\AF$ and $\AG$. We have,
\begin{enumerate}
	\item $\AF\equiv_{ND}^{\Lab_{\stb}}\AG \ToT \AF^{\k(\stb)} = \AG^{\k(\stb)}$,
	\item \mbox{$\AF\equiv_{ND}^{\Lab_{\sigma}}\!\AG \ToT \AF^{\k(\adm)} = \AG^{\k(\adm)}$  for $\sigma\in\{\semi,\eag,\prf,\id\}$,}
	\item $\AF\equiv_{ND}^{\Lab_{\sigma}}\!\AG \ToT \AF^{\k(\com)} = \AG^{\k(\com)}$ for $\sigma\in\{\adm,\com\}$ and
	\item $\AF\equiv_{ND}^{\Lab_{\grd}}\!\AG \ToT \AF^{\k(\grd)} = \AG^{\k(\grd)}$.
	\end{enumerate}
\end{theorem}

\subsection{Summary of Results and Conclusion}

The following Figure~\ref{fig:labelrel} presents a comprehensive overview of the state of the art in case of labelling-based semantics. Analogously to Figure~\ref{fig:extrel} the entry ``$\k$'' in row $M$ and column $\sigma$ indicates that $\equiv_M^{\Lab_{\sigma}}$ is characterizable through $\k$ given the finiteness restriction. The abbreviation ``$\iden$'' stands for identity map and the question mark represents an open problem.\footnote{In contrast to extension-based semantics the labelling versions of conflict-free-based semantics like stage, naive, cf2 as well as stage2 semantics (cf.\ \cite{Cam11,GagD14}) as well as weak expansion equivalence at all were not considered so far and thus, represent open problems too.} A grey-highlighted entry reflects the situation that extension-based and labelling-based version do not coincide. 

\begin{figure}[H]
\center
\begin{tikzpicture}
\tikzstyle{literal}=[text centered, text width=15mm, minimum height=5mm]
\draw[help lines] (0,-4) grid (9,5) ;

\node[literal] (S1) at (1.5,4.5)  {$\textcolor{white}{d}\stb\textcolor{white}{p}$};
\node[literal] (S1) at (2.5,4.5)  {$\textcolor{white}{d}\semi\textcolor{white}{p}$};
\node[literal] (S1) at (3.5,4.5)  {$\textcolor{white}{d}\eag\textcolor{white}{p}$};
\node[literal] (S1) at (4.5,4.5)  {$\textcolor{white}{d}\adm\textcolor{white}{p}$};
\node[literal] (S1) at (5.5,4.5)  {$\textcolor{white}{d}\prf\textcolor{white}{p}$};
\node[literal] (S1) at (6.5,4.5)  {$\textcolor{white}{d}\id\textcolor{white}{p}$};
\node[literal] (S1) at (7.5,4.5)  {$\textcolor{white}{d}\grd\textcolor{white}{p}$};
\node[literal] (S1) at (8.5,4.5)  {$\textcolor{white}{d}\com\textcolor{white}{p}$};

\node[literal] (S1) at (1.5,3.5)  {\textcolor{gray!50}{\scriptsize{$ ? $}}};
\node[literal] (S1) at (2.5,3.5)  {\scriptsize{$ \k(\adm)$}};
\node[literal] (S1) at (3.5,3.5)  {\scriptsize{$ \k(\adm)$}};
\node[literal] (S1) at (4.5,3.5)  {\textcolor{gray!50}{\scriptsize{$ \k(\com) $}}};
\node[literal] (S1) at (5.5,3.5)  {\scriptsize{$ \k(\adm)$}};
\node[literal] (S1) at (6.5,3.5)  {\scriptsize{$ \k(\adm)$}};
\node[literal] (S1) at (7.5,3.5)  {\textcolor{gray!50}{\scriptsize{$ ? $}}};
\node[literal] (S1) at (8.5,3.5)  {\textcolor{gray!50}{\scriptsize{$ ? $}}};

\node[literal] (S1) at (1.5,2.5)  {\scriptsize{$ \k(\stb)$}};
\node[literal] (S1) at (2.5,2.5)  {\scriptsize{$ \k(\adm)$}};
\node[literal] (S1) at (3.5,2.5)  {\scriptsize{$ \k(\adm)$}};
\node[literal] (S1) at (4.5,2.5)  {\textcolor{gray!50}{\scriptsize{$ \k(\com) $}}};
\node[literal] (S1) at (5.5,2.5)  {\scriptsize{$ \k(\adm)$}};
\node[literal] (S1) at (6.5,2.5)  {\scriptsize{$ \k(\adm)$}};
\node[literal] (S1) at (7.5,2.5)  {\scriptsize{$ \k(\grd)$}};
\node[literal] (S1) at (8.5,2.5)  {\scriptsize{$ \k(\com)$}};

\node[literal] (S1) at (1.5,1.5)  {\scriptsize{$ \k(\stb)$}};
\node[literal] (S1) at (2.5,1.5)  {\scriptsize{$ \k(\adm)$}};
\node[literal] (S1) at (3.5,1.5)  {\scriptsize{$ \k(\adm)$}};
\node[literal] (S1) at (4.5,1.5) {\textcolor{gray!50}{\scriptsize{$ \k(\com) $}}};
\node[literal] (S1) at (5.5,1.5)  {\scriptsize{$ \k(\adm)$}};
\node[literal] (S1) at (6.5,1.5)  {\scriptsize{$ \k(\adm)$}};
\node[literal] (S1) at (7.5,1.5)  {\scriptsize{$ \k(\grd)$}};
\node[literal] (S1) at (8.5,1.5)  {\scriptsize{$ \k(\com)$}};

\node[literal] (S1) at (1.5,0.5)  {\scriptsize{$ \k(\stb)$}};
\node[literal] (S1) at (2.5,0.5)  {\scriptsize{$ \k(\adm)$}};
\node[literal] (S1) at (3.5,0.5)  {\scriptsize{$ \k(\adm)$}};
\node[literal] (S1) at (4.5,0.5) {\textcolor{gray!50}{\scriptsize{$ \k(\com) $}}};
\node[literal] (S1) at (5.5,0.5)  {\textcolor{gray!50}{\scriptsize{$ \k(\adm) $}}};
\node[literal] (S1) at (6.5,0.5)  {\textcolor{gray!50}{\scriptsize{$ \k(\adm) $}}};
\node[literal] (S1) at (7.5,0.5)  {\textcolor{gray!50}{\scriptsize{$ \k(\grd) $}}};
\node[literal] (S1) at (8.5,0.5)  {\textcolor{gray!50}{\scriptsize{$ \k(\com) $}}};

\node[literal] (S1) at (1.5,-0.5)  {{\scriptsize{$ \k(\stb) $}}};
\node[literal] (S1) at (2.5,-0.5) {\textcolor{gray!50}{\scriptsize{$ \k(\adm) $}}};
\node[literal] (S1) at (3.5,-0.5) {\textcolor{gray!50}{\scriptsize{$ \k(\adm) $}}};
\node[literal] (S1) at (4.5,-0.5) {\textcolor{gray!50}{\scriptsize{$ \k(\com) $}}};
\node[literal] (S1) at (5.5,-0.5)  {\textcolor{gray!50}{\scriptsize{$ \k(\adm) $}}};
\node[literal] (S1) at (6.5,-0.5)  {\textcolor{gray!50}{\scriptsize{$ \k(\adm) $}}};
\node[literal] (S1) at (7.5,-0.5)  {\textcolor{gray!50}{\scriptsize{$ \k(\grd) $}}};
\node[literal] (S1) at (8.5,-0.5)  {\textcolor{gray!50}{\scriptsize{$ \k(\com) $}}};

\node[literal] (S1) at (1.5,-1.5)  {\scriptsize{$ \iden $}};
\node[literal] (S1) at (2.5,-1.5)  {\scriptsize{$ \iden $}};
\node[literal] (S1) at (3.5,-1.5)  {\scriptsize{$ \iden $}};
\node[literal] (S1) at (4.5,-1.5)  {\scriptsize{$ \iden $}};
\node[literal] (S1) at (5.5,-1.5)  {\scriptsize{$ \iden $}};
\node[literal] (S1) at (6.5,-1.5)  {\scriptsize{$ \iden $}};
\node[literal] (S1) at (7.5,-1.5)  {\scriptsize{$ \iden $}};
\node[literal] (S1) at (8.5,-1.5)  {\scriptsize{$ \iden $}};

\node[literal] (S1) at (1.5,-2.5)  {\scriptsize{$ \iden $}};
\node[literal] (S1) at (2.5,-2.5)  {\scriptsize{$ \iden $}};
\node[literal] (S1) at (3.5,-2.5)  {\scriptsize{$ \iden $}};
\node[literal] (S1) at (4.5,-2.5)  {\scriptsize{$ \iden $}};
\node[literal] (S1) at (5.5,-2.5)  {\scriptsize{$ \iden $}};
\node[literal] (S1) at (6.5,-2.5)  {\scriptsize{$ \iden $}};
\node[literal] (S1) at (7.5,-2.5)  {\scriptsize{$ \iden $}};
\node[literal] (S1) at (8.5,-2.5)  {\scriptsize{$ \iden $}};

\node[literal] (S1) at (1.5,-3.5)  {\scriptsize{$ \iden $}};
\node[literal] (S1) at (2.5,-3.5)  {\scriptsize{$ \iden $}};
\node[literal] (S1) at (3.5,-3.5)  {\scriptsize{$ \iden $}};
\node[literal] (S1) at (4.5,-3.5)  {\scriptsize{$ \iden $}};
\node[literal] (S1) at (5.5,-3.5)  {\scriptsize{$ \iden $}};
\node[literal] (S1) at (6.5,-3.5)  {\scriptsize{$ \iden $}};
\node[literal] (S1) at (7.5,-3.5)  {\scriptsize{$ \iden $}};
\node[literal] (S1) at (8.5,-3.5)  {\scriptsize{$ \iden $}};

\node[literal] (S1) at (0.5,3.5)  {L};
\node[literal] (S1) at (0.5,2.5)  {E};
\node[literal] (S1) at (0.5,1.5)  {N};
\node[literal] (S1) at (0.5,0.5)  {S};
\node[literal] (S1) at (0.5,-0.5)  {ND};
\node[literal] (S1) at (0.5,-1.5)  {D};
\node[literal] (S1) at (0.5,-2.5)  {LD};
\node[literal] (S1) at (0.5,-3.5)  {U};

\end{tikzpicture}
\caption{\label{fig:labelrel}Labelling-based Characterizations for Finite AFs}
\end{figure}

In contrast to extension-based semantics we observe a much more homogeneous picture. Firstly, there is no need for the more sophisticated $\sigma$-*-kernels. Secondly, normal deletion equivalence w.r.t.\ labelling-based semantics is naturally incorporated in the overall picture in the sense that it coincides with its corresponding expansion, normal expansion and strong expansion equivalence notions. 

The following Figure~\ref{fig:rellaball} applies to each one of the eight labelling-based semantics considered in this section. In comparison to Figure~\ref{fig:prelimrel} where preliminary relations are depicted it illustrates (to a certain extent) a collapse of the diversity of the introduced equivalence notions in case of labelling-based semantics. 

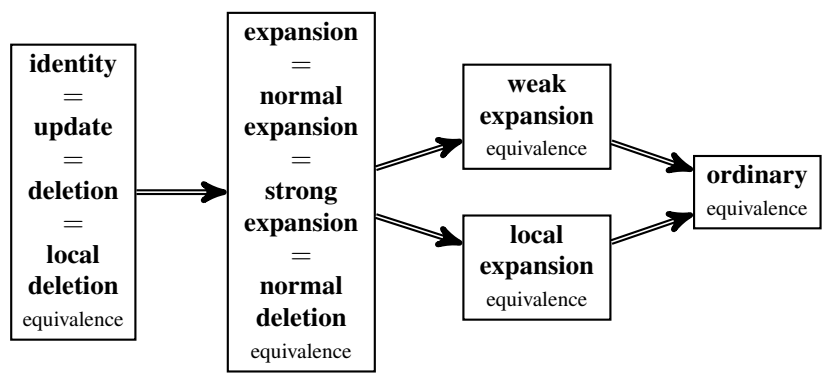
\begin{figure}[H]
\centering
\begin{tikzpicture}[scale=1.0]

\node (I) at (-1.6,1.5) [rectangle, thick, draw,text width = 1.4cm, text centered]{\textbf{identity}\\
$=$\\
\textbf{update}\\
$=$\\
\textbf{deletion}\\
$=$\\
\textbf{local}\\
\textbf{deletion}\\
\footnotesize{equivalence}};

\node (B) at (1.4,1.5) [rectangle, thick, draw,text width = 1.7cm, text centered]{\textbf{expansion}\\
$=$\\
\textbf{normal}\\
\textbf{expansion}\\
$=$\\
\textbf{strong}\\
\textbf{expansion}\\
$=$\\
\textbf{normal}\\
\textbf{deletion}\\
\footnotesize{equivalence}};

\node (F) at (4.5,0.5) [rectangle, thick, draw,text width = 1.7cm, text centered]{\textbf{local}\\
\textbf{expansion}\\
\footnotesize{equivalence}};

\node (W) at (4.5,2.5) [rectangle, thick, draw,text width = 1.7cm, text centered]{\textbf{weak}\\
\textbf{expansion}\\
\footnotesize{equivalence}};

\node (O) at (7.4,1.5) [rectangle, thick, draw,text width = 1.43cm, text centered]{\textbf{ordinary}\\
\footnotesize{equivalence}};

\draw[->, thick, double] (I) to [thick,double] (B);

\draw[->, thick, double] (B) to [thick,double] (W);
\draw[->, thick, double] (B) to [thick,double] (F);
\draw[->, thick, double] (F) to [thick,double] (O);
\draw[->, thick, double] (W) to [thick,double] (O);

\end{tikzpicture}
\caption{\label{fig:rellaball}Relations for $\sigma\in\{\stb,\semi,\eag,\adm,\prf,\id,\grd,\com\}$ -
Labelling-based Versions and Finite AFs}
\end{figure}

\section{Final Remarks} \label{sec:sumcon3}

In this chapter we motivated and discussed several notions of equivalence in the context of abstract argumentation and provided an exhaustive number of characterization theorems for extension-based as well as labelling-based semantics. In general we may state that Dung's abstract argumentation frameworks are a very compact formalism since the majority of the considered equivalence notion possess only little space for redundancy. Moreover, most of these notions collapse to identity if self-loop-free AFs are considered. This means, in this case any subframework of the AF in question may play a decisive role w.r.t.\ further evaluations and thus, cannot be locally replaced by another. This insight is sometimes used as an argument against the usefulness of the study of equivalence notions in the context of abstract argumentation. Obviously, we agree that if you are expecting much space for simplification, then the results are somehow disappointing but let us not lose sight of the fact that this is only clear \textit{after} it has been proved. Furthermore, as already stated, the results underline that in case of abstract argumentation (almost) everything is meaningful similar to other non-monotonic formalisms available in the literature (cf.\ \cite{DBLP:journals/tocl/LifschitzPV01} for logic programs, \cite{DBLP:conf/lpnmr/Turner04} for causal theories, \cite{Turner01} for default logic and \cite{DBLP:journals/amai/Truszczynski06} for nonmonotonic logics in general). However, one decisive difference to these formalisms is that equivalence notions in case of abstract argumentation can be decided syntactically. Indeed, kernels are interesting from several perspectives: First, they allow to decide the corresponding notion of equivalence by a simple check for topological equality and second, all kernels we have obtained so far can be efficiently constructed from a given argumentation framework. This means, if
a certain equivalence notion is characterizable through such a kernel, then we
have tractability of the associated decision problem.

\chapter{Verifiability of Argumentation Semantics} \label{cha:ver}

Over the last 20 years a series of abstract argumentation semantics were introduced. The
motivations of these semantics range from the desired treatment of specific examples
to fulfilling a number of abstract principles. The comparison via abstract criteria of the different semantics available is a topic which emerged quite recently in the community (\cite{BarG07} can be seen as the first paper in this line). 
In this chapter we take a further step towards a comprehensive understanding of argumentation semantics. In particular, we study the following question: Do we really need the entire AF $\AF$ to compute a certain argumentation semantics $\sigma$? In other words, is it possible to unambiguously determine acceptable sets  w.r.t.\ $\sigma$, given only partial information of the underlying framework $\AF$. In order to solve this problem let us start with the following reflections:

\begin{enumerate}
	\item As a matter of fact, one basic requirement of almost all existing semantics (exemptions are given in \cite{JakV99,Ari12,GroM15})
	is that of conflict-freeness, i.e. arguments within a reasonable position are not allowed to attack each other. Consequently, knowledge about conflict-free sets is an essential part for computing semantics.
	\item The second step is to ask the following: Which information on top on conflict-free sets has to be added? Imagine the set of conflict-free sets given by $\{\emptyset, \{a\}, \{b\}\}$. Consequently, there has to be at least one attack between $a$ and $b$. Unfortunately, this information is not sufficient to compute any standard semantics (except naive extensions, which are defined as $\subseteq$-maximal conflict-free sets) since we know nothing precise about the neighborhood of $a$ and $b$. The following three AFs possess exactly the mentioned conflict-free sets, but differ with respect to other semantics.
	
	
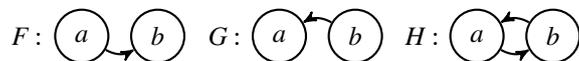
\begin{center}
\begin{tikzpicture}
	\node (a1) at (-0.6,0) [circle, thick, draw, label = left:$\F:$]{$a$};
    \node (b1) at (0.4,0) [circle, thick, draw]{$b$};
    
    \node (a2) at (2,0) [circle, thick, draw, label = left:$\G:$] {$a$};
    \node (b2) at (3.0,0) [circle, thick, draw] {$b$};
		
		\node (a3) at (4.6,0) [circle, thick, draw, label = left:$\H:$] {$a$};
    \node (b3) at (5.6,0) [circle, thick, draw] {$b$};

\draw[->,thick] (a1) to [thick, bend right] (b1);
\draw[->,thick] (b2) to [thick, bend right] (a2);
\draw[->,thick] (a3) to [thick, bend right] (b3);
\draw[->,thick] (b3) to [thick, bend right] (a3);
\end{tikzpicture}
\end{center}
\vspace{-8pt}	
	\item The final step is to try to minimize the added information. In other words, which kind of knowledge about the neighborhood is somehow dispensable in the light of computation? Clearly, this will depend on the considered semantics. For instance, in case of stage semantics \cite{Ver96}, which requests conflict-free sets of maximal range, we do not need any information about incoming attacks. This information can not be omitted in case of admissible-based semantics since incoming attacks require counterattacks.

\end{enumerate}
\vspace{-6pt} 

%
%
%
%
%
%
%
%
%
%
    %
%
\noindent
The above considerations motivate the introduction of 
\emph{verification classes} specifying a certain amount of information. 
In a first step,
we study the relation of these classes to each other. We therefore
introduce the notion of being \textit{more informative}, capturing the intuition that a certain class can reproduce the information of another. We present a hierarchy w.r.t.\ this ordering,
containing
15 different verification classes only. This is
because
many syntactically different classes collapse to the same amount of information. 

We then formally define the essential property of a semantics $\sigma$ being \textit{verifiable} w.r.t.\ a certain verification class. We present a general theorem stating that any \textit{rational} semantics is exactly verifiable w.r.t.\ one of the $15$ different verification classes. Roughly speaking, a semantics is rational if attacks inbetween two self-loops can be omitted without affecting the set of extensions. An important aside hereby is that even the most informative class contains indeed less information than the entire framework by itself.  

We consider a representative set of 
standard semantics. All of them satisfy rationality and thus, are exactly verifiable w.r.t.\ a certain class. Since the theorem does not provide an answer to which verification class perfectly matches a certain rational semantics we study this problem one by one for any considered semantics. As a result, only six different classes are essential to classify the considered standard semantics.

In the last section we study an application of the concept of verifiability. More precisely, we address the question of strong equivalence for semantics lying inbetween known semantics, so-called \textit{intermediate semantics}. 
Strong equivalence is the natural counterpart to ordinary equivalence in monotonic theories
(see \cite{strong,Bau16} for abstract argumentation and
\cite{Mah86,DBLP:journals/tocl/LifschitzPV01,DBLP:conf/lpnmr/Turner04,DBLP:journals/amai/Truszczynski06} for other nonmonotonic theories). We provide characterization theorems relying on the notion of verifiability and thus, contributing to a more abstract understanding of the different features argumentation semantics offer. 

\section{Preliminaries and Strongly Admissible Sets}

In this chapter we will consider finite AFs and extension-based semantics only. In Section~\ref{sec:flagship} we introduced the so-called \textit{range} $\E^\oplus_{\F}$ of a certain set $\E$ w.r.t.\ a given AF $\F = (A,R)$. It is given by $\E^\oplus_\F = \E\cup \E^+_\F$ where $\E^+_\F = \{b\mid (a,b)\in R, a\in \E\}$. 
Similarly, we will use $\E^\ominus_\F$
to denote the \textit{anti-range} of $\E$ defined as
$E\cup E^-_\F$ with $E^-_\F = \{b\mid (b,a)\in R, a\in \E\}$. If clear from context we drop the indices and simply write $\E^\oplus$, $\E^+$, $\E^\ominus$ or $\E^-$ respectively.

We now introduce the concepts of \textit{intermediate semantics} and \textit{rationality} which will play a decisive role in this chapter. Remember that we use $\sigma\subseteq\tau$ to indicate that $\sigma(\AF)\subseteq\tau(\AF)$ for each AF $\AF$. 

\begin{definition}
Given three semantics $\rho$,$\sigma$,$\tau$. If we have $\rho\subseteq\sigma$ and $\sigma\subseteq\tau$,
we say that $\sigma$ is \linebreak \textit{\interm{\rho}{\tau}}.
\end{definition}
Several examples of intermediate semantics can be found in Proposition~\ref{Pro:semrel}. For instance, we have that $\semi$ is \interm{\stb}{\prf}. 
		
\begin{definition}
\label{def:semantics_conditions}
A semantics $\sigma$ \emph{rational} if for each AF $\F$, $\sigma(\F) = \sigma(\F^l)$. We have
$\F^l = (A(\F),
         R(\F) \sm \{(a,b) \in R(\F) \mid (a,a), (b,b) \in R(\F), a \neq b\})$.
\end{definition}

Indeed, all semantics introduced in Definition~\ref{def:extsem} are rational. 
A prominent semantics that is based on conflict-free sets, but is not rational 
is the $\cfzwei$-semantics~\cite{BarGG05}, since here chains of self-loops can have an 
influence on the SCCs of an AF (see also \cite{GagW13}).



In order to present an exhaustive analysis of intermediate semantics (cf.\ Section~\ref{sec:intermediate}) we provide a missing characterizing kernel for strongly admissible sets firstly introduced in \cite{BarG07}. 
We will see that,
besides grounded \cite{strong} and resolution-based grounded semantics \cite{BarDG11,DvoLOW14},
strongly admissible sets are characterizable through the grounded kernel. Consider the following 
self-referential
definition
taken
from \cite{Cam14}.

\begin{definition} \label{def:strad} Given an AF $\AF = (A,R)$. A set $S\subseteq A$ is \textit{strongly admissible}, i.e. $S\in\sad(\AF)$ iff any $a\in S$ is defended by a strongly admissible set $S' \subseteq S\sm\{a\}$. 
\end{definition} 

The following properties are needed to prove the characterization theorem. $(1)$ and $(2)$ are already shown in \cite{BarG07}, $(3)$ is an immediate consequence of the former.

\begin{proposition} \label{pro:strad:prop} Given two AFs $\AF$ and $\G$, it holds that 

\begin{enumerate}
	\item $\grd(\AF)\subseteq\sad(\AF)\subseteq\adm(\AF)$,
	\item if $S\in\grd(\AF)$ we have: $S'\subseteq S$ for all $S'\in\sad(\AF)$, and
	\item $\sad(\AF) = \sad(\AG)$ implies $\grd(\AF) = \grd(\AG)$. 
\end{enumerate}
\end{proposition}

We now provide
an alternative criterion for
being a strongly admissible set.
In contrast to the former it allows one to construct strongly admissible sets step by step which enables a construction method. 

\begin{definition} \label{def:stradnew} Given an AF $\AF = (A,R)$. A set $S\subseteq A$ is \textit{strongly admissible}, i.e. $S\in\sad(\AF)$ iff there are finitely many and pairwise disjoint sets $A_1,...,A_n$, s.t. $S = \bigcup_{1\leq i \leq n} A_i$ and $A_1\subseteq\Gamma_{\AF}(\emptyset)$\footnote{Hereby, $\Gamma$ is the so-called \textit{characteristic function} (cf.\ Section~\ref{sec:flagship}). Note that the term $\Gamma_{\AF}(\emptyset)$ can be equivalently replaced by $\{a\in A\mid a \text{ is unattacked} \}$.}  and furthermore, $\bigcup_{1\leq i \leq j} A_i \text{ defends } A_{j+1} \text{ for } 1\leq j \leq n-1$.

\end{definition} 
The following proof shows that both definitions are indeed equivalent.

\begin{proposition} \label{pro:stradnew} Definitions~\ref{def:strad} and \ref{def:stradnew} are equivalent.
\end{proposition} 

\begin{proof} For the proof we use $S\in\sad_{k}(\AF)$ as a shorthand for $S\in\sad(\AF)$ in the sense of Definition~$k$.\\
$(\oT)$ Given $S\in\sad_{\ref{def:stradnew}}(\AF)$. Hence, there is a finite partition, s.t. $S = \bigcup_{1\leq i \leq n} A_i$, $A_1\subseteq\Gamma_{\AF}(\emptyset)$ and $\bigcup_{1\leq i \leq j} A_i$  defends $A_{j+1} \text{ for } 1\leq j \leq n-1$. Observe that $\bigcup_{1\leq i \leq j} A_i \in \sad_{\ref{def:stradnew}}(\AF)$ for any $j\leq n$. Let $a\in S$. Consequently, there is an index $i^*$, s.t. $a\in A_{i^*}$. Furthermore, since $\bigcup_{1\leq i \leq i^*-1} A_i \text{ defends } A_{i^*}$ by definition, we deduce that $\bigcup_{1\leq i \leq i^*-1} A_i\subseteq S\sm\{a\}$ defends $a$. We have to show now that (the smaller set w.r.t.\ $\subseteq$) $\bigcup_{1\leq i \leq i^*-1} A_i\in\sad_{\ref{def:strad}}(\AF)$. Note that $\bigcup_{1\leq i \leq i^*-1} A_i\in\sad_{\ref{def:stradnew}}(\AF)$.  Since we are dealing with finite AFs we may iterate our construction. Hence, no matter which elements are chosen we end up with a $\subseteq$-chain, s.t. $\emptyset\subseteq \bigcup_{1\leq i \leq i_e} A_i \subseteq S_e\sm{a_e} $ and $\emptyset$ defends $a_{e}$ for some index $i_e$, set $S_e$ and element $a_e$. This means, the question whether $S\in\sad_{\ref{def:strad}}(\AF)$ can be decided positively by proving  $\emptyset\in\sad_{\ref{def:strad}}(\AF)$. Since the empty set does not contain any elements we find $\emptyset\in\sad_{\ref{def:strad}}(\AF)$ concluding $\sad_{\ref{def:stradnew}}\subseteq\sad_{\ref{def:strad}}$.\\
$(\To)$ Given $S\in\sad_{\ref{def:strad}}(\AF)$, consider the following sets~$S_i$: $S_1 = \left(\Gamma(\emptyset)\sm \emptyset\right) \cap S$, $S_2 = \left(\Gamma(S_1)\sm S_1\right) \cap S$, $S_3 = \left(\Gamma(\bigcup_{i=1}^2 S_i)\sm \bigcup_{i=1}^2 S_i\right) \cap S$, \dots , $S_n = \left(\Gamma(\bigcup_{i=1}^{n-1} S_{i})\sm \bigcup_{i=1}^{n-1} S_{i}\right) \cap S$. Since we are dealing with finite AFs there has to be a natural $n\in\N$, s.t. $S_n = S_{n+1} = S_{n+2} = \dots$. Consider now the union of these sets, i.e. $\bigcup_{i=1}^{n} S_{i}$. We show now that $\bigcup_{i=1}^{n} S_{i}\in\sad_{\ref{def:stradnew}}(\AF)$ and $\bigcup_{i=1}^{n} S_{i} = S$. By construction we have $S_1 \subseteq \Gamma(\emptyset)$. Moreover, $\bigcup_{1\leq i \leq j} S_i \text{ defends } S_{j+1} \text{ for } 1\leq j \leq n-1$. This can be seen as follows. By definition $S_{j+1} = \left(\Gamma(\bigcup_{i=1}^{j} S_{i})\sm \bigcup_{i=1}^{j} S_{i}\right) \cap S$. This means, $S_{j+1} \subseteq \Gamma\left(\bigcup_{i=1}^{j} S_{i}\right)$. Since $\Gamma\left(\bigcup_{i=1}^{j} S_{i}\right)$ contains all elements defended by $\bigcup_{i=1}^{j} S_{i}$ we obtain $\bigcup_{i=1}^{n} S_{i}\in\sad_{\ref{def:stradnew}}(\AF)$. Obviously, $\bigcup_{i=1}^{j} S_{i} \subseteq S$. In order to derive a contradiction we suppose $S\not\subseteq \bigcup_{i=1}^{n} S_{i}$. This means there is a nonempty set $S^*$, s.t. $S = S^* \cup \bigcup_{i=1}^{n} S_{i}$. Let $S^* = \{s_1,\dots,s_k\}$. Observe that no element $s_i$ is defended by $\bigcup_{i=1}^{n} S_{i}$ (*). Since $S\in\sad_{\ref{def:strad}}(\AF)$ we obtain a set $S^*_1 \subseteq S\sm\{s_1\}$, s.t. $S^*_1\in\sad_{\ref{def:strad}}(\AF)$ and $S^*_1$ defends $s_1$. We now iterate this procedure ending up with a set $S^*_k \subseteq S^*_{k-1}\sm\{s_k\}\subseteq \bigcup_{i=1}^{n} S_{i}$, s.t. $S^*_k\in\sad_{\ref{def:strad}}(\AF)$ and $S^*_k$ defends $s_k$ contradicting (*) and concluding the proof. 
\end{proof}


\begin{example} Consider the following AF $\AF$.
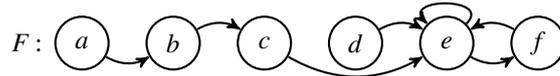
\vspace{-8pt}
\begin{center}
\begin{tikzpicture}
	\node (a) at (-0.6,0) [minimum size = 7mm, circle, thick, draw, label = left:$\AF:$]{$a$};
    \node (b) at (0.6,0) [minimum size = 7mm, circle, thick, draw]{$b$};
    \node (c) at (1.8,0) [minimum size = 7mm, circle, thick, draw] {$c$};
		\node (d) at (3,0) [minimum size = 7mm, circle, thick, draw] {$d$};
    \node (e) at (4.2,0) [minimum size = 7mm, circle, thick, draw] {$e$};
		\node (f) at (5.4,0) [minimum size = 7mm, circle, thick, draw] {$f$};

\draw[->,thick] (e) to [thick,loop,distance=0.5cm] (e);
\draw[->,thick] (a) to [thick, bend right] (b);
\draw[->,thick] (b) to [thick, bend left] (c);
\draw[->,thick] (c) to [thick, bend right] (e);
\draw[->,thick] (e) to [thick, bend right] (f);
\draw[->,thick] (f) to [thick, bend right] (e);
\draw[->,thick] (d) to [thick, bend left] (e);

\end{tikzpicture}
\end{center}
We have $\Gamma_{\AF}(\emptyset) = \{a,d\}$. Hence, for all $S\subseteq\{a,d\}$, $S\in\sad(\AF)$. Furthermore, $\Gamma_{\AF}(\{a\}) = \{a,c\}$, $\Gamma_{\AF}(\{d\}) = \{d,f\}$ and $\Gamma_{\AF}(\{a,d\}) = \{a,d,c,f\}$. This means, additionally $\{a,c\},\{d,f\},\{a,d,c\},\{a,d,f\}, $ $\{a,d,c,f\}\in\sad(\AF)$. Finally, $\Gamma_{\AF}(\{a,c\}) = \{a,c,f\}$ justifying the last missing set $\{a,c,f\}\in\sad(\AF)$. 
\end{example}

The following corollary is an immediate consequence of Definition~\ref{def:stradnew}. It is essential to prove the characterization theorem for strongly admissible sets.

\begin{corollary} \label{cor:strad}
Given an AF $\AF$ and two sets $B,B'\subseteq A(F)$. If $B$ defends $B'$, then $B\cup B'$ is strongly admissible if $B$ is. 
\end{corollary}

The grounded kernel is insensitive w.r.t.\ strongly admissible sets,
which then allows us to state the main result for strongly admissible sets.

\begin{lemma} \label{lem:strad} For any AF $\AF$,  $\sad(\AF) = \sad\left(\AF^{k(\grd)}\right)$.
\end{lemma}
\begin{proof} The grounded kernel is node- and loop-preserving, i.e.\ $A(\AF) = A\left(\AF^{k(\grd)}\right)$ and $L(\AF)~=~L\left(\AF^{k(\grd)}\right)$. Furthermore, $\cf(\AF)\! =\! \cf\left(\AF^{k(\grd)}\right)$ and $\Gamma_{\AF}(\emptyset)~=~\Gamma_{\AF^{k(\grd)}}(\emptyset)$ as shown in \cite[Lemma~6]{strong}.\\
($\subseteq$) Given $S\in\sad(\F)$. The proof is by induction on $n$ indicating the number of  sets forming a suitable (according to Definition~\ref{def:stradnew}) partition of $S$. Let $n=1$. In consideration of the grounded kernel we observe $\Gamma_{\F}(\emptyset) = \Gamma_{\F^{k(\grd)}}(\emptyset)$, i.e. the set of unattacked arguments does not change. Since $S\subseteq\Gamma_{\F}(\emptyset)$ is assumed we are done. Assume now that the assertion is proven for any $k$-partition. Let $S$ be a $(k+1)$-partition, i.e. $S = \bigcup_{i=1}^{k+1} A_i$. According to induction hypothesis as well as Corollary~\ref{cor:strad} it suffices to prove $\bigcup_{i=1}^k A_i$ defends $A_{k+1}$ in $\F^{k(\grd)}$. Assume not, i.e. there are arguments $b\in A(\F)\sm S$, $c\in A_{k+1}$ s.t. $(b,c)\in R\left(\F^{k(\grd)}\right) \subseteq R(\AF)$ and for all $a\in \bigcup_{i=1}^k A_i$, $(a,b)\notin R\left(\F^{k(\grd)}\right)$ (*). Since $\bigcup_{i=1}^k A_i$ defends $A_{k+1}$ in $\F$ we deduce the existence of an argument $a\in\bigcup_{i=1}^k A_i$
 s.t. $(a,b)\in R(\F)$. Thus, $(a,b)$ is redundant w.r.t.\ the grounded kernel. According to
 Definition~\ref{def:kernel} and due to the conflict-freeness of $\bigcup_{i=1}^k A_i$ we have $(a,a)\notin R(\F)$ and $(b,a),(b,b)\in R(\F)$. Consequently, $(b,a)\in\F^{k(\grd)}$.
 Since $\bigcup_{i=1}^k A_i$ is a strong admissible $k$-partition in $\F$ we obtain by induction hypothesis that $\bigcup_{i=1}^k A_i$ is strongly admissible in $\F^{k(\grd)}$ and therefore, admissible in $\F^{k(\grd)}$ (Proposition~\ref{pro:strad:prop}). Hence there has to be an argument $a\in \bigcup_{i=1}^k A_i$, s.t.\ $(a,b)\in R\left(\F^{k(\grd)}\right)$, contradicting~(*).\\
($\supseteq$) Assume $S\in\sad\left(\F^{k(\grd)}\right)$. We show $S\in\sad(\F)$ by induction on $n$ indicating that $S$ is a $n$-partition in $\F^{k(\grd)}$. Due to $\Gamma_{\F}(\emptyset) = \Gamma_{\F^{k(\grd)}}(\emptyset)$ the base case is immediately clear. For the induction step let $S$ be a $(k+1)$-partition, i.e. $S = \bigcup_{i=1}^{k+1} A_i$. By induction hypothesis we may assume that $\bigcup_{i=1}^k A_i$ is strongly admissible in $\F$. Using Corollary~\ref{cor:strad} it suffices to prove $\bigcup_{i=1}^k A_i$ defends $A_{k+1}$ in $\F$. Assume not, i.e. there are arguments $b\in A(\F)\sm S$, $c\in A_{k+1}$ s.t. $(b,c)\in R(\F)$ and for all $a\in \bigcup_{i=1}^k A_i$, $(a,b)\notin R(\F)$. We even have $(a,b)\notin R\left(\F^{k(\grd)}\right)$ since $R\left(\F^{k(\grd)}\right)\subseteq R(\F)$. Consequently, $(b,c)$ has to be deleted in $\F^{k(\grd)}$. Definition~\ref{def:kernel} requires $(c,c)\in R\left(\F^{k(\grd)}\right)$ contradicting the conflict-freeness of $S$ in $\F^{k(\grd)}$.
\end{proof}


\begin{theorem} For any two AFs $\F$ and $\G$, we have
$\F\equiv^{\sad}_{E} \G \ToT \F^{k(\grd)} = \G^{k(\grd)}$.
\end{theorem}

\begin{proof} ($\To$) We show the contrapositive, i.e. $\F^{k(\grd)}~\neq~\G^{k(\grd)} \To \F\not\equiv^{\sad}_{E} \G$. Assuming $\F^{k(\grd)} \neq \G^{k(\grd)}$ implies $\F\not\equiv^{\grd}_{E} \G$ (cf.\ Theorem~\ref{the:strong}). This means, there is an AF $\H$, s.t. $\grd(F\dcup H) \neq \grd(G\dcup H)$.
Due to statement 3 of Proposition~\ref{pro:strad:prop}, we deduce $\sad(F\dcup H) \neq \sad(G\dcup H)$ proving $\F\not\equiv^{\sad}_{E} \G$.\\
($\oT$) Given $\F^{k(\grd)} = \G^{k(\grd)}$. Since expansion equivalence is a congruence w.r.t.\ $\dcup$ we obtain $(\F\dcup\H)^{k(\grd)} = (\G\dcup\H)^{k(\grd)}$ for any AF $\H$. Consequently, $\sad\left((\F\dcup\H)^{k(\grd)}\right) = \sad\left((\G\dcup\H)^{k(\grd)}\right)$. Due to Lemma~\ref{lem:strad} we deduce $\sad(\F\dcup\H) = \sad(\G\dcup\H)$, concluding the proof. 
\end{proof}

\section{Verifiability}
\label{sec:verify}
 
In this section we study the question whether we really need
the entire AF $\F$ to compute the extensions of a given semantics.
Consider naive semantics.
Obviously, in order to determine naive extensions it suffices
to 
know
all conflict-free sets.
Conversely, knowing $\cf(\F)$ only does not allow to reconstruct $\F$ unambiguously.
This means, knowledge about $\cf(\F)$ is indeed less information than the entire AF by itself.
In fact, most of the existing semantics do not need information about the entire AF.
We will categorize the amount of information by taking the conflict-free sets as a basis
and distinguish between different amounts of knowledge about the neighborhood of these sets.
One the one hand this is natural since conflict-freeness is the most basic concept used by argumentation semantics
and on the other hand the neighborhood (i.e.\ arguments attacking and being attacked by an argument)


\begin{definition}
\label{def:classes}
We call a function $\r^x : \powerset{\m{U}}\times \powerset{\m{U}} \to \left(\powerset{\m{U}}\right)^n$
($n>0$),
which is expressible via basic set operations only,
\emph{neighborhood function}.
A neighborhood function $\r^x$ induces the \emph{verification class}
mapping each AF $\F$ to
$\widetilde{\F}^x = \{(S,\r^x(S^\oplus_\F,S^\ominus_\F)) \mid S \in \cf(\F)\}$.
%
\end{definition}

We coined the term neighborhood function because the induced verification classes apply these functions to the neighborhoods, i.e. range and anti-range of conflict-free sets. The notion of {\it expressible via basic set operations} simply means that (in case of $n = 1$) the expression $\r^x(A,B)$ is in the language generated by the BNF
$X ::= A\mid B \mid (X\cup X) \mid (X\cap X) \mid (X\setminus X)$.
Consequently, in case of $n=1$, we may distinguish eight
set theoretically different neighborhood functions, namely
\vspace{-5pt}
\begin{align*}
\r^\epsilon(S,S') &= \emptyset &
\r^{+}(S,S') &=  S  &
\r^{-}(S,S') &=  S' &
\r^{\mp}(S,S') &= S' \sm S \\[-3pt]
\r^{\pm}(S,S') &=  S \sm S' &
\r^{\cap}(S,S') &=  S \cap S' &
\r^{\cup}(S,S') &= S \cup S' &
\r^{\Delta}(S,S') &= (S \cup S') \sm (S \cap S')
\end{align*}

\vspace{-6pt}
The names of the neighborhood functions are inspired by
their usage in the verification classes they induce (cf.\ Definition~\ref{def:classes}).
A verification class encapsulates a certain amount of information about an AF, as the following example illustrates.

\begin{example}
\label{ex:verification_class}
Consider the following AF $\F=(\{a,b,c\},\{(a,b),(b,a),(b,b),(c,b)\})$.
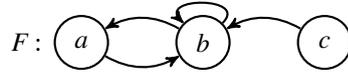
\begin{center}
\begin{tikzpicture}
	\node (a) at (-0.6,0) [circle, thick, draw, label = left:$\F:$]{$a$};
    \node (b) at (1,0) [circle, thick, draw]{$b$};

    \node (c) at (2.6,0) [circle, thick, draw] {$c$};

\draw[->,thick] (b) to [thick,loop,distance=0.5cm] (b);

\draw[->,thick] (b) to [thick, bend right] (a);
\draw[->,thick] (a) to [thick, bend right] (b);
\draw[->,thick] (c) to [thick, bend right] (b);
\end{tikzpicture}
\end{center}
Now take, for instance, the verification class induced by $\r^+$, that is
$\widetilde{\F}^+ = \{(S,\r^+(S^\oplus_\F,S^\ominus_\F))\mid S\in\cf(\F)\} = \{(S,S^\oplus_\F)\mid S\in\cf(\F)\}$,
storing information about conflict-free sets together with their associated ranges w.r.t.\ $\F$.
It contains the following tuples:
$(\emptyset,\emptyset)$, $(\{a\},\{a,b\})$, $(\{c\},\{b,c\})$, and $(\{a,c\},\{a,b,c\})$.
For he verification class induced by $\r^\pm$,
on the other hand,
we have
$\widetilde{\F}^\pm = \{(\emptyset,\emptyset),(\{a\},\emptyset),(\{c\},\{b\}),(\{a,c\},\emptyset)\}$.
\end{example}

 Intuitively, it should be clear that the set $\widetilde{\F}^+$ suffices to compute stage extensions (i.e., range-maximal conflict-free sets) of~$\F$. This intuitive understanding of \textit{verifiability} will be formally specified in Definition~\ref{def:verifiable}.  
Note that a neighborhood function $\r^x$ may return $n$-tuples. Consequently, in consideration of the eight
basic functions we obtain (modulo reordering, duplicates, empty set) $2^7+1$ syntactically different neighborhood functions and therefore the same number of verification classes. As usual, we
denote the $n$-ary combination of basic functions $(\r^{x_1}(S,S'), \dots, \r^{x_n}(S,S'))$ as $\r^x(S,S')$ by $x=x_1\dots x_n$.

With the following definition we can put neighborhood functions into relation w.r.t.\ their information.
This will help us to show that actually many of the induced classes collapse to the same amount of information.

\begin{definition}
\label{def:informative}
Given neighborhood functions $\r^x$ and $\r^y$ returning $n$-tuples and $m$-tuples, respectively,
we say that $\r^x$ is \textit{more informative} than $\r^y$,
for short $\r^x \succeq \r^y$,
iff there is a function
$\delta : \left(\powerset{\m{U}}\right)^n \to \left(\powerset{\m{U}}\right)^m$
such that for any two sets of arguments $S,S' \subseteq \m{U}$,
we have $\delta(\r^x(S,S')) = \r^y(S,S')$.
We denote the strict part of $\succeq$ by $\succ$,
i.e.\ $\r^x \succ \r^y$ iff $\r^x \succeq \r^y$ and $\r^y \not\succeq \r^x$.
Finally, $\r^x \approx \r^y$ ($\r^x$ \emph{represents} $\r^y$ and vice versa)
in case $\r^x \succeq \r^y$ and $\r^y \succeq \r^x$.
\end{definition}

\begin{figure}[t]
\centering
\begin{tikzpicture}
\matrix (a) [matrix of math nodes, column sep=0.41cm, row sep=0.5cm]{
& & & +- \\
+\pm & +\mp & \pm\mp & & \cap\cup & -\pm & -\mp \\[0.5cm]
+ & \pm & \cap & \Delta & \cup & \mp & - \\
& & & \epsilon \\};

\foreach \i/\j in {2-1/1-4, 2-2/1-4,  2-3/1-4, 2-5/1-4, 2-6/1-4, 2-7/1-4,
                   3-1/2-1, 3-1/2-2,
                   3-2/2-1, 3-2/2-3, 3-2/2-6,
                   3-3/2-1, 3-3/2-5, 3-3/2-7,
                   3-4/2-3, 3-4/2-5,
                   3-5/2-2, 3-5/2-5, 3-5/2-6,
                   3-6/2-2, 3-6/2-3, 3-6/2-7,
                   3-7/2-6, 3-7/2-7,
                   4-4/3-1, 4-4/3-2, 4-4/3-3, 4-4/3-4, 4-4/3-5, 4-4/3-6, 4-4/3-7}
         \draw[->] (a-\i) -- (a-\j);
\end{tikzpicture}
\vspace{-10pt}
\caption{Representatives of neighborhood functions and their relation w.r.t.\ information;
         a node $x$ stands for the neighborhood function $\r^x$;
         an arrow from $x$ to $y$ means $\r^x \prec \r^y$.}
\label{fig:verification_classes}
\end{figure}
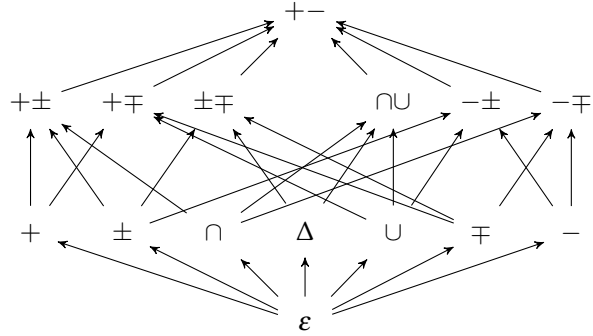

It turns out that many neighborhood functions amount to the same amount of information.
In particular, $\r^{+-}$ represents all $\r^{x_1,\dots,x_n}$ with $n>2$.

\begin{lemma}
\label{lemma:verification_classes}
All neighborhood functions are represented by the ones
depicted in Figure~\ref{fig:verification_classes}
and
the $\prec$-relation represented by arcs in Figure~\ref{fig:verification_classes} holds.
\end{lemma}

\begin{proof}
We begin by showing that all neighborhood functions are represented in Figure~\ref{fig:verification_classes}.
Clearly, each neighborhood function $\r^x$ represents itself, i.e. $\r^x \approx \r^x$.
All neighborhood functions for $n=1$
are depicted in Figure~\ref{fig:verification_classes}.
We turn to $n=2$.
Consider the neighborhood functions $\r^{+\pm}$, $\r^{+\cap}$, and $\r^{\pm\cap}$,
defined as
$\r^{+\pm}(S,S') = (S,S \sm S')$,
$\r^{+\cap}(S,S') = (S,S \cap S')$, and
$\r^{\pm\cap}(S,S') = (S \sm S',S \cap S')$
for $S,S' \subseteq \m{U}$.
Observe that $S = (S \sm S') \cup (S \cap S')$.
Hence, we can easily define functions in the spirit of Definition~\ref{def:informative}
mapping the images of the function to one another.

\begin{itemize}
\item
$\delta_1(\r^{+\pm}(S,S')) = \delta_1(S,S \sm S') =_{def} (S,S \sm (S \sm S')) = (S,S \cap S') = \r^{+\cap}(S,S')$;
\item
$\delta_2(\r^{+\cap}(S,S')) = \delta_2(S,S \cap S') =_{def} (S \sm (S \cap S'),S \cap S') = (S \sm S',S \cap S') = \r^{\pm\cap}(S,S')$;
\item
$\delta_3(\r^{\pm\cap}(S,S')) = \delta_3(S \sm S',S \cap S') =_{def} ((S \sm S') \cup (S \cap S'),S \sm S')$ $= (S,S \sm S')$ = $\r^{+\pm}(S,S')$.
\end{itemize}

Therefore, $\r^{+\pm} \approx \r^{+\cap} \approx \r^{\pm\cap}$.
In particular, they are all represented by $\r^{\pm}$.
We can apply the same reasoning to other combinations of neighborhood functions
and get the following equivalences w.r.t.\ information content:

\begin{itemize}
\item
$\r^{+\mp} \approx \r^{+\cup} \approx \r^{\mp\cup}$;
\item
$\r^{\pm\mp} \approx \r^{\pm\Delta} \approx \r^{\mp\Delta}$;
\item
$\r^{\cap\cup} \approx \r^{\cap\Delta} \approx \r^{\cup\Delta}$;
\item
$\r^{-\pm} \approx \r^{-\cup} \approx \r^{\pm\cup}$; and
\item
$\r^{-\mp} \approx \r^{-\cap} \approx \r^{\mp\cap}$,
\end{itemize}
with the functions stated first
acting as representatives in Figure~\ref{fig:verification_classes}.

For the remaining functions returning $2$-tuples we get
$\r^{+-} \approx \r^{+\Delta} \approx \r^{-\Delta}$
by the fact that $\r^{\Delta}(S,S') = (S \cup S') \sm (S \cap S')$.
\pagebreak
\begin{itemize}[leftmargin=15pt]
\item
$\delta_4(\r^{+-}(S,S')) = \delta_4(S,S') =_{def} (S,(S\cup S') \sm (S\cap S')) = \r^{+\Delta}(S,S')$;
\item
$\delta_5(\r^{+\Delta}(S,S')) = \delta_5(S,(S\cup S') \sm (S\cap S')) =_{def} ((S\sm((S\cup S') \sm (S\cap S')))\cup ((S\cup S') \sm (S\cap S'))\sm S,(S\cup S') \sm (S\cap S')) = (S',(S\cup S') \sm (S\cap S')) = \r^{-\cap}(S,S')$;
\item
$\delta_6(\r^{-\Delta}(S,S')) = \delta_6(S',(S\cup S') \sm (S\cap S')) =_{def} ((S'\sm((S\cup S') \sm (S\cap S')))\cup ((S\cup S') \sm (S\cap S'))\sm S',S') = (S,S') = \r^{+-}(S,S')$.
\end{itemize}

Finally, every neighborhood function $\r^{x_1 \dots x_n}$
with $x_1,\dots,x_n \in \{+,-,\pm,\mp,\cup,\cap,\Delta\}$ and
$n \geq 3$
is represented by $\r^{+-}$ since we can compute all possible sets
from $S$ and $S'$.
Therefore $\r^{+-}$ (together with all of these equally informative function) is the most informative
neighborhood function.


Now consider two functions $\r^x$ and $\r^y$ such that
there is an arrow from $x$ to $y$ in Figure~\ref{fig:verification_classes}.
It is easy to see that $\r^y \succeq \r^x$ since,
for sets of arguments $S$ and $S'$,
$\r^x(S,S')$ is either contained in $\r^y(S,S')$ or
obtainable from $\r^y(S,S')$ by basic set operations.
The fact that $\r^x  \not\succeq \r^y$, entailing $\r^y \succ \r^x$,
follows from the impossibility of finding a function $\delta$
such that $\delta(\r^x(S,S')) = \r^y(S,S')$.
%

\end{proof}

If the information provided by a neighborhood function 
is sufficient to compute the extensions under a semantics,
we say that the semantics is verifiable by
the class induced by the neighborhood function.

\begin{definition} \label{def:verifiable}
A semantics $\sigma$ is \textit{verifiable} by
the verification class induced by
the neighborhood function
$\r^x$ returning $n$-tuples (or simply, $x$\textit{-veri\-fi\-able}) iff
there is a function (also called \textit{criterion})
$\gamma_{\sigma} : \left(\powerset{\m{U}}\right)^n \times \powerset{\m{U}} \to \powerset{\powerset{\m{U}}}$
s.t. for every AF $\F \in \m{U}$ we have:
$\gamma_{\sigma}\left(\widetilde{\F}^x, A(\F)\right) = \sigma(\F)$.
Moreover, $\sigma$ is \emph{exactly $x$-verifiable} iff
$\sigma$ is $x$-verifiable and there is no verification class induced by $\r^y$ with
$\r^y \prec \r^x$ such that $\sigma$ is $y$-verifiable.
\end{definition}

Observe that if a semantics $\sigma$ is $x$-verifiable then for any two AFs $\F$ and $\G$ with $\widetilde{\F}^x = \widetilde{\G}^x$ and\linebreak $A(\F) = A(\G)$
it must hold that $\sigma(F) = \sigma(G)$.

We proceed with a list of criteria showing that well-known semantics mentioned in Definition~\ref{def:extsem} are verifiable by a verification class induced by a certain neighborhood function. 
In the following, we abbreviate the tuple $\left(\widetilde{\F}^x,A(\F)\right)$ by $\widetilde{\F}^x_A$. 


\begin{align*}
\gamma_{\nav}\left(\widetilde{\F}^\epsilon_A\right) =  \Big\{& S\mid 
           S\in\widetilde{\F}, S \textit{ is }\subseteq\textit{-maximal in }\widetilde{\F} \Big\}; \\
\gamma_{\stg}\left(\widetilde{\F}^+_A\right) = \Big\{ & S\mid 
           (S,S^+)\in\widetilde{\F}^+,
             S^+ \textit{ is }\subseteq\textit{-maximal in } \{C^+\mid (C,C^+)\in\widetilde{\F}^+\} \Big\}; \\
\gamma_{\stb}\left(\widetilde{\F}_A^+\right) = \Big\{ & S\mid 
           (S,S^+)\in\widetilde{\F}^+, S^+= A\Big\}; \\
\gamma_{\adm}\left(\widetilde{\F}^\mp_A\right) = \Big\{ & S\mid 
           (S,S^\mp)\in\widetilde{\F}^\mp, S^\mp = \emptyset\Big\}; \\
\gamma_{\prf}\left(\widetilde{\F}^\mp_A\right) = \Big\{ & S\mid 
           S\in\gamma_{\adm}\left(\widetilde{\F}^\mp_A\right),
             S \textit{ is }\subseteq\textit{-maximal in } \gamma_{\adm}\left(\widetilde{\F}^\mp_A\right) \Big\}; \\
\gamma_{\semi}\left(\widetilde{\F}^{+\mp}_A\right) = \Big\{ & S\mid
           S\in\gamma_{\adm}\left(\widetilde{\F}^\mp_A\right), 
             S^+ \textit{ is }\subseteq\textit{-maximal in } 
             \Big\{C^+\mid (C,C^+,C^\mp)\in\widetilde{\F}^{+\mp}, C\in\gamma_{\adm}\left(\widetilde{\F}^\mp_A\right)\Big\} \Big\};\\
\gamma_{\id}\left(\widetilde{\F}^\mp_A\right) = \Big\{ & S\mid
           S \textit{ is }\subseteq\textit{-maximal in } 
           \Big\{C \mid C \in \gamma_{\adm}\left(\widetilde{\F}^\mp_A\right), C \subseteq \bigcap \gamma_\prf\left(\widetilde{\F}^\mp_A\right)\Big\} \Big\};\\
\gamma_{\eag}\left(\widetilde{\F}^{+\mp}_A\right) = \Big\{ & S\mid
           S \textit{ is }\subseteq\textit{-maximal in } 
           \Big\{C \mid C \in \gamma_{\adm}\left(\widetilde{\F}^\mp_A\right), C \subseteq \bigcap \gamma_\semi\left(\widetilde{\F}^{+\mp}_A\right)\Big\}\Big\};
\end{align*}
\begin{align*}
\gamma_{\sad}\left(\widetilde{\F}^{-\pm}_A\right) = \Big\{ & S \mid 
           (S,S^-,S^\pm) \in \widetilde{\F}^{-\pm}, 
           \exists (S_0,S_0^-,S_0^\pm),\dots,(S_n,S_n^-,S_n^\pm) \in \widetilde{\F}^{-\pm}: \\
           & (\emptyset = S_0 \subset \dots \subset S_n = S \wedge
           \forall i \in \{1,\dots,n\} : S_i^- \subseteq S_{i-1}^\pm) \Big\};\\
\gamma_{\grd}\left(\widetilde{\F}^{-\pm}_A\right) = \Big\{ & S \mid
           S \in \gamma_{\sad}\left(\widetilde{\F}^{-\pm}_A\right), 
            \forall (\bar{S},\bar{S}^-,\bar{S}^\pm) \in \widetilde{\F}^{-\pm} : 
             \bar{S} \supset S \Rightarrow (\bar{S}^- \sm S^\pm) \neq \emptyset) \Big\}; \\
\gamma_{\com}\left(\widetilde{\F}^{+-}_A\right) = \Big\{ & S \mid 
           (S,S^+,S^-) \in \widetilde{\F}^{+-}, (S^- \sm S^+) = \emptyset, 
           \forall (\bar{S},\bar{S}^+,\bar{S}^-) \in \widetilde{\F}^{+-} : 
             \bar{S} \supset S \Rightarrow (\bar{S}^- \sm S^+) \neq \emptyset) \Big\}.
\end{align*}
It is easy to see that the naive semantics
is verifiable by the verification class induced by $\r^\epsilon$ 
since the naive extensions can be determined by the conflict-free sets.
Stable and stage semantics, on the other hand,
utilize the range of each conflict-free set in addition.
Hence they are verifiable by the verification class induced by $\r^+$.
Now consider admissible sets.
Recall that a conflict-free $S$ set is admissible if and only if
it attacks all attackers.
This is captured exactly by the condition $S^\mp = \emptyset$,
hence admissible sets are verifiable by the verification class induced by $\r^\mp$.
The same holds for preferred semantics,
since we just have to determine the maximal conflict-free sets
with $S^\mp = \emptyset$.
Semi-stable semantics, however, needs the range of each conflict-free set
in addition, see $\gamma_\semi$,
which makes it verifiable by the verification class induced by $\r^{+\mp}$.
Finally consider the criterion $\gamma_\com$.
The first two conditions for a set of arguments $S$ stand
for conflict-freeness and admissibility, respectively.
Now assume the third condition does not hold,
i.e., there exists a tuple
$(\bar{S},\bar{S}^+,\bar{S}^-) \in \widetilde{F}^{+-}$
with $\bar{S} \supset S$ and $\bar{S}^- \sm S^+ = \emptyset$.
This means that every argument attacking $\bar{S}$ is attacked
by $S$, i.e., $\bar{S}$ is defended by $S$.
Hence $S$ is not a complete extension,
showing that $\gamma_\com(\widetilde{F}^{+-}_A) = \com(F)$ for each $F \in \m{U}$.
One can verify that all criteria from the list are adequate
in the sense that they describe the extensions of the corresponding semantics.

The concepts of verifiability and being more informative behave correctly
insofar as
more informative neighborhood functions do not lead to a loss of verification capacity.

\begin{proposition}
If a semantics $\sigma$ is $x$-verifiable,
then $\sigma$ is verifiable by all verification classes induced by some $\r^y$ with $\r^y \succeq \r^x$.
\end{proposition}
\begin{proof}
As $\sigma$ is verifiable by the verification class induced by $\r^x$ it holds that
there is some $\gamma_\sigma$ such that for all $\F \in \m{U}$,
$\gamma_\sigma(\widetilde{\F}^x,A(\F)) = \sigma(\F)$.
Now let $\r^y \succeq \r^x$, meaning that there is some $\delta$ such that
$\delta(\r^y(S,S'))~=~\r^x$.
We define
$\gamma'_\sigma(\widetilde{\F}^y,A(\F)) = \gamma_\sigma(\{(S,\delta(\Ss)) \mid (S,\Ss) \in \widetilde{\F}^y\},A(\F))$ and
observe that
$\{(S,\delta(\Ss)) \mid (S,\Ss) \in \widetilde{\F}^y\} = \widetilde{\F}^x$,
hence
$\gamma'_\sigma(\widetilde{\F}^y,A(\F)) = \sigma(\F)$ for each $F \in \m{U}$.
\end{proof}

In order to prove unverifiability  of a semantics $\sigma$ w.r.t.\ a class induced by a certain $\r^x$
it suffices to present two AFs $\F$ and $\G$ such that
$\sigma(\F)\neq\sigma(\G)$
but, $\widetilde{\F}^x = \widetilde{\G}^x$ and $A(\F)~=~A(\G)$.
Then the verification class induced by $\r^x$ does not provide enough
information to verify $\sigma$. 
In the following we will use this strategy to show exact verifiability.
Consider a semantics $\sigma$ which is verifiable by a class induced by $\r^x$.
If $\sigma$ is unverifiable by all verifiability classes induced by $\r^y$ with
$\r^y \prec \r^x$
we have that $\sigma$ is exactly verifiable by $\r^x$.
The following examples study this issue for the
semantics under consideration.

\begin{example}
\label{ex:exact_verify}
The complete semantics is ${+-}$-verifiable
as seen before.
The following AFs show that it is even exactly verifiable by that class.

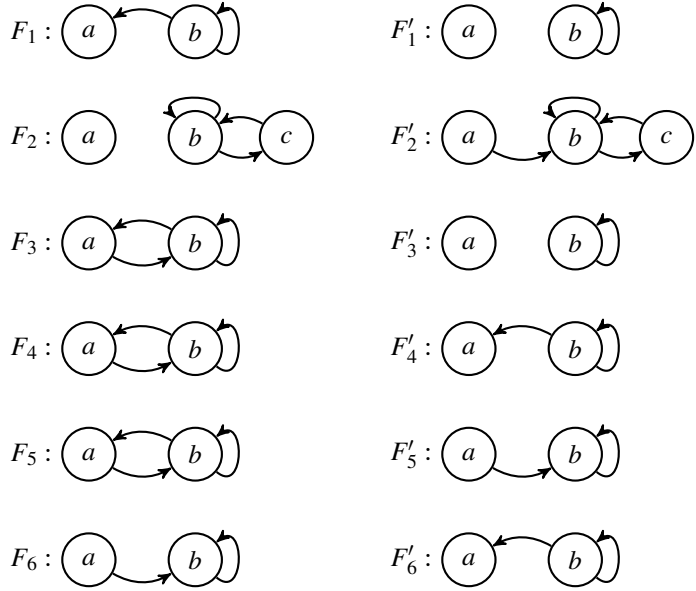
\begin{center}
\begin{tikzpicture}
	\node (a1) at (-0.6,0) [circle, thick, draw, label = left:$\F_1:$]{$a$};
    \node (b1) at (0.8,0) [circle, thick, draw]{$b$};

    \node (a1') at (4.4,0) [circle, thick, draw, label = left:$\F_1':$] {$a$};
    \node (b1') at (5.8,0) [circle, thick, draw] {$b$};

\draw[->,thick] (b1) to [thick,loop,distance=0.5cm,out=-40,in=40] (b1);
\draw[->,thick] (b1') to [thick,loop,distance=0.5cm,out=-40,in=40] (b1');

\draw[->,thick] (b1) to [thick, bend right] (a1);


	\node (a2) at (-0.6,-1.4) [circle, thick, draw, label = left:$\F_2:$]{$a$};
    \node (b2) at (0.8,-1.4) [circle, thick, draw]{$b$};
    \node (c2) at (2,-1.4) [circle, thick, draw]{$c$};

    \node (a2') at (4.4,-1.4) [circle, thick, draw, label = left:$\F_2':$] {$a$};
    \node (b2') at (5.8,-1.4) [circle, thick, draw] {$b$};
    \node (c2') at (7,-1.4) [circle, thick, draw] {$c$};

\draw[->,thick] (b2) to [thick,loop,distance=0.5cm] (b2);
\draw[->,thick] (b2') to [thick,loop,distance=0.5cm] (b2');

\draw[->,thick] (b2) to [thick, bend right] (c2);
\draw[->,thick] (c2) to [thick, bend right] (b2);
\draw[->,thick] (a2') to [thick, bend right] (b2');
\draw[->,thick] (c2') to [thick, bend right] (b2');
\draw[->,thick] (b2') to [thick, bend right] (c2');


	\node (a3) at (-0.6,-2.8) [circle, thick, draw, label = left:$\F_3:$]{$a$};
    \node (b3) at (0.8,-2.8) [circle, thick, draw]{$b$};

    \node (a3') at (4.4,-2.8) [circle, thick, draw, label = left:$\F_3':$] {$a$};
    \node (b3') at (5.8,-2.8) [circle, thick, draw] {$b$};

\draw[->,thick] (b3) to [thick,loop,distance=0.5cm,out=-40,in=40] (b3);
\draw[->,thick] (b3') to [thick,loop,distance=0.5cm,out=-40,in=40] (b3');

\draw[->,thick] (a3) to [thick, bend right] (b3);
\draw[->,thick] (b3) to [thick, bend right] (a3);


	\node (a4) at (-0.6,-4.2) [circle, thick, draw, label = left:$\F_4:$]{$a$};
    \node (b4) at (0.8,-4.2) [circle, thick, draw]{$b$};

    \node (a4') at (4.4,-4.2) [circle, thick, draw, label = left:$\F_4':$] {$a$};
    \node (b4') at (5.8,-4.2) [circle, thick, draw] {$b$};

\draw[->,thick] (b4) to [thick,loop,distance=0.5cm,out=-40,in=40] (b4);
\draw[->,thick] (b4') to [thick,loop,distance=0.5cm,out=-40,in=40] (b4');

\draw[->,thick] (a4) to [thick, bend right] (b4);
\draw[->,thick] (b4) to [thick, bend right] (a4);
\draw[->,thick] (b4') to [thick, bend right] (a4');


    \node (a5) at (-0.6,-5.6) [circle, thick, draw, label = left:$\F_5:$]{$a$};
    \node (b5) at (0.8,-5.6) [circle, thick, draw]{$b$};

    \node (a5') at (4.4,-5.6) [circle, thick, draw, label = left:$\F_5':$] {$a$};
    \node (b5') at (5.8,-5.6) [circle, thick, draw] {$b$};

\draw[->,thick] (b5) to [thick,loop,distance=0.5cm,out=-40,in=40] (b5);
\draw[->,thick] (b5') to [thick,loop,distance=0.5cm,out=-40,in=40] (b5');

\draw[->,thick] (a5) to [thick, bend right] (b5);
\draw[->,thick] (b5) to [thick, bend right] (a5);
\draw[->,thick] (a5') to [thick, bend right] (b5');

	\node (a6) at (-0.6,-7) [circle, thick, draw, label = left:$\F_6:$]{$a$};
    \node (b6) at (0.8,-7) [circle, thick, draw]{$b$};

    \node (a6') at (4.4,-7) [circle, thick, draw, label = left:$\F_6':$] {$a$};
    \node (b6') at (5.8,-7) [circle, thick, draw] {$b$};

\draw[->,thick] (b6) to [thick,loop,distance=0.5cm,out=-40,in=40] (b6);
\draw[->,thick] (b6') to [thick,loop,distance=0.5cm,out=-40,in=40] (b6');

\draw[->,thick] (b6') to [thick, bend right] (a6');
\draw[->,thick] (a6) to [thick, bend right] (b6);

\end{tikzpicture}
\end{center}

First consider the AFs $\F_1$ and $\F_1'$,
and observe that
$\widetilde{\F_1}^{+\pm} = \{(\emptyset,\emptyset,\emptyset),(\{a\},\emptyset,\emptyset)\} = \widetilde{\F_1'}^{+\pm}$.
On the other hand $\F_1$ and $\F_1'$
differ in their complete extensions
since $\com(\F_1) = \{\emptyset\}$ but
$\com(\F_1') = \{\{a\}\}$.
Therefore complete semantics is unverifiable
by the verification class induced by $\r^{+\pm}$.
Likewise, this can be shown for the classes induced by
$\r^{-\mp}$, $\r^{\pm\mp}$, $\r^{-\pm}$, $\r^{+\mp}$,
and $\r^{\cap\cup}$, respectively:

\begin{itemize} 
\item 
$\widetilde{\F_2}^{-\mp} =
= \widetilde{\F_2'}^{-\mp}$,
but
$\com(\F_2) = \{\{a\},\{a,c\}\} \neq \{\{a,c\}\} = \com(\F_2')$.
\item
$\widetilde{\F_3}^{\pm\mp} = 
\widetilde{\F_3'}^{\pm\mp}$,
but
$\com(\F_3) = \{\emptyset,\{a\}\} \neq \{\{a\}\} = \com(\F_3')$.
\item
$\widetilde{\F_4}^{-\pm} = 
\widetilde{\F_4'}^{-\pm}$,
but
$\com(\F_4) = \{\emptyset,\{a\}\} \neq \{\emptyset\} = \com(\F_4')$.
\item
$\widetilde{\F_5}^{+\mp} = 
\widetilde{\F_5'}^{+\mp}$,
but
$\com(\F_5) = \{\emptyset,\{a\}\} \neq \{\{a\}\} = \com(\F_5')$.
\item
$\widetilde{\F_6}^{\cap\cup} = 
\widetilde{\F_6'}^{\cap\cup}$,
but
$\com(\F_6) = \{\{a\}\} \neq \{\emptyset\} = \com(\F_6')$.
\end{itemize}

Hence 
complete semantics
is exactly verifiable by the verification class induced by $\r^{+-}$.
\end{example}

\begin{example}
\label{ex:semi_eager_unverifiable}
Consider the
semi-stable and eager semantics
and recall that they are ${+\mp}$-verifiable.
To show exact verifiability it suffices to show
unverifiability by the classes induced by $\r^+$, $\r^\cup$, and $\r^\mp$
(cf.\ Figure~\ref{fig:verification_classes});
$F_1$ and $F_6$ are taken from Example~\ref{ex:exact_verify} above.

\begin{itemize} 
\item
$\widetilde{\F_1}^{+} = 
\widetilde{\F_1'}^{+}$,
but
$\semi(\F_1) = \eag(\F_1) = \{\emptyset\} \neq \{\{a\}\} = \semi(\F_1') = \eag(\F_1')$.
\item
$\widetilde{\F_6}^{\cup} = 
\widetilde{\F_6'}^{\cup}$,
but
$\semi(\F_6) = \eag(\F_6) = \{\{a\}\} \neq \{\emptyset\} = \semi(\F_6') = \eag(\F_6')$.
\item
$\widetilde{\F_7}^{\mp} = 
\widetilde{\F_7'}^{\mp}$,
but
$\semi(\F_7) = \{\{b\}\} \neq \{\{a\},\{b\}\} = \semi(\F_7')$ and
$\eag(\F_7) = \{\{b\}\} \neq \{\emptyset\} = \eag(\F_7')$.
\end{itemize}
\begin{center}
\begin{tikzpicture}
	\node (a1) at (-0.6,0) [circle, thick, draw, label = left:$\F_7:$]{$a$};
    \node (b1) at (0.8,0) [circle, thick, draw]{$b$};
    \node (c1) at (2,0) [circle, thick, draw]{$c$};

    \node (a1') at (4.4,0) [circle, thick, draw, label = left:$\F_7':$] {$a$};
    \node (b1') at (5.8,0) [circle, thick, draw] {$b$};
    \node (c1') at (7.2,0) [circle, thick, draw] {$c$};

\draw[->,thick] (c1) to [thick,loop,distance=0.5cm] (c1);
\draw[->,thick] (c1') to [thick,loop,distance=0.5cm] (c1');

\draw[->,thick] (b1) to [thick, bend right] (a1);
\draw[->,thick] (a1) to [thick, bend right] (b1);
\draw[->,thick] (b1') to [thick, bend right] (a1');
\draw[->,thick] (a1') to [thick, bend right] (b1');
\draw[->,thick] (b1) to [thick, bend right] (c1);
\end{tikzpicture}
\end{center}
Hence, both the semi-stable and eager semantics
are exactly verifiable by the verification class induced by $\r^{+\mp}$.
\end{example}

\begin{example}
\label{ex:grd_sad_unverifiable}
Now consider the grounded and strong admissible semantics
and recall that they are ${-\pm}$-verifiable.
In order to show exact verifiability we have to show
unverifiability by the classes induced by $\r^\pm$, $\r^-$, and $\r^\cup$
(cf.\ Figure~\ref{fig:verification_classes}); again,
the AFs from Example~\ref{ex:exact_verify} can be reused.

\begin{itemize} 
\item
$\widetilde{\F_1}^{\pm} = \widetilde{\F_1'}^{\pm}$,
but
$\grd(\F_1) = \{\emptyset\} \neq \{\{a\}\} = \grd(\F_1')$ and
$\sad(\F_1) = \{\emptyset\} \neq \{\emptyset,\{a\}\} = \sad(\F_1')$.
\item
$\widetilde{\F_2}^{-} = \widetilde{\F_2'}^{-}$,
but
$\grd(\F_2) = \{\{a\}\} \neq \{\{a,c\}\} = \grd(\F_2')$ and
$\sad(\F_2) = \{\emptyset,\{a\}\} \neq \{\emptyset,\{a\},\{a,c\}\} = \sad(\F_2')$.
\item
$\widetilde{\F_6}^{\cup} = \widetilde{\F_6'}^{\cup}$,
but
$\grd(\F_6) = \{\{a\}\} \neq \{\emptyset\} = \grd(\F_6')$ and
$\sad(\F_6) = \{\emptyset,\{a\}\} \neq \{\emptyset\} = \sad(\F_6')$.
\end{itemize}

Hence, both the grounded and strong admissible semantics
are exactly verifiable by the verification class induced by $\r^{+\mp}$.
\end{example}

\begin{example}
\label{ex:last_unverifiable}
Finally consider stable, stage, admissible, preferred and ideal semantics.
They are either $+$-verifiable ($\stb$ and $\stg$) or $\mp$-verifiable ($\adm$, $\prf$, and $\id$).
We have to show unverifiability w.r.t.\ the verification class induced by $\r^\epsilon$.
Consider, for instance, the AFs $\F_4$ and $\F_4'$ from Example~\ref{ex:exact_verify}.
%
%
%
We have
$\widetilde{\F_4}^{\epsilon} = \widetilde{\F_4'}^{\epsilon}$,
but $\adm(\F_4) = \{\emptyset,\{a\}\} \neq  \{\emptyset\} = \adm(\F_4')$,
$\stb(\F_4) = \{\{a\}\} \neq  \emptyset = \stb(\F_4')$, and
$\sigma(\F_4) = \{\{a\}\} \neq \{\emptyset\} = \sigma(\F_4')$ for
$\sigma \in \{\stg,\prf,\id\}$,
showing exactness of the respective verification classes.
\end{example}

The insights obtained through Examples~\ref{ex:exact_verify} ,~\ref{ex:semi_eager_unverifiable},~\ref{ex:grd_sad_unverifiable} and~\ref{ex:last_unverifiable}
show that the
verification classes obtained
from the criteria given above
are indeed exact. Figure~\ref{fig:classes_of_semantics} shows the resulting relation between the semantics
under consideration
with respect to their exact verification classes.

\begin{figure}[t]
\centering
\begin{tikzpicture}
\path
node[snode] (na) {$\epsilon$: $\nav$}
++(-3,1) node[snode] (stb) {$+$: $\stb$, $\stg$} 
++(3,0) node[snode] (pref) {$\mp$: $\adm$, $\prf$, $\id$}
++(-1.5,1) node[snode] (semi) {$+\mp$: $\semi$, $\eag$}
++(4.5,0) node[snode] (grd) {$-\pm$: $\grd$, $\sad$}   
++(-3,1) node[snode] (co) {$+-$: $\com$};
				\path [->, thick,>=stealth]
			(na) edge (stb)
            (na) edge (pref)
            (na) edge (grd)
            (pref) edge (semi)
            (stb) edge (semi)
            (semi) edge (co)
            (grd) edge (co)
			;	
\end{tikzpicture}
\caption{Semantics and their Exact Verification Classes}
\label{fig:classes_of_semantics}
\end{figure}
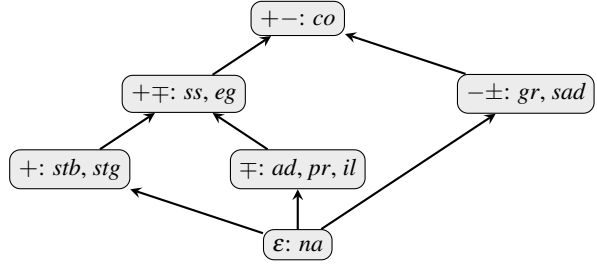

\iflong
We turn now to the main theorem stating that any \textit{rational} semantics is exactly verifiable by one of the $15$ different verification classes. 
\fi

\begin{theorem}
Every semantics which is 
rational
is exactly verifiable by a verification class induced by
one of the neighborhood functions
presented in Figure~\ref{fig:verification_classes}.
\end{theorem}
\begin{proof}
First of all note that by Lemma~\ref{lemma:verification_classes},
$\r^{\epsilon}$ is the least informative neighborhood function and for every other
neighborhood function $\r^x$ 
it holds that $\r^\epsilon \preceq \r^{x}$.
Therefore, if a semantics is verifiable by the verification class induced by any $\r^x$
then it is exactly verifiable by a verification class induced by some $\r^y$
with $\r^\epsilon \preceq \r^y \preceq \r^x$.
Moreover, if a semantics is exactly verifiable by a class, then it is by definition also verifiable by this class.
Hence it remains to show that every semantics which is rational is verifiable by a verification class
presented in Figure~\ref{fig:verification_classes}.

We show the contrapositive, i.e.,
if a semantics is not verifiable by a verification class
induced by one of the neighborhood functions
presented in Figure~\ref{fig:verification_classes}
then it is not rational.
Assume a semantics $\sigma$ is not verifiable by one of the verification classes.
This means $\sigma$ is not verifiable by the verification class induced by $\r^{+-}$.
Hence there exist two AFs $\F$ and $\G$ such that
$\widetilde{\F}^{+-} = \widetilde{\G}^{+-}$ and
$A(\F) = A(\G)$,
but $\sigma(\F) \neq \sigma(\G)$.
For every argument $a$ which is not self-attacking,
a tuple $(\{a\},\{a\}^+,\{a\}^-)$ is contained in $\widetilde{\F}^{+-}$ (and in $\widetilde{\G}^{+-}$).
Hence $\F$ and $\G$ have the same not-self-attacking arguments and,
moreover these arguments have the same ingoing and outgoing attacks in $\F$ and $\G$.
This, together with $A(\F)=A(\G)$ implies that $\F^l = \G^l$ (see Definition~\ref{def:semantics_conditions}) holds.
But since $\sigma(\F) \neq \sigma(\G)$
we get that $\sigma$ is not rational,
which was to show.
\end{proof}



Note that the criterion giving evidence for verifiability
of a semantics by a certain class
has access to the set of arguments of a given AF.
In fact, only the criterion for stable semantics makes use of that --
it can be omitted for the other semantics.

\section{Intermediate Semantics} \label{sec:intermediate}
A type of semantics which has aroused quite some interest in the literature (see e.g.\ \cite{BarG07a} and \cite{NieOZ11})
are intermediate semantics,
i.e.\ semantics which 
yield results
lying
between two existing semantics. The introduction of $\sigma$-$\tau$-intermediate semantics can be motivated by deleting 
\textit{undesired} (or add \textit{desired}) $\tau$-extensions\footnote{Recently, the so-called \textit{extension removal problem} was studied \cite{}. Here, instead of sticking to an intermediate semantics the authors studied whether it is possible - and if so how - to modify a given
AF in such a way that certain undesired extensions are no longer generated.} while 
guaranteeing all reasonable positions w.r.t.\ $\sigma$. 
In other words, \interm{\sigma}{\tau} semantics can be seen 
as sceptical or credulous acceptance shifts within the range of $\sigma$ and $\tau$. 

A natural question is whether we can make any statements
about compatible kernels of intermediate semantics.
In particular, if semantics $\sigma$ and $\tau$ are characterizable by some kernel $\k$,
is then every \interm{\sigma}{\tau} semantics characterizable by $\k$.
The following example answers this question negatively.

\begin{example}
\label{ex:stagle}
Recall from Theorem~\ref{the:strong} that both
stable and stage semantics are compatible with $k(\stb)$,
i.e.\ $\F\equiv^{\stb}_E\G \ToT \F\equiv^{\stg}_E~\G \ToT  \F^{k(\stb)} = \G^{k(\stb)}$.
Now we define the following \interm{\stb}{\stg} semantics, say \textit{stagle} semantics:
Given an AF $\F = (A,R)$, $S\in\sta(\F)$ iff
$S\in \cf(\F)$, $S^\oplus_\F \cup S^\ominus_\F = A$ and
for every $T\in \cf(\F)$ we have $S^\oplus_\F \not\subset T^\oplus_\F$.
Obviously, it holds that $\stb\subseteq\sta\subseteq\stg$ and 
$\stb\neq\sta$ as well as $\sta\neq\stg$,
as witnessed by the AF $\F$: 

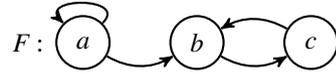
\begin{center}
\begin{tikzpicture}
    \node (A) at (1.0,0.0) [circle, thick, draw, label = left:$\F:$]{$a$};
    \node (B) at (2.5,0.0) [circle, thick, draw]{$b$};
    \node (C) at (4.0,0.0) [circle, thick, draw]{$c$};
		
\draw[->,thick] (A) to [thick,loop,distance=0.5cm] (A);
\draw[->,thick] (B) to [thick,bend right] (C);
\draw[->,thick] (C) to [thick,bend right] (B);
\draw[->,thick] (A) to [thick,bend right] (B);

\end{tikzpicture}
\end{center}

It is easy to verify that $\stb(\F) = \emptyset \subset \sta(\F) = \{\{b\}\} \subset \stg(\F) = \{\{b\},\{c\}\}$.
We proceed by showing that stagle semantics is not characterizable by $\k(\stb)$.
To this end consider $\F^{\k(\stb)}$ which is depicted below.
\begin{center}
\begin{tikzpicture}
    \node (A) at (1.0,0.0) [circle, thick, draw, label = left:$\F^{k(\stb)}:$]{$a$};
    \node (B) at (2.5,0.0) [circle, thick, draw]{$b$};
    \node (C) at (4.0,0.0) [circle, thick, draw]{$c$};

\draw[->,thick] (A) to [thick,loop,distance=0.5cm] (A);
\draw[->,thick] (B) to [thick,bend right] (C);
\draw[->,thick] (C) to [thick,bend right] (B);

\end{tikzpicture}
\end{center}
Now, $\sta\left(\F^{k(\stb)}\right) = \{\{b\},\{c\}\}$ witnesses $\F\not\equiv^{\sta}\F^{k(\stb)}$ and therefore, $\F\not\equiv^{\sta}_E\F^{k(\stb)}$. Since $\F^{k(\stb)} = \left(\F^{k(\stb)}\right)^{k(\stb)}$ we are done, i.e. stagle semantics is indeed not characterizable by the stable kernel.

\end{example}

It is the main result of this section that characterizability of intermediate semantics w.r.t.\ a certain kernel can be guaranteed if verifiability w.r.t.\ a certain class is presumed. The provided characterization theorems generalize former results presented in Theorem~\ref{the:strong}. Moreover, due to the abstract character of the theorems the results are applicable to semantics which may be defined in the future.

Before turning to the characterization theorems we state some implications of verifiability. 
In particular, under the assumption that $\sigma$ is verifiable by a certain class, equality of certain kernels 
implies expansion equivalence w.r.t.\ $\sigma$.

\iflong
\begin{proposition}
\label{pro:cf-range-verifiable}
For any $+$-verifiable semantics $\sigma$ we have
$\F^{k(\stb)} = \G^{k(\stb)} \To \F\equiv^\sigma_E\G$.
\end{proposition}

\begin{proof}
%
It was shown that 
\mbox{$\F^{k(\stb)} = \G^{k(\stb)} \To (\F\dcup\H)^{k(\stb)} = (\G\dcup\H)^{k(\stb)}$} (cf.\ In \cite{strong} or Item 5 of Fact~\ref{fact:kernel})(i).
Consider now a $+$-verifiable semantics $\sigma$. In order to show 
$\sigma(\F) = \sigma\left(\F^{k(\stb)}\right)$ (ii)
we prove $\widetilde{\F}^+ = \widetilde{\F^{k(\stb)}}^+$ (*) first.
It is easy to see that $S\in\cf(\F)$ iff $S\in\cf\left(\F^{k(\stb)}\right)$.
Furthermore, since $\k(\stb)$ deletes an attack $(a,b)$ only if $a$ is self-defeating
we deduce that ranges does not change as long as conflict-free sets are considered.
Thus, $\sigma(\F) =_{\text{ \tiny (Def.)}} \gamma_{\sigma}(\widetilde{\F}^+) =_{\text{ \tiny (*)}} \gamma_{\sigma}\left(\widetilde{\F^{k(\stb)}}^+\right) =_{\text{ \tiny (Def.)}} \sigma\left(\F^{k(\stb)}\right)$. 
Now assume that $\F^{k(\stb)} = \G^{k(\stb)}$ and
let $S\in\sigma(\F\dcup~\H)$ for some AF $\H$.
We have to show that $S\in\sigma(\G\dcup\H)$.
Applying (ii) we obtain $S\in\sigma\left((\F\dcup\H)^{k(\stb)}\right)$. Furthermore, using (i) we deduce
$S\in\sigma\left((\G\dcup\H)^{k(\stb)}\right)$.
Finally, $S\in\sigma(\G\dcup\H)$ by applying (ii),
which concludes the proof.
\end{proof}

The following results can be shown in a similar manner.

\begin{proposition}
\label{pro:cf-range-verifiable}
\label{pro:cf-range-?-verifiable}
\label{pro:cf-range-inrange-verifiable}
\label{pro:cf-inrange-pm-verifiable}
\label{pro:cf-verifiable}
For a semantics $\sigma$ it holds that
\begin{itemize}
\item if $\sigma$ is $+$-verifiable then $\F^{k(\stb)} = \G^{k(\stb)} \To \F\equiv^\sigma_E\G$.
\item if $\sigma$ is $+\mp$-verifiable then $\F^{k(\adm)} = \G^{k(\adm)} \To \F\equiv^\sigma_E\G$.
\item if $\sigma$ is $+-$-verifiable then $\F^{k(\com)} = \G^{k(\com)} \To \F\equiv^\sigma_E\G$.
\item if $\sigma$ is $-\pm$-verifiable then $\F^{k(\grd)} = \G^{k(\grd)} \To \F\equiv^\sigma_E\G$.
\item if $\sigma$ is $\epsilon$-verifiable then $\F^{k(\nav)} = \G^{k(\nav)} \To \F\equiv^\sigma_E\G$.
\end{itemize}
\end{proposition}

We proceed with general characterization theorems. The first one states that \interm{\stb}{\stg} semantics are characterizable via the stable kernel if $+$-verifiability is given. Consequently, stagle semantics as defined in Example~\ref{ex:stagle} can not be $+$-verifiable. 

\begin{theorem}
\label{the:stablefirst}
Given a semantics $\sigma$ which is $+$-verifiable and \interm{\stb}{\stg}, then \linebreak
$\F^{k(\stb)} = \G^{k(\stb)} \ToT \F\equiv^\sigma_E\G$.
\end{theorem}
\begin{proof}
($\To$) Follows directly from Proposition~\ref{pro:cf-range-verifiable}.

%
($\oT$)
We show the contrapositive, i.e.\ $\F^{k(\stb)} \neq \G^{k(\stb)} \To \F\not\equiv^\sigma_E\G$.
Assuming $\F^{k(\stb)} \neq \G^{k(\stb)}$ implies $\F \not\equiv^\stb_E \G$,
i.e.\ there exists an AF $\H$ such that $\stb(\F \dcup \H) \neq \stb(\G \dcup \H)$ and therefore, $\stb(\F \dcup \H) \neq \stb(\G \dcup \H)$.
Let $B = A(\F) \cup A(\G) \cup A(\H)$ and $\H' = (B \cup \{a\},\{(a,b),(b,a) \mid b \in B\})$.
It is easy to see that $\stb(\F \dcup \H') = \stb(\F \dcup \H) \cup \{\{a\}\}$ and
$\stb(\G \dcup \H') = \stb(\G \dcup \H) \cup \{\{a\}\}$.
Since now both \linebreak $\stb(\F \dcup \H') \neq \emptyset$ and $\stb(\G \dcup \H') \neq \emptyset$
it holds that $\stb(\F \dcup \H') = \stb(\F \dcup \H')$ and $\stb(\G \dcup \H') =$\linebreak $\stb(\G \dcup \H')$.
Hence $\sigma(\F \dcup \H') \neq \sigma(\F \dcup \H')$,
showing that $\F \not\equiv^\stb_E \G$.

\end{proof}

The following theorems can be shown in a similar manner.

\begin{theorem}
\label{the:grdfirst}
Given a semantics $\sigma$ which is $-\pm$-verifiable and \interm{\grd}{\sad},
it holds that
$\F^{k(\grd)} = \G^{k(\grd)} \ToT \F\equiv^\sigma_E\G$.
\end{theorem}

%
%
%

\begin{theorem} \label{the:admfirst}
Given a semantics $\sigma$ which is $+\mp$-verifiable and \interm{\rho}{\adm}
for any
$\rho\in\{\semi,\id,\eag\}$,
it holds that
$\F^{k(\adm)} = \G^{k(\adm)} \ToT \F\equiv^\sigma_E\G$.
\end{theorem}


Recall that complete semantics is a \interm{\semi}{\adm} semantics. Furthermore, it is not characterizable by the admissible kernel as already observed in \cite{strong}. Consequently, 
it
is not $+\mp$-verifiable (as we have shown in Example~\ref{ex:exact_verify} with considerable effort).

\section{Conclusions} \label{sec:sumcon4}

In this chapter we have initiated a, to the best of our knowledge, novel
approach 
contributed to the analysis and comparison of abstract argumentation semantics.
The main idea of our approach
is to provide a novel categorization 
in terms of the amount of information
required for testing whether a set of arguments is an extension of a certain
semantics.
The resulting notion of verification
classes allows us to categorize any new semantics (given it is ``rational'')
with respect to the information needed and compare it to other semantics.
Thus our work is 
in the tradition of the principle-based evaluation due to 
Baroni and Giacomin
\cite{BarG07}
and 
paves the way for a more general view on argumentation
semantics, their 
common features, and their inherent differences.

Using our notion of verifiability, we were able to show characterizability
for certain intermediate semantics w.r.t.\ some classical kernels.
Concerning concrete semantics, our results yield the following observation:
While preferred, semi-stable, ideal and eager semantics coincide w.r.t.\
strong equivalence, verifiability of these semantics differs.
In fact, preferred and ideal semantics manage to be verifiable with strictly less information.

For future work we envisage an extension of the
notion of verifiability classes in order to categorize
semantics not captured by the approach followed in this chapter,
such as $\cfzwei$ \cite{BarGG05}.

\chapter{Summary and Final Remarks} \label{chap:conc}

The field of Artificial Intelligence (AI) brings together people from theoretical sciences like mathematics or philosophy and application oriented researchers. 
For a particular application the latter group is engaged with several design decisions ranging from \textit{how to represent the relevant knowledge?} and \textit{how to update a given knowledge base?} over \textit{how to simplify a given knowledge base?} to \textit{how much computing power is sufficient for the intelligent agent?}. 
In consideration of the large variety of existing logical formalisms it is of utmost importance to select the most adequate one for the specific purpose in mind. 
The presented habilitation treatise tackles several fundamental intrinsic properties of knowledge representation formalisms and thus contributes to an informed choice for the designer. 
The results can be compactly summarized as follows (for detailed summaries and conclusions we refer to the Sections~\ref{sec:discussion}, \ref{sec:sumcon}, \ref{sec:sumcon2}, \ref{sec:sumcon3}~and~\ref{sec:sumcon4}).
First, we tackled the open question whether strong equivalence of two theories (which guarantees their mutual replaceability without loss of information) in a certain formalism can be decided via ordinary equivalence in another classical logic-like formalism. 
We showed that the important case of
considering only finite knowledge bases guarantees the existence of a canonical characterizing formalism. This means that the search for characterizing logics for a given representation formalism -- analogously to the logic of here-and-there in case of normal logic programs under stable model semantics \citep{DBLP:journals/tocl/LifschitzPV01} -- is not doomed to failure. 
Secondly, we studied argumentation semantics
which play the flagship role in Dung's abstract argumentation theory in very detail \cite{Dung95}. 
In particular, we compared a representative number of semantics regarding the following properties: existence and uniqueness, expressibility, replaceability and verifiability. It turned out that the considered properties are highly sensitive to the chosen semantics as well as structural properties of the considered AFs. This variety can be seen as a positive feature since it allows the designer to choose the most appropriate one for the current task.

In the last few years AI has stepped more and more into the public. 
Omniscient question answering systems like IBM's \textit{Watson} \cite{Watson} or a superhuman chess, shogi as well as Go playing program like DeepMind's \textit{AlphaZero} \cite{Silver1140} left a remarkable impression in the wider public. Especially, the latter program is considered to be one of the greatest breakthroughs in AI. Even the most optimistic researchers did not expect that a game like Go will be played by a machine at super human level within the next 10 years since it is an extremely complex game, more so than chess. The Danish chess grandmaster Peter Heine Nielsen \cite{Nielsen} said about the chess skills of AlphaZero the following:

\begin{quotation} I always wondered how it would be if a superior species landed on earth and showed us how they played chess. Now I know. 
\end{quotation}

The public and scientific opinion about what AI is capable to achieve changed so many times since the by now famous \textit{Dartmouth Summer Workshop} in 1956 \cite{dartmouth} which is considered by most researchers as the beginning of AI research. The history of artificial intelligence is marked by the alternation of periods of enthusiasm equipped with broad funding opportunities and cycles of disappointment and criticism, so-called \textit{AI winters}. Right now we are definitely in a hype cycle. Several major biennial conferences in AI decided to become an annual event like the prestigious \textit{International Joint Conference on Artificial Intelligence} (from 2015 onwards) \cite{IJCAI} as well as the leading knowledge representation conference \textit{International Conference on Principles of Knowledge Representation and Reasoning} (from 2020 onwards) \cite{KR}. Moreover, the German government decided to spend more than 3 billion Euros for research and development of artificial intelligence over the next years up to 2025 \cite{Initiative}. In order to stop the \textit{AI brain drain} the initiative will include a funding of 100 new professorships. The goal is clear: Germany wants to become a worldwide leader in AI.

I'm sure that the next few years will be an exciting time regarding new AI developments. One of the main challenges will be how to transfer research results into industry and finally into our daily life. It is not too far-fetched to envisage that the expertise of programs like AlphaZero and Watson might be used to optimize dangerous missions like fire fighting or to assist elderly people in their daily decisions. What is not achieved by both programs so far is the ability to explain why certain decisions are derived. This means, the actions performed by these AIs are neither transparent, nor easily understandable by humans. In order to attain a broad approval for AI technologies explanation components has to be implemented. This issue or this demand is known as \textit{explainable AI} \cite{explainable}. Argumentation theory provides the formal ingredients for being able to have a discussion or justifying a certain proposal via explicitly mentioning the pros and cons. In this regard, I believe that the presented formal study of intrinsic properties of Dung's abstract argumentation formalism may contribute to the development of a new AI generation.

%
%

%
%
%
%
%
%
%
%
%
%
%
%
%
%
%
%
%
%

\bibliographystyle{apalike}
\bibliography{habil}

\begin{thebibliography}{}

\bibitem[Alchourr{\'o}n et~al., 1985]{AGM}
Alchourr{\'o}n, C.~E., G{\"a}rdenfors, P., and Makinson, D. (1985).
\newblock On the logic of theory change: Partial meet contraction and revision
  functions.
\newblock {\em Journal of Symbolic Logic}, 50:510--530.

\bibitem[Arieli, 2012]{Ari12}
Arieli, O. (2012).
\newblock Conflict-tolerant semantics for argumentation frameworks.
\newblock In {\em Proc.\ JELIA}, pages 28--40.

\bibitem[Baroni et~al., 2014]{inputoutput}
Baroni, P., Boella, G., Cerutti, F., Giacomin, M., van~der Torre, L. W.~N., and
  Villata, S. (2014).
\newblock On the input/output behavior of argumentation frameworks.
\newblock {\em Artificial Intelligence}, 217:144--197.

\bibitem[Baroni et~al., 2011a]{BarCG11}
Baroni, P., Caminada, M., and Giacomin, M. (2011a).
\newblock An introduction to argumentation semantics.
\newblock {\em The Knowledge Engineering Review}, 26:365--410.

\bibitem[Baroni et~al., 2018a]{BaroniCG18}
Baroni, P., Caminada, M., and Giacomin, M. (2018a).
\newblock Abstract argumentation frameworks and their semantics.
\newblock In Baroni, P., Gabbay, D., Giacomin, M., and van~der Torre, L.,
  editors, {\em Handbook of Formal Argumentation}, chapter~4. College
  Publications.

\bibitem[Baroni et~al., 2013]{BarCDG13}
Baroni, P., Cerutti, F., Dunne, P.~E., and Giacomin, M. (2013).
\newblock Automata for infinite argumentation structures.
\newblock {\em Artificial Intelligence}, 203:104--150.

\bibitem[Baroni et~al., 2011b]{BarDG11}
Baroni, P., Dunne, P.~E., and Giacomin, M. (2011b).
\newblock On the resolution-based family of abstract argumentation semantics
  and its grounded instance.
\newblock {\em Artif. Intell.}, 175(3-4):791--813.

\bibitem[Baroni and Giacomin, 2007a]{BarG07a}
Baroni, P. and Giacomin, M. (2007a).
\newblock Comparing argumentation semantics with respect to skepticism.
\newblock In {\em Proc.\ ECSQARU}, pages 210--221.

\bibitem[Baroni and Giacomin, 2007b]{BarG07}
Baroni, P. and Giacomin, M. (2007b).
\newblock On principle-based evaluation of extension-based argumentation
  semantics.
\newblock {\em Artificial Intelligence}, 171:675--700.

\bibitem[Baroni et~al., 2005]{BarGG05}
Baroni, P., Giacomin, M., and Guida, G. (2005).
\newblock {SCC}-recursiveness: a general schema for argumentation semantics.
\newblock {\em Artificial Intelligence}, 168:162--210.

\bibitem[Baroni et~al., 2018b]{BarGL18}
Baroni, P., Giacomin, M., and Liao, B. (2018b).
\newblock Locality and modularity in abstract argumentation.
\newblock In Baroni, P., Gabbay, D., Giacomin, M., and van~der Torre, L.,
  editors, {\em Handbook of Formal Argumentation}, chapter~19. College
  Publications.

\bibitem[Baumann, 2011]{split}
Baumann, R. (2011).
\newblock Splitting an argumentation framework.
\newblock In {\em Logic Programming and Nonmonotonic Reasoning - 11th
  International Conference, {LPNMR} 2011, Vancouver, Canada, May 16-19, 2011.
  Proceedings}, pages 40--53.

\bibitem[Baumann, 2012a]{normal}
Baumann, R. (2012a).
\newblock Normal and strong expansion equivalence for argumentation frameworks.
\newblock {\em Artificial Intelligence}, 193:18--44.

\bibitem[Baumann, 2012b]{minimal}
Baumann, R. (2012b).
\newblock What does it take to enforce an argument? minimal change in abstract
  argumentation.
\newblock In {\em {ECAI} 2012 - 20th European Conference on Artificial
  Intelligence. Including Prestigious Applications of Artificial Intelligence
  {(PAIS-2012)} System Demonstrations Track, Montpellier, France, August 27-31
  , 2012}, pages 127--132.

\bibitem[Baumann, 2014a]{Bau14}
Baumann, R. (2014a).
\newblock Context-free and context-sensitive kernels: Update and deletion
  equivalence in abstract argumentation.
\newblock In {\em {ECAI} 2014 - 21st European Conference on Artificial
  Intelligence, 18-22 August 2014, Prague, Czech Republic - Including
  Prestigious Applications of Intelligent Systems {(PAIS} 2014)}, pages 63--68.

\bibitem[Baumann, 2014b]{diss}
Baumann, R. (2014b).
\newblock {\em Metalogical Contributions to the Nonmonotonic Theory of Abstract
  Argumentation}.
\newblock College Publications - Studies in Logic.

\bibitem[Baumann, 2016]{Bau16}
Baumann, R. (2016).
\newblock Characterizing equivalence notions for labelling-based semantics.
\newblock In {\em Principles of Knowledge Representation and Reasoning:
  Proceedings of the Fifteenth International Conference, {KR} 2016, Cape Town,
  South Africa, April 25-29, 2016.}, pages 22--32.

\bibitem[Baumann, 2018]{Bau18}
Baumann, R. (2018).
\newblock On the nature of argumentation semantics: Existence and uniqueness,
  expressibility, and replaceability.
\newblock In Baroni, P., Gabbay, D., Giacomin, M., and van~der Torre, L.,
  editors, {\em Handbook of Formal Argumentation}, chapter~14. College
  Publications.
\newblock also appears in IfCoLog Journal of Logics and their Applications
  4(8):2779-2886.

\bibitem[Baumann and Brewka, 2010]{Bau10}
Baumann, R. and Brewka, G. (2010).
\newblock Expanding argumentation frameworks: Enforcing and monotonicity
  results.
\newblock In {\em Computational Models of Argument: Proceedings of {COMMA}
  2010, Desenzano del Garda, Italy, September 8-10, 2010.}, pages 75--86.

\bibitem[Baumann and Brewka, 2013a]{zoo}
Baumann, R. and Brewka, G. (2013a).
\newblock Analyzing the equivalence zoo in abstract argumentation.
\newblock In {\em Computational Logic in Multi-Agent Systems - 14th
  International Workshop, {CLIMA} XIV, Corunna, Spain, September 16-18, 2013.
  Proceedings}, pages 18--33.

\bibitem[Baumann and Brewka, 2013b]{BauB13}
Baumann, R. and Brewka, G. (2013b).
\newblock Spectra in abstract argumentation: An analysis of minimal change.
\newblock In {\em Logic Programming and Nonmonotonic Reasoning, 12th
  International Conference, {LPNMR} 2013, Corunna, Spain, September 15-19,
  2013. Proceedings}, pages 174--186.

\bibitem[Baumann and Brewka, 2015a]{BauB15}
Baumann, R. and Brewka, G. (2015a).
\newblock {AGM} meets abstract argumentation: Expansion and revision for dung
  frameworks.
\newblock In \cite{DBLP:conf/ijcai/2015}, pages 2734--2740.

\bibitem[Baumann and Brewka, 2015b]{zoo2}
Baumann, R. and Brewka, G. (2015b).
\newblock The equivalence zoo for {Dung}-style semantics.
\newblock {\em Journal of Logic and Computation}.

\bibitem[Baumann et~al., 2012]{split2}
Baumann, R., Brewka, G., Dvor{\'{a}}k, W., and Woltran, S. (2012).
\newblock Parameterized splitting: {A} simple modification-based approach.
\newblock In {\em Correct Reasoning - Essays on Logic-Based {AI} in Honour of
  Vladimir Lifschitz}, pages 57--71.

\bibitem[Baumann et~al., 2011]{tafa}
Baumann, R., Brewka, G., and Wong, R. (2011).
\newblock Splitting argumentation frameworks: An empirical evaluation.
\newblock In {\em Theorie and Applications of Formal Argumentation - First
  International Workshop, {TAFA} 2011. Barcelona, Spain, July 16-17, 2011,
  Revised Selected Papers}, pages 17--31.

\bibitem[Baumann et~al., 2016a]{compactj}
Baumann, R., Dvor{\'{a}}k, W., Linsbichler, T., Spanring, C., Strass, H., and
  Woltran, S. (2016a).
\newblock On rejected arguments and implicit conflicts: The hidden power of
  argumentation semantics.
\newblock {\em Artificial Intelligence}, 241:244--284.

\bibitem[Baumann et~al., 2014a]{compact}
Baumann, R., Dvor{\'{a}}k, W., Linsbichler, T., Strass, H., and Woltran, S.
  (2014a).
\newblock Compact argumentation frameworks.
\newblock In {\em {ECAI} 2014 - 21st European Conference on Artificial
  Intelligence, 18-22 August 2014, Prague, Czech Republic - Including
  Prestigious Applications of Intelligent Systems {(PAIS} 2014)}, pages 69--74.

\bibitem[Baumann et~al., 2014b]{compact2}
Baumann, R., Dvor{\'{a}}k, W., Linsbichler, T., Strass, H., and Woltran, S.
  (2014b).
\newblock Compact argumentation frameworks.
\newblock {\em CoRR}, abs/1404.7734.

\bibitem[Baumann et~al., 2017]{BauDLS17}
Baumann, R., Dvo\v{r}\'{a}k, W., Linsbichler, T., and Woltran, S. (2017).
\newblock A general notion of equivalence for abstract argumentation.
\newblock In {\em {IJCAI}, Proceedings of the 26th International Joint
  Conference on Artificial Intelligence, Melbourne, China, Australia, August
  19-25 ,2017}, page to appear.

\bibitem[Baumann et~al., 2016b]{BauLW16a}
Baumann, R., Linsbichler, T., and Woltran, S. (2016b).
\newblock Verifiability of argumentation semantics.
\newblock {\em CoRR}, abs/1603.09502.

\bibitem[Baumann et~al., 2016c]{BauLW16}
Baumann, R., Linsbichler, T., and Woltran, S. (2016c).
\newblock Verifiability of argumentation semantics.
\newblock In {\em Computational Models of Argument - Proceedings of {COMMA}
  2016, Potsdam, Germany, 12-16 September, 2016}, pages 83--94.

\bibitem[Baumann and Spanring, 2015]{BauS15}
Baumann, R. and Spanring, C. (2015).
\newblock Infinite argumentation frameworks - {On} the existence and uniqueness
  of extensions.
\newblock In {\em Advances in Knowledge Representation, Logic Programming, and
  Abstract Argumentation - Essays Dedicated to Gerhard Brewka on the Occasion
  of His 60th Birthday}, pages 281--295.

\bibitem[Baumann and Spanring, 2017]{BauS17}
Baumann, R. and Spanring, C. (2017).
\newblock A study of unrestricted abstract argumentation frameworks.
\newblock In {\em {IJCAI}, Proceedings of the 26th International Joint
  Conference on Artificial Intelligence, Melbourne, China, Australia, August
  19-25 ,2017}, page to appear.

\bibitem[Baumann and Strass, 2012]{BauS12}
Baumann, R. and Strass, H. (2012).
\newblock Default reasoning about actions via abstract argumentation.
\newblock In {\em Computational Models of Argument: Proceedings of {COMMA}
  2012,, September 10-12, 2012, Vienna, {Austria}}, pages 297--309.

\bibitem[Baumann and Strass, 2013]{characteristic}
Baumann, R. and Strass, H. (2013).
\newblock On the maximal and average numbers of stable extensions.
\newblock In {\em Theory and Applications of Formal Argumentation - Second
  International Workshop, {TAFA} 2013, Beijing, China, August 3-5, 2013,
  Revised Selected papers}, pages 111--126.

\bibitem[Baumann and Strass, 2015]{BauSt15}
Baumann, R. and Strass, H. (2015).
\newblock Open problems in abstract argumentation.
\newblock In {\em Advances in Knowledge Representation, Logic Programming, and
  Abstract Argumentation - Essays Dedicated to Gerhard Brewka on the Occasion
  of His 60th Birthday}, pages 325--339.

\bibitem[Baumann and Strass, 2016]{BauS16}
Baumann, R. and Strass, H. (2016).
\newblock An abstract logical approach to characterizing strong equivalence in
  logic-based knowledge representation formalisms.
\newblock In {\em Principles of Knowledge Representation and Reasoning:
  Proceedings of the Fifteenth International Conference, {KR} 2016, Cape Town,
  South Africa, April 25-29, 2016.}, pages 525--528.

\bibitem[Baumann and Strass, 2017]{BauSt17}
Baumann, R. and Strass, H. (2017).
\newblock On the number of bipolar boolean functions.
\newblock {\em J. Log. Comput.}, 27(8):2431--2449.

\bibitem[Baumann and Ulbricht, 2018]{BauU18}
Baumann, R. and Ulbricht, M. (2018).
\newblock If nothing is accepted - repairing argumentation frameworks.
\newblock In {\em Principles of Knowledge Representation and Reasoning:
  Proceedings of the Sixteenth International Conference, {KR} 2018, Tempe,
  Arizona, 30 October - 2 November 2018}, pages 108--117.

\bibitem[Baumann and Woltran, 2016]{equisurvey}
Baumann, R. and Woltran, S. (2016).
\newblock The role of self-attacking arguments in characterizations of
  equivalence notions.
\newblock {\em Journal of Logic and Computation}, 26(4):1293--1313.

\bibitem[BBC, 2017]{Nielsen}
BBC (2017).
\newblock Google's 'superhuman' deepmind ai claims chess crown.
\newblock https://www. bbc.com/news/technology-42251535.

\bibitem[Belardinelli et~al., 2015]{BelGM15}
Belardinelli, F., Grossi, D., and Maudet, N. (2015).
\newblock Formal analysis of dialogues on infinite argumentation frameworks.
\newblock In \cite{DBLP:conf/ijcai/2015}, pages 861--867.

\bibitem[Bench-Capon and Dunne, 2007]{Bench-CaponD07}
Bench-Capon, T. J.~M. and Dunne, P.~E. (2007).
\newblock Argumentation in artificial intelligence.
\newblock {\em Artif.\ Intell.}, 171(10-15):619--641.

\bibitem[Besnard and Hunter, 2001]{hunterlink}
Besnard, P. and Hunter, A. (2001).
\newblock A logic-based theory of deductive arguments.
\newblock {\em Artificial Intelligence}, 128:203--235.

\bibitem[Besnard and Hunter, 2009]{BesH09}
Besnard, P. and Hunter, A. (2009).
\newblock Argumentation based on classical logic.
\newblock In {\em Argumentation in Artificial Intelligence}, pages 133--152.
  Springer.

\bibitem[Birkhoff, 1973]{birkhoff73lattice}
Birkhoff, G. (1973).
\newblock {\em Lattice Theory}, volume~25 of {\em American Mathematical Society
  Colloquium Publications}.
\newblock American Mathematical Society, Providence, Rhode Island, third
  edition.

\bibitem[Brewka, 1992]{Bre92}
Brewka, G. (1992).
\newblock Nonmonotonic reasoning: Logical foundations of commonsense.
\newblock {\em SIGART Bull.}, 3(2):28--29.
\newblock Reviewer-Marek, V. W.

\bibitem[Brewka et~al., 2013a]{brewka13adfs}
Brewka, G., Ellmauthaler, S., Strass, H., Wallner, J.~P., and Woltran, S.
  (2013a).
\newblock Abstract dialectical frameworks revisited.
\newblock In {\em Proceedings of the Twenty-Third International Joint
  Conference on Artificial Intelligence (IJCAI)}, pages 803--809. IJCAI/AAAI.

\bibitem[Brewka et~al., 2013b]{ADF2}
Brewka, G., Strass, H., Ellmauthaler, S., Wallner, J.~P., and Woltran, S.
  (2013b).
\newblock Abstract dialectical frameworks revisited.
\newblock In {\em {IJCAI}, Proceedings of the 23rd International Joint
  Conference on Artificial Intelligence, Beijing, China, August 3-9, 2013},
  pages 803--809.

\bibitem[Brewka and Woltran, 2010a]{brewka-woltran10adfs}
Brewka, G. and Woltran, S. (2010a).
\newblock Abstract dialectical frameworks.
\newblock In {\em Proceedings of the Twelfth International Conference on the
  Principles of Knowledge Representation and Reasoning (KR)}, pages 102--111.

\bibitem[Brewka and Woltran, 2010b]{ADF}
Brewka, G. and Woltran, S. (2010b).
\newblock Abstract dialectical frameworks.
\newblock In {\em Principles of Knowledge Representation and Reasoning:
  Proceedings of the Twelfth International Conference, {KR} 2010, Toronto,
  Ontario, Canada, May 9-13, 2010}.

\bibitem[Cabalar and Di{\'{e}}guez, 2014]{DBLP:conf/kr/CabalarD14}
Cabalar, P. and Di{\'{e}}guez, M. (2014).
\newblock Strong equivalence of non-monotonic temporal theories.
\newblock In Baral, C., Giacomo, G.~D., and Eiter, T., editors, {\em Principles
  of Knowledge Representation and Reasoning: Proceedings of the Fourteenth
  International Conference, {KR} 2014, Vienna, Austria, July 20-24, 2014}.
  {AAAI} Press.

\bibitem[Caminada, 2006]{semi06}
Caminada, M. (2006).
\newblock Semi-stable semantics.
\newblock In {\em Computational Models of Argument: Proceedings of {COMMA}
  2006, September 11-12, 2006, Liverpool, {UK}}, pages 121--130.

\bibitem[Caminada, 2011]{Cam11}
Caminada, M. (2011).
\newblock A labelling approach for ideal and stage semantics.
\newblock {\em Argument and Computation}, 2:1--21.

\bibitem[Caminada, 2014]{Cam14}
Caminada, M. (2014).
\newblock Strong admissibility revisited.
\newblock In {\em Proc.\ {COMMA}}, pages 197--208.

\bibitem[Caminada and Amgoud, 2007]{cameva}
Caminada, M. and Amgoud, L. (2007).
\newblock On the evaluation of argumentation formalisms.
\newblock {\em Artificial Intelligence}, 171:286--310.

\bibitem[Caminada et~al., 2012]{CamCD12}
Caminada, M.~W., Carnielli, W.~A., and Dunne, P.~E. (2012).
\newblock Semi-stable semantics.
\newblock {\em Journal of Logic and Computation}, 22:1207--1254.

\bibitem[Caminada and Verheij, 2010]{CamV10}
Caminada, M.~W. and Verheij, B. (2010).
\newblock On the existence of semi-stable extensions.
\newblock In {\em Benelux Conference on Artificial Intelligence}.

\bibitem[Caminada and Oren, 2014]{CamO14}
Caminada, M. W.~A. and Oren, N. (2014).
\newblock Grounded semantics and infinitary argumentation frameworks.
\newblock In {\em Benelux Conference on Artificial Intelligence}.

\bibitem[Clark, 1978]{clark78negation}
Clark, K.~L. (1978).
\newblock Negation as failure.
\newblock In Gallaire, H. and Minker, J., editors, {\em Logic and Data Bases},
  pages 293--322. Plenum Press.

\bibitem[Cohn, 1981]{cohn81universal}
Cohn, P.~M. (1981).
\newblock {\em {Universal algebra}}.
\newblock Mathematics and Its Applications. Springer, Dordrecht.

\bibitem[Dantsin et~al., 1997]{DanEGV97}
Dantsin, E., Eiter, T., Gottlob, G., and Voronkov, A. (1997).
\newblock Complexity and expressive power of logic programming.
\newblock In {\em Proceedings of the Twelfth Annual {IEEE} Conference on
  Computational Complexity, Ulm, Germany, June 24-27, 1997}, pages 82--101.

\bibitem[Davey and Priestley, 2002]{davey-priestley}
Davey, B. and Priestley, H. (2002).
\newblock {\em {Introduction to Lattices and Order}}.
\newblock Cambridge University Press, second edition.

\bibitem[Delgrande and Peppas, 2015]{DelgrandeP15}
Delgrande, J.~P. and Peppas, P. (2015).
\newblock Belief revision in horn theories.
\newblock {\em Artificial Intelligence}, 218:1--22.

\bibitem[Delgrande et~al., 2013]{DelgrandePW13}
Delgrande, J.~P., Peppas, P., and Woltran, S. (2013).
\newblock Agm-style belief revision of logic programs under answer set
  semantics.
\newblock In {\em Logic Programming and Nonmonotonic Reasoning, 12th
  International Conference, {LPNMR} 2013, Corunna, Spain, September 15-19,
  2013. Proceedings}, pages 264--276.

\bibitem[Delgrande et~al., 2008]{DelgrandeSTW08}
Delgrande, J.~P., Schaub, T., Tompits, H., and Woltran, S. (2008).
\newblock Belief revision of logic programs under answer set semantics.
\newblock In {\em Principles of Knowledge Representation and Reasoning:
  Proceedings of the Eleventh International Conference, {KR} 2008, Sydney,
  Australia, September 16-19, 2008}, pages 411--421.

\bibitem[Denecker et~al., 2000]{denecker00approximations}
Denecker, M., Marek, V., and Truszczy\'nski, M. (2000).
\newblock {Approximations, Stable Operators, Well-Founded Fixpoints and
  Applications in Nonmonotonic Reasoning}.
\newblock In {\em Logic-Based Artificial Intelligence}, pages 127--144. Kluwer
  Academic Publishers.

\bibitem[Diestel, 2006]{Die06}
Diestel, R. (2006).
\newblock {\em Graph Theory}.
\newblock Electronic library of mathematics. Springer.

\bibitem[Diller et~al., 2015]{DilHLRW15}
Diller, M., Haret, A., Linsbichler, T., R{\"{u}}mmele, S., and Woltran, S.
  (2015).
\newblock An extension-based approach to belief revision in abstract
  argumentation.
\newblock In \cite{DBLP:conf/ijcai/2015}, pages 2926--2932.

\bibitem[Dung, 1995]{Dung95}
Dung, P.~M. (1995).
\newblock On the acceptability of arguments and its fundamental role in
  nonmonotonic reasoning, logic programming and n-person games.
\newblock {\em Artificial Intelligence}, 77:321--357.

\bibitem[Dunne et~al., 2015]{DunDLW15}
Dunne, P.~E., Dvor{\'{a}}k, W., Linsbichler, T., and Woltran, S. (2015).
\newblock Characteristics of multiple viewpoints in abstract argumentation.
\newblock {\em Artificial Intelligence}, 228:153--178.

\bibitem[Dunne et~al., 2013]{DunDLW13}
Dunne, P.~E., Dvo\v{r}{\'a}k, W., Linsbichler, T., and Woltran, S. (2013).
\newblock Characteristics of multiple viewpoints in abstract argumentation.
\newblock In {\em 4th Workshop on Dynamics of Knowledge and Belief}, pages
  16--30.

\bibitem[Dunne et~al., 2016]{DunSLW16}
Dunne, P.~E., Spanring, C., Linsbichler, T., and Woltran, S. (2016).
\newblock Investigating the relationship between argumentation semantics via
  signatures.
\newblock In {\em {IJCAI}, Proceedings of the Twenty-Fifth International Joint
  Conference on Artificial Intelligence, 2016, New York, NY, USA, 9-15 July
  2016}, pages 1051--1057.

\bibitem[Dvo{\v{r}}{\'{a}}k and Dunne, 2018]{DvoD18}
Dvo{\v{r}}{\'{a}}k, W. and Dunne, P.~E. (2018).
\newblock Computational problems in formal argumentation and their complexity.
\newblock In Baroni, P., Gabbay, D., Giacomin, M., and van~der Torre, L.,
  editors, {\em Handbook of Formal Argumentation}, chapter~14. College
  Publications.
\newblock also appears in IfCoLog Journal of Logics and their Applications
  4(8):2557--2622.

\bibitem[Dvor{\'{a}}k et~al., 2011]{DvoDW11}
Dvor{\'{a}}k, W., Dunne, P.~E., and Woltran, S. (2011).
\newblock Parametric properties of ideal semantics.
\newblock In {\em {IJCAI}, Proceedings of the 22nd International Joint
  Conference on Artificial Intelligence, Barcelona, Catalonia, Spain, July
  16-22, 2011}, pages 851--856.

\bibitem[Dvo{\v{r}}{\'a}k and Spanring, 2016]{DvoS16}
Dvo{\v{r}}{\'a}k, W. and Spanring, C. (2016).
\newblock Comparing the expressiveness of argumentation semantics.
\newblock {\em Journal of Logic and Computation}.

\bibitem[Dvo\v{r}\'ak and Gaggl, 2012]{DvoG12}
Dvo\v{r}\'ak, W. and Gaggl, S.~A. (2012).
\newblock Incorporating stage semantics in the scc-recursive schema for
  argumentation semantics.
\newblock In {\em Proceedings of the 14th International Workshop on
  Non-Monotonic Reasoning (NMR 2012)}.

\bibitem[Dvo\v{r}{\'a}k et~al., 2014]{DvoLOW14}
Dvo\v{r}{\'a}k, W., Linsbichler, T., Oikarinen, E., and Woltran, S. (2014).
\newblock Resolution-based grounded semantics revisited.
\newblock In {\em Proc.\ {COMMA}}, pages 269--280.

\bibitem[Dvo\v{r}{\'a}k and Woltran, 2011]{DvoW11}
Dvo\v{r}{\'a}k, W. and Woltran, S. (2011).
\newblock On the intertranslatability of argumentation semantics.
\newblock {\em Journal of Artificial Intelligence Research}, 41:445--475.

\bibitem[Dyrkolbotn, 2014]{Dyr14}
Dyrkolbotn, S.~K. (2014).
\newblock How to argue for anything: Enforcing arbitrary sets of labellings
  using afs.
\newblock In {\em Principles of Knowledge Representation and Reasoning:
  Proceedings of the Fourteenth International Conference, {KR} 2014, Vienna,
  Austria, July 20-24, 2014}.

\bibitem[Eiter and Fink, 2003]{DBLP:conf/iclp/EiterF03}
Eiter, T. and Fink, M. (2003).
\newblock Uniform equivalence of logic programs under the stable model
  semantics.
\newblock In Palamidessi, C., editor, {\em Logic Programming, 19th
  International Conference, {ICLP} 2003, Mumbai, India, December 9-13, 2003,
  Proceedings}, volume 2916 of {\em Lecture Notes in Computer Science}, pages
  224--238. Springer.

\bibitem[Eiter et~al., 2013]{EitFPTW13}
Eiter, T., Fink, M., P{\"{u}}hrer, J., Tompits, H., and Woltran, S. (2013).
\newblock Model-based recasting in answer-set programming.
\newblock {\em Journal of Applied Non-Classical Logics}, 23(1-2):75--104.

\bibitem[Ferrucci et~al., 2010]{Watson}
Ferrucci, D.~A., Brown, E.~W., Chu{-}Carroll, J., Fan, J., Gondek, D.,
  Kalyanpur, A., Lally, A., Murdock, J.~W., Nyberg, E., Prager, J.~M.,
  Schlaefer, N., and Welty, C.~A. (2010).
\newblock Building watson: An overview of the deepqa project.
\newblock {\em {AI} Magazine}, 31(3):59--79.

\bibitem[Gabbay, 1985]{Gabbay85}
Gabbay, D.~M. (1985).
\newblock Theoretical foundations for non-monotonic reasoning in expert
  systems.
\newblock In Apt, K.~R., editor, {\em Logics and Models of Concurrent Systems},
  pages 439--457, Berlin, Heidelberg. Springer Berlin Heidelberg.

\bibitem[Gabbay et~al., 1994]{GabHR94}
Gabbay, D.~M., Hogger, C.~J., and Robinson, J.~A., editors (1994).
\newblock {\em Handbook of Logic in Artificial Intelligence and Logic
  Programming: Nonmonotonic Reasoning and Uncertain Reasoning}, volume~3.
\newblock Oxford University Press, Inc., New York, NY, USA.

\bibitem[Gaggl and Dvo{\v{r}}{\'{a}}k, 2016]{GagD14}
Gaggl, S.~A. and Dvo{\v{r}}{\'{a}}k, W. (2016).
\newblock Stage semantics and the {SCC-recursive} schema for argumentation
  semantics.
\newblock {\em Journal of Logic and Computation}, 26(4):1149--1202.

\bibitem[Gaggl and Woltran, 2013]{GagW13}
Gaggl, S.~A. and Woltran, S. (2013).
\newblock The cf2 argumentation semantics revisited.
\newblock {\em Journal of Logic and Computation}, 23:925--949.

\bibitem[Gebser et~al., 2012]{gebser12asp}
Gebser, M., Kaminski, R., Kaufmann, B., and Schaub, T. (2012).
\newblock Answer set solving in practice.
\newblock {\em Synthesis Lectures on Artificial Intelligence and Machine
  Learning}, 6(3):1--238.

\bibitem[Gelfond and Lifschitz, 1988]{gelfond-lifschitz88thestablemodel}
Gelfond, M. and Lifschitz, V. (1988).
\newblock The stable model semantics for logic programming.
\newblock In {\em ICLP}, pages 1070--1080.

\bibitem[Goguen and Burstall, 1984]{goguen84introducing}
Goguen, J.~A. and Burstall, R.~M. (1984).
\newblock Introducing institutions.
\newblock In {\em Logics of Programs}, pages 221--256. Springer.

\bibitem[Gottlob, 1992]{Got92}
Gottlob, G. (1992).
\newblock Complexity results for nonmonotonic logics.
\newblock {\em Journal of Logic and Computation}, 2(3):397--425.

\bibitem[Grossi and Modgil, 2015]{GroM15}
Grossi, D. and Modgil, S. (2015).
\newblock On the graded acceptability of arguments.
\newblock In {\em Proc.\ IJCAI}, pages 868--874.

\bibitem[Holzinger, 2018]{explainable}
Holzinger, A. (2018).
\newblock explainable ai (ex-ai).
\newblock https://gi.de/informatiklexikon/explainable-ai-ex-ai/.

\bibitem[IJCAI, 2019]{IJCAI}
IJCAI (2019).
\newblock International joint conferences on artificial intelligence
  organization.
\newblock https://www.ijcai.org/.

\bibitem[Jakobovits and Vermeir, 1999]{JakV99}
Jakobovits, H. and Vermeir, D. (1999).
\newblock Robust semantics for argumentation frameworks.
\newblock {\em JLC}, 9(2):215--261.

\bibitem[KR, 2019]{KR}
KR (2019).
\newblock Principles of knowledge representation and reasoning.
\newblock http://www.kr.org.

\bibitem[Kraus et~al., 1990]{Kraus90}
Kraus, S., Lehmann, D., and Magidor, M. (1990).
\newblock Nonmonotonic reasoning, preferential models and cumulative logics.
\newblock {\em Artificial Intelligence}, 44(1):167 -- 207.

\bibitem[Lifschitz et~al., 2001]{DBLP:journals/tocl/LifschitzPV01}
Lifschitz, V., Pearce, D., and Valverde, A. (2001).
\newblock Strongly equivalent logic programs.
\newblock {\em {ACM} Transactions on Computational Logic}, 2(4):526--541.

\bibitem[Lifschitz and Turner, 1994]{LifT94}
Lifschitz, V. and Turner, H. (1994).
\newblock Splitting a logic program.
\newblock In Hentenryck, P.~V., editor, {\em International Conference on Logic
  Programming}, pages 23--37. MIT Press.

\bibitem[Linsbichler et~al., 2016]{LinPS16}
Linsbichler, T., P{\"{u}}hrer, J., and Strass, H. (2016).
\newblock A uniform account of realizability in abstract argumentation.
\newblock In {\em {ECAI} 2016 - 22nd European Conference on Artificial
  Intelligence, 29 August-2 September 2016, The Hague, The Netherlands -
  Including Prestigious Applications of Artificial Intelligence {(PAIS} 2016)},
  pages 252--260.

\bibitem[Linsbichler et~al., 2015]{LinSW15}
Linsbichler, T., Spanring, C., and Woltran, S. (2015).
\newblock The hidden power of abstract argumentation semantics.
\newblock In {\em Theory and Applications of Formal Argumentation - Third
  International Workshop, {TAFA} 2015, Buenos Aires, Argentina, July 25-26,
  2015, Revised Selected Papers}, pages 146--162.

\bibitem[Maher, 1986]{Mah86}
Maher, M.~J. (1986).
\newblock Eqivalences of logic programs.
\newblock In {\em Proc.\ ICLP}, pages 410--424.

\bibitem[Makinson, 1994]{Makinson94}
Makinson, D. (1994).
\newblock Handbook of logic in artificial intelligence and logic programming
  (vol. 3).
\newblock chapter General Patterns in Nonmonotonic Reasoning, pages 35--110.
  Oxford University Press, Inc., New York, NY, USA.

\bibitem[McCarthy et~al., 1955]{dartmouth}
McCarthy, J., Minsky, M.~L., Rochester, N., and Shannon, C.~E. (1955).
\newblock A {PROPOSAL} {FOR} {THE} {DARTMOUTH} {SUMMER} {RESEARCH} {PROJECT}
  {ON} {ARTIFICIAL} {INTELLIGENCE}.
\newblock http://www-formal.stanford.edu/jmc/history/dartmouth/dartmouth.html.

\bibitem[Nieves et~al., 2011]{NieOZ11}
Nieves, J.~C., Osorio, M., and Zepeda, C. (2011).
\newblock A schema for generating relevant logic programming semantics and its
  applications in argumentation theory.
\newblock {\em Fundam. Inform.}, 106(2-4):295--319.

\bibitem[Oikarinen and Woltran, 2011]{strong}
Oikarinen, E. and Woltran, S. (2011).
\newblock Characterizing strong equivalence for argumentation frameworks.
\newblock {\em Artificial Intelligence}, 175:1985--2009.

\bibitem[Ore, 1944]{ore44galois}
Ore, O. (1944).
\newblock Galois connexions.
\newblock {\em Transactions of the American Mathematical Society},
  55(3):493--513.

\bibitem[Osorio et~al., 2005]{osorio05inferring}
Osorio, M., Zepeda, C., Nieves, J.~C., and Cort{\'e}s, U. (2005).
\newblock Inferring acceptable arguments with answer set programming.
\newblock In {\em Proceedings of the Sixth Mexican International Conference on
  Computer Science (ENC)}, pages 198--205.

\bibitem[Prakken, 2010]{Pra10}
Prakken, H. (2010).
\newblock An abstract framework for argumentation with structured arguments.
\newblock {\em Argument and Computation}, 1:93--124.

\bibitem[Qi and Yang, 2008]{QiY08}
Qi, G. and Yang, F. (2008).
\newblock A survey of revision approaches in description logics.
\newblock In {\em Proceedings of the 21st International Workshop on Description
  Logics (DL2008), Dresden, Germany, May 13-16, 2008}.

\bibitem[Rahwan and Simari, 2009]{RahwanIS09}
Rahwan, I. and Simari, G.~R. (2009).
\newblock {\em Argumentation in Artificial Intelligence}.
\newblock Springer.

\bibitem[Reiter, 1980]{reiter80alogic}
Reiter, R. (1980).
\newblock A {L}ogic for {D}efault {R}easoning.
\newblock {\em Artificial Intelligence}, 13:81--132.

\bibitem[Rudin, 1976]{Rud76}
Rudin, W. (1976).
\newblock {\em Principles of Mathematical Analysis}.
\newblock International series in pure and applied mathematics. McGraw-Hill.

\bibitem[Scherer and Lotito, 2018]{Initiative}
Scherer, M. and Lotito, M. (2018).
\newblock Germany announces 3 billion "artificial intelligence (ai) made in
  germany" initiative.
\newblock
  https://www.littler.com/publication-press/publication/germany-announces-eu3-billion-artificial-intelligence-ai-made-germany.

\bibitem[Silver et~al., 2018]{Silver1140}
Silver, D., Hubert, T., Schrittwieser, J., Antonoglou, I., Lai, M., Guez, A.,
  Lanctot, M., Sifre, L., Kumaran, D., Graepel, T., et~al. (2018).
\newblock A general reinforcement learning algorithm that masters chess, shogi,
  and go through self-play.
\newblock {\em Science}, 362(6419):1140--1144.

\bibitem[Spanring, 2012]{Spanring12}
Spanring, C. (2012).
\newblock Intertranslatability results for abstract argumentation semantics.
\newblock Master's thesis, Uni Wien.

\bibitem[Spanring, 2015]{Spa16}
Spanring, C. (2015).
\newblock Hunt for the collapse of semantics in infinite abstract argumentation
  frameworks.
\newblock In {\em 2015 Imperial College Computing Student Workshop, {ICCSW}
  2015, September 24-25, 2015, London, United Kingdom}, pages 70--77.

\bibitem[Strass, 2013]{strass13approximating}
Strass, H. (2013).
\newblock Approximating operators and semantics for abstract dialectical
  frameworks.
\newblock {\em Artificial Intelligence}, 205:39--70.

\bibitem[Strass, 2015]{Str15}
Strass, H. (2015).
\newblock The relative expressiveness of abstract argumentation and logic
  programming.
\newblock In {\em Proceedings of the Twenty-Ninth {AAAI} Conference on
  Artificial Intelligence, January 25-30, 2015, Austin, Texas, {USA.}}, pages
  1625--1631.

\bibitem[Tarski, 1955]{Tar55}
Tarski, A. (1955).
\newblock A lattice-theoretical fixpoint theorem and its applications.
\newblock {\em Pacific Journal of Mathematics}, 5(2):285--309.

\bibitem[Truszczy\'nski, 2006]{DBLP:journals/amai/Truszczynski06}
Truszczy\'nski, M. (2006).
\newblock Strong and uniform equivalence of nonmonotonic theories -- an
  algebraic approach.
\newblock {\em Annals of Mathematics and Artificial Intelligence},
  48(3--4):245--265.

\bibitem[Truszczy\'nski, 2007]{DBLP:conf/aaai/Truszczynski07}
Truszczy\'nski, M. (2007).
\newblock The modal logic {S4F}, the default logic, and the logic
  here-and-there.
\newblock In {\em Proceedings of the Twenty-Second {AAAI} Conference on
  Artificial Intelligence, July 22--26, 2007, Vancouver, British Columbia,
  Canada}, pages 508--514. {AAAI} Press.

\bibitem[Truszczy\'nski and Woltran, 2008]{DBLP:conf/aaai/TruszczynskiW08}
Truszczy\'nski, M. and Woltran, S. (2008).
\newblock Hyperequivalence of logic programs with respect to supported models.
\newblock In Fox, D. and Gomes, C.~P., editors, {\em Proceedings of the
  Twenty-Third {AAAI} Conference on Artificial Intelligence, {AAAI} 2008,
  Chicago, Illinois, USA, July 13--17, 2008}, pages 560--565. {AAAI} Press.

\bibitem[Turner, 1996]{Tur96}
Turner, H. (1996).
\newblock Splitting a default theory.
\newblock In Clancey, W.~J. and Weld, D.~S., editors, {\em Conference on
  Artificial Intelligence and Innovative Applications of Artificial
  Intelligence Conference}, pages 645--651. AAAI Press / The MIT Press.

\bibitem[Turner, 2001a]{Turner01}
Turner, H. (2001a).
\newblock Strong equivalence for logic programs and default theories (made
  easy).
\newblock In Eiter, T., Faber, W., and Truszczy\'nski, M., editors, {\em Logic
  Programming and Nonmonotonic Reasoning, 6th International Conference, {LPNMR}
  2001, Vienna, Austria, September 17-19, 2001, Proceedings}, volume 2173 of
  {\em Lecture Notes in Computer Science}, pages 81--92. Springer.

\bibitem[Turner, 2001b]{DBLP:conf/lpnmr/Turner01}
Turner, H. (2001b).
\newblock Strong equivalence for logic programs and default theories (made
  easy).
\newblock In {\em LPNMR}, pages 81--92.

\bibitem[Turner, 2004]{DBLP:conf/lpnmr/Turner04}
Turner, H. (2004).
\newblock Strong equivalence for causal theories.
\newblock In {\em Logic Programming and Nonmonotonic Reasoning, 7th
  International Conference, {LPNMR} 2004, Fort Lauderdale, FL, USA, January
  6-8, 2004, Proceedings}, pages 289--301.

\bibitem[van Emden and Kowalski, 1976]{vanemden76thesemantics}
van Emden, M.~H. and Kowalski, R.~A. (1976).
\newblock {The Semantics of Predicate Logic as a Programming Language}.
\newblock {\em Journal of the ACM}, 23(4):733--742.

\bibitem[Verheij, 1996]{Ver96}
Verheij, B. (1996).
\newblock Two approaches to dialectical argumentation: Admissible sets and
  argumentation stages.
\newblock In {\em Proceedings of the biannual International Conference on
  Formal and Applied Practical Reasoning (FAPR)}, pages 357--368. Universiteit.

\bibitem[Verheij, 2003]{Ver03}
Verheij, B. (2003).
\newblock Deflog: {On} the logical interpretation of prima facie justified
  assumptions.
\newblock {\em Journal of Logic and Computation}, 13(3):319--346.

\bibitem[Weydert, 2011]{weydert}
Weydert, E. (2011).
\newblock Semi-stable extensions for infinite frameworks.
\newblock In {\em Benelux Conference on Artificial Intelligence}, pages
  336--343.

\bibitem[Williams and Antoniou, 1998]{WilliamsA98}
Williams, M. and Antoniou, G. (1998).
\newblock A strategy for revising default theory {Ex\-tensions}.
\newblock In {\em Principles of Knowledge Representation and Reasoning:
  Proceedings of the Sixth International Conference, {KR} 1998, Trento, Italy,
  June 2-5, 1998.}, pages 24--35.

\bibitem[Woltran, 2008]{DBLP:journals/tplp/Woltran08}
Woltran, S. (2008).
\newblock A common view on strong, uniform, and other notions of equivalence in
  answer-set programming.
\newblock {\em {TPLP}}, 8(2):217--234.

\bibitem[Wu et~al., 2009]{wu09completeextensions}
Wu, Y., Caminada, M., and Gabbay, D.~M. (2009).
\newblock {Complete Extensions in Argumentation Coincide with 3-Valued Stable
  Models in Logic Programming}.
\newblock {\em Studia Logica}, 93(2--3):383--403.

\bibitem[Yang and Wooldridge, 2015]{DBLP:conf/ijcai/2015}
Yang, Q. and Wooldridge, M., editors (2015).
\newblock {\em Proceedings of the Twenty-Fourth International Joint Conference
  on Artificial Intelligence, {IJCAI} 2015, Buenos Aires, Argentina, July
  25-31, 2015}. {AAAI} Press.

\bibitem[Zorn, 1935]{zorn1935}
Zorn, M. (1935).
\newblock A remark on method in transfinite algebra.
\newblock {\em Bulletin of the American Mathematical Society}, 41:667--670.

\end{thebibliography}



\end{document}